\documentclass[10.5pt, twocolumn]{article}
\usepackage[margin=0.75in]{geometry}

\usepackage{amssymb,amsmath,amsthm,amsfonts}

\usepackage{fancyhdr}
\usepackage{amssymb}

\usepackage{xcolor}
\usepackage{makecell}
\usepackage{colortbl}
\usepackage{natbib}

\definecolor{mydarkblue}{rgb}{0,0.08,0.45}
\usepackage[colorlinks=true,
    linkcolor=mydarkblue,
    citecolor=mydarkblue,
    filecolor=mydarkblue,
    urlcolor=mydarkblue,
    pdfview=FitH]{hyperref}

\usepackage{microtype}
\usepackage{graphicx}
\usepackage{subcaption}
\usepackage{booktabs} 

\usepackage{mathtools}

\newtheorem{theorem}{Theorem}[section]

\newtheorem{lemma}[theorem]{Lemma}
\newtheorem{corollary}[theorem]{Corollary}

\newtheorem{assumption}[theorem]{Assumption}

\usepackage[textsize=tiny]{todonotes}

\usepackage[utf8]{inputenc}
\usepackage[T1]{fontenc}
\usepackage{url}
\usepackage{nicefrac}
\usepackage{cancel}
\usepackage{mathrsfs}

\usepackage{yfonts}
\usepackage[mathscr]{euscript}

\usepackage{xparse}
\usepackage[most]{tcolorbox}
\newcommand\smallO{
  \mathchoice
    {{\scriptstyle\mathcal{O}}}
    {{\scriptstyle\mathcal{O}}}
    {{\scriptscriptstyle\mathcal{O}}}
    {\scalebox{.7}{$\scriptscriptstyle\mathcal{O}$}}
  }

\usepackage{nccmath}
\usepackage{bbold}
\usepackage{empheq}
\usepackage{pifont}

\usepackage{enumitem}
\usepackage{wrapfig}

\usepackage{thmtools}
\usepackage{thm-restate}

\usepackage{multirow}
\usepackage{rotating}
\usepackage{bm}
\usepackage{tabularx}

\usepackage{tikz,pgfplots}
\usetikzlibrary{shapes}
\usetikzlibrary{decorations.markings}
\usetikzlibrary{plotmarks}
\usetikzlibrary{patterns,patterns.meta}
\usetikzlibrary{hobby}
\usetikzlibrary{arrows.meta}
\tikzset{>={Latex[length=4,width=4]}}
\usetikzlibrary{calc,intersections,decorations.markings}
\usetikzlibrary{decorations.pathreplacing,calligraphy}
\usepackage{siunitx}
\usepackage{ctable}

\usepackage{color, colortbl}

\usepackage{multicol}

\newcommand{\cut}[1]{}

\definecolor{myteal}{RGB}{27,158,119}
\definecolor{myorange}{RGB}{217,95,2}
\definecolor{myred}{RGB}{231,41,138}
\definecolor{mypurple}{RGB}{152,78,163}
\definecolor{myblue}{RGB}{55,126,184}
\definecolor{mygreen}{RGB}{0,100,0}
\definecolor{myroyalblue}{HTML}{4169E1}
\definecolor{mygrey}{RGB}{100,100,100}

\definecolor{PIIcolor}{HTML}{4e79a7}
\definecolor{PIcolor}{HTML}{7D476e}
\definecolor{PIVcolor}{HTML}{9a9c07}
\definecolor{PIIIcolor}{HTML}{337B29}
\definecolor{FPPcolor}{HTML}{7D476e}
\definecolor{FACcolor}{HTML}{4e79a7}
\definecolor{F0color}{HTML}{cc6805}
\definecolor{KPPcolor}{HTML}{337B29}
\definecolor{Mutedred}{HTML}{CD5C5C}

\definecolor{region1}{HTML}{CAC7FF}
\definecolor{region2}{rgb}{0.683,0.9335,0.9726}
\definecolor{region2b}{HTML}{66dde4}
\definecolor{region3a}{rgb}{0.407,0.764,0.686}
\definecolor{region3b}{rgb}{0.729411, 1.0, 0.7882}
\definecolor{region3}{rgb}{0.729411, 1.0, 0.7882}
\definecolor{region4a}{rgb}{1.0,0.8745,0.729411}
\definecolor{region4b}{rgb}{1.0,0.7019,0.729411}
\definecolor{region1c}{HTML}{E4D3FC}
\definecolor{region1b}{HTML}{FEE0F9}

\usepackage[framemethod=tikz]{mdframed}
\global\mdfdefinestyle{exampledefault}{%
outerlinewidth=0pt,innerlinewidth=0pt,
roundcorner=5pt,backgroundcolor=myblue,
leftmargin=0.2cm,rightmargin=0.2cm
}

\definecolor{myblue}{HTML}{D2E4FC}
\definecolor{Gray}{gray}{0.92}

\usetikzlibrary{decorations.markings,positioning}

\graphicspath{{./figures/}}

\newcommand{\lr}{\text{lr}}
\newcommand{\gO}{\mathcal{O}}

\def\R{\mathbb{R}}

\def\dif{\mathop{}\!\mathrm{d}}

\def\EE{{\mathbb E}\,}

\def\defas{\stackrel{\text{def}}{=}}

\DeclareDocumentCommand{\Prto} {o} {
  \IfNoValueTF {#1}
  {\overset{\Pr}{\longrightarrow}}
  { \xrightarrow[ #1 \to \infty]{\Pr }}
}

\DeclareDocumentCommand{\Asto} {o} {
  \IfNoValueTF {#1}
  {\overset{\text{\rm a.s.}}{\longrightarrow}}
  { \xrightarrow[ #1 \to \infty]{\text{\rm a.s.} }}
}

\DeclareDocumentCommand{\law} {o} {
  \IfNoValueTF {#1}
  {\overset{\text{law}}{=}}
  { \xrightarrow[ #1 \to \infty]{\Pr }}
}

\DeclareMathOperator{\E}{\mathbb{E}}

\DeclareMathOperator{\Exp}{\mathbb{E}}

\newcommand{\beq}{ \begin{equation} }
\newcommand{\eeq}{ \end{equation} }

\newcommand{\ip}[1]{\langle {#1} \rangle }

\newcommand{\vertiii}[1]{{\left\vert\kern-0.25ex\left\vert\kern-0.25ex\left\vert #1 
    \right\vert\kern-0.25ex\right\vert\kern-0.25ex\right\vert}}

\def\gL{{\mathcal{L}}}

\def\gO{{\mathcal{O}}}

\def\sR{{\mathbb{R}}}

\DeclareDocumentCommand{\Wkto} {o} {
\IfNoValueTF {#1}
 {\overset{\rm law}{\longrightarrow}}
 { \xrightarrow[ #1 \to \infty]{\operatorname{law}}}
}

\makeatletter
\newcommand{\vast}{\bBigg@{4}}
\newcommand{\Vast}{\bBigg@{5}}
\makeatother

\usepackage{algorithm}
\usepackage[noend]{algorithmic}
\usepackage{fancyvrb}

\usepackage{xspace}
\usepackage{textcomp}
\usepackage[capitalize,noabbrev]{cleveref}

\newcommand{\AdamW}{{\scshape AdamW}\xspace}
\newcommand{\AdamDW}{{\scshape AdamDW}\xspace}
\newcommand{\Adam}{{\scshape Adam}\xspace}
\newcommand{\NadamW}{{\scshape NadamW}\xspace}
\newcommand{\Adamnesterov}{{\scshape NAdamW}\xspace}
\newcommand{\Logadam}{{\scshape Log-AdamW}\xspace}

\newcommand{\Logadamnesterov}{{\scshape Log-NAdamW}\xspace}
\newcommand{\Dana}{{\scshape Dana}\xspace}
\newcommand{\ADana}{{\scshape ADana}\xspace}

\newcommand{\ADanad}{{\scshape ADana-D}\xspace}
\newcommand{\Danadecaying}{{\scshape Dana-decaying}\xspace}
\newcommand{\Danaconstant}{{\scshape Dana-constant}\xspace}

\newcommand{\GeneralNesterov}{{\scshape Generalized Nesterov}\xspace}
\newcommand{\Danastar}{{\scshape Dana-Star}\xspace}
\newcommand{\DanaMKfour}{{\scshape Dana-MK4}\xspace}
\newcommand{\DanastarMKfour}{{\scshape Dana-Star-MK4}\xspace}

\newcommand{\SGD}{{\scshape SGD}\xspace}

\newcommand{\Ademamix}{{\scshape Ademamix}\xspace}
\newcommand{\Shampoo}{{\scshape Shampoo}\xspace}
\newcommand{\Muon}{{\scshape Muon}\xspace}
\newcommand{\Longadam}{{\scshape Long Adam}\xspace} 

\newcommand{\Soap}{{\scshape Soap}\xspace}
\newcommand{\Mars}{{\scshape Mars}\xspace}
\newcommand{\Dmuon}{{\scshape D-Muon}\xspace}
\newcommand{\ScheduleFreeAdamW}{{\scshape Schedule-Free AdamW}\xspace}
\newcommand{\Lion}{{\scshape Lion}\xspace}
\newcommand{\Adopt}{{\scshape Adopt}\xspace}
\newcommand{\Sophia}{{\scshape Sophia}\xspace}
\newcommand{\Signum}{{\scshape Signum}\xspace}
\newcommand{\Prodigy}{{\scshape Prodigy}\xspace}
\newcommand{\Cautious}{{\scshape Cautious}\xspace}
\newcommand{\Adammini}{{\scshape Adam-Mini}\xspace}
\newcommand{\Scion}{{\scshape Scion}\xspace}
\newcommand{\Kron}{{\scshape Kron}\xspace}
\newcommand{\Adafactor}{{\scshape Adafactor}\xspace}

\usepackage{fancyhdr}
\begin{document}

\title{Logarithmic-time Schedules for Scaling Language Models with Momentum}

\date{}
\author{Damien Ferbach \\ {\small Mila \& Universit\'e de Montr\'eal} \\ {\small   \texttt{ferbach.damien@gmail.com}} \and Courtney Paquette \\ {\small McGill University \& Mila} \\ {\small  \texttt{courtney.paquette@mcgill.ca}} \and Gauthier Gidel \\ {\small Mila \& Universit\'e de Montr\'eal} \\ {\small \texttt{gidelgau@mila.quebec}} \and Katie Everett \\ {\small Google DeepMind \& MIT} \\ {\small \texttt{everettk@google.com}} \and Elliot Paquette \\ {\small McGill University \& Mila} \\ {\small  \texttt{elliot.paquette@mcgill.ca}}}











\renewcommand{\thefootnote}{\fnsymbol{footnote}}

\twocolumn[
  \begin{@twocolumnfalse}
    \maketitle
  \end{@twocolumnfalse}
]
\begin{abstract}
{\small In practice, the hyperparameters $(\beta_1,\beta_2)$ and weight-decay $\lambda$ in \AdamW are typically kept at fixed values. Is there any reason to do otherwise? We show that for large-scale language model training, the answer is yes: by exploiting the power-law structure of language data, one can design time-varying schedules for $(\beta_1,\beta_2,\lambda)$ that deliver substantial performance gains. 

We study \emph{logarithmic-time scheduling}, in which the optimizer’s gradient memory horizon grows with training time. Although naïve variants of this are unstable, we show that suitable damping mechanisms restore stability while preserving the benefits of longer memory. Based on this, we present \ADana, an \AdamW-like optimizer that couples log-time schedules with explicit damping to balance stability and performance. We empirically evaluate \ADana across transformer scalings (45M to 2.6B parameters), comparing against \AdamW, \Muon, and \Ademamix. 

When properly tuned, \ADana achieves up to 40\% compute efficiency relative to a tuned \AdamW, with gains that persist—and even improve—as model scale increases. We further show that similar benefits arise when applying logarithmic-time scheduling to \Ademamix, and that logarithmic-time weight-decay alone can yield significant improvements. Finally, we present variants of \ADana that mitigate potential failure modes and improve robustness.

\textbf{Codebase:} \url{https://github.com/mlexpos/adana/}
}
\end{abstract}
\renewcommand{\thefootnote}{\arabic{footnote}} 
\setcounter{footnote}{0}                        



\section{Introduction}
\label{sec:intro}

A key empirical finding in language models is the emergence of scaling laws, which predict that the training loss decays as a power law in available compute $C$, data $N$, and model size $D$ \citep{kaplan2020scaling}. In this work, we study optimization algorithms within the compute-optimal scaling regime for training language models \citep{hoffmann2022chinchilla}.

\begin{figure}[t]
\centering
\includegraphics[width=0.45\textwidth]{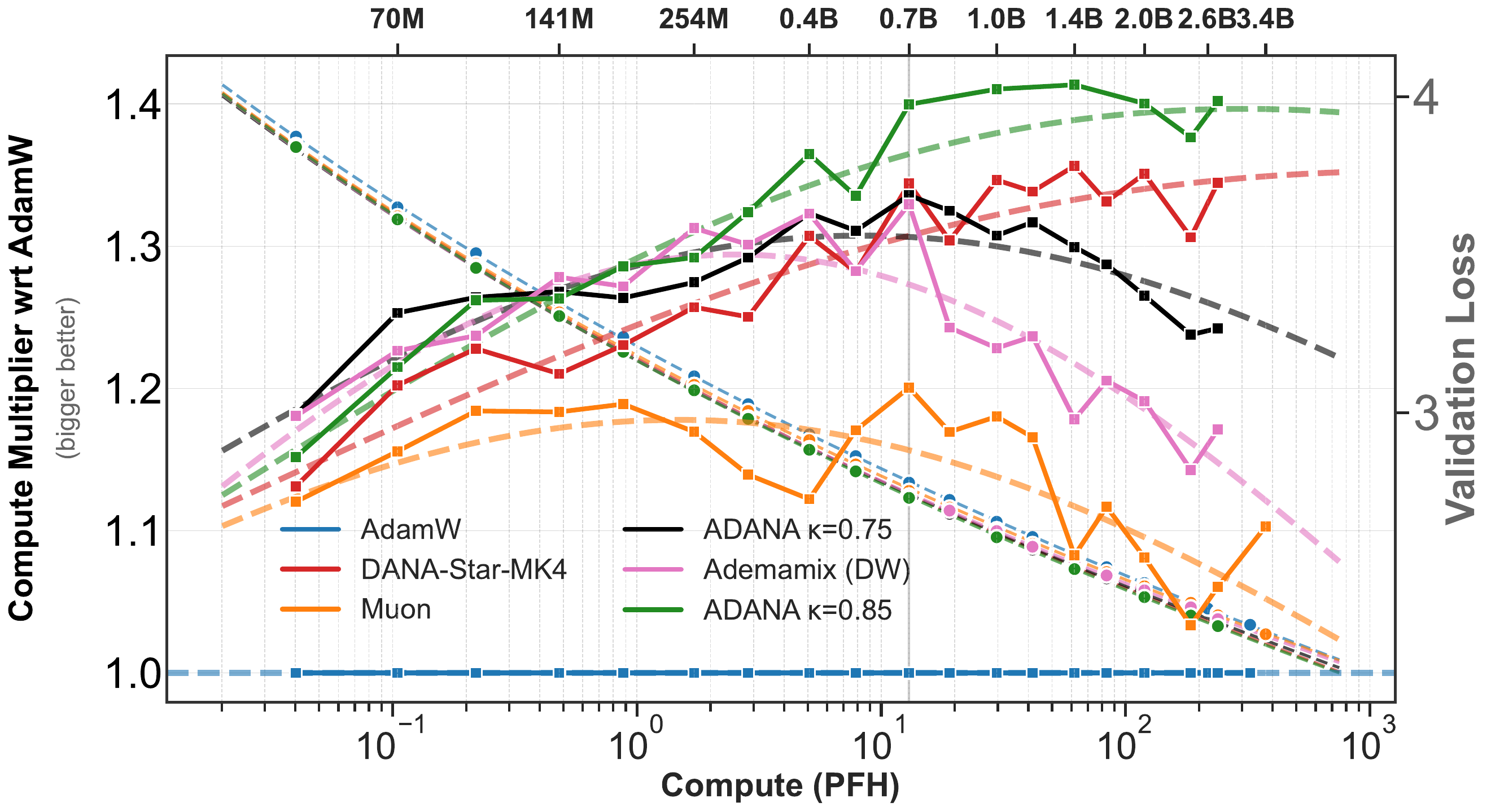}
\caption{\small \textbf{Scaling laws and compute multiplier vs compute $C$} for \ADana variants with transformers (architecture in Tab.~\ref{table:enoki_large_41_compact}) on FineWeb. \textbf{Left axis:} compute-multiplier relative to \AdamW. \ADana and \DanastarMKfour ($\kappa=0.85$) benefits increase with scale; ${\sim}40\%$ compute savings. \Ademamix (DW)$^0$ (configured at $\kappa=0.75$; Sec.~\ref{sec:appendix_ademamix}), \ADana ($\kappa=0.75$), and \Muon outperform \AdamW but with diminishing gains at larger scales; increasing $\kappa$ to $0.85$ improves scaling. \textbf{Right axis:} validation loss fit to a broken power-law $L = a + b C^{-c} + e C^{-f}$ with shared saturation $a$ across optimizers (Sec.~\ref{subsec:single_vs_broken}).}
\label{fig:scaling_main}

\end{figure}\footnotetext{We specify "(DW)" when an optimizer outside the \ADana class uses decaying weight-decay (e.g., \Ademamix-DecayWD); See Sec.~\ref{sec:weight_decay} for details.}
Scaling strategies differ fundamentally from standard hyperparameter tuning: hyperparameters cannot be re-swept at every scale without incurring prohibitive computational cost. Instead, it is essential to identify \textit{functional forms} that prescribe how optimization hyperparameters should scale with model size $D$, enabling reliable extrapolation from small to large models. A desirable property of functional forms is that performance\footnote{ Performance is measured by compute-multiplier and compute-efficiency (higher better; 1.4 is 40\% savings in compute).
For a target loss reached by algorithms A and B using respectively the compute $C^{\text{A}}$, and 
$C^{\text{B}}$, the \textbf{\textit{compute-multiplier}} of B w.r.t A is $\nicefrac{C^{\text{A}}}{C^{\text{B}}}$ and \textbf{\textit{compute-efficiency}} is $\nicefrac{(C^{\text{A}}-C^{\text{B}})}{C^{\text{B}}}$; Sec~\ref{sec:compute_saved_analysis} for discussion and compare with Sec. C.2.1 \citep{qiu2025hyperparameter} for original usage. } at large scale matches or improves upon that observed at small scale. Fig.~\ref{fig:scaling_main} illustrates this principle: algorithms with comparable small-scale performance can behave very differently under scaling—some degrade, while others maintain or improve performance. 
This divergence highlights the importance of scale-aware functional forms. While most prior work has focused on learning rate (LR) scaling (e.g., \cite{yang2022tensor}), we instead study the momentum parameters $(\beta_1, \beta_2)$ and weight-decay $\lambda$ in \AdamW, which are typically treated as fixed constants.

\begin{algorithm}[t]
\caption{ \ADana \\ \small{(Adaptive DAmped Nesterov Accel.)}}
\label{alg:adana}
{\small \begin{algorithmic}[1]
\REQUIRE Parameters $\theta_0$, 1st/2nd moments $m_0 = 0$, $v_0 = 0$, peak LR $\gamma^* > 0$, LR schedule $\gamma(t)$, numerical stability constant $\epsilon > 0$
\REQUIRE \textcolor{red}{\textbf{(Specific to \ADana)}} $\delta > 0$ (typically $\delta = 8$), $\kappa \in (0,1)$ (spectral dimension), $\omega > 0$ (weight-decay)
\STATE \textbf{Define schedules:}
\STATE \quad \textcolor{myteal}{$\beta_1(t) = \beta_2(t) = 1 - \nicefrac{\delta}{(\delta + t)}$} \COMMENT{Logarithmic time}
\STATE \quad \textcolor{myorange}{$\lambda(t) = \nicefrac{\omega}{t}$} \COMMENT{Decaying weight-decay}
\STATE \quad \textcolor{myroyalblue}{$\alpha(t) = \tilde{\alpha} \cdot (1+t)^{1-\kappa}$} \COMMENT{Damped Nesterov sched.}
\FOR{$t = 0, 1, 2, \ldots, T-1$}
    \STATE Sample minibatch of size $B$
    \STATE Compute stochastic gradient $g_{t+1}$
    \STATE $m_{t+1} = \textcolor{myteal}{\beta_1(t)} \cdot m_{t} + \textcolor{myteal}{(1 - \beta_1(t))} \cdot g_{t+1}$ \COMMENT{1st moment}
    \STATE $v_{t+1} = \textcolor{myteal}{\beta_2(t)} \cdot v_{t} + \textcolor{myteal}{(1 - \beta_2(t))} \cdot g_{t+1}^2$ \COMMENT{2nd moment}
    \STATE $\theta_{t+1} = \theta_t - \gamma(t) \left ( \gamma^* \cdot \frac{g_{t+1} + \textcolor{myroyalblue}{\alpha(t)} \cdot m_{t+1}}{\sqrt{v_{t+1}} + \epsilon} + \textcolor{myorange}{\lambda(t)} \cdot \theta_t \right )$ 
\ENDFOR
\end{algorithmic}}
\end{algorithm}

A good scaling strategy should leverage the intrinsic structures in language.  Beyond the well-known Zipfian distribution of token frequencies \citep{piantadosi2014zipf}, the next-token prediction problem itself exhibits power-law structure. Information-theoretic studies of natural language \cite{Shannon1951,Hilberg1990} show that the per-token entropy---the uncertainty in predicting the next token given a context of length $T$---decays as $T^{\beta - 1}$, where $\beta \approx 0.88$ is the Hilberg exponent \citep{Takahira2016} (see also \Cref{sec:hilberg_hypothesis}). In other words, the marginal benefit of each additional token of context diminishes as a power-law. This has a direct consequence for optimization: after $T$ steps, one additional token of training data does not meaningfully improve next-token prediction---the informative signal is spread over a timescale of length $\Theta(T)$. 

Fixed momentum and weight-decay parameters impose a constant memory horizon with exponential forgetting, creating a fundamental mismatch with this growing timescale. \emph{Logarithmic-time scheduling} resolves this by letting the optimizer's effective memory grow with time.  Because naïve log-time schedules can lead to instability (Fig.~\ref{fig:short_average_impact} and Thm. 7 \cite{even2021continuized}), we introduce a damping mechanism that moderates momentum while preserving acceleration.

Based on these principles, we introduce \ADana (Alg.~\ref{alg:adana}), which incorporates \textbf{scalable functional forms for momentum and weight-decay (WD).} 
The innovations are:
\begin{enumerate}[topsep=0pt, itemsep=1pt, parsep=0pt,leftmargin = *]
\item \textit{Logarithmic-time} schedules for both 1st and 2nd moment and the independent weight-decay parameters:
\[
\text{\textcolor{myteal}{$\beta_1(t) = \beta_2(t) = 1 - \frac{\delta}{(\delta + t)}$}}, \quad \text{\textcolor{myorange}{$\lambda(t) = \frac{\omega}{t}$}},
\]
where $\omega, \delta$ are constants. 
\item Damping schedule \textcolor{myroyalblue}{$\alpha(t) =  \tilde{\alpha} \cdot (1+t)^{1-\kappa}$} on momentum which was theoretically shown on a toy model that mimics scaling behavior to optimally balance \textit{acceleration} and \textit{stability} to stochastisicity.  
The exponent $\kappa>0$ is straightforward to tune, and empirically appears transferable across model scales (Fig.~\ref{fig:alpha_scaling_chinchilla}); See Sec.~\ref{sec:building_adana} \& \ref{sec:appendix_building_adana}. In simplified settings, the optimal $\kappa$ can be tied to the power-law exponent of the data covariance \citep{ferbach2025dimension}; we speculate that for language, it may also be directly tied to the Hilberg exponent (see \Cref{sec:hilberg_hypothesis}).  When batch size $B = 32$, the constant $\tilde{\alpha}$ was empirically observed to be $\tilde{\alpha} \approx 1$.
\end{enumerate}
Importantly, \ADana requires no major structural changes to architectures tuned for \AdamW and uses only 1 more hyperparameter, making it inexpensive to deploy. 

\begin{figure}[t]
\centering
\includegraphics[width=\columnwidth]{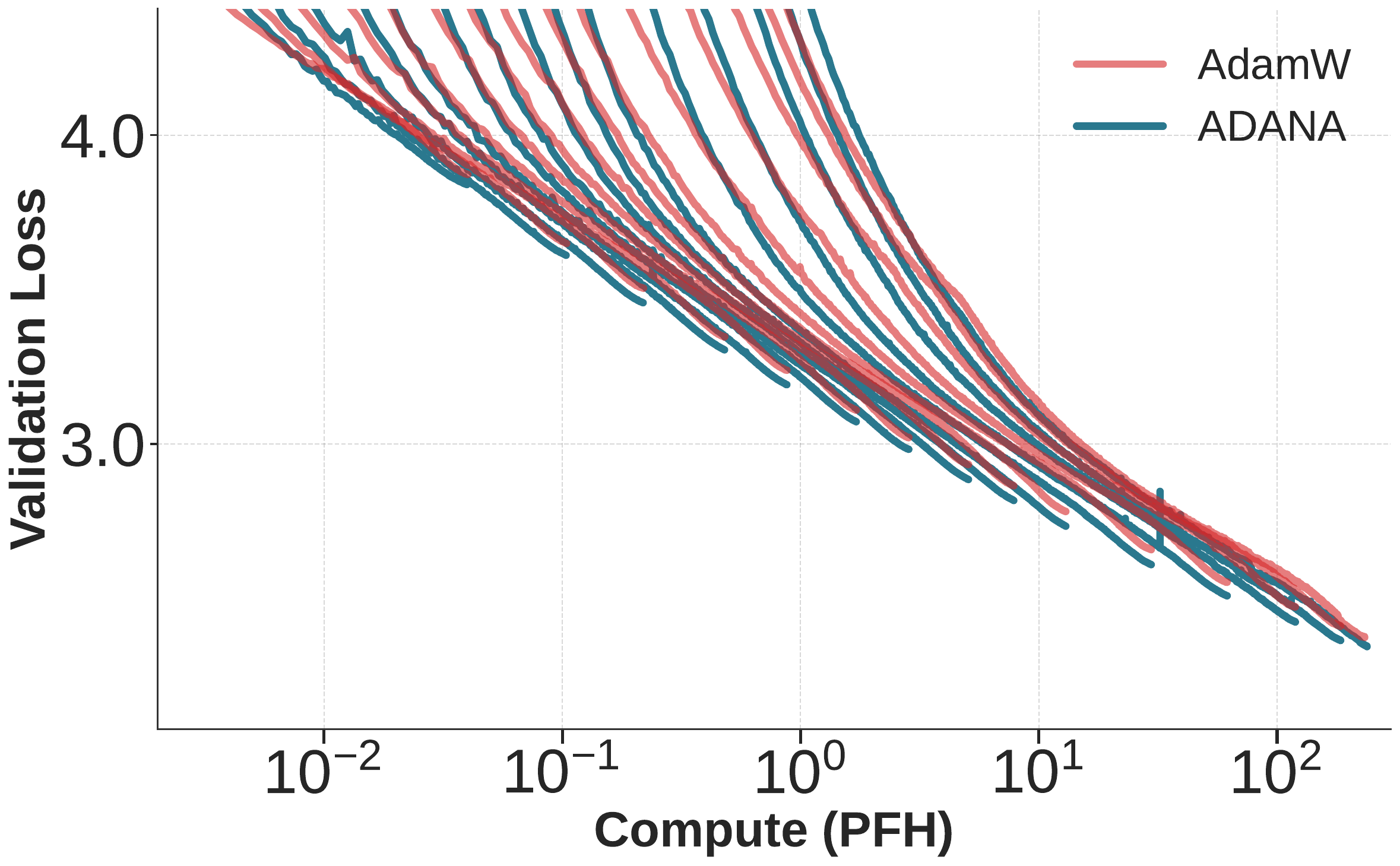}
\caption{{ \small \textbf{Training loss curves.} Validation loss over training across scales from 45.7M to 2.62B parameters for \AdamW and \ADana ($\kappa=0.85$) on FineWeb \cite{penedo2024fineweb}; Train at compute-optimal scaling $D = 20 N$ where $N$ is the total number of parameters and $D$ is the number of tokens; Final validation loss follows scaling law \cite{hoffmann2022chinchilla}. \ADana shows better loss than \AdamW along the majority of training, especially at larger scales and consistently outperforms at the end of training.
\vspace{-2mm}}
}
\label{fig:loss_curves}
\end{figure}

We conducted a \textbf{comprehensive empirical evaluation} of \ADana (Alg.~\ref{alg:adana}), \AdamW (Alg.~\ref{alg:adamW}), \Muon (Alg.~\ref{alg:muon}) and \Ademamix (Alg.~\ref{alg:ademamix}) on scaling ladders of decoder-only transformers (with 45M-2.6B parameters) using FineWeb \cite{penedo2024fineweb}. Combined, \ADana's innovations yielded \textbf{compute-efficiency improvements of up to \textcolor{red}{40\%}\footnotemark[1] against \AdamW, with gains that persist as model size increases}. Moreover, applying log-time weight-decay schedules to some algorithms boosted their performance by about \textcolor{red}{10\%} against their constant weight-decay counterparts.

Finally, we analyze potential failure modes of \ADana and introduce \textbf{robust variants of \ADana to sparse gradients and sensitivity to the $\kappa$ parameter}. These variants retain comparable performance while we show they improve stability in theoretical and synthetic experimental settings. 

\paragraph{Related Work.}
A related algorithm \Ademamix (Alg.~\ref{alg:ademamix}) has been shown to perform well on language models \citep{semenov2025benchmarking}. It uses multiple momentum buffers and more hyperparameters than \AdamW, which typically must be tuned at each scale and whose roles are not always transparent. In its original form, \Ademamix uses constant momentum parameters $\beta_i(t)\equiv\beta_i$ together with a constant damping parameter $\alpha(t)\equiv\tilde{\alpha}$; as a heuristic, the authors introduce a linear damping schedule $\alpha(t)=\tilde{\alpha} \cdot \tfrac{t}{T}$, where $T$ denotes the training horizon. We show this schedule closely mirrors a damped schedule of the form $\alpha(t)=T^{-\kappa}(\delta+t)$ when scaling its hyperparameters in a particular way (comparison with \ADana; Sec.~\ref{sec:appendix_ademamix}). Note that in all our experiments, \Ademamix hyperparameters are not swept but set to mirror \ADana schedules. 
Empirically, \Ademamix improves upon \AdamW (Fig.~\ref{fig:scaling_main}), though it consistently under performs against \ADana at scale.

In \citet{semenov2025benchmarking}, the authors benchmark a wide range of optimization algorithms on LLM training (see Sec.~\ref{sec:related_work} for optimizers). They report strong and consistent performance from \Mars \citep{yuan2411mars} and \Ademamix relative to \AdamW, particularly as the number of training iterations increases, while most other methods perform similarly to or worse than \AdamW. They also study the effects of batch size, warm-up duration, weight-decay, and the use of $z$-loss. Concurrently, \citet{wen2025fantastic} evaluate a related set of optimizers (see Sec.~\ref{sec:related_work}) and find that matrix-preconditioned methods perform best at small scales but exhibit diminishing gains relative to \AdamW at larger scales. 
Details and related work are provided in Sec.~\ref{sec:related_work}.


\begin{table}[h!]
\centering
\scriptsize
\caption{ {\small \textbf{Enoki scaling ladder} \cite{charles2025communication}. Head dim. = $64$; Embed. dim. $ = \text{Heads} \times \text{Head
dim.}$; MLP Hidden
dim. $= \text{Heads} \times 256$; $\text{Iteration} = \frac{\text{tokens}}{32 \times 2048}$; $\text{Non-embd params (P)} = 12 \times \text{(embd dim)}^2 \times \text{transformer layers}$; $\text{Tot. params (D)} = \text{non-embd} + 2 \times \text{embd dim} \times 50304$. \vspace{-0.1cm} }}
\label{table:enoki_large_41_compact}
\begin{tabular}{c | c c c c }
\toprule
\textbf{Model (D)} 
& \textbf{Layers} 
& \textbf{Heads} 
& \begin{minipage}{0.1\textwidth}
\begin{center} \textbf{Non-emb.}\\
\textbf{(P)}
\end{center}
\end{minipage}
& \begin{minipage}{0.1\textwidth}
\begin{center}
\textbf{Compute}\\ \textbf(C, PFH)
\end{center}
\end{minipage} \\
\midrule
186M  
& 10 & 14 & 96.3M    & $8.8\cdot 10^{-1}$ \\ 

254M  
& 12 & 16 & 151.0M   & $1.7\cdot 10^0$ \\ 

664M  
& 18 & 24 & 509.6M   &  $1.3\cdot 10^1$ \\ 

1.41B 
& 24 & 32 & 1.21B    &  $6.2\cdot 10^1$ \\ 

2.62B 
& 30 & 40 & 2.36B    & $2.2\cdot 10^2$ \\ 
\bottomrule
\end{tabular}
\end{table}

\section{Experimental Set-up}
\label{sec:exp_setup_main}
We evaluate optimizers on decoder-only transformers trained on FineWeb \citep{penedo2024fineweb}, following \citet{semenov2025benchmarking}. Experiments are in the compute-optimal regime: for a fixed compute budget $C$, we choose the number of non-embedding parameters $P$ and training tokens $N$ to minimize loss. We adopt the framework of \citet{hoffmann2022chinchilla}, modeling total training compute\footnote{Technically for $C$, we use $D = \text{non-embd} + \text{embd dim} \times 50304$; see Sec.~\ref{sec:scaling_laws} for complete discussion about compute.} as ``$C = 6ND$'' and predicting the optimal token count to scale as $N = 20D$, where $D$ is the total parameter count including embeddings. Experiments in the main text are on the \emph{Enoki} model, a GPT-3–like model with RoPE, QK-LayerNorm, pre-LN blocks, residual scaling, no weight tying, and no $z$-loss. We fix the head dimension at 64 while proportionally increasing width and depth (Table~\ref{table:enoki_large_41_compact} \& \ref{table:training_config_main}).  We also include a short set of experiments on the \emph{Qwen3} model in \Cref{sec:appendix_qwen}, which broadly confirm those on Enoki.

\setlength{\textfloatsep}{8pt} 
\renewcommand{\arraystretch}{.9}
{\footnotesize
\ctable[
caption={\textbf{Training hyperparameters.}},
label={table:training_config_main},
pos=t!
]{l | c}{\tnote[1]{Except embedding and normalization layers and biases.}}
{
\toprule
\textbf{Parameter} & \textbf{Value} \\
\midrule
Sequence length & 2048 \\
Batch size (sequences) & 32 \\
Tokens per batch & 65,536 \\
Vocabulary size & 50,304 \\
Warmup fraction & 0.02 (2\% of total steps) \\
Final LR fraction of peak LR, $\gamma^*$ & 0.10 \\
LR rule, $\gamma(t)$ & cosine decay \\
Precision/Optimizer state precision & bfloat16/float32 \\
Gradient clipping & 0.5 (global norm) \\
Numerical stability $\epsilon$ & 1e-8\\ 
Weight-decay\tmark[1] & independent \\
\bottomrule
}
}


\paragraph{Initialization.} Embeddings and standard linear layers use fan-in scaling ($\text{std}=1/\sqrt{\text{fan}_{\text{in}}}$); residual-branch projections (attention and MLP outputs) are further scaled by depth ($\text{std}=1/\sqrt{2\,\text{fan}_{\text{in}}\,L}$, with $L$ layers); all biases are zero.

Details for experimental setup in Sec.~\ref{sec:architecture_details} \& ~\ref{sec:setup}; Algs. used in Sec.~\ref{sec:algorithms_appendix}; Baselining \& hyperparamter procedures Sec.~\ref{sec:baselining}; Compute measurements/scaling law procedures in Sec.~\ref{sec:scaling_laws}. 


\paragraph{LR search.} To enable fair comparison across scales, we fit power-law scaling rules to the top-$K$ ($K=5$) peak learning rates $\gamma^*$ for each optimizer, ranked by (weighted) final validation loss at each model size.\footnote{Search strategy: first explore log-scale LRs with factor-2 spacing, then refine to get top-5 (Sec.~\ref{sec:search_strategy}) for heads 6–24; For heads $>24$, used the fitted rule.} We fit a saturated power law $\gamma^*(P) = a \cdot (b + P)^{-d}$, where $P$ is the number of non-embedding parameters, $a,b,d >0$ are fitted parameters. See Sec.~\ref{sec:lr_scaling} for fitting methodology, explicit formulas, and visualizations. 

\paragraph{Batch.} All experiments are performed with global batch 32 and $\tilde{\alpha} \equiv 1$ (Fig.~\ref{fig:alpha_scaling_chinchilla}). We do a preliminary study and provide additional discussion on larger batches in \Cref{sec:alpha,sec:appendix_batch}. These hint that \ADana's compute efficiency gains extend to larger batch regimes, with appropriate batch-dependent scaling rules (especially on $\tilde{\alpha}$), although additional hyperparmeter scaling experiments are needed.\footnote{As it stands, we used around 22 NVIDIA-H100 equivalent years for this project; a full batch study could easily double this.}

\section{Building ADANA}
\label{sec:building_adana}
Consider the risk function $\mathscr{R} : \R^d \to \R$ and the problem:
$\min_{\theta \in \R^d} \mathscr{R}(\theta) \defas \E_x[\mathscr{L}(\theta; x)].$
We denote (mini-batched) stochastic gradients as $g_{t+1} = \tfrac{1}{B} \sum_{i \in B} \nabla \mathscr{L}(\theta_t, x_{t+1}^i)$ where $\{x^i_{t+1}\}$ is data generated at each iteration and not reused. For a detailed discussion of this section, see Sec.~\ref{sec:appendix_building_adana}.



We build \ADana (Alg.~\ref{alg:adana}) by revisiting momentum and Nesterov-style acceleration through the lens of logarithmic-time scheduling of \AdamW's adaptive 1st/2nd moments. 
The key modifications are: (1) time-dependent momentum coefficients \textcolor{myteal}{$\beta_1(t) = \beta_2(t) = 1 - \nicefrac{\delta}{\delta + t}$}, (2) time-dependent \textcolor{myorange}{\emph{independent} weight-decay $\lambda(t) = \nicefrac{\omega}{t}$} (see the definition in \eqref{eq:decoupled_wd}), and (3) a \textcolor{myroyalblue}{damped Nesterov-style update} with scaling factor \textcolor{myroyalblue}{$\alpha(t) = (1+t)^{1-\kappa}$} where $\kappa >0$ is fixed for all scales. Specifically, \ADana modifies the \Adamnesterov update rule as
\begin{equation} \label{eq:ADana_short}
\theta_{t+1} \! = \! \theta_t \!- \!\gamma(t) \left ( \gamma^* \frac{\textcolor{myroyalblue}{\alpha(t)} m_{t+1} + g_{t+1}}{\sqrt{v_{t+1}} + \epsilon} + \textcolor{myorange}{\lambda(t)} \theta_t \right ).
\end{equation}
The parameter update combines the current gradient $g_{t+1}$ with the scaled momentum $\alpha(t) \cdot m_{t+1}$, both normalized by the adaptive second moment. Figure~\ref{fig:loss_curves} shows the training dynamics of \ADana and \AdamW on different scales showcasing compute-optimal scaling laws. 

\subsection{Scheduling $\beta_1$: Benefit of Log-time Momentum}

We begin by considering the widely used \AdamW algorithm which maintains a first moment estimate $m_t$ and second moment estimate $v_t$, updating as:
\begin{align}
\text{\emph{($1^{st}$ mom.)}}\,\, &m_{t+1} = \beta_1 m_t + (1-\beta_1) g_{t+1}, \notag \\ 
\text{\emph{($2^{\text{nd}}$ mom.)}}\,\, &v_{t+1} = \beta_2 v_t + (1-\beta_2) g_{t+1}^2, \tag{\AdamW} \\
\text{\emph{(param.)}}\, \, &\theta_{t+1} = \theta_t - \gamma(t) \left ( \frac{\gamma^* \cdot m_{t+1}}{\sqrt{v_{t+1}} + \epsilon} + \lambda \theta_t \right ), \notag
\end{align}
where $\beta_1, \beta_2 \in (0,1)$ are 1st/2nd moment hyperparameters, $\gamma(t) > 0$ is LR schedule, $\gamma^* > 0$ is peak LR, $\lambda > 0$ is the \textit{independent} WD coefficient, and $\epsilon > 0$ numerical stability. The notation $g_{t+1}^2$ denotes element-wise squaring. A variant of \AdamW, \Adamnesterov, (c.f. \cite{Dozat}) replaces the update rule of \AdamW by \textcolor{mypurple}{(indicates the change)}
\begin{equation} \tag{\footnotesize \Adamnesterov}
\theta_{t+1} = \theta_t - \gamma(t) \left ( \gamma^* \frac{m_{t+1}+\textcolor{mypurple}{g_{t+1}}}{\sqrt{v_{t+1} } + \epsilon} + \lambda \theta_t \right ).
\end{equation}
A common practical choice is to set $\beta_1$ and $\beta_2$ to fixed constants independent of model size and training time. However we recall a potential fundamental limitation in this setting: \textit{constant $\beta_1$ may not improve scaling behavior beyond stochastic gradient descent (SGD).}

\paragraph{Constant $\beta_1$ may not add any benefit over SGD.} 
Consider the standard momentum update, 
\[m_{t+1} = \beta_1 m_t + (1-\beta_1) g_{t+1}, \theta_{t+1} =
\theta_t - \gamma g_{t+1} - m_{t+1}.\] 
Unrolling the recursion gives $m_{t+1} \asymp \sum_{j=0}^t
\beta_1^j g_{t+1-j}$, revealing that constant momentum implements an exponentially weighted moving
average of past gradients. In the small batch stochastic setting, gradients do not change appreciably over the
resulting finite memory horizon, implying $m_{t+1} \approx g_{t+1}$. As a consequence, the
update reduces to SGD with a rescaled learning rate. This observation has been proven in specific settings \citep{paquette2021dynamics, ferbach2025dimension,zhang2019which}. 



\hspace{-0.3cm}\begin{minipage}{0.48\textwidth}
\begin{table}[H]
\centering
\label{tab:nesterov_variants}
\vspace{0.5em}
\centering
\begin{minipage}{0.9\textwidth}
\centering
\resizebox{\textwidth}{!}{%
\begin{tabular}{l|cccc}
\toprule
\multicolumn{5}{c}{$\theta_{t+1} = \theta_t - \gamma(t) \left ( \gamma^* \cdot  \frac{\tilde{\gamma} \times g_{t+1}+\alpha(t) \cdot m_{t+1} }{\sqrt{v_{t+1}} + \epsilon} + \lambda(t) \cdot \theta_t \right )$} \\
\textbf{Optimizer} & $\tilde{\gamma}$ & $\alpha(t)$ & $\beta_1$ & $\beta_2$ \\
\midrule
\AdamW (DW) & $0.0$ & $1.0$ & $0.9$ & $0.999$ \\
\Logadamnesterov (DW) & $1.0$ & $\delta + t$ & $\nicefrac{\delta}{(\delta+t)}$ & $\nicefrac{\delta}{(\delta+t)}$ \\
\Logadam (DW) & $0.0$ & $1.0$ & $\nicefrac{\delta}{(\delta+t)}$ & $\nicefrac{\delta}{(\delta+t)}$  \\
\ADana No Gradient & $0.0$ & $(1+t)^{1-\kappa}$ & $\nicefrac{\delta}{(\delta+t)}$ & $\nicefrac{\delta}{(\delta+t)}$ \\
\bottomrule
\end{tabular}%
}
\end{minipage}
\end{table}
\begin{figure}[H]
    \centering
    \includegraphics[width=0.9\linewidth]{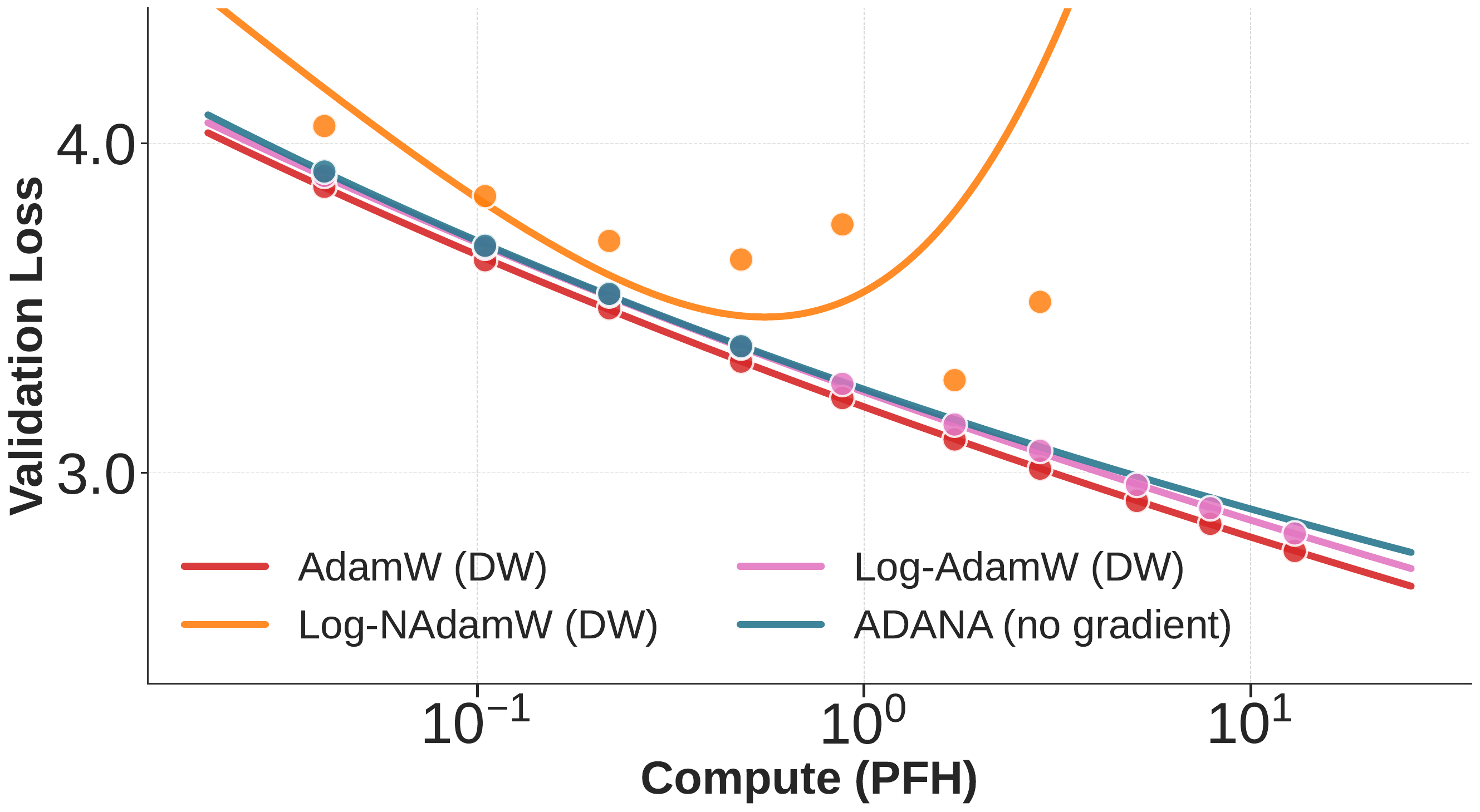}
    \caption{ \small
    \textbf{Instabilities in Logarithmic-time Momentum} \Logadamnesterov (DW) is unstable due to the undamped momentum learning rate $\alpha(t) = \delta+t$. \Logadam (DW) and \ADana-no-gradient respectively use $\alpha(t)=1$ and $\alpha(t) = (1+t)^{1-\kappa}$, $\kappa=0.75$ without the stabilizing gradient term and show degraded performance against baseline \AdamW (DW). Both stabilizing gradient and damped momentum learning rate $\alpha(t)$ are necessary to achieve good performances.}
    \label{fig:short_average_impact}
\end{figure}
\end{minipage}

\subsubsection{Log-Time Momentum Benefit} 
To overcome the limitations of fixed momentum, the effective memory of $m_{t+1}$ must grow with training time. A natural and theoretically motivated approach is \textit{logarithmic-time momentum}, achieved by scheduling
\(
\beta_1(t) = 1- \nicefrac{\delta}{\delta + t}.
\)
 Under this schedule, gradients from a fixed \emph{fraction} of the training history contribute meaningfully to the momentum term.  Notably, this is the same momentum schedule that enables acceleration in deterministic convex optimization \citep{nesterov1983method}. The term “logarithmic time” reflects that, under the change of variables $\tau=\log t$, this schedule corresponds to a \textit{constant} effective momentum in $\tau$ (see Sec~\ref{sec:log_time_explanation}).

\paragraph{Performance Gains vs. Instability.} 
While log-time momentum offers potential performance gains, it comes at the cost of stability. Consider the stochastic analogue of Nesterov’s accelerated method, which uses logarithmic-time momentum, $\beta_1(t) = 1-\nicefrac{\delta}{\delta+t}$:
\begin{equation*} 
\begin{aligned}
\text{\emph{(mom.)}}\, \, & m_{t+1} = \beta_1(t) \cdot m_t + (1-\beta_1(t)) \cdot g_{t+1}  \\
\text{\emph{(param.)}} \, \, & \theta_{t+1} = \theta_t - \gamma\cdot \Big(g_{t+1} + \frac{t+\delta}{\delta} \cdot m_{t+1}\Big). 
\end{aligned}
\end{equation*}
The additional gradient term $g_{t+1}$ in the update is essential. In Fig.~\ref{fig:short_average_impact}, we compare \AdamW to a variant with log-momentum but without the additional gradient term (\Logadam, Alg.~\ref{alg:log_adam}). Removing that additional gradient significantly degrades performance. Moreover, with stochastic gradients, instability arises from noise accumulation due to log-momentum. To give some intuition, writing $g_{t+1}=\nabla \mathscr{R}(\theta_t)+\varepsilon$ and unrolling the momentum recursion shows that the accumulating $\varepsilon$'s grow in time, while only a single $\varepsilon$ in the extra gradient term. Since there is only one LR to control both terms, the update becomes dominated by noise/accumulating $\varepsilon$ term and "diverges". This behavior has been rigorously established on a toy scaling model (PLRF\footnotemark[3]), where stochastic Nesterov diverges \cite{even2021continuized}, and is corroborated by our transformer experiments in Fig.~\ref{fig:short_average_impact}, where \AdamW with log-momentum and an additional gradient term (\Logadamnesterov) exhibits instability.

\subsubsection{Damped Nesterov acceleration.} 
\label{sec:alpha} 
The failure of stochastic Nesterov underscores a key design principle: log-momentum can be beneficial, but only when its contribution is sufficiently damped relative to the gradient to avoid instability.
This motivates a \textit{generalized Nesterov/\Dana update rule} \cite{ferbach2025dimension}:
\begin{equation*} 
\begin{aligned}
\text{\emph{(mom.)}}\quad & m_{t+1} = \Big(1-\frac{\delta}{\delta+t}\Big) \cdot m_t + \frac{\delta}{\delta+t} \cdot g_{t+1}\\
\text{\emph{(param.)}}\quad & \theta_{t+1} = \theta_t - \gamma \cdot \Big(g_{t+1} + \alpha(t) \cdot m_{t+1}\Big),
\end{aligned}
\end{equation*}
where $\alpha(t)$ is an additional, time-dependent damping factor. Standard Nesterov corresponds to $\alpha(t) \approx t$ which is unstable for stochastic gradients. Larger $\alpha(t) \gtrsim t$ further amplifies the momentum contribution and worsens divergence, while $\alpha(t) = 0$ removes momentum entirely, reducing the method to SGD. The only viable regime is $\alpha(t) \lesssim t$, where the momentum contribution is deliberately damped so that it remains comparable in magnitude to the stochastic gradient. This motivates the question: \textit{How should one choose a damping schedule that suppresses noise accumulation while preserving the benefits of log-momentum?}

\paragraph{Damping Schedule with Performance Gains \& Stability. }
The key design in \ADana is the damping schedule
\begin{equation}
\alpha(t) = \tilde{\alpha} \cdot (1+t)^{1-\kappa}, \quad \text{where $\kappa >0$.} 
\end{equation}
This form is motivated by a theoretical analysis of \Dana on the power-law random features model $\operatorname{PLRF}(\rho)$ with spectral exponent $\rho > \nicefrac{1}{2}$.\footnote{\label{fn:plrf}The power-law random features (PLRF) model is a toy model which exhibits scaling laws similar to those observed empirically \citep{maloney2022solvable, bahri2021explaining}; the data $x$ in PLRF has covariance whose $j$-th eigenvalue is $j^{-2\rho}$; see Sec.~\ref{sec:appendix_synthetic}.} On PLRF, \citet{ferbach2025dimension} show that a damping schedule of this form has the following properties, regardless of model size:
\begin{tabbing}
\hspace{2.0cm}\=:\quad \= \kill
$\kappa \in [0, \nicefrac{1}{2\rho})$ \>: \> Diverges (including $\kappa=0$, Nesterov). \\[2pt]
$\kappa \in [\nicefrac{1}{2\rho}, 1)$ \>: \> \parbox[t]{\dimexpr\linewidth-2.7cm}{Stable and accelerates over \SGD. Optimal at $\kappa=\nicefrac{1}{2\rho}$, for which reason we call $\kappa$ the \textit{spectral dimension.}} \\[2pt]
$\kappa \in [1, \infty)$                \>: \> Reduces to \SGD performance.
\end{tabbing}







The $\kappa$ dependence picture generalizes to \ADana, small $\kappa$ underperform \AdamW up to a critical point, after which they outperform \AdamW up to $\kappa=1$, at which they have similar performance to \AdamW (DW). See Figure \ref{fig:hummingbird_combined} for a $\kappa$-cross-section and Figure \ref{fig:optimal_kappa} for scale-dependence.

\begin{figure}[t]
\centering
\includegraphics[width=\linewidth]{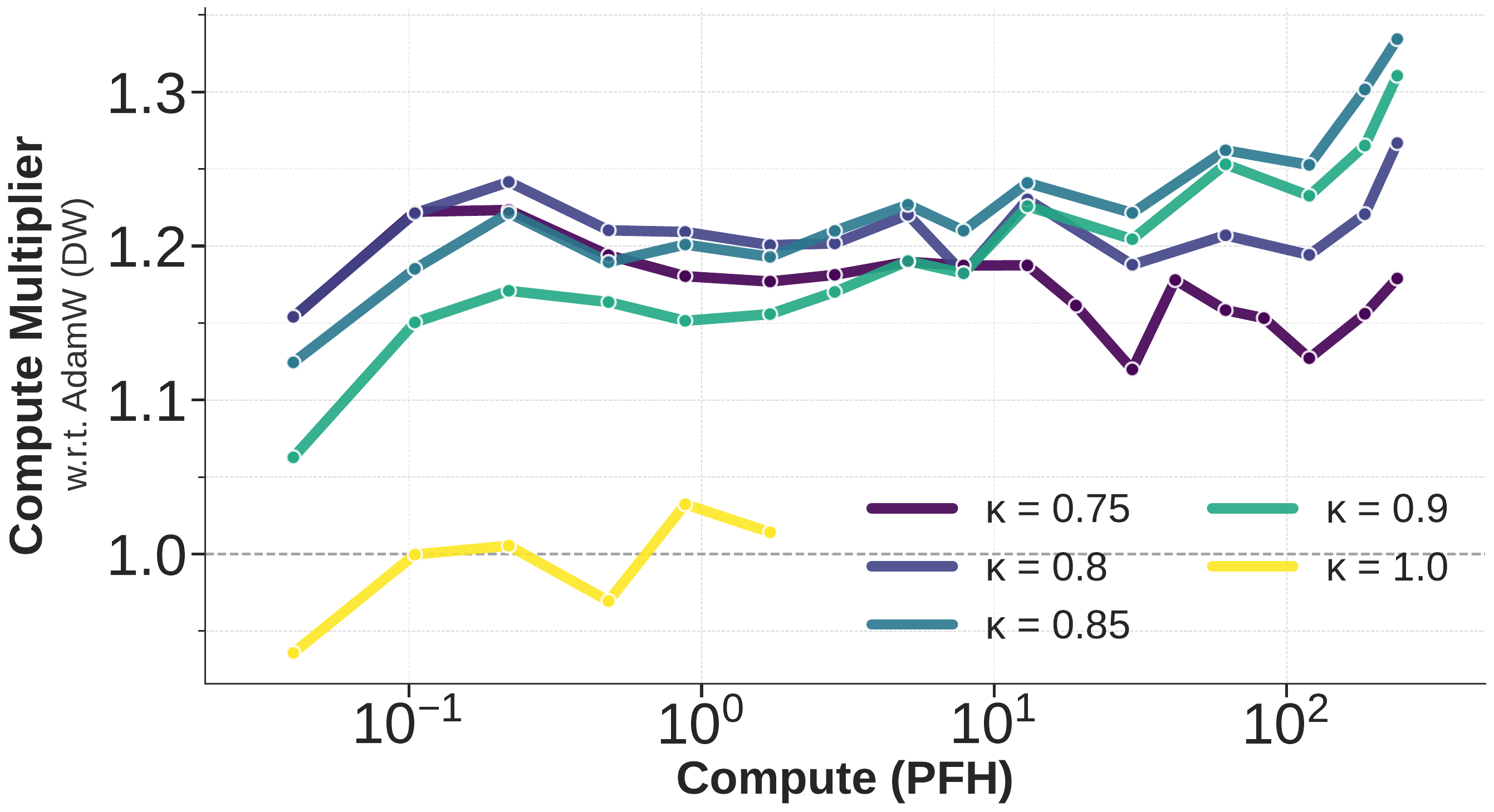}
\caption{ \small \textbf{Impact \& sensitivity of $\kappa$ on scaling performance of \ADana Alg.~\ref{alg:adana}} on Enoki transformer scaling ladder with FineWeb; $\kappa=0.85$ allows for best performance. At the largest scales compared with \AdamW-log-time WD, the optimal $\kappa=0.85$ yields more than $30\%$ ($40\%$ w.r.t \AdamW with constant WD) compute gain improvement against less than $20\%$ for $\kappa=0.75$; Performance improves across scales for more conservative $\kappa\geq 0.85$ while degrading for the more aggressive $\kappa\leq 0.8$; $\kappa=1.0$ similar to the baseline as predicted by toy model PLRF.}
\label{fig:optimal_kappa}
\end{figure}

\paragraph{Spectral dimension, $\kappa$, on LLMs.} 
In Fig.~\ref{fig:optimal_kappa}, across a transformer scaling ladder using Enoki scaling, all \ADana runs with $\kappa \in [0.75, 0.9]$ substantially outperform the \AdamW baseline. Notably, these gains do not diminish as model size increases, with optimal performance using $\kappa \approx 0.85$. 

In Sec.~\ref{sec:appendix_data_exponent} we present measurements of the covariance spectra decay of the activations, in the spirit of comparison to the PLRF.  We note that another interpretation of $\kappa$ is related to fundamental information theoretic properties of the data, which we discuss in \Cref{sec:hilberg_hypothesis}; this point of view is not easily reconcilable with the PLRF interpretation, and we emphasize that a properly explaining $\kappa$-dependence of \ADana in the LLM setting is an open problem.


\begin{figure}
\centering
\includegraphics[width=\linewidth]{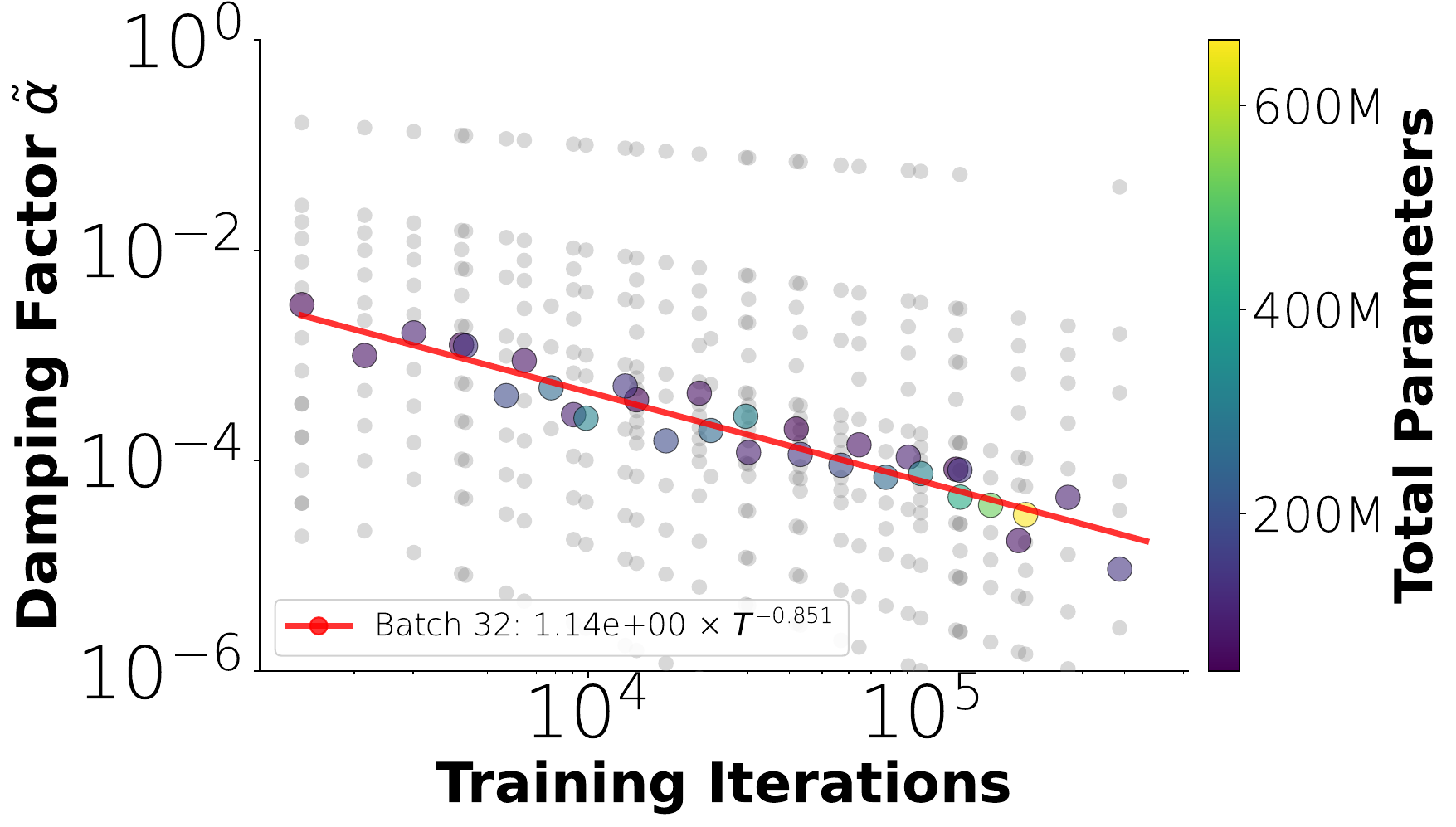}
\caption{ \small \textbf{Dependence of $\alpha(t)$ on time and scale.} Sweeps of constant $\tilde{\alpha}$ in \ADana with $\alpha(t) = \tilde{\alpha} \cdot (1+t)$. Vary both the iterations, going beyond Chinchilla scale, and model sizes. Fitting optimal $\tilde{\alpha}$ recovers $\alpha(t) = (1+t)^{1-\kappa}$. Note here $\tilde{\alpha} \approx 1$ for $B = 32$ and $\kappa\approx 0.85$.}
\label{fig:alpha_scaling_chinchilla}
\end{figure}

\paragraph{Alternative $\alpha(t)$ Schedules.} 
Other damping schedules are possible. In particular, \cite{ferbach2025dimension} showed on PLRF\footnotemark[5] that a linear schedule $\alpha(t) = \tilde{\alpha}\times (\delta+t)$, with sufficiently small $\tilde{\alpha}$, can also achieve stability, albeit with slightly worse performance than $(1+t)^{1-\kappa}$. The constant that yields the best performance is $\tilde{\alpha} = T^{-\kappa}$, where $T$ is the runtime horizon and $\kappa$ is an explicit quantity dependent on the power-law exponent of the covariance spectrum. In Fig.~\ref{fig:comparison_adana_variants}, we compare \ADana with $(1+t)^{1-\kappa}$ to \ADana with $\alpha(t) = T^{-\kappa}(\delta + t)$ (see also \Ademamix Sec.~\ref{sec:appendix_ademamix}). While this variant of \ADana consistently outperforms \AdamW with 15\% compute gains, it performs noticeably worse than \ADana with $(1+t)^{1-\kappa}$. 

\paragraph{Correct functional dependence on time and scale in $\alpha(t)$.} We investigate whether alternative damping schedules can improve performance and, more importantly, whether our proposed functional form captures the correct dependence on both training iterations and model size (Fig.~\ref{fig:alpha_scaling_chinchilla}). We run \ADana with a linear schedule $\alpha(t) = \tilde{\alpha} \cdot (1+t)$ and sweep the constant $\tilde{\alpha}$ under two settings: (i) training far beyond the Chinchilla horizon while fixing head size, to probe iteration $T$ dependence, and (ii) varying head size, to probe scaling with model size (see Sec.~\ref{appendix:verification_alpha_scaling} \& Sec.~\ref{sec:appendix_batch}, Fig.~\ref{fig:alpha_scaling_batch}). Fitting the optimal $\tilde{\alpha}$ across both regimes reveals $\tilde{\alpha} = T^{-\kappa}$ with $\kappa \approx 0.83$ (Fig.~\ref{fig:alpha_scaling_chinchilla}). Notably, the optimal choice of $\kappa$ appears largely independent of model size.  This empirically recovers the schedule $\alpha(t) = (1+t)^{1-\kappa}$, predicted to be optimal in the PLRF model, and confirms the functional dependence of $(1+t)^{1-\kappa}$ on both iterations and scale. 

\begin{figure}[t]
\centering
\begin{subfigure}{0.45\textwidth}
    \centering
    \includegraphics[width=\linewidth]{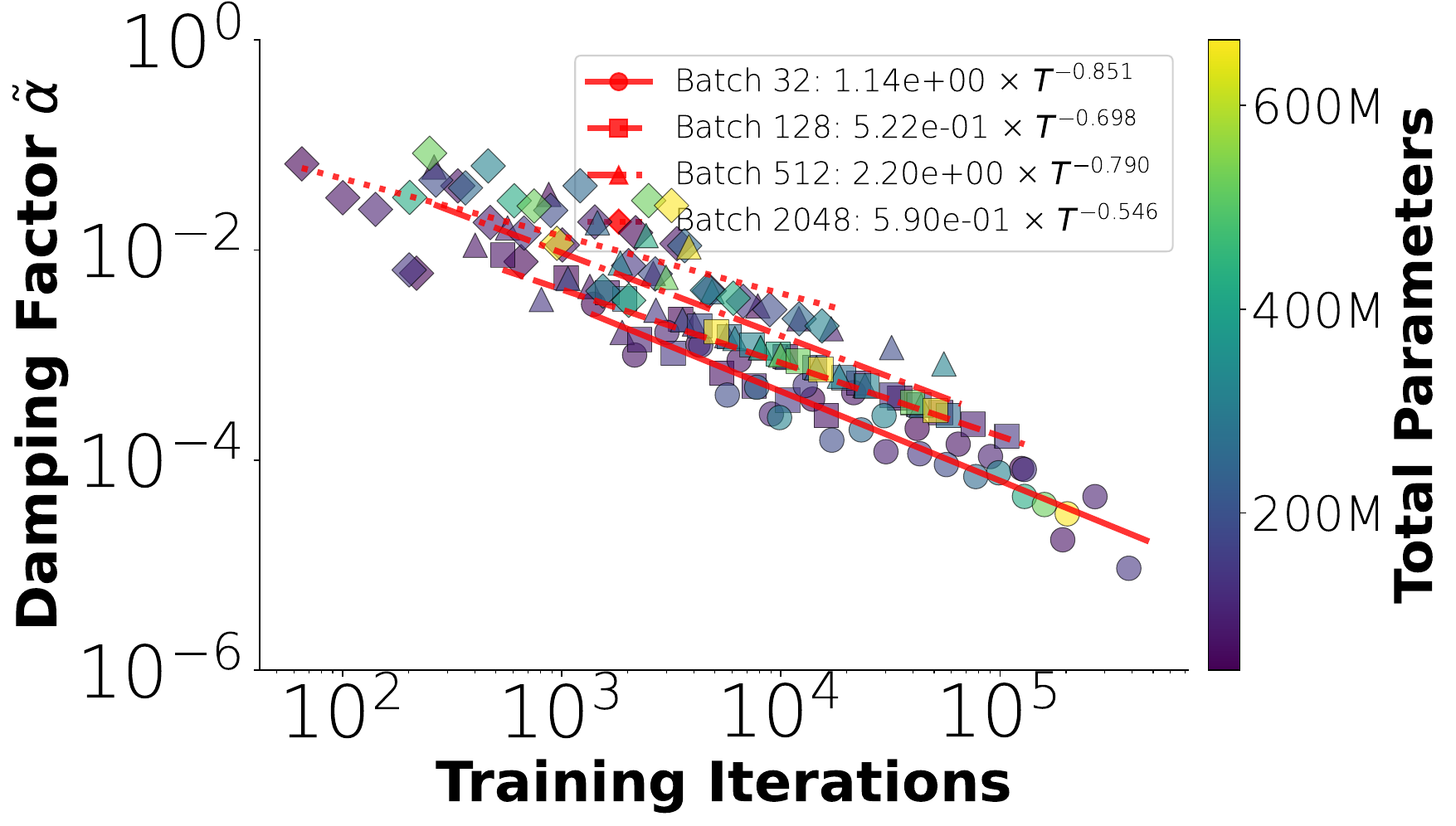}
    \caption{\small \textbf{Larger Batch Allow Larger Damping Schedule.}}
    \label{fig:batch_effect}
\end{subfigure}
\vspace{0.25cm}
\begin{subfigure}{0.45\textwidth}
    \centering
    \includegraphics[width=\linewidth]{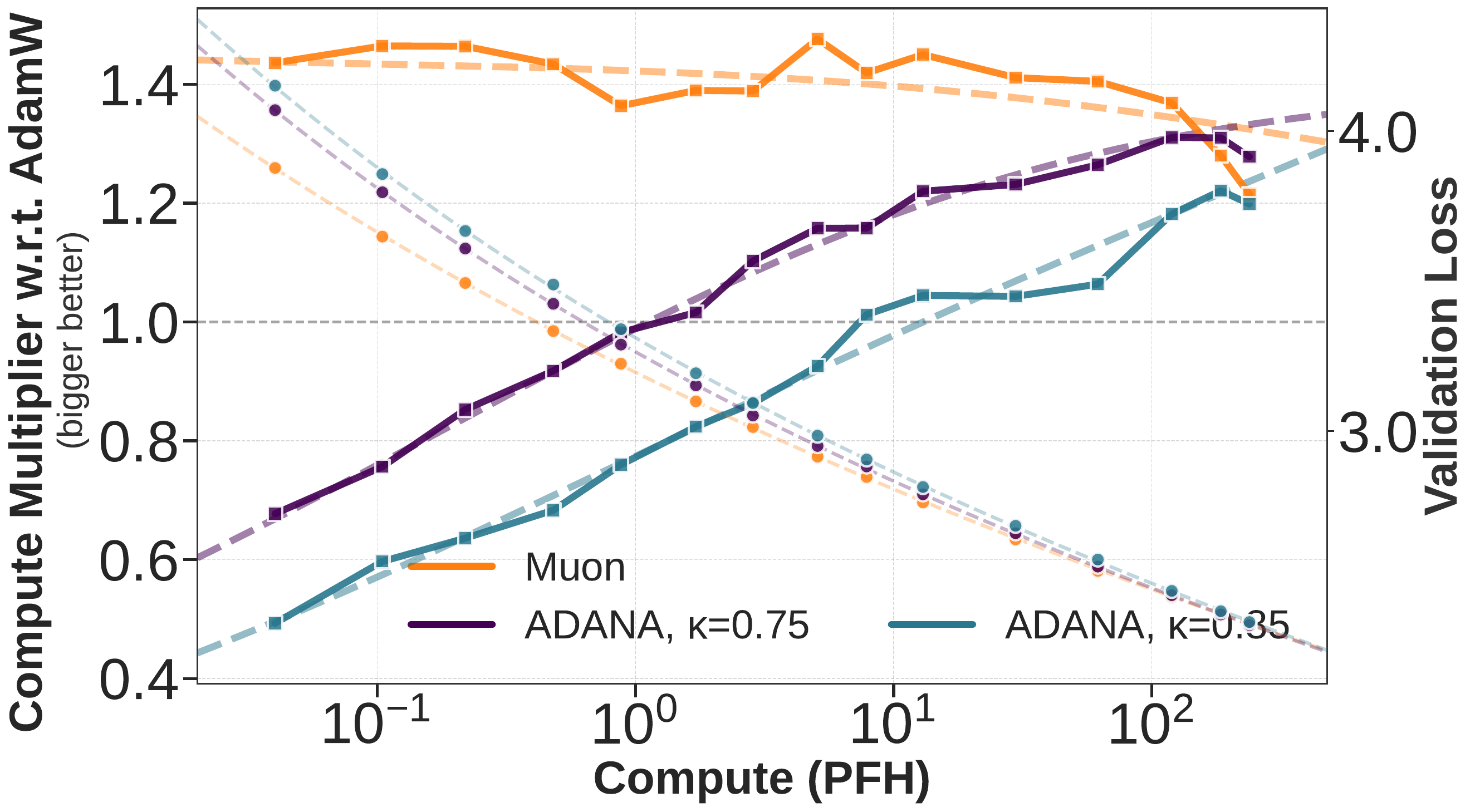}
    \caption{\small \textbf{Large Batch $B=256$ Comparison of \Muon, \AdamW and \ADana with $\alpha(t)=(1+t)^{1-\kappa}$ and $\kappa\in \{0.75,0.85\}$.}}
    \label{fig:batch_scaling_256}
\end{subfigure}
\caption{\textbf{(a)} Same setting as in \Cref{fig:alpha_scaling_chinchilla}: we use $\alpha(t) = \tilde{\alpha}\times (1+t)$ and sweep $\tilde{\alpha}$ across model size and iterations. Increasing batch size allows for larger constant $\tilde{\alpha}$. \textbf{(b)} Under $\alpha(t) = (1+t)^{1-\kappa}$ with large batch $B=256$, \ADana with $\kappa = 0.75$ outperforms $\kappa=0.85$ hence allowing larger damping schedule. The performance increases with scale for \ADana with both $\kappa \in \{0.75, 0.85\}$ eventually outperforming \Muon.}
\end{figure}

\paragraph{Effect of Batch Size} Increasing the batch size $B>0$ reduces the stochastic noise in the gradient estimates. Therefore, larger batch sizes should reduce the need for Nesterov damping $\alpha(t) = \tilde{\alpha} (1+t)^{1-\kappa}$. In the limit of very large batch $B \rightarrow \infty$ we expect no damping to be required for full-batch Nesterov $\tilde{\alpha}=1,\ \kappa=0$. Hence we expect that increasing $B$ favors larger $\alpha(t)$ schedules in the form of either larger optimal $\tilde{\alpha}$, or smaller $\kappa$.

Sweeping through $\tilde{\alpha}$ in the setting $\alpha(t)= \tilde{\alpha}\times (1+t)$ in \Cref{fig:batch_effect} shows that increasing batch favors larger $\tilde{\alpha}$. We fitted $\tilde{\alpha} = \hat{\alpha} \times T^{-\kappa}$ where $\hat{\alpha},\kappa$ are fitted variables across model size and iterations. While no clear behavior for $\hat{\alpha},\ \kappa$ seem to emerge from \Cref{fig:batch_effect}, we tested two strategies for scaling the damping factor $\alpha(t)$ at larger batch regimes. The first strategy consists in using smaller exponent $\kappa$ while keeping $\tilde{\alpha}=1$ and was used at batch $256$ in \Cref{fig:batch_scaling_256}. The second strategy consist in keeping $\kappa=0.85$ but using a larger constant $\tilde{\alpha}$ and was used in \Cref{fig:batch_512_scaling} with batch $512$. In both cases, we observe that \ADana's performance consistently increases with scale with strong performance against \AdamW at the largest scales. Note that at small scale, \ADana's performance is strongly degraded which can be explained by an excessively large batch for the given scale and a small number of iterations ($T=1743$ at the smallest scale, $\text{Head}=6$, for batch size $B = 256$).

\subsection{Scheduling $\beta_2$ with Log-time}
\label{sec:beta2}
Now that we have seen the benefits of log-time for $\beta_1$, we turn now to how to schedule $\beta_2$. Using a long-momentum schedule $\beta_1(t) = 1-\nicefrac{\delta}{(\delta+t)}$ increases the contribution of past gradients in the update. If using a constant $\beta_2$, the contribution of these gradients in the second-moment update $v_t$ is killed exponentially fast. In the case of time sparsity, where a single gradient coordinate is updated only rarely, this can lead to diverging updates due to division by exponentially small numbers. For example, assume that $\epsilon=0$ for simplicity and that the first gradient has non-zero first coordinate $(g_0)_1\neq 0$ is  and that for any $t\geq 0$, $(g_t)_0$, we see that the update ratio becomes \[
\left(\frac{m_t}{\sqrt{v_t}}\right)_1 = 
\frac{(1-\beta_1(0))\prod_{s=1}^{t-1}\beta_1(s)}{\sqrt{(1-\beta_2(0))\prod_{s=1}^{t-1}\beta_2(s)}}.
\]
This leads to the stability condition $\nicefrac{\beta^2_1}{{\beta_2}}<1$ in order for the ratio to not explode (see \Cref{thm:long_adam_divergence} for more details); In particular, any schedule $\beta_2 \geq\beta_1 $ satisfies this stability condition. The diagonal choice $\beta_2 = \beta_1 $ has additionally been observed to be practically optimal and avoids various potential pitfalls of \AdamW when $\beta_2 > \beta_1$ \citep{orvieto2025adam}.
In \Cref{fig:comparison_adana_variants}, we compare the performance of \ADana with constant $\beta_2=0.999$ and logarithmic time $\beta_2 = 1-\nicefrac{\delta}{(\delta+t)}$. While both perform similarly at small scale, \ADana with constant $\beta_2=0.999$ diverges at $36$ heads showing the instability arising from constant $\beta_2$ used with long momentum.

\begin{table}[t]
\centering
\caption{\textbf{Optimizers in Ablation Study (Fig.~\ref{fig:ablation_small}).} Details  Sec.~\ref{sec:ablation_adana}.}
\label{tab:adana_variants_small}
\centering
\begin{minipage}{0.45\textwidth}
\centering
\resizebox{\textwidth}{!}{%
\begin{tabular}{l|ccc}
\toprule
& \multicolumn{3}{c}{$\theta_{t+1} = \theta_t - \gamma(t) \left ( \gamma^* \cdot \frac{g_{t+1}+\alpha(t) \cdot m_{t+1} }{\sqrt{v_{t+1}} + \epsilon} + \lambda(t) \cdot \theta_t \right )$} \\
\textbf{Optimizer} & $\alpha(t)$ & $\beta_1$ & $\beta_2$ \\
\midrule
\ADana 1 & $(1+t)^{1-\kappa}$ & Log & Log \\
\ADana 2 & $(1+T)^{1-\kappa}t/T$ & Log & Log \\
\ADana 3 & $(1+t)^{1-\kappa}$ & Log & $0.999$ \\
\Ademamix (DW) & $(1+T)^{1-\kappa}t/T$ & Approx Log & $0.999$ \\
\bottomrule
\end{tabular}%
}
\end{minipage}
\end{table}

\begin{figure}[t]
    \centering
    \includegraphics[width=0.9\linewidth]{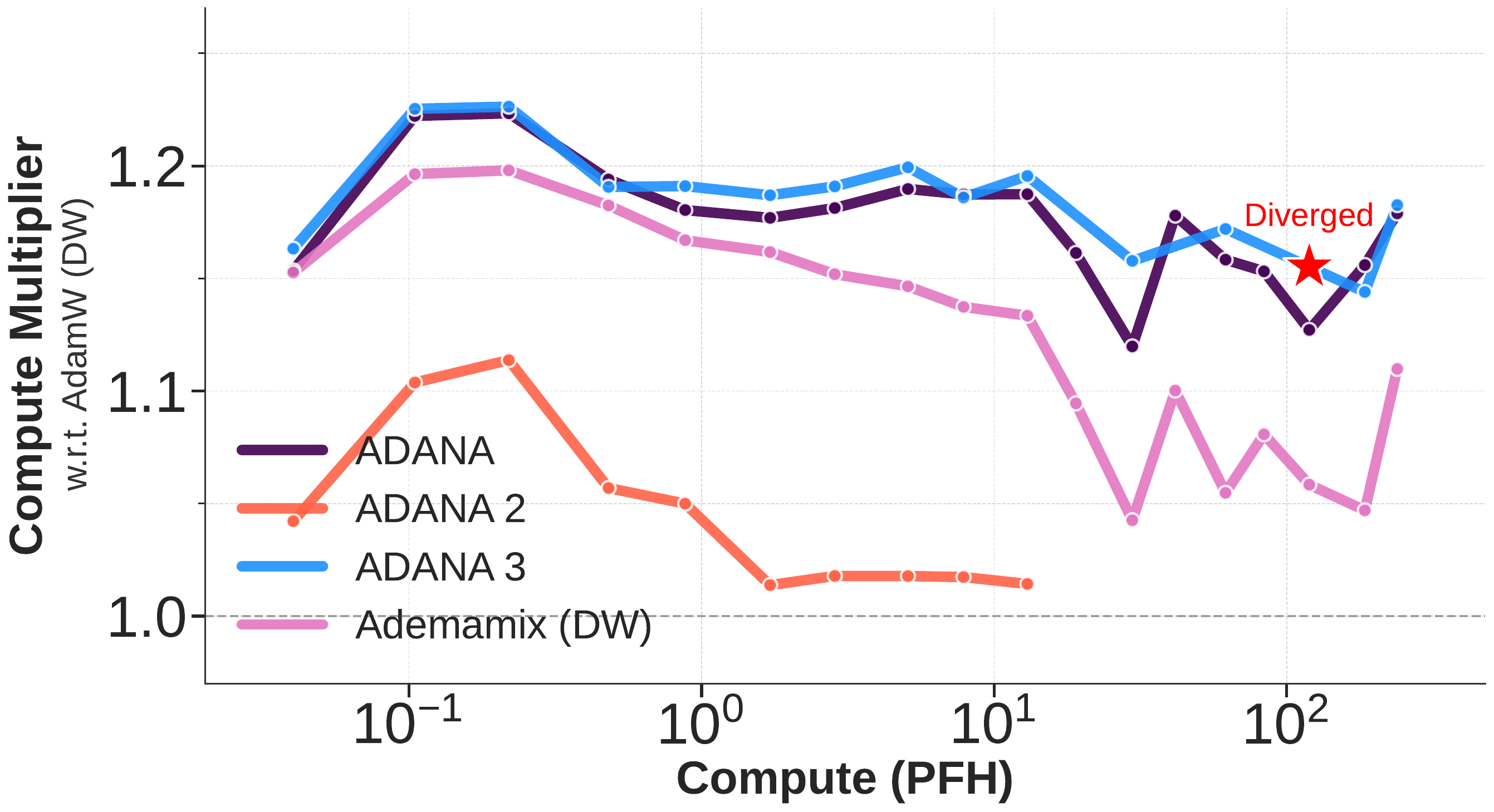}
    \caption{\small \textbf{Ablation Study on \ADana.} Comparison of compute gains relative to \AdamW for different variants of \ADana. \Ademamix and \ADana 1 significantly underperform \ADana, especially at larger scales due to the $\alpha(t)$ schedule. $\beta_2$ constant do not impact performance much but diverges at $36$ heads.}
    \label{fig:ablation_small}
\end{figure}



\begin{figure*}[t]
\centering

\begin{subfigure}[t]{0.32\textwidth}
    \centering
    \includegraphics[width=0.96\linewidth]{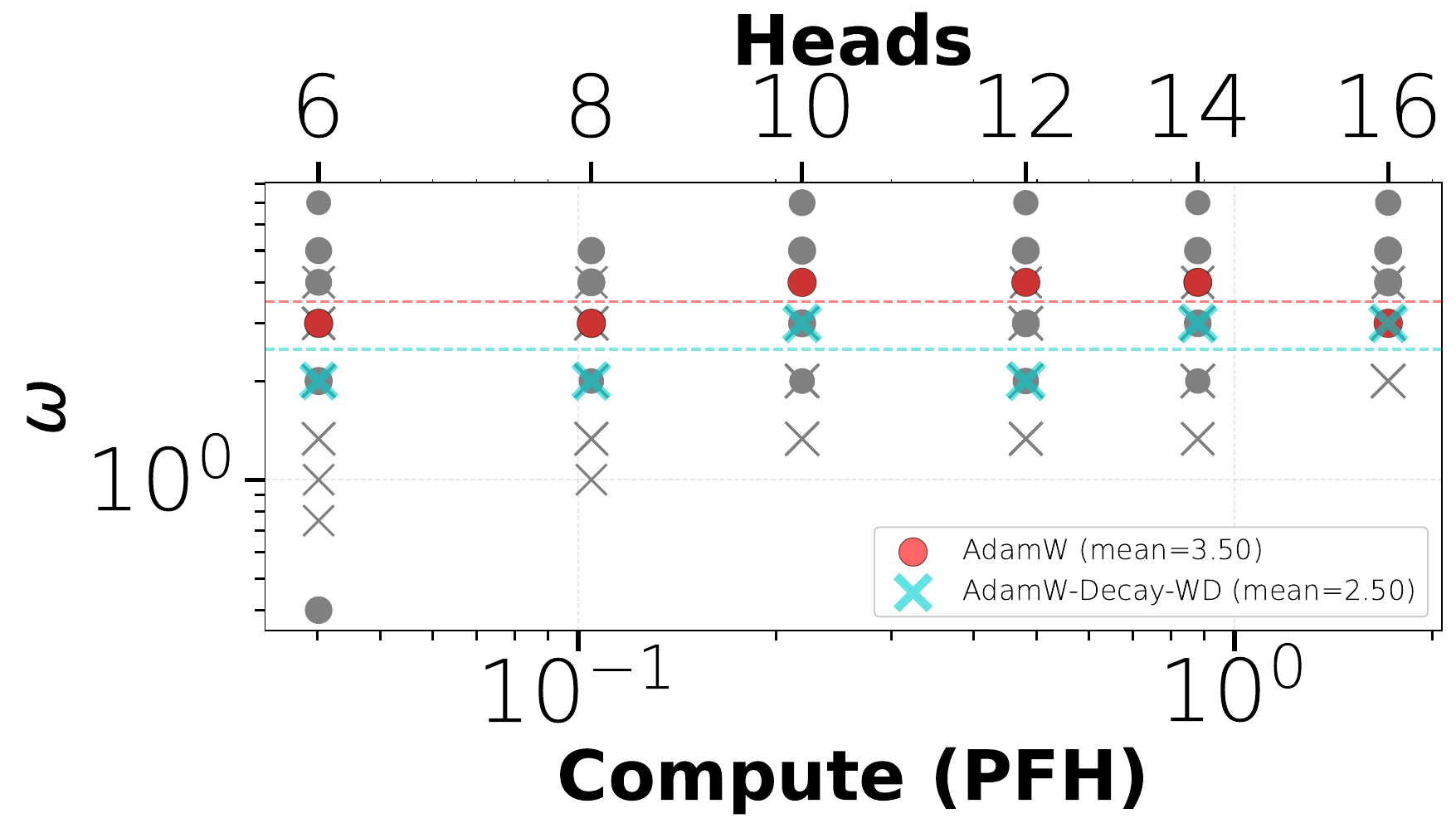}
\end{subfigure}
\begin{subfigure}[t]{0.32\textwidth}
    \centering
    \includegraphics[width=0.96\textwidth]{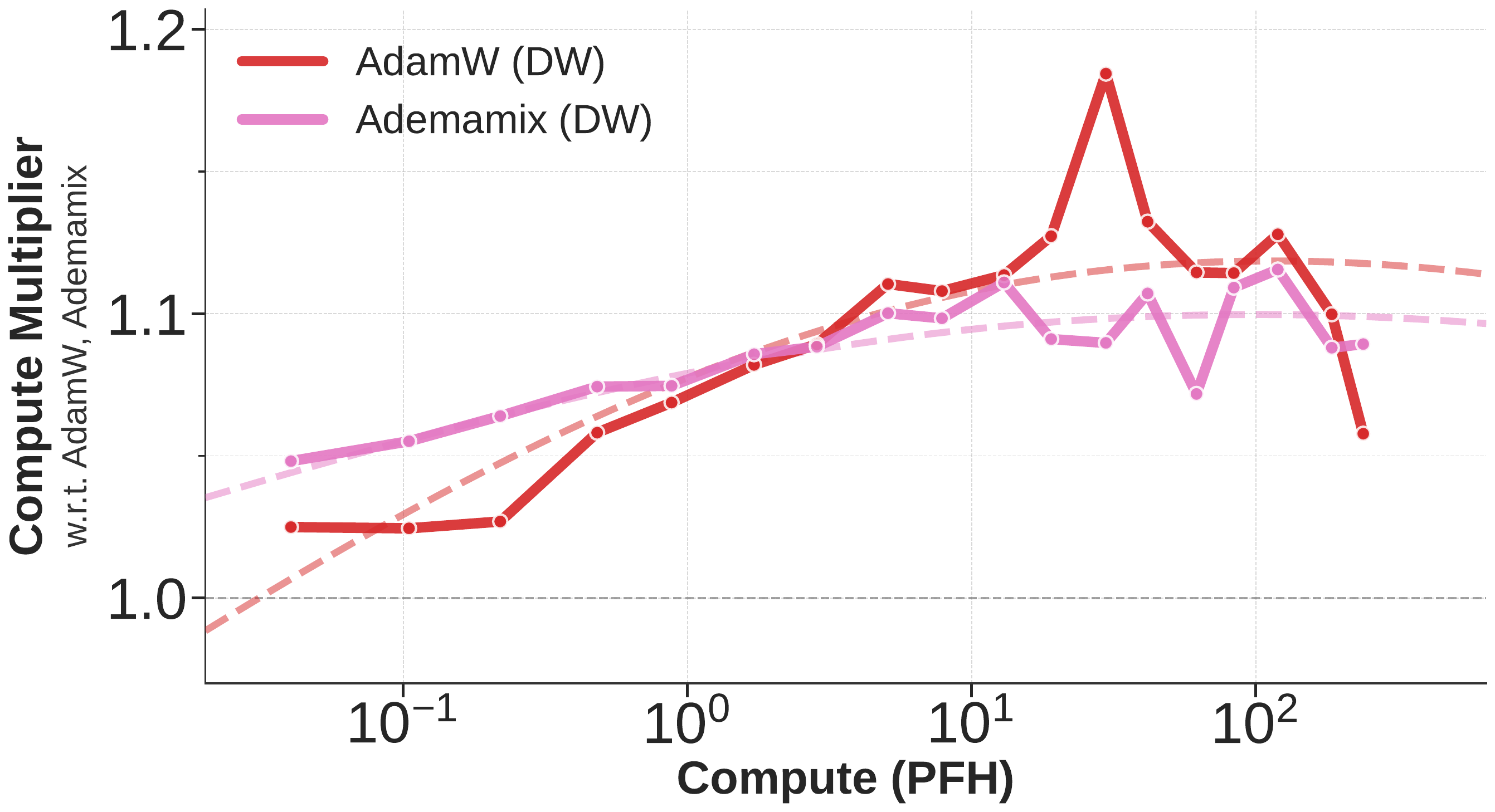}
\end{subfigure}
\begin{subfigure}[t]{0.32\textwidth}
    \centering
\includegraphics[width=0.96\linewidth]{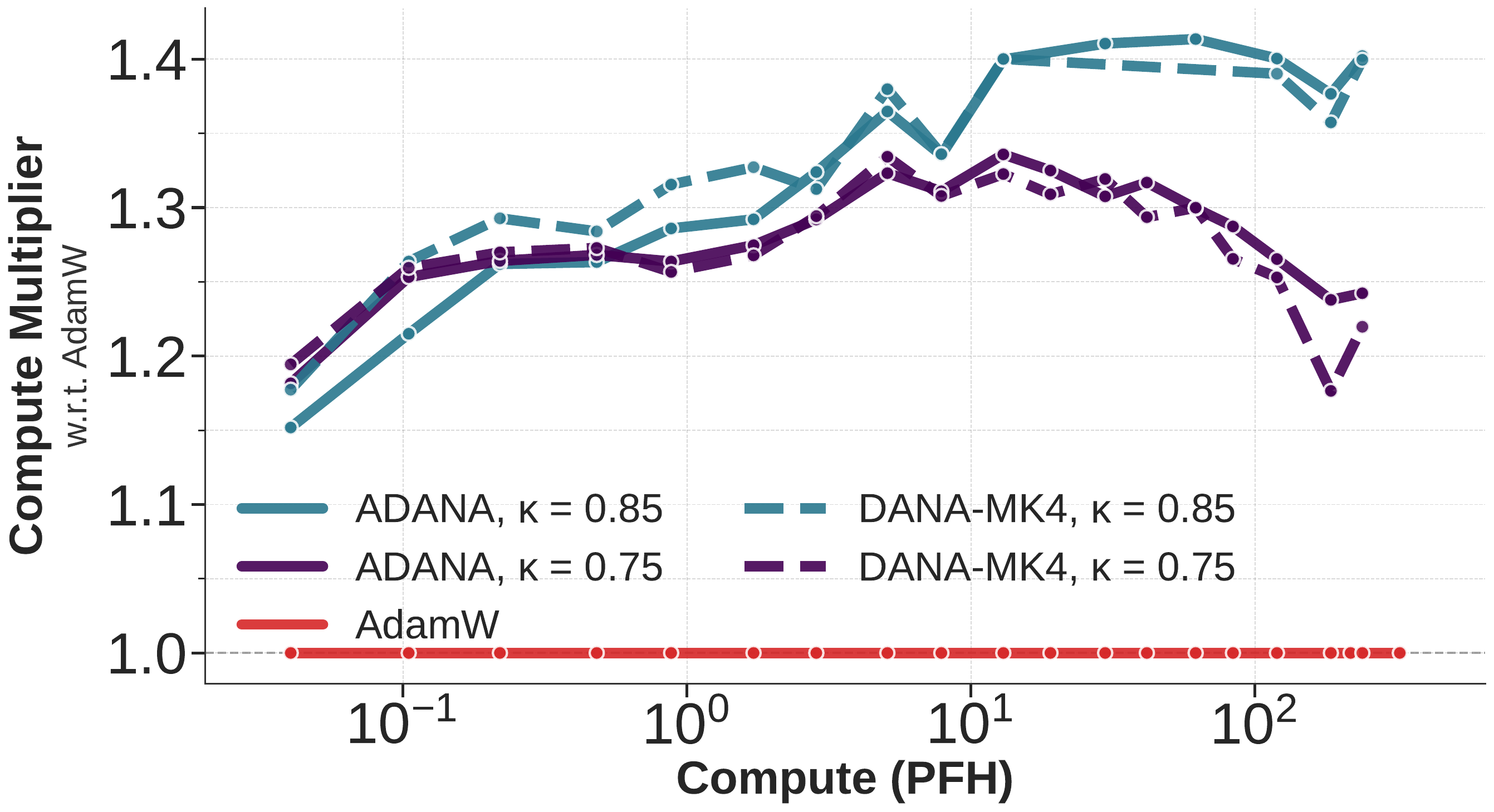}
\end{subfigure}
    \caption{\textbf{Left: Sweeps of $\omega$ for constant weight-decay $\lambda = \nicefrac{\omega}{T}$ and log-time weight-decay $\lambda(t) = \nicefrac{\omega}{(T/10 + t)}$.} The optimal weight decay seems in both cases constant across scales. \textbf{Center: Compute gains for \AdamW and \Ademamix with logaritmic-time versus constant weight-decay.} All experiments used $\omega=4.0$. \textbf{Right: Comparison of \ADana and \DanaMKfour with $\kappa\in \{0.75, 0.85\}$.} Both algorithms are optimal with $\kappa=0.85$ and share the same performance across scales.
    \vspace{-0.3cm}
    }
    \label{fig:opt_weight_decay_sweeps}
\end{figure*}

\subsection{Logarithmic-time Weight-Decay}
\label{sec:weight_decay}

We analyze weight-decay schedules and show that log-time scheduling provides substantial improvements for free and across optimizers. We use decoupled weight-decay with the \citep{loshchilov2017decoupled} parameterization, sometimes referred to as \emph{independent weight-decay}\footnote{Caveat emptor: PyTorch and Optax use a different parameterization; c.f. \cite{wortsman2023small}}:
\begin{equation} \tag{Independent-WD}
\label{eq:decoupled_wd}
    \theta_{t+1} = \theta_t -\gamma(t) (\gamma^* g_{t+1} - \lambda(t) \theta_t).
\end{equation}
The logarithmic-time weight-decay schedule $\lambda(t) = \nicefrac{\omega}{t}$ is naturally suited to problems in which information must accumulated on growing time scales.




We compare two schedules: (i) constant weight-decay, most standard in prior work, under the scaling $\lambda = \nicefrac{\omega}{T}$, and (ii) logarithmic-time approximation $\lambda(t) = \nicefrac{\omega}{(\nicefrac{T}{10} + t)}$, which is a version of $\lambda(t) = \nicefrac{\omega}{t}$ with a warmup period of $\nicefrac{T}{10}$ (denoted by \textit{decaying weight-decay, DW}). The choice of $10$ made \AdamW (DW) perform well; see Figure~\ref{fig:WD_Omega} for other choices besides $10$ and its effect on \ADana and \AdamW. Joint sweeps over $(\gamma, \omega)$ shown in (Fig.~\ref{fig:opt_weight_decay_sweeps}, left) evidence that in both cases the optimal $\omega$ remains constant across model scales, validating the proposed scaling (c.f.\ Sec~\ref{sec:extended_weight_decay} (Fig.~\ref{fig:wd_heatmap_adamw} \&  \ref{fig:wd_heatmap_adamw_decaying}). Empirically, log-time weight-decay consistently outperforms constant decay for both \AdamW and \Ademamix (Fig.~\ref{fig:opt_weight_decay_sweeps}, center). 
Importantly the compute efficiency gains are significant (about $10\%$) and appear to increase with scale

\paragraph{Previous works on weight-decay scaling.}
Finally, several recent works have established scaling rules for weight-decay that align with our analysis, in particular the schedule $\lambda = \nicefrac{\omega}{T}$ (see Sec.~\ref{sec:related_work}). In our framework, \citet{wang2024set,bergsma2025power} show that the \AdamW timescale $\tau \coloneqq \nicefrac{B}{\lambda D}$, where $D$ is the dataset size and $B$ the batch size, should remain constant across scales. Noting that $T \coloneqq \nicefrac{D}{B}$ is the number of iterations, this implies $\lambda \propto \nicefrac{1}{T}$. However, to our knowledge, no prior work apply weight-decay in log time as $\lambda(t) = \nicefrac{\omega}{t}$. In \cite{golatkar2019time}, the authors show that weight-decay in SGD improves generalization when applied early in training but can be neutral or harmful if applied later. This motivates our time-decaying weight-decay schedule, which enforces stronger regularization early and gradually relaxes it as training progresses.


\subsection{Summary of Main Experimental Results}
\label{sec:experiments}

We present comprehensive experiments comparing ADANA and its variants against \AdamW, \Muon, and \Ademamix across transformer scales from 45M to 2.6B parameters. 
Fig.~\ref{fig:scaling_main} shows the main scaling comparison. We summarize:

\begin{enumerate}
    \item \textbf{\ADana outperforms \AdamW} across all model sizes, with compute-efficiency gap increasing at larger scales; $\kappa$ parameter strongly influences performance with the optimal at $\kappa=0.85$. \textbf{\DanastarMKfour} significantly outperforms \AdamW but worse performance than \textbf{\ADana}, with a gap which seems to decrease with scale.
    \item Both \textbf{\Muon} and \textbf{\Ademamix} with appropriate scheduling (see Sec.~\ref{sec:appendix_ademamix}) show competitive performance at smaller scales but the advantage diminishes at larger scales.
    \item \textbf{\AdamW with weight-decay in log-time} (\AdamW (DW)) achieves substantial improvement compared to \AdamW, by itself around $10\%$ compute gains.
\end{enumerate}

\section{Hardening \ADana: From \ADana to \DanastarMKfour}

\label{subsec:hardening}

While \ADana performs well across a wide range of settings, we identify potential failure modes that motivate hardened variants. We summarize the key issues and solutions here; full details (including definitions and diagnostic experiments) are provided in \Cref{sec:appendix_failure_modes} and \Cref{sec:log_momentum_beta_2}.


\begin{figure*}[t]  
\centering

\begin{subfigure}[t]{0.32\textwidth}
    \centering
    \includegraphics[width=\linewidth]{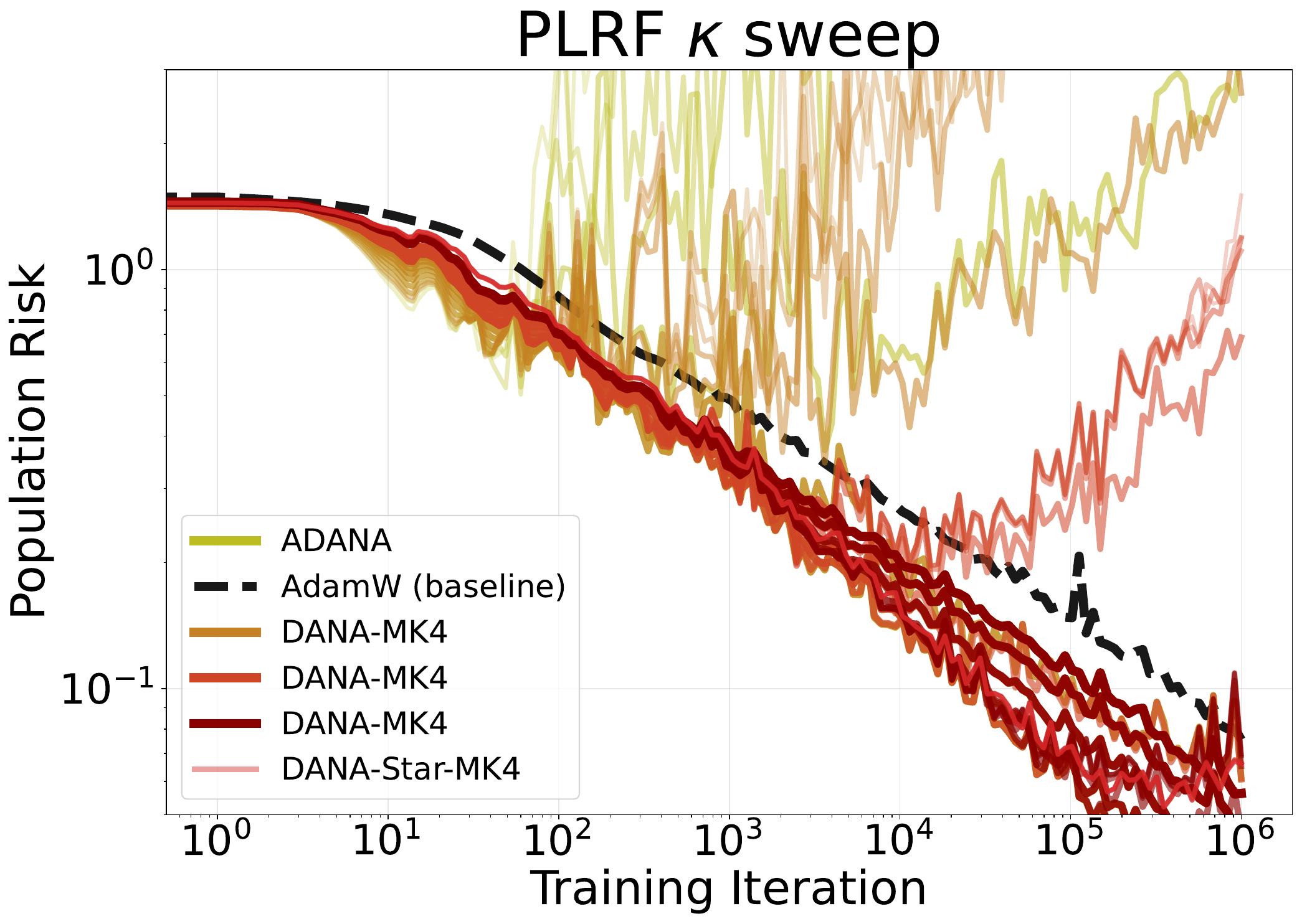}
\end{subfigure}
\hfill
\begin{subfigure}[t]{0.32\textwidth}
    \centering
    \includegraphics[width=\linewidth]{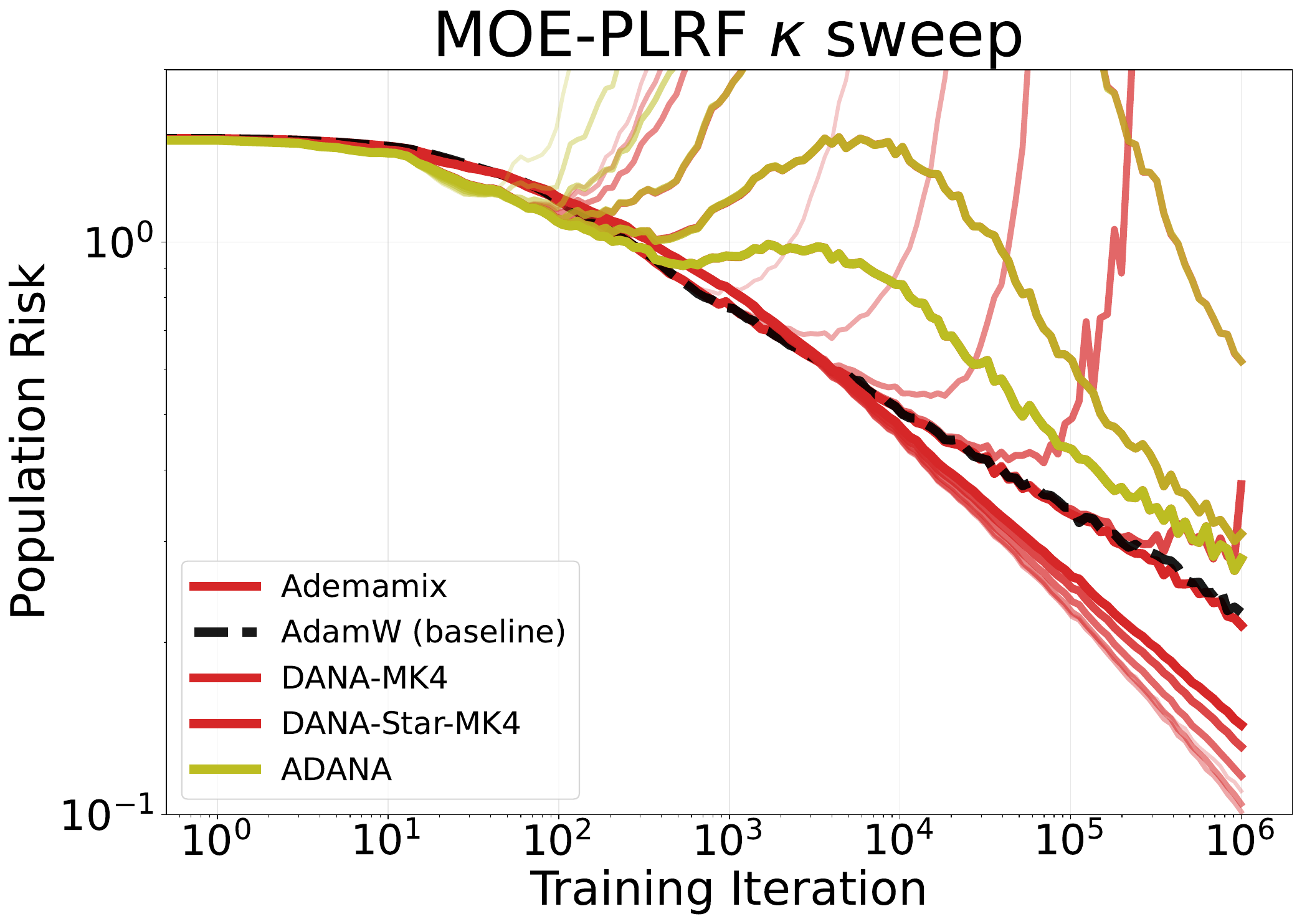}
\end{subfigure}
\hfill
\begin{subfigure}[t]{0.32\textwidth}
    \centering
    \includegraphics[width=\linewidth]{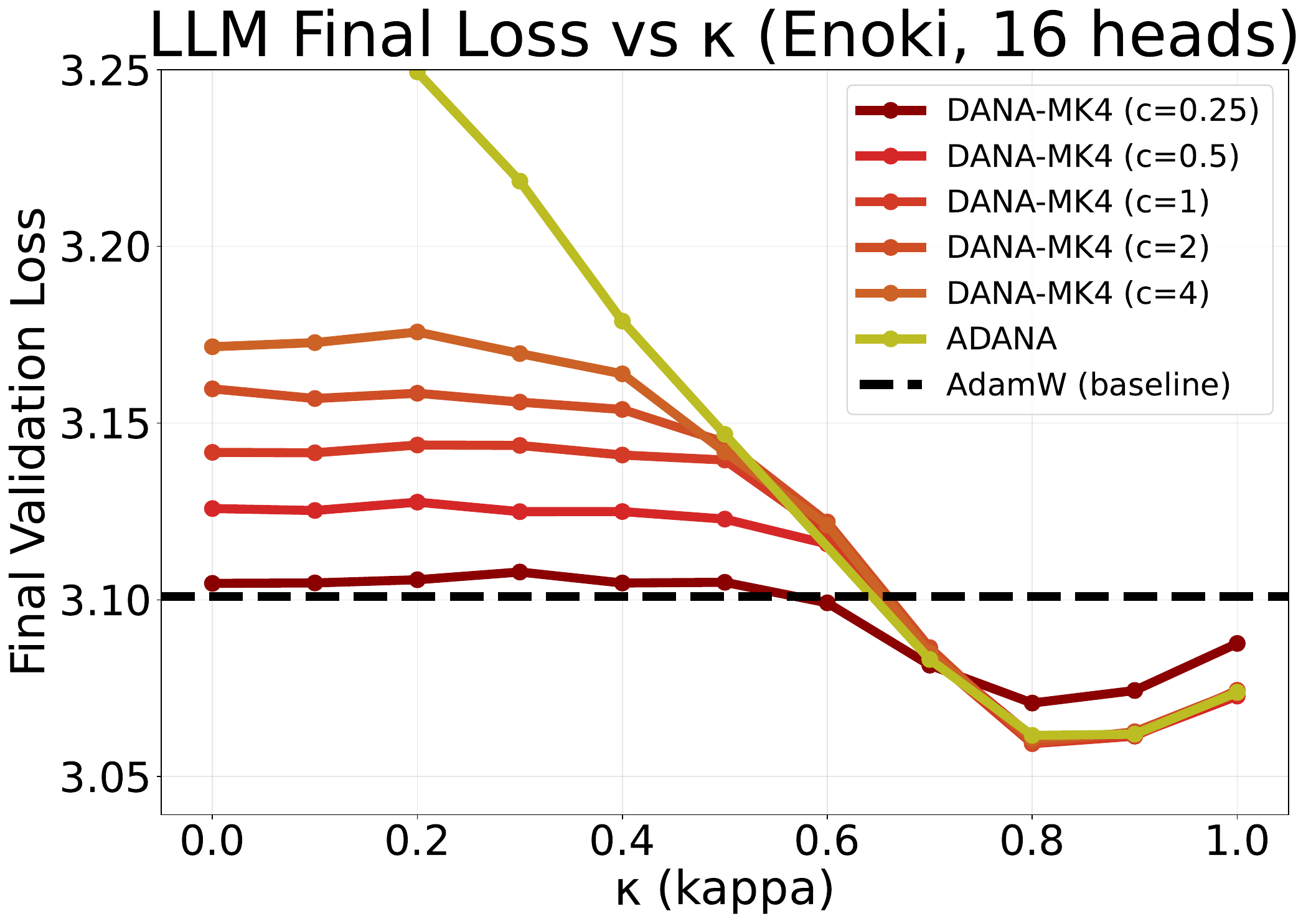}
\end{subfigure}

\caption{ \small \textbf{Left: PLRF $\kappa$ sweep ($m=1$, batch size 1).} Population risk vs.\ training iteration for different $\kappa$ and \texttt{clipsnr} values. \ADana (clipsnr$=16$, olive) performs well at intermediate $\kappa$ but diverges at low $\kappa$. \DanaMKfour variants with lower clipsnr (red shades) are more robust (darker $\implies$ lower \texttt{clipsnr}). \textbf{Center: MOE-PLRF $\kappa$ sweep ($m=1000$, batch size 10).} Sparse gradients due to the mixture-of-experts structure cause \ADana, \Ademamix, and \DanaMKfour to diverge while \DanastarMKfour remains stable. \textbf{Right: LLM $\kappa$ sensitivity.} Final validation loss vs.\ $\kappa$ for varying \texttt{clipsnr}. \ADana shows strong dependence on $\kappa$. \DanaMKfour variants with lower clipsnr values are more robust, with clipsnr$=0.25$ showing nearly flat performance for small $\kappa$. \AdamW baseline (dashed) is outperformed by all variants at $\kappa \approx 0.80$. \vspace{-0.4cm}}
\label{fig:hummingbird_combined}
\end{figure*}


\paragraph{Inhomogeneous Spectral Dimensions.} The spectral dimension $\kappa$ depends on the effective data exponent on PLRF\footnotemark[3]. Within a single transformer, different layers may have different effective spectral dimensions. See, e.g., Fig.~\ref{fig:layer_plots_after_training_head_41}. If $\kappa$ is too large for some layers, the log-time momentum term can dominate and cause instability. This motivates \DanaMKfour. The key modification is clipping the damping schedule $\alpha(t)$ by a sign-to-noise ratio (SNR) estimator. Specifically, define normalized momentum factor
\begin{equation}
\texttt{norm} = \tfrac{1}{\sqrt{v_{t+1}} + \epsilon}, \quad \texttt{mfac} = |m_{t+1}| \cdot \texttt{norm}.
\end{equation}
The quantity $\texttt{mfac}$ measures the magnitude of the momentum relative to the second moment, serving as a proxy for the signal-to-noise ratio. The clipped scaling factor then is
\begin{equation}
\alpha_{\texttt{fac}} = \min\left\{ t^{1-\kappa} \cdot \texttt{mfac}, \, \texttt{clipsnr} \right\},
\end{equation}
where $\texttt{clipsnr}$ is a hyperparameter controlling the maximum allowed scaling (see Fig.~\ref{fig:hummingbird_combined} for sensitivity to $\texttt{clipsnr}$; if not specified, $\texttt{clipsnr} =2$). This clipping prevents the momentum term from dominating, providing automatic adaptation to the local effective dimension. The term $\alpha_{\texttt{fac}} \cdot \text{sign}(m_{t+1})$ replaces the momentum term in Line 9 (See Alg.~\ref{alg:dana_mk4}). This makes \ADana robust to a range of effective spectral dimensions. See Fig.~\ref{fig:opt_weight_decay_sweeps} (right) comparing \ADana and \DanaMKfour with $\kappa\in \{0.75, 0.85\}$. 



\paragraph{Sparse Gradient Vulnerability.} The long second moment buffer (necessary to balance the first moment average) is vulnerable to gradients that are sparsely updated. This occurs in embedding layers, where the Zipf-law distribution of token frequencies means rare tokens contribute gradients infrequently (see \Cref{fig:tau_stats_summary}).  It potentially arises in Mixture of Experts models and transformer blocks, through mechanisms like attention sinks \citep{xiao2023efficient}.

We prove in Sec.~\ref{sec:proofs} that naive logarithmic-time momentum diverges \textit{under sparse gradients}:
\begin{tcolorbox}[colback=gray!5, colframe=gray!50, boxrule=0.5pt, left=6pt, right=6pt, top=3pt, bottom=3pt]
\textbf{Thm.~\ref{thm:long_adam_divergence}.} \Adam with $\epsilon=0$ and $\beta_1(t) = 1-\nicefrac{\delta}{\delta+t}$ will diverge for \emph{all} choices of $\beta_2$ in the presence of sparse updates.  The same holds for \ADana.
\end{tcolorbox}



The solution is \Danastar, which introduces a probability estimator $\tau$ that tracks how frequently each parameter receives meaningful updates (See Alg.~\ref{alg:danastar} \& Sec.~\ref{sec:appendix_failure_modes} for full details). The effective time $t_{\texttt{eff}} = \max\{t \cdot \tau, 1\}$ rescales the iteration count by the update probability, ensuring that rarely-updated parameters do not accumulate excessive momentum and the 2nd moment buffer is not vanishing due to long breaks between updates. 

\medskip

\noindent \textbf{Validation on Synthetic Experiments.} 
We evaluate the proposed variants on PLRF\footnotemark[5] and its Mixture-of-Experts extension (MOE-PLRF; see Sec.~\ref{sec:appendix_synthetic}), which isolate the effects of $\kappa$ misspecification and sparse gradients, respectively. On PLRF (Fig.~\ref{fig:hummingbird_combined}, left), \ADana is highly sensitive to $\kappa$: too small values yield insufficient acceleration, while overly aggressive choices lead to divergence. In contrast, the stabilized \DanaMKfour variants remain robust across a wide range of $\kappa$, with SNR clipping automatically limiting the noise. On MOE-PLRF (Fig.~\ref{fig:hummingbird_combined}, center), which introduces sparse gradients, these differences are amplified: \ADana exhibits a narrow optimal $\kappa$ regime, whereas \DanaMKfour and \DanastarMKfour maintain stable performance throughout.

\paragraph{Validation on LLMs.} We observe the same trends in language model training (Fig.~\ref{fig:hummingbird_combined}, right). While \ADana achieves strong performance when $\kappa$ is well tuned, it is sensitive to misspecification. $\texttt{clipsnr}$ in \DanaMKfour slightly reduces peak gains but yields markedly flatter performance across $\kappa$, improving robustness—a practical advantage when optimal hyperparameters vary across layers.

\begin{table}[t]
\centering
\caption{\textbf{\ADana variant comparison.} }
\label{table:adana_variants}
\begin{tabular}{l|l}
\toprule
\textbf{Variant} & \textbf{Key Feature} \\
\midrule
\ADana & Base logarithmic-time scheduling \\
\DanaMKfour & + SNR clipping ($\kappa$ robustness) \\
\Danastar & + $\tau$ estimator (sparse gradients) \\
\DanastarMKfour & + both \\
\bottomrule
\end{tabular}
\end{table}



\paragraph{Conclusions \& Limitations.} In this work, we introduced \ADana, which leverages log-time schedules and damping—motivated by the power-law structure of language data—to achieve substantial improvements (40\%) across a scaling ladder. Our functional forms for $\beta_1,\beta_2, \lambda$, and $\alpha$ 
are easily tunable at small scales, require no architectural modifications, and transfer effectively to larger models, with promising indications that they generalize across architectures (Qwen3, Sec.~\ref{sec:appendix_qwen}). 
Log-time weight-decay schedules, alone, boosts performance for many optimizers, including \AdamW. Together, these results demonstrate the potential of scale-aware scheduling that leverages language’s power-law structure to improve LLM optimization. 

While our experiments focus on small batch sizes, exploring larger batch regimes is important; we provide preliminary results in Sec.~\ref{sec:appendix_batch}, which suggest \ADana maintains its scaling advantage; larger batches reduce noise, which allows for even more aggressive momentum scheduling.  We have not deeply investigated \Muon, but we believe that log-time scheduling could equally benefit \Muon; moreover the mechanisms that lead log-time scheduling are likely complementary to those that lead to \Muon's performance gains.

\subsection*{Acknowledgements \& Disclosure of Funding}

This research was enabled in part by compute resources, software and technical help provided by Mila (mila.quebec) as well as by support provided by Calcul Québec (\href{https://www.calculquebec.ca/}{Calcul Québec}) and the Digital Research Alliance of Canada (alliancecan.ca). No Google computing resources or data were used in this submission. The Google DeepMind author’s involvement was limited to an advisory role.

D. Ferbach is supported by a Fonds de recherche du Québec – Nature et technologies (FRQNT) doctoral training scholarship (DOI: \href{https://doi.org/10.69777/363626}{https://doi.org/10.69777/363626}). G. Gidel is a CIFAR AI Chair, he is supported by a Discovery Grant from the Natural Science and Engineering Research Council (NSERC) of Canada and a  Google x Mila research grant.
C. Paquette is a Canadian Institute for Advanced Research (CIFAR) AI
chair, Quebec AI Institute (Mila) and a Sloan Research Fellow in
Computer Science (2024). C. Paquette is supported by a Discovery Grant from the
 Natural Science and Engineering Research Council (NSERC) of Canada,
 NSERC CREATE grant Interdisciplinary Math and Artificial Intelligence
 Program (INTER-MATH-AI), Google x Mila research grant, Fonds de recherche du Quebec Nature et technologies (FRQNT) NOVA Grant (DOI: \href{https://doi.org/10.69777/363444}{https://doi.org/10.69777/363444}), and CIFAR AI Catalyst Grant. Research by E. Paquette was supported by a Discovery Grant from the Natural Science and Engineering Research Council (NSERC) and Google Research Grant. Additional revenues related to this work: C. Paquette has 20\% part-time employment at Google DeepMind. 

 The authors would like to thank Jeffrey Pennington, Lechao Xiao, and Atish Agarwala for their careful reading and helpful feedback that improved the paper.

\newpage

\paragraph{Impact Statement.} 
The work presented in this paper is an empirical study on how to scale hyperparameters for different optimisers in the context of LLM pretraining. While a common positive foreseeable impact of algorithms and hyperparameter schedules with better scaling properties is that they would allow more efficient model training and hence lower energy consumption of AI, it is important to acknowledge that, 1) this research itself had had a negative environmental impact using around 22 years of compute with a H-100 GPU (details in the appendix) 2) many times across history, cost-lowering technological improvements have nevertheless led to an increase in consumption due to the commonly named Jevons paradox. Regarding our experiments, we include results on language models trained on a public language dataset and a synthetic dataset. We do not release any pretrained models, but because the dataset we used is based on Common Crawl, which has potential issues of fairness, harm, accountability, and transparency, we may be contributing to reinforcing these issues by supporting the creation of FineWeb as a standard in NLP, thereby obfuscating these concerns. Finally, we anticipate that one potential negative impact of this line of research is that fundamental advances in scaling will benefit more entities with greater computational resources, thereby enhancing the centralisation of technological power.

\bibliography{references}
\bibliographystyle{plainnat}  


\onecolumn
\appendix


\section*{Appendix Overview}
\label{sec:appendix_overview}

This appendix provides supplementary material organized into the following sections.

\paragraph{Appendix~\ref{sec:related_work}: Related Work.}
We survey related work on adaptive optimization methods, scaling laws for language models, and momentum-based acceleration techniques. This includes discussion of \Adam variants, second-order methods, and prior work on hyperparameter scheduling.

\paragraph{Appendix~\ref{sec:appendix_synthetic}: Synthetic Experiments.}
We describe the Power-Law Random Features (PLRF) model and its Mixture-of-Experts extension (MOE-PLRF), which serve as theoretical testbeds for studying optimizer behavior. These models capture the essential structure of language data while permitting rigorous analysis.

\paragraph{Appendix~\ref{sec:setup}: Experimental Setup.}
We detail our scaling setup, training configuration, hyperparameter scaling methodology, and compute budget calculations. This section also includes detailed specifications of the Enoki model architecture (attention mechanisms, feed-forward networks, normalization, initialization strategy, and parameter count formulas). This section provides the foundation for fair comparison across model scales.

\paragraph{Appendix~\ref{sec:algorithms_appendix}: Algorithms.}
We present complete algorithmic descriptions for all optimizers studied: \AdamW, \Logadam, \Muon, \Ademamix, \Dana variants, and the hardened \ADana family (\DanaMKfour, \Danastar, \DanastarMKfour). Each algorithm includes its specific hyperparameters and update rules.

\paragraph{Appendix~\ref{sec:scaling_laws}: Scaling Laws.}
We analyze single versus broken power-law fits, compare compute formulas (Kaplan $6P$, Chinchilla $6D$, DeepSeek $M$), examine optimizer outscaling behavior, and survey measured exponents from the literature. This section also includes sensitivity analysis for the spectral dimension parameter $\kappa$.

\paragraph{Appendix~\ref{sec:baselining}: Baselining Procedure.}
We describe our systematic hyperparameter search procedure, including the two-stage learning rate search strategy, weight-decay grid, and other training details. Learning rate scaling laws for the Enoki architecture are derived and analyzed.

\paragraph{Appendix~\ref{sec:appendix_qwen}: Qwen3 Experiments.}
We present comprehensive results on the Qwen3 architecture, including model implementation details, scaling law fits, learning rate scaling analysis, and a comparison of scaling rules between Enoki and Qwen3. This validates that our findings generalize beyond the Enoki architecture.

\paragraph{Appendix~\ref{sec:appendix_batch}: Batch Sizes Experiments} We extend our analysis of \ADana to larger batch and show some promising preliminary experiments on \ADana's performance in a larger batch setting.

\paragraph{Appendix~\ref{sec:appendix_data_exponent}: Data Exponent Analysis.}
We measure the activation covariance spectra in transformers to estimate the data exponent $\rho$, connecting empirical observations to PLRF theory and explaining the optimal range of $\kappa$ values.

\paragraph{Appendix~\ref{sec:appendix_ademamix}: AdEMAMix Correspondence.}
We analyze the relationship between \Ademamix's warmup schedules and \Dana's logarithmic-time scheduling, showing how proper scheduling enables similar performance gains.

\paragraph{Appendix~\ref{sec:extended_weight_decay}: Additional Weight Decay Details.}
We give additional details on the sweeps performed for our weight-decay study and additional intuition on why log-time weight-decay might improve performance.

\paragraph{Appendix~\ref{sec:appendix_building_adana}: Building ADANA Theory.}
We develop the theoretical foundations of \ADana, including the role of log-momentum in Nesterov acceleration and the analysis of sparse gradient challenges.

\paragraph{Appendix~\ref{sec:hilberg_hypothesis}: Connection With Hilberg Exponent.}
We discuss a connection between the Hilberg exponent —a concept from information theory—and the logarithmic time schedules in \ADana.

\paragraph{Appendix~\ref{sec:appendix_failure_modes}: Failure Modes \& Hardened Variants.}
We detail the failure modes of naive logarithmic-time scheduling under sparse gradients and describe the design of hardened variants (\DanaMKfour, \Danastar) that address these issues.

\paragraph{Appendix~\ref{sec:appendix_nesterov}: Nesterov Formulations.}
We derive the equivalence between different formulations of Nesterov's accelerated method and show how they relate to the \Dana family of algorithms.

\paragraph{Appendix~\ref{sec:proofs}: Proofs.}
We provide complete proofs of theoretical results stated in the main text, including convergence guarantees and stability analysis.

\paragraph{Appendix~\ref{sec:appendix_compute}: Compute Usage.}
We report the total compute used for all experiments, broken down by model size and optimizer.

\section{Additional Related Work}
\label{sec:related_work}

\paragraph{Momentum and Accelerated SGD Methods.} 
Prior work has established convergence guarantees for stochastic gradient descent with classical momentum (SGD-M) (i.e., momentum parameters held fixed) in both strongly convex and non-strongly convex settings \cite{flammarion2015from, gadat2016stochastic, orvieto2019role, yan2018unified, sebbouh2020almost}\footnote{This is a non-exhaustive list of work on stochastic momentum methods.}.
For quadratic objectives, SGD-M exhibits linear convergence of the iterates (though not in $L^2$) under additional assumptions \cite{loizou2017momentum}. Under additional conditions on the stochastic gradient noise, linear convergence to a neighborhood of the optimum has been shown in \cite{kidambi2018on, can2019accelerated}.
The batch size plays a critical role in the performance of both SGD and SGD-M. For small batch sizes, SGD-M does not necessarily outperform plain SGD \cite{kidambi2018on, paquette2021dynamics, zhang2019which}, whereas acceleration can emerge in the large-batch regime \cite{lee2022trajectory, bollapragada2024on, dang2024accelerated}.


Convergence results for stochastic Nesterov's accelerated method \cite{nesterov2004introductory} (S-NAG) have been established for both strongly convex and non-strongly convex settings \cite{kulunchakov2019generic, assran2020on,aybat2018robust}. Under stronger assumptions—such as the strong growth condition \cite{vaswani2019fast} or additive noise on stochastic gradients \cite{laborde2019nesterov}—convergence to the optimum at an accelerated rate can be guaranteed. As noted in Thm. 7 \cite{even2021continuized} (see also references therein), a naive implementation of stochastic Nesterov acceleration fails to converge in the non-strongly convex setting. 

\ADana, along with \Ademamix \cite{pagliardini2024ademamix}, can be viewed as special instances of  "accelerated SGD" methods (see nonexhaustive list \cite{jain2018accelerating,allen2017katyusha,ghadimi2012optimal, ghadimi2013optimal,Liu2020accelerating, gupta2024nesterov,defazio2024road,krichene2017acceleration}). Recent work by \cite{morwani2025connections} establishes connections between \Ademamix, \ScheduleFreeAdamW \cite{defazio2024road,defazio2025gradients}, and accelerated SGD variants. In particular, the authors propose removing the EMA on the short-term average in \Ademamix—an approach similar in spirit to \ADana—and show that this simplified variant remains competitive in practice. The Lion algorithm \citep{chen2023symbolic} is also related to \ADana as it combines a gradient term and momentum term in the update. However, it doesn't scale the damping schedule (taken constant) neither uses logarithmic time for the momentum (also taken constant). It additionally uses a sign update instead of a division by the second moments.

On the theoretical side, \cite{varre2022accelerated} study accelerated SGD for least-squares regression, demonstrating that acceleration can be achieved using schedules similar to the \Danaconstant schedule introduced in \cite{ferbach2025dimension}. Relatedly, \cite{yarotsky2025sgd,yarotsky2025corner} analyze the acceleration of stochastic momentum methods on quadratic problems under power-law spectral assumptions.

\paragraph{Scaling laws and PLRF.} Scaling laws and compute-optimality for the PLRF model under one-pass SGD have been analyzed by \cite{paquette20244+, bordelon2024dynamical, lin2024scaling} with an extension to feature learning in \cite{bordelon2025how,defilippis2025scaling}. Particularly, \cite{paquette20244+,ferbach2025dimension} developed scaling laws for SGD and \Dana under a deterministic equivalent and showed that there exists 4 distinct phases. They used their scaling law to find compute-optimality. Related work on ridge-regression gradient descent includes \cite{cui2021generalization} and \cite{defilippis2024dimension}. Other theoretical guarantees for scaling laws beyond PLRF, typically for gradient flow, have been established in works such as \cite{nam2024exactly}. Other works have studied SGD under power-law covariance including \cite{carratino2018learning,dieuleveut2016nonparametric,pillaud2018statistical} on the least-squares problem. 


\paragraph{Architecture Scaling.} In \cite{mcleish2025gemstones}, the authors study the impact of architecture choices such as depth vs width scaling or optimizers choices for training large language models. They release a model suite of checkpoints and highlight the sensitivity of scaling laws fits to training choices, significantly impacting the final results.

\paragraph{Hyperparameters Transfer at Scale.}
There is extensive literature on the optimal scaling of optimizer hyperparameters for large neural networks, with particular emphasis on initialization and learning rate~\cite{yang2021tensoriv, everett2024scaling, wortsman2023small}, batch size~\cite{mccandlish2018empirical, shallue2019measuring, zhang2024does}, and the epsilon parameter in adaptive optimizers~\cite{yang2023tensorivb, wortsman2023small, everett2024scaling}. In contrast, momentum hyperparameters in large-scale models are typically treated as fixed constants rather than scaled quantities. Notably, the GPT-3 model~\cite{brown2020language} used $\beta_2 = 0.95$ instead of the more common values $\beta_2 = 0.99$ or $0.999$, a choice that may improve stability at large batch sizes.

A recent line of work focuses on scaling optimizer hyperparameters—such as learning rate and weight-decay—to ensure stable and non-degenerate training dynamics as model size increases. A seminal contribution is \cite{yang2022tensor}, which demonstrates how to optimally transfer hyperparameters from small-scale settings, where tuning is feasible, to large-scale regimes, where tuning is computationally prohibitive. Building on this framework, \cite{dey2025don} introduce the CompleteP parametrization, extending prior approaches to enable \emph{depth transfer}. \cite{mlodozeniec2025completed} further refine this framework by enabling transfer across batch size and number of training iterations, and additionally demonstrate the feasibility of per-module hyperparameter transfer, yielding substantial performance gains.

\paragraph{Batch Scaling.}
\citet{mccandlish2018empirical} showed that training parallelism scales nearly linearly with batch size up to a critical batch size determined by the noise scale of the problem, beyond which performance saturates. \citet{zhang2024does} study how this critical batch size scales during pretraining and show that it correlates more strongly with dataset size than with model size. \citet{marek2025small} investigate the impact of small batch sizes in training, showing that smaller batches often lead to better final losses and that vanilla SGD remains stable when training LLMs with very small batch sizes. They also propose a scaling rule for \AdamW hyperparameters in the small-batch regime. In particular, they suggest scaling \AdamW’s $\beta_2$ with batch size $B$ according to $\log(\beta_2)\propto B$, so that the second-moment EMA averages over a comparable number of tokens across batch sizes. Since their experiments fix the total number of training tokens $D$, we have $B = D / T$, and thus their result is equivalent to $\log(\beta_2)\propto 1/T$, recovering our proposed logarithmic-time scaling $\beta_2 = 1 - \frac{\delta}{T}$.

\paragraph{Optimizers Benchmarking and Performance Comparison.} 
In \cite{semenov2025benchmarking}, the authors compare the performance of \Soap \citep{vyas2024soap}, \Mars \citep{yuan2411mars}, \Dmuon \citep{liu2025muon}, \ScheduleFreeAdamW \citep{defazio2024road}, \Lion \citep{chen2023symbolic}, \Adopt \citep{taniguchi2024adopt}, \Prodigy \citep{mishchenko2023prodigy}, \Sophia \citep{liu2023sophia}, and \Signum \citep{bernstein2018signsgd}. They report strong and consistent performance from \Mars and \Ademamix relative to \AdamW, particularly as the number of training iterations increases, while other optimizers exhibit similar or worse performance. The authors also study the influence of batch size, warm-up, and weight-decay, and find that the use of 
$z$-loss has little impact and can even degrade performance.

Similarly, \cite{wen2025fantastic} compare the performance of \Muon, \NadamW \citep{Dozat}, \Cautious \citep{liang2024cautious}, \Adammini \citep{zhang2024adam}, \Scion \citep{pethick2025training}, \Kron \citep{li2022black}, \Soap, and \Sophia. They find that matrix-preconditioned optimizers perform best at small scale, but that their advantages diminish relative to \AdamW as scale increases. They attribute this discrepancy with \cite{semenov2025benchmarking} to the use of larger batch sizes, which tend to favor matrix-based methods.

While both \cite{semenov2025benchmarking, wen2025fantastic} show that \Muon achieves over $\sim 10\%$ compute gains relative to \AdamW at scale, recent work in \cite{qiu2025hyperparameter} demonstrates that, under appropriate weight-decay scaling, \Muon and \Shampoo can yield improvements of $\sim 30\%$-$40\%$. Specifically, \cite{qiu2025hyperparameter} advocate combining $\mu P$ scaling \citep{yang2022tensor} with weight-decay scaling of the form $\lambda \sim \frac{\omega}{\lr \times \text{width}}$. \citet{shah2025practical} similarly report substantial gains for \Muon over \AdamW, showing in particular that \Muon remains highly effective at large batch sizes, thereby improving the Pareto frontier of compute–time trade-offs under a fixed validation-loss constraint. 
Finally, \cite{zhao2024deconstructing} show that, aside from \SGD, many optimizers—including \Lion, \Signum, and \Adafactor\cite{shazeer2018adafactor}—exhibit performance comparable to \AdamW.

\paragraph{Choosing \AdamW Hyper-parameters} Other works focused on designing efficient or simplified choices of hyper-parameters for \AdamW. \citet{zhang2022adam} identified a condition on $\beta_2$ to be $\frac{\beta_1(s)}{\sqrt{\beta_2(s)}}<1$. More recently \cite{orvieto2025adam} highlighted that picking $\beta_1 = \beta_2$ has empirical and theoretical benefits, hence aligning with our schedules choices for \ADana. Additionally, \cite{zhao2024deconstructing} show that \AdamW with tied $\beta_1=\beta_2$ behaves similarly to \Signum.

\paragraph{Weight-decay Scheduling and Scaling.}
\begin{itemize}[leftmargin = *]
\item \textbf{Constant weight-decay.} 
In \cite{mlodozeniec2025completed}, the Complete(d) parametrization recommends for hidden and unembedding weights a peak learning rate scaled as $\gamma^* \asymp m_N^{-1}m_L^{\alpha-1}\sqrt{\frac{m_B}{m_D}}$ where $m_N, m_B, m_D$ scale respectively as width, batch size and token count and $\alpha\in [1/2,1]$ controls the depth scaling. They predict similarly the decoupled weight-decay $\lambda$ to scale as $\lambda \asymp m_N\sqrt{\frac{m_B}{m_D}}$ on hidden weights. Note that they are not in the independent weight-decay formulation, hence $\lambda$ is hit by $\gamma^*$ in the update. With $\alpha=1$ and rewriting the number of iterations $T = \sqrt{\frac{m_D}{m_B}}$, we see that $\lambda \asymp \frac{1}{\gamma^* \times T}$ which matches our scaling rule $\lambda \asymp \frac{1}{T}$ in the independent weight-decay formulation. \cite{xiao2024rethinking} study alternative scaling rules for learning rate and weight decay. They compare $\mu$P scaling with width-independent weight decay to a regime in which both the peak learning rate and the independent weight decay scale inversely with model width, $\gamma^*, \lambda \propto m_N^{-1}$. They show that the latter scaling yields better performance at large widths, highlighting the importance of identifying appropriate hyperparameter scaling laws to optimize performance at scale.

\item \textbf{Weight-decay schedules.} \citet{defazio2025gradients} shows that constant weight-decay scheduling causes gradient norms to increase near the end of training and proposes to decay the weight-decay schedule on normalized layers as $\lambda(t) \propto \gamma(t) \lambda$ to counter this effect, resulting in lower loss along training. While this schedule decays when using decaying schedules on $\gamma(t)$, it generally does not match the schedule $1/t$. 
In \cite{xie2020understanding}, the authors propose an adaptive method for weight-decay of the form $\lambda(t) \propto \frac{1}{v_t^\rho}$ where $v_t$ is the second moment of the gradients and $\rho$ typically set to $1/2$. 
\end{itemize}


\section{Synthetic Experiments}
\label{sec:appendix_synthetic}

In this section, we discuss two important toy setups studied extensively in the theoretical scaling law literature because of their rich behavior and because they phenomenologically capture aspects of scaling law setups that occur in practice.

\subsection{Power-law Random Features (PLRF)}
\label{sec:plrf_model}

The four-parameter model called \textit{power-law random features} (PLRF) \citep{maloney2022solvable,bahri2021explaining,paquette20244+} is studied extensively in the theoretical scaling law literature because of its rich behavior and because it phenomenologically captures many aspects of scaling law setups that occur in practice \citep{paquette20244+, bordelon2024dynamical, lin2024scaling}.

This model assumes the following. For a data vector $x \in \mathbb{R}^v$ we embed this vector in $\mathbb{R}^d$ using a matrix $W \in \mathbb{R}^{v \times d}$ and construct noiseless targets by dotting a fixed $b \in \mathbb{R}^v$ with the sample $x$. This leads to the formal problem statement:
\begin{equation}
    \label{eq:plrf_loss}
    \min_{\theta \in \mathbb{R}^d} \big \{ \tfrac{1}{2} \mathscr{P}(\theta) \defas \tfrac{1}{2} \EE_x \big [ (\ip{W^Tx, \theta} - \ip{x,b})^2  \big ] \big \}.
\end{equation}

The matrix $W$ allows the model to have variable capacity ($d$) independent of the data dimension, and we choose the matrix $W$ to have entries distributed as $N(0, 1/d)$. The key structural assumptions are as follows.

\begin{assumption}[Power-law data and targets, $\rho$ and $\eta$]
\label{assump:plrf}
The samples $x \in \mathbb{R}^v$ are distributed according to $(x_j) \sim j^{-\rho} z_j$ for all $1 \le j \le v$ and $\{z_j\}_{j=1}^v \sim N(0,1)$. The targets are scalars constructed by dotting the sample $x$ with a signal $b \in \mathbb{R}^v$ whose entries $(b_j) = j^{-\eta}$.
\end{assumption}

Power-law type data distributions are ubiquitous in language, vision, and many other tasks, and these distributions are largely responsible for making this model phenomenologically similar to scaling law setups \citep[Fig.~2,3]{maloney2022solvable}. Moreover, this setup allows for a theoretical analysis of compute-optimal neural scalings -- finding optimal allocation of $d$ given a fixed compute budget \citep{paquette20244+,ferbach2025dimension}. It was shown in \citep{paquette20244+} for a large portion of the $(\rho, \eta)$ this model captures the Chinchilla scaling rule. Without the random matrix $W$, $(\rho, \eta)$ are related to what is known in the literature as source and capacity conditions \citep{caponnetto2007optimal,carratino2018learning, dieuleveut2016nonparametric, pillaud2018statistical}.

The hidden dimensionality $v$ is assumed to be large and proportionate to $d$, so that $v/d \to r \in (1, \infty)$. In the case that $2\rho > 1$, this assumption can be relaxed, in that one can take $v$ larger as a function of $d$ or even $v=\infty$. It should be noted that in many scaling law setups, such as \citep{hoffmann2022chinchilla}, the task scales with the parameter count, so that it is natural to assume $v$ grows as $d$ grows.

The PLRF model exhibits scaling laws: the population risk decays as a power of training time, with the exponent depending on $\rho$ and $\beta$. \Danadecaying and more generally the \Dana algorithms were studied under the PLRF model and were shown to exhibit significant exponent speed-up when $\kappa = 1/(2\rho)$ with $\rho$ the data exponent. In \Cref{sec:appendix_data_exponent}, we measure the activation covariance spectra in transformers and find power-law exponents consistent with the optimal $\kappa$ values observed in practice.

\subsection{Mixture-of-Experts PLRF (MOE-PLRF)}
\label{sec:moe_plrf}

We introduce the \textit{Mixture-of-Experts (MOE) PLRF} to add sparse gradients coming from a mixture-of-experts. In this setup, we have $m$ experts and each expert has their own PLRF model with the data being generated with the same power-law exponents $(\rho, \eta)$. Each sample gets routed to one expert's parameters and only that expert's parameters will be updated. We denote the $i$th expert's parameter by $\theta^{(i)} \in \R^d$. We define the probability of routing to expert $i$ as power-law distributed with exponent $\zeta > 0$, that is,
\begin{equation}
p(i) = \frac{i^{-\zeta}}{\sum_{j=1}^m j^{-\zeta}}.
\end{equation}
The experts each share the same random features matrix $W \in \R^{v \times d}$ where $W_{ij} \sim N(0, 1/d)$. The optimization problem we are interested in solving is
\begin{equation} \label{eq:moe_model}
\begin{gathered}
\min_{\Theta \in \R^{d \times m}} \E_{i \sim p, x} \Big[ \tfrac{1}{2} \big (\ip{W^Tx, \theta^{(i)}} - \ip{x, b} \big )^2 \Big], \\
\text{where } \Theta = [\theta^{(1)} | \theta^{(2)} | \hdots | \theta^{(m)}].
\end{gathered}
\end{equation}

For a batch of $B$ samples $\{x_t^k, b_t^k\}_{k=1}^B$ such that $(x_t^k)_j \sim N(0,j^{-2\alpha})$ and $(b_t^k)_j = j^{-\beta}$ for every $k$, we assign each $(x_t^k, b_t^k)$ to an expert $i_t^k \sim p$. This can be represented by a routing matrix $R \in \R^{m \times B}$,
\[
R_{ik} = \begin{cases}
    1, & \text{if expert $i$ is assigned to sample $k$}\\
    0, & \text{otherwise.}
\end{cases}
\]
Let $X_t \in \R^{v \times B}$ be the matrix of data inputs where column $k$ contains $x_t^k$ and let $Y_t \in \R^B$ be the vector of targets where the $k$th entry contains $\ip{x_t^k, b_t^k}$.

 Frequently-selected experts (small $i$) receive updates with high probability, while rarely-selected experts (large $i$) receive updates with probability proportional to $i^{-\zeta}$. The Zipfian selection distribution mimics the frequency distribution of tokens in natural language: common words appear frequently while rare words appear infrequently.

 This setup creates exactly the sparse gradient structure used to analyze \Adam and its variants with log-time $\beta_1$ and $\beta_2$ (\Cref{sec:sparse_gradient_necessary_condition} and \Cref{sec:appendix_failure_modes}). Particularly, we show that divergent behavior for \Adam-like algorithms when $\beta_1(t) = 1-\delta/(t+\delta)$ and $\beta_2$ held fixed independent of $d$ (see Theorem~\ref{thm:long_adam_divergence} in Section~\ref{sec:sparse_gradient_necessary_condition}).

\subsection{Experimental Validation}

\begin{figure}[t]
\centering
\includegraphics[width=0.5\columnwidth]{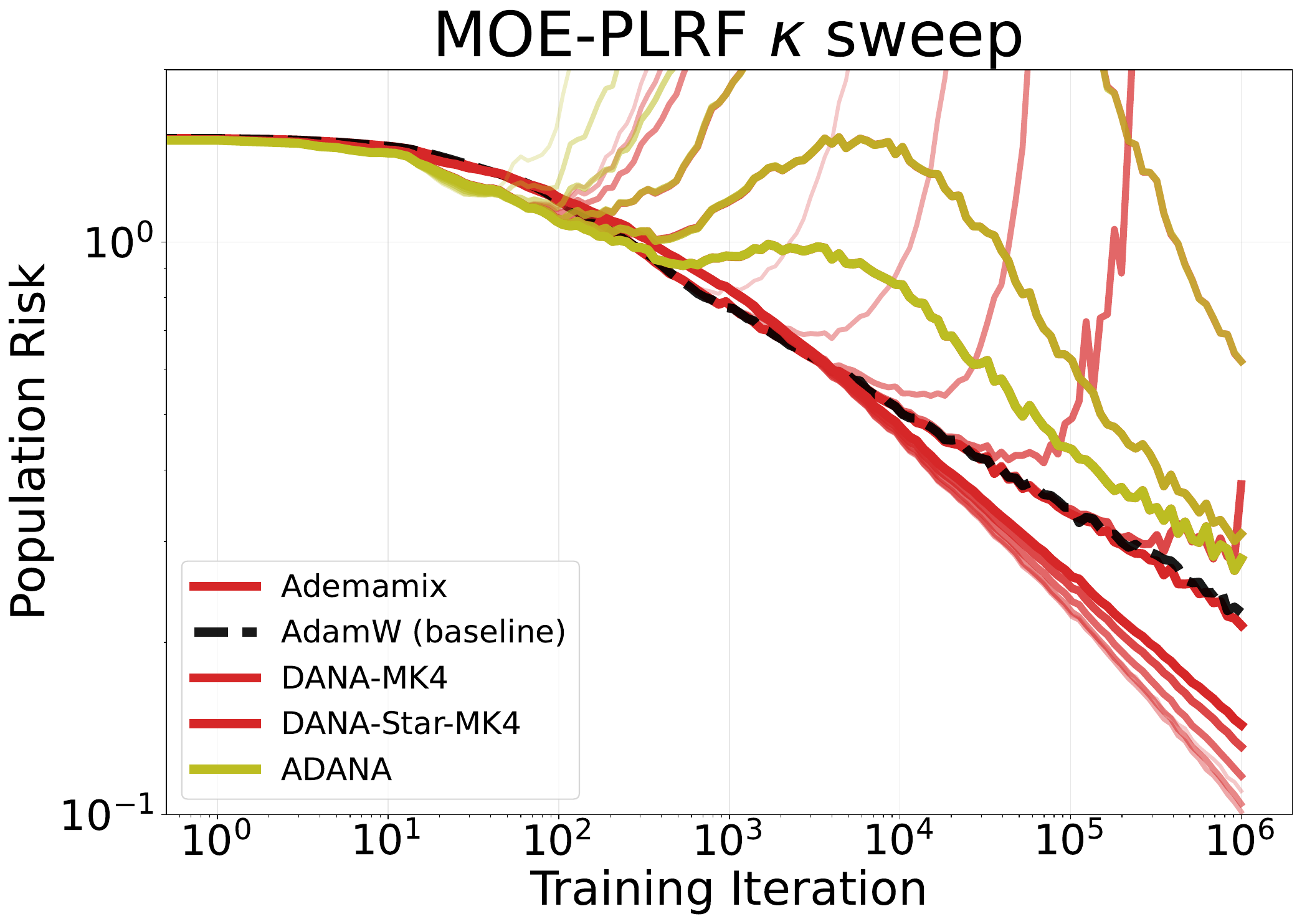}
\caption{\textbf{MOE-PLRF Population risk vs.\ log-time for different $\kappa$ values.}
The optimal $\kappa$ matches the theoretical prediction
$\kappa = 1/(2\rho)$. \Adam baseline shown as a black dashed line.}
\label{fig:hummingbird_appendix}
\end{figure}

Figure~\ref{fig:hummingbird_appendix} shows the population risk versus training time for different values of $\kappa$.
When $\kappa$ is too small, acceleration is insufficient and convergence mirrors
\Adam. When $\kappa$ is too large, instability dominates and the loss grows.
Between these regimes lies a stable band where acceleration occurs.
The center of this band coincides with the theoretical prediction
$\kappa = 1/(2\rho)$ derived from the PLRF scaling law analysis.

We run MOE-PLRF experiments with $m = 100$ experts, data exponent $\rho = 1.5$,
Zipf exponent $\zeta = 1.0$ for expert selection, and $v = 1000$ features per expert.

\paragraph{Connection to Transformer Training.} The MOE-PLRF model connects to transformer training in several important ways. First, token embeddings can be viewed through this lens: each vocabulary token is an ``expert'' that receives gradients only when that token appears in the training batch. The Zipfian distribution of natural language means rare tokens behave like tail experts in MOE-PLRF. Second, attention patterns create similar sparsity: positions that rarely attend to certain keys have sparse gradient patterns for those key projections. Third, modern architectures increasingly use explicit mixture-of-experts layers, where the connection is direct.

The success of \Danastar on MOE-PLRF predicts its benefits for transformer training where sparsity is present and the gradient structure becomes more pronounced.


\section{Experimental Setup: Summary}
\label{sec:setup}

We evaluate our optimizers on decoder-only transformer architectures trained on the FineWeb dataset \citep{penedo2024fineweb} with cross-entropy loss. We adapt the implementation used by \citet{semenov2025benchmarking}.

\subsection{Scaling set-up}
\label{sec:enoki}

\paragraph{Model architecture} We used two different architectures for our experiments: Enoki and Qwen3 \cite{qwen3}. We describe below the Enoki model while details on Qwen3 can be found in \Cref{sec:appendix_qwen}. The Enoki model is a GPT-3-like decoder-only transformer with the following modern improvements: Rotary Positional Embedding (RoPE) \citep{su2024roformer}, QK-LayerNorm applied before RoPE \citep{henry2020query}, residual initialization scaling $1/\sqrt{\text{depth}}$, pre-layer-norm on attention and MLP \citep{xiong2020layer}, and no weight tying \citep{press2017using}. We do not use $z$-loss \citep{chowdhery2023palm}.

To study the performance of optimizers across scales, we design a scaling rule that specifies how to increase each component of the architecture as the number of parameters grows. Following \citet{charles2025communication}, the Enoki model fixes the head dimension at 64 and scales the number of transformer layers $n_{\text{layer}}$ and attention heads $(\text{\# of heads} = \text{heads})$ proportionally:
\begin{align}
n_\text{layer} = \lfloor \nicefrac{3}{4} \times \text{heads} \rfloor.
\end{align}
In this way, the model's width $n_\text{embd} = \text{heads} \times 64$ grows proportionally with the depth $n_{layer}$.

\paragraph{Initialization Strategy}
We use the ScaledGPT initialization scheme. Weights from embedding matrices and linear layers are initialized with mean zero and variance $1/\text{fan}_{\text{in}}$. Weights from projections on the residual branch (attention output and MLP output) have an additional depth-dependent scaling factor, yielding variance $1/(2 \cdot \text{fan}_{\text{in}} \cdot n_{\text{layer}})$. All bias terms are initialized to zero.

\paragraph{Compute Budget and Chinchilla Scaling} Denote $D$ the total number of parameters (including embeddings), $P$ the number of non-embedding parameters, $N$ the number of training tokens and $C$ the compute budget in floating-point operations. The compute budget is usually approximated as $C = 6 \cdot N \cdot P$ \citet{kaplan2020scaling} ignoring optimizer FLOPs, which are typically a small fraction of total FLOPs). For the Enoki model, $P = 12 \cdot n_{\text{embd}}^2 \cdot n_{\text{layer}}$, and the total parameter count is $D = P + 2 \cdot n_{\text{embd}} \cdot 50304$.

Importantly, we follow Chinchilla-optimal scaling \citep{hoffmann2022chinchilla} where training tokens and parameters are scaled proportionally: $N = 20 \cdot D$.

\paragraph{Training Configuration}
All experiments use batch size 32 sequences with sequence length 2048, yielding 65,536 tokens per step. The vocabulary size is 50,304 (GPT-2 tokenizer). All runs use the same random seed for initialization and data ordering. We additionally use warmup and cosine decay on the learning rate as well as decoupled weight-decay \citep{loshchilov2017decoupled}. More details can be found in \Cref{sec:additional_training_details}.

\subsection{Hyperparameter Scaling}
\label{sec:hp_scaling}

To enable fair comparison across model scales, we fit power-law scaling rules to the optimal learning rates found at each model size. For each optimizer, we collect the top-$K$ ($K=5$) learning rates (ranked by final validation loss) at each scale and fit a saturated power-law using weight $(K-k+1)^2\times P$ for the $k$-th best loss at non-embedding parameters $P$:
\begin{equation}
\gamma(P) = a \cdot (b + P)^{d},
\end{equation}
where $P$ is the number of non-embedding parameters and $a, b, d$ are fitted coefficients with $a, b > 0$. The saturation term $b$ prevents divergence at small scales, while the exponent $d$ (typically between $-0.4$ and $-0.9$) controls how aggressively learning rate decreases with model size.

We observe that different optimizers exhibit distinct scaling behavior. Notably, for the Enoki architecture, \AdamW exibits a learning rate scaling as square root of non-embedding parameters $P$ with $d=-0.52$, aligned with recent studies \citep{yang2022tensor,dey2025don,mlodozeniec2025completed}. \Muon maintains relatively high learning rates at scale ($d \approx -0.42$), while \ADana ($\kappa=0.85$) requires more aggressive reduction ($d \approx -0.62$). For Qwen3, the smallest model sizes (below 100M non-embedding parameters) deviate from power-law scaling and are excluded from the fit. The complete fitting methodology, explicit formulas for all optimizers, and visualization of the scaling laws are provided in Appendix~\ref{sec:lr_scaling}.


\subsubsection{Compute Saved Analysis}
\label{sec:compute_saved_analysis}
To quantify the practical benefit of each optimizer, we measure ``compute saved'' through the compute multiplier metric. Let a scaling ladder of compute budgets $C_1, C_2, \cdots C_k$ where $C_{i}$ corresponds to training at the compute-optimal horizon \citep{hoffmann2022chinchilla} model size $i$. For each optimizer B under study (e.g. \ADana, \Muon, \Ademamix ... ) and each model size $i$, we have access to the final loss $\gL_i^B$ of optimizer B under budget $C_{i}$. We can do the same for any baseline algorithm A (often \AdamW or \AdamDW) to get $L_i^B$. This naturally gives rise to an affine by part function $\gL^B: \sR \rightarrow \sR$ such that $\forall i\in \{1,\cdots, k\}$, $\gL^B(C_{i})=\gL_i^B$ with linear interpolation between adjacent points. Outside the observed range, the function is extended to $+\infty$ and $-\infty$ using the slope of the nearest segment, i.e. with slope respectively $\frac{\gL^B_k - \gL^B_{k-1}}{C^B_k - C^B_{k-1}}$ and $\frac{\gL^B_2 - \gL^B_1}{C^B_2 - C^B_1}$.

Fix a model scale $i$ and let $C^B = C_i$ denote the compute budget used by optimizer $B$, yielding final loss $\mathcal{L}^B(C^B)$. We define $C^{A}$ as the compute required by a baseline optimizer $A$ to reach the same loss,
\[
\mathcal{L}^A(C^{A}) = \mathcal{L}^B(C^B),
\]
where $C^{A}$ is obtained by inverting the piecewise-affine extrapolation of $\mathcal{L}^A$.

The compute multiplier and compute efficiency (or relative compute saved) are then respectively
\[
\frac{C^{A}}{C^B}\quad \text{and}\quad  \frac{C^{A} - C^B}{C^B},
\]
which quantifies the additional compute optimizer $A$ would need to match optimizer $B$ at compute $C^B$. Positive values indicate improved compute efficiency. (see also \cite{qiu2025hyperparameter,liu2025muon}). 


\subsection{Architecture Details}
\label{sec:architecture_details}


\subsubsection{Enoki Model Implementation}
\label{sec:enoki_implementation}

\begin{table}[h]
\centering
\caption{\textbf{ScaledGPT Initialization Strategy.} The ScaledGPT initialization scheme uses fan-in variance scaling for general linear layers and embeddings, with additional depth-dependent scaling for residual projection layers. Here $L$ denotes n\_layer, $d_{\text{embd}}$ the embedding dimension, $d_{\text{qkv}}$ the total QKV dimension, and $d_{\text{hidden}}$ the MLP hidden dimension.}
\label{table:init}
\begin{tabular}{c|c}
\toprule
\textbf{Layer Type} & \textbf{Init Std Dev} \\
\midrule
Token Embedding & $\sigma = 1/\sqrt{d_{\text{embd}}}$ \\
LM Head & $\sigma = 1/\sqrt{d_{\text{embd}}}$ \\
Attention QKV (c\_attn) & $\sigma = 1/\sqrt{d_{\text{embd}}}$ \\
Attention Output (c\_proj) & $\sigma = 1/\sqrt{2 \times d_{\text{qkv}} \times L}$ \\
MLP Input (c\_fc) & $\sigma = 1/\sqrt{d_{\text{embd}}}$ \\
MLP Output (c\_proj) & $\sigma = 1/\sqrt{2 \times d_{\text{hidden}} \times L}$ \\
All Bias Terms & $b = 0$ \\
\bottomrule
\end{tabular}
\end{table}

The Enoki model takes its name from the enoki mushroom, reflecting its architectural design with a relatively small head dimension of $64$ and hence many heads to keep the width comparable to the depth.

The transformer architecture underlying Enoki follows the GPT-3 paradigm~\citep{brown2020language} with several modern stabilization improvements inspired by recent scaling studies~\citep{wortsman2023small}. We retain the GPT-3 tokenizer and vocabulary size of $50,304$ tokens.

\paragraph{Attention mechanism.}
The attention implementation employs multi-head self-attention with causal masking. Queries, keys, and values are computed through separate linear projections from the input hidden state. A critical stabilization technique is the application of QK-LayerNorm~\citep{dehghani2023scaling} to the query and key vectors \emph{before} applying Rotary Position Embeddings (RoPE)~\citep{su2024roformer}.

\paragraph{Feed-forward network.}
The MLP component uses a standard two-layer architecture with GELU activation~\citep{hendrycks2016gaussian}. The hidden dimension is set to $4 \times d_{\text{model}}$, following the GPT convention. All linear layers omit bias terms, following~\citet{chowdhery2023palm}.

\paragraph{Normalization and residual connections.}
Enoki employs pre-LayerNorm architecture, applying LayerNorm before each attention and MLP sublayer. A final LayerNorm precedes the language modeling head. No bias terms are used in the LayerNorm layers.

\paragraph{Scaling rule.}
The Enoki scaling rule fixes the head dimension and scales width and depth according to the number of heads $h$:
\begin{align}
d_h &= 64 \text{ (fixed)} \\
n_{\text{head}} &= h \\
n_{\text{layer}} &= \lfloor 3h/4 \rfloor \\
n_{\text{embd}} &= 64h \\
d_{\text{mlp}} &= 4 \times n_{\text{embd}}
\end{align}
This yields a relatively shallow but wide architecture compared to standard scaling rules, which we found beneficial for optimizers that benefit from larger gradient signals per layer.

\paragraph{Parameter count formulas.}
For Enoki:
\begin{align}
\text{per\_layer} &= 12 \times n_{\text{embd}}^2 \\
\text{non\_emb} &= n_{\text{layer}} \times \text{per\_layer} \\
\text{total} &= \text{non\_emb} + 2 \times n_{\text{embd}} \times V
\end{align}
where $V = 50304$ is the vocabulary size. The factor of 12 arises from the attention projections ($4 \times n_{\text{embd}}^2$ for Q, K, V, and output) plus the MLP projections ($8 \times n_{\text{embd}}^2$ for the up and down projections with $4\times$ expansion).

\renewcommand{\arraystretch}{1.1}
\begin{table}[t]
\centering
\caption{\textbf{Architectural Details of the Enoki Model.} The Enoki model fixes head dimension to 64 and scales width and depth linearly with the number of attention heads. MLP hidden dimension is $4 \times d_{\text{embd}}$. Head count is $\tfrac{4}{3}n_{\text{layer}}$.}
\label{table:enoki_sizes}
\begin{tabular}{c|ccc}
\toprule
\textbf{$D$ (Total)} & \textbf{Layers} & \textbf{Embd Dim} & \textbf{MLP Hidden} \\
\midrule
70.39M & 6 & 512 & 2,048 \\
140.97M & 9 & 768 & 3,072 \\
254.02M & 12 & 1,024 & 4,096 \\
423.69M & 15 & 1,280 & 5,120 \\
664.14M & 18 & 1,536 & 6,144 \\
1.41B & 24 & 2,048 & 8,192 \\
2.62B & 30 & 2,560 & 10,240 \\
\bottomrule
\end{tabular}
\end{table}

\begin{table}[t]
\centering
\caption{\textbf{Scaling of Parameters and Compute in the Enoki Model.} Training tokens follow Chinchilla scaling ($N = 20D$). Compute budget $C = 6NP$ follows \citet{kaplan2020scaling}. PFdays = petaflop-days.}
\label{table:enoki_compute}
\begin{tabular}{c|ccc}
\toprule
\textbf{$D$ (Total)} & \textbf{$P$ (Non-embd)} & \textbf{$N$ (Tokens)} & \textbf{$C$ (PFdays)} \\
\midrule
45.71M & 7.08M & 0.914B & $4.5 \times 10^{-4}$ \\
70.39M & 18.87M & 1.41B & $1.8 \times 10^{-3}$ \\
140.97M & 63.70M & 2.82B & $1.3 \times 10^{-2}$ \\
186.48M & 96.34M & 3.73B & $2.5 \times 10^{-2}$ \\
254.02M & 150.99M & 5.08B & $5.3 \times 10^{-2}$ \\
423.69M & 294.91M & 8.47B & $1.7 \times 10^{-1}$ \\
664.14M & 509.61M & 13.28B & $4.7 \times 10^{-1}$ \\
1.41B & 1.21B & 28.28B & $2.4$ \\
2.62B & 2.36B & 52.34B & $8.6$ \\
\bottomrule
\end{tabular}
\end{table}

\subsubsection{Initialization Details}
\label{sec:initialization_details}

Proper initialization is crucial for stable training of deep transformers, and the choice of initialization scheme interacts non-trivially with the learning rate schedule and optimizer hyperparameters.

\paragraph{The connection between initialization and learning rates.}
There is a symmetry between initialization, learning rate, and parameter multipliers~\citep{yang2021tensoriv} where the same training dynamics can be achieved with different but equivalent (initialization, learning rate) pairs.

\paragraph{Our initialization scheme.}
We use fan-in scaling for all linear layers, with residual projections (attention output and MLP down projection) incorporating an additional depth-dependent factor:
\begin{align}
W &\sim \mathcal{N}\left(0, \frac{1}{\text{fan\_in}}\right) \\
W_{\text{proj}} &\sim \mathcal{N}\left(0, \frac{1}{2 \times \text{fan\_in} \times n_{\text{layer}}}\right)
\end{align}
This scheme ensures that the forward pass variance is approximately preserved regardless of layer width. The depth scaling on residual projections prevents the residual stream variance from growing unboundedly with depth, which is essential for stable training of deep transformers. We do not use parameter multipliers or per-layer learning rates.

\paragraph{Embedding initialization.}
Token embeddings are initialized separately from the transformer blocks, typically with standard deviation $1/\sqrt{n_{\text{embd}}}$:
\begin{equation}
E \sim \mathcal{N}\left(0, \frac{1}{n_{\text{embd}}}\right)
\end{equation}
This scaling ensures that the norm of individual embedding vectors is approximately 1 in expectation, preventing the embedding layer from dominating or being dominated by subsequent layers.

\paragraph{Initialization in related work.}
The OLMo model family~\citep{groeneveld2024olmo} uses a similar fan-in initialization with depth scaling on residual projections. Qwen3~\citep{qwen3} follows similar principles while incorporating adjustments for the SwiGLU activation. Compared to the parameterization strategies in~\citet{everett2024scaling}, our parameterization is most similar to their standard parameterization with global learning rate. Their standard parameterization implementation uses a small constant of $0.01$ for the embedding initialization standard deviation, which is similar to our choice of $1/\sqrt{n_{\text{embd}}} = 1/\sqrt{50304} \approx 0.0045$.


\section{Algorithms Used in Baselining \& Referenced} \label{sec:algorithms_appendix}
To be explicit about the algorithms used in our baselining experiments and referenced throughout the text, we include descriptions below. We specify for each algorithm hyperparamters which are specific to that algorithm (e.g., for \Adam, $\beta_1, \beta_2$) and hyperparameters that are common across all algorithms (e.g., learning rate $\gamma$).

\subsection{\AdamW} The main algorithm we used for comparison in our benchmarking is \AdamW \cite{loshchilov2017decoupled} with fixed weight-decay and fixed 1st-/2nd-moment parameters $\beta_1$ and $\beta_2$ respectively.

\begin{algorithm}[h!]
\caption{\AdamW (\Adam \cite{kingma2015adam} with independent weight-decay)}
\label{alg:adamW}
\begin{algorithmic}[1]
\REQUIRE Initial parameters $\theta_0$, 1st/2nd moments $m_0 = 0$, $v_0 = 0$, resp., peak LR $\gamma^* > 0$, LR schedule $\gamma(t)$, stability constant $\epsilon > 0$, $\omega > 0$ (weight-decay), number of iterations $T$
\REQUIRE \textcolor{red}{\textbf{(Specific hyperparameters to \AdamW)}} 1st/2nd moment $\beta_1, \beta_2 \in (0,1)$
\STATE \textbf{Define schedules:}
\STATE \quad $\beta_1(t) \equiv \beta_1$, $\beta_2(t) \equiv \beta_2$ \COMMENT{Constant, typically $\beta_1 = 0.9$ and $\beta_2 = 0.99$}
\STATE \quad $\lambda(t) \equiv \omega$ \COMMENT{Independent weight decay constant}
\FOR{$t = 0, 1, 2, \ldots, T-1$}
    \STATE Sample minibatch, compute stochastic gradient $g_{t+1}$
    \STATE $m_{t+1} = \beta_1 \cdot m_{t} + (1 - \beta_1) \cdot g_{t+1}$ \COMMENT{1st moment}
    \STATE $v_{t+1} = \beta_2 \cdot v_{t} + (1 - \beta_2) \cdot g_{t+1}^2$ \COMMENT{2nd moment}
    \STATE $\theta_{t+1} = \theta_t - \gamma(t) \left ( \gamma^*\frac{ m_{t+1}}{\sqrt{v_{t+1}} + \epsilon} + \lambda \cdot \theta_t \right )$ 
\ENDFOR
\end{algorithmic}
\end{algorithm}

The \Adam algorithm \cite{kingma2015adam} without weight-decay, sets $\lambda = 0$ in \AdamW. We also note that we are using \textbf{independent weight-decay}, denote for clarity in this section as $\lambda_{\text{ind}}$. The standard implementations of \AdamW in \texttt{PyTorch} \cite{paszke2019pytorch} and \texttt{Optax} \cite{deepmind2020jax} use \textit{coupled weight-decay}, denote by $\lambda_{\text{cwd}}$. Coupled weight-decay contains the learning rate, that is,
\[
\text{(independent weight-decay)} = \text{(learning rate $\times$ coupled weight-decay)}, \quad \lambda_{\text{ind}} = \gamma^* \times \lambda_{\text{cwd}}.
\]
It was shown in \citet{wortsman2023small} that independent weight-decay reduces learning rate sensitivity when transferring small models to large models.

\subsection{\Logadam}
\label{subsec:log_adam}

\Logadam \Cref{alg:log_adam} applies the full logarithmic time change to \AdamW---scheduling both the momentum parameters $\beta_1$, $\beta_2$ and the weight-decay $\lambda$---but \emph{without} the Nesterov-type update that adds an extra gradient term. This serves as an ablation to determine whether the benefits of \ADana come primarily from the log-time scheduling itself, or whether the Nesterov acceleration (which includes the additional $g_{t+1}$ term in the update) is essential.

\begin{algorithm}[h!]
\caption{\Logadam (\AdamW with log-time scheduling, no Nesterov term)}
\label{alg:log_adam}
\begin{algorithmic}[1]
\REQUIRE Initial parameters $\theta_0$, 1st/2nd moments $m_0 = 0$, $v_0 = 0$, resp., peak LR $\gamma^* > 0$, LR schedule $\gamma(t)$, stability constant $\epsilon > 0$, $\omega > 0$ (weight-decay), number of iterations $T$
\REQUIRE \textcolor{red}{\textbf{(Specific hyperparameters to \Logadam)}} $\delta > 0$ (schedule parameter, typically $\delta = 8$)
\STATE \textbf{Define schedules:}
\STATE \quad $\beta_1(t) = 1 - \frac{\delta}{\delta + t}$, $\beta_2(t) = 1 - \frac{\delta}{\delta + t}$ \COMMENT{Log-time momentum}
\STATE \quad $\lambda(t) = \frac{\omega}{\nicefrac{T}{10} + t}$ \COMMENT{Log-time decaying weight-decay}
\FOR{$t = 0, 1, 2, \ldots, T-1$}
    \STATE Sample minibatch, compute stochastic gradient $g_{t+1}$
    \STATE $m_{t+1} = \beta_1(t) \cdot m_{t} + (1 - \beta_1(t)) \cdot g_{t+1}$ \COMMENT{1st moment}
    \STATE $v_{t+1} = \beta_2(t) \cdot v_{t} + (1 - \beta_2(t)) \cdot g_{t+1}^2$ \COMMENT{2nd moment}
    \STATE $\theta_{t+1} = \theta_t - \gamma(t) \left ( \gamma^*\frac{ m_{t+1}}{\sqrt{v_{t+1}} + \epsilon} + \lambda(t) \cdot \theta_t \right )$ \COMMENT{No extra $g_{t+1}$ term}
\ENDFOR
\end{algorithmic}
\end{algorithm}

Compared to \ADana and \Dana, the key omission is the Nesterov-type update structure. In \ADana (with $\lambda=\epsilon=0$), the parameter update includes an additional gradient term: $\theta_{t+1} = \theta_t - \gamma (g_{t+1} + \alpha(t) m_{t+1}) / \sqrt{v_{t+1}}$, where the $g_{t+1}$ provides immediate responsiveness while the momentum $m_{t+1}$ captures long-horizon information. \Logadam uses only the momentum $m_{t+1}$, testing whether log-time scheduling alone---without this acceleration structure---is sufficient for improved optimization on language modeling tasks.

\subsection{\Logadamnesterov}
\label{subsec:log_adam_nesterov}

We define \Logadamnesterov using \Logadam on which we add a Nesterov-type update that adds an extra gradient term. We additionally add a momentum schedule $\alpha(t)=\frac{\delta+t}{\delta}$ on the momentum update. \Logadamnesterov serves as an ablation to show that a damping schedule o $\alpha(t)$ is necessary for stability. Indeed, \Logadamnesterov is unstable in the stochastic setting as shown on \Cref{fig:short_average_impact}.

\begin{algorithm}[h!]
\caption{\Logadamnesterov (\AdamW with log-time scheduling, additional gradient term and momentum scheduling)}
\label{alg:log_adam_nesterov}
\begin{algorithmic}[1]
\REQUIRE Initial parameters $\theta_0$, 1st/2nd moments $m_0 = 0$, $v_0 = 0$, resp., peak LR $\gamma^* > 0$, LR schedule $\gamma(t)$, stability constant $\epsilon > 0$, $\omega > 0$ (weight-decay), number of iterations $T$
\REQUIRE \textcolor{red}{\textbf{(Specific hyperparameters to \Logadamnesterov)}} $\delta > 0$ (schedule parameter, typically $\delta = 8$)
\STATE \textbf{Define schedules:}
\STATE \quad $\beta_1(t) = 1 - \frac{\delta}{\delta + t}$, $\beta_2(t) = 1 - \frac{\delta}{\delta + t}$ \COMMENT{Log-time momentum}
\STATE \quad $\lambda(t) = \frac{\omega}{\nicefrac{T}{10} + t}$ \COMMENT{Log-time decaying weight-decay}
\STATE \quad $\alpha(t) = \frac{\delta+t}{\delta}$ \COMMENT{Momentum learning rate}
\FOR{$t = 0, 1, 2, \ldots, T-1$}
    \STATE Sample minibatch, compute stochastic gradient $g_{t+1}$
    \STATE $m_{t+1} = \beta_1(t) \cdot m_{t} + (1 - \beta_1(t)) \cdot g_{t+1}$ \COMMENT{1st moment}
    \STATE $v_{t+1} = \beta_2(t) \cdot v_{t} + (1 - \beta_2(t)) \cdot g_{t+1}^2$ \COMMENT{2nd moment}
    \STATE $\theta_{t+1} = \theta_t - \gamma(t) \left ( \gamma^*\frac{g_{t+1}+ \alpha(t)m_{t+1}}{\sqrt{v_{t+1}} + \epsilon} + \lambda(t) \cdot \theta_t \right )$ \COMMENT{Combination of $g_{t+1}$ and scheduled $m_{t+1}$ term}
\ENDFOR

\end{algorithmic}
\end{algorithm}

\subsection{\Muon}
\label{subsec:muon}


\paragraph{Background.}
\Muon (MomentUm Orthogonalized by Newton-schulz) was originally proposed by \citet{jordan2024muon} and demonstrated strong results training small-scale language models. The key insight is to replace standard gradient updates with orthogonalized updates: rather than moving in the direction of the gradient, \Muon approximately projects the momentum buffer onto the nearest orthogonal matrix. This constrains updates isometrically across all directions, fundamentally differing from \AdamW's element-wise adaptive approach.

For 2D weight matrices $W \in \R^{A \times B}$, \Muon internally maintains SGD-style momentum (not \Adam's exponential moving average), then applies Newton-Schulz iterations to approximate the polar decomposition $M = U\Sigma V^\top \mapsto UV^\top$. This orthogonalization step can be efficiently computed in bfloat16 on GPU using a quintic iteration with carefully chosen coefficients.

\paragraph{Learning rate convention.}
A critical practical consideration is the learning rate scaling convention. The theoretical \Muon update has RMS equal to $1/\sqrt{\max(A, B)}$ for full-rank matrices, where $A$ and $B$ are the matrix dimensions. To match the typical RMS of \AdamW updates (approximately $0.2$--$0.4$), the Moonlight technical report \citep{liu2025muon} introduced a rescaling factor:
\begin{equation}
\gamma_{\text{adjusted}} = \gamma \cdot \sqrt{\max(A, B)} \cdot \texttt{matched\_adamw\_rms}.
\end{equation}
This ``matched AdamW RMS'' convention ensures that \Muon updates have comparable magnitude to \AdamW regardless of parameter shape, enabling transfer of learning rate schedules between the two optimizers.

\paragraph{Hybrid approach.}
For parameters that are not 2D matrices---including embeddings, layer norms, and biases---\Muon falls back to \AdamW. 

\begin{algorithm}[h!]
\caption{\Muon \citep{jordan2024muon, liu2025muon}}
\label{alg:muon}
\begin{algorithmic}[1]
\REQUIRE Initial 2D parameters $W_0 \in \R^{A \times B}$, momentum buffer $M_0 = 0$
\REQUIRE Peal LR $\gamma^* > 0$, LR schedule $\gamma(t)$, weight-decay $\lambda > 0$, momentum $\mu \in (0,1)$, number of iterations $T$
\REQUIRE \textcolor{red}{\textbf{(Specific to \Muon)}} NS iterations $K$ (typically 5), \texttt{matched\_adamw\_rms} $\rho$ (typically 0.2--0.3), Nesterov flag
\STATE \textbf{Newton-Schulz coefficients:} $a = 3.4445$, $b = -4.7750$, $c = 2.0315$
\FOR{$t = 0, 1, 2, \ldots, T-1$}
    \STATE Sample minibatch, compute stochastic gradient $G_{t+1} \in \R^{A \times B}$
    \STATE $M_{t+1} = \mu \cdot M_t + G_{t+1}$ \COMMENT{SGD-style momentum}
    \IF{Nesterov}
        \STATE $\tilde{M} = G_{t+1} + \mu \cdot M_{t+1}$
    \ELSE
        \STATE $\tilde{M} = M_{t+1}$
    \ENDIF
    \STATE \textbf{// Newton-Schulz orthogonalization}
    \STATE $X_0 = \tilde{M} / (\|\tilde{M}\|_F + \epsilon)$ \COMMENT{Normalize to spectral norm $\leq 1$}
    \IF{$A > B$}
        \STATE $X_0 \gets X_0^\top$ \COMMENT{Work with wider matrix}
    \ENDIF
    \FOR{$k = 0, 1, \ldots, K-1$}
        \STATE $X_{k+1} = a X_k + (b X_k X_k^\top + c (X_k X_k^\top)^2) X_k$ \COMMENT{Quintic NS iteration}
    \ENDFOR
    \IF{$A > B$}
        \STATE $O_{t+1} \gets X_K^\top$
    \ELSE
        \STATE $O_{t+1} \gets X_K$
    \ENDIF
    \STATE \textbf{// Apply scaled update with weight-decay}
    \STATE $\gamma_{\text{adj}} = \gamma(t)\gamma^* \cdot \rho \cdot \sqrt{\max(A, B)}$ \COMMENT{Matched AdamW RMS scaling}
    \STATE $W_{t+1} = (1 - \gamma(t) \lambda) \cdot W_t - \gamma_{\text{adj}} \cdot O_{t+1}$
\ENDFOR
\end{algorithmic}
\end{algorithm}

The Newton-Schulz iteration approximates the matrix sign function, effectively computing $(MM^\top)^{-1/2} M \approx UV^\top$ where $M = U\Sigma V^\top$ is the SVD. The quintic coefficients $(a, b, c) = (3.4445, -4.7750, 2.0315)$ are optimized to maximize convergence rate near zero while maintaining stability.

\subsection{\Ademamix}
\label{subsec:ademamix}


\paragraph{Background.}
\Ademamix \citep{pagliardini2024ademamix} extends \Adam by adding a second exponential moving average (EMA) of the gradients, with two very different timescales. The design inspiration is that ``a single EMA cannot both give significant weight to recent gradients, and give non-negligible weight to older gradients.'' 


\paragraph{Notation convention.}
In the original \Ademamix paper, $\beta_1 \approx 0.9$ denotes the fast (short-range) momentum and $\beta_3 \approx 0.9999$ (but usually swept) denotes the slow (long-range) momentum. \textbf{For consistency with our \ADana notation, we reverse these roles}: in our presentation, $\beta_1$ denotes the \emph{long-range} momentum (large, close to 1) and $\beta_3$ denotes the \emph{short-range} momentum (smaller, typically 0.9). This aligns with the DANA framework where $\beta_1(t) = 1 - \delta/t$ represents the logarithmic-time schedule for long-horizon memory.

\paragraph{Warmup schedules.}
Training with constant large momentum is unstable early in training. \Ademamix addresses this with warmup schedules: the long momentum $\beta_1(t)$ gradually increases from $\beta_3$ to its target value, and the mixing coefficient $\alpha(t)$ linearly warms up from 0 to its final value over $T_\alpha$ iterations. See \Cref{sec:appendix_ademamix} for detailed analysis of how these schedules relate to \Dana's logarithmic-time schedule.

\begin{algorithm}[h!]
\caption{\Ademamix \citep{pagliardini2024ademamix} (notation: $\beta_1$ = long momentum, $\beta_3$ = short momentum)}
\label{alg:ademamix}
\begin{algorithmic}[1]
\REQUIRE Initial parameters $\theta_0$, fast/slow EMAs $m_1^{(0)} = m_3^{(0)} = 0$, second moment $v_0 = 0$
\REQUIRE Learning rate $\gamma > 0$, weight-decay $\lambda > 0$, stability constant $\epsilon > 0$, number of iterations $T$
\REQUIRE \textcolor{red}{\textbf{(Specific to \Ademamix)}} Short momentum $\beta_3 \approx 0.9$, long momentum $\beta_1 \approx 0.9999$, second moment $\beta_2 \approx 0.999$, mixing coefficient $\alpha > 0$ (typically 5--10), warmup horizon $T_\alpha$
\FOR{$t = 0, 1, 2, \ldots, T-1$}
    \STATE Sample minibatch, compute stochastic gradient $g_{t+1}$
    \STATE \textbf{// Warmup schedules (optional)}
    \STATE $\hat{\beta}_1(t) = \texttt{warmup\_scheduler}(\beta_3, \beta_1, t, T)$ \COMMENT{Gradually increase to $\beta_1$}
    \STATE $\hat{\alpha}(t) = \min(t / T_\alpha, 1) \cdot \alpha$ \COMMENT{Linear warmup of mixing coefficient}
    \STATE \textbf{// Momentum updates}
    \STATE $m^{(3)}_{t+1} = \beta_3 \cdot m^{(3)}_{t} + (1 - \beta_3) \cdot g_{t+1}$ \COMMENT{Fast/short-range EMA}
    \STATE $m^{(1)}_{t+1} = \hat{\beta}_1(t) \cdot m^{(1)}_{t} + (1 - \hat{\beta}_1(t)) \cdot g_{t+1}$ \COMMENT{Slow/long-range EMA}
    \STATE $v_{t+1} = \beta_2 \cdot v_{t} + (1 - \beta_2) \cdot g_{t+1}^2$ \COMMENT{Second moment}
    \STATE \textbf{// Bias correction for fast EMA and second moment}
    \STATE $\hat{m}^{(3)}_{t+1} = m^{(3)}_{t+1} / (1 - \beta_3^{t+1})$
    \STATE $\hat{v} = v_{t+1} / (1 - \beta_2^{t+1})$
    \STATE \textbf{// Parameter update (no bias correction on slow EMA)}
    \STATE $\theta_{t+1} = \theta_t - \gamma(t)\left( \gamma^*\frac{\hat{m}^{(3)}_{t+1} + \hat{\alpha}(t) \cdot m^{(1)}_{t+1}}{\sqrt{\hat{v}} + \epsilon} - \lambda \cdot \theta_t\right)$
\ENDFOR
\end{algorithmic}
\end{algorithm}

Note that in its original formulation, \cite{pagliardini2024ademamix} use a different weight-decay formulation

\[\theta_{t+1} = \theta_t - \gamma(t)\gamma^*\left(\frac{\hat{m}^{(3)}_{t+1} + \hat{\alpha}(t) \cdot m^{(1)}_{t+1}}{\sqrt{\hat{v}} + \epsilon} - \lambda \cdot \theta_t\right)\]

but for consistency with our framework we wrote \Cref{alg:ademamix} with independent weight-decay \citep{loshchilov2017decoupled,wortsman2023small}.

The slow EMA $m^{(1)}$ does not use bias correction, as the warmup schedule handles the initial transient. The mixing coefficient $\alpha$ (typically 5--10) controls the relative contribution of old versus recent gradients. When $\alpha = 0$, \Ademamix reduces to standard \Adam. 

\subsection{Naive Implementation of Stochastic Nesterov's Accelerated Method (non-strongly convex setting)}
Nesterov's accelerated method \cite{nesterov1983method} is known to achieve optimal rates for minimizing non–strongly convex objective functions. A stochastic variant of Nesterov's method \cite{Nesterov1988On} is defined by the updates
\begin{equation}
\begin{aligned}
\theta_{t+1} = y_t - \gamma g_t,
\quad
y_{t+1} = \theta_{t+1} + \mu_t(\theta_{t+1} - \theta_t),
\quad
\text{where } \mu_t = 1-\frac{3}{t+3}, \, \ g_{t+1} = \nabla \mathcal{L}(y_t).
\end{aligned}
\end{equation}
Here $\gamma>0$ denotes the learning rate. For simplicity, we write $\mathcal{L}(y_t) \defas \mathcal{L}(y_t, x_{t+1})$ to emphasize that the gradient is stochastic. While Nesterov's accelerated method is originally formulated using full-batch (deterministic) gradients, here we replace the deterministic gradient with a stochastic one, yielding the most naive stochastic implementation of the method.

We note that \cite{Nesterov1988On,beck2009gradient} use the parameter choice $\mu_t = 1 - \frac{3}{t+3}$, while later works \cite{lan2012optimal, flammarion2015from} observe that one may alternatively use $\mu_t = 1-\frac{2}{t+1}$ with similar acceleration guarantees with full-batch gradients. However, \cite{flammarion2015from} shows that this naive stochastic implementation diverges on power-law random features models (see Section~\ref{sec:appendix_synthetic}). This highlights the need to modify the algorithm to ensure stability in the presence of gradient noise (see Section~\ref{sec:appendix_building_adana} for further discussion).

In Section~\ref{sec:standard_nesterov_formulations}, we show that the above update can be written approximately as an exponential moving average (EMA) of the momentum, yielding the equivalent representation
\begin{equation}
\label{eq:nesterov_non_convex_rewritten}
\begin{aligned}
\text{\emph{(mom.)}}\quad & m_{t+1} = \left(1-\frac{\delta}{\delta+t}\right) m_t + \frac{\delta}{\delta+t} \cdot g_{t+1},\\
\text{\emph{(param.)}}\quad & \theta_{t+1} = \theta_t - \gamma\left(g_{t+1} + \frac{t+\delta}{\delta} m_{t+1}\right).
\end{aligned}
\end{equation}
We leave $\delta$ as a free hyperparameter, as it was shown in \cite{ferbach2025dimension} (and previously noted in \cite{lan2012optimal, flammarion2015from}) that multiple choices of $\delta$ yield comparable accelerated performance. We write the full version of stochastic Nesterov's accelerated method in Algorithm~\ref{alg:sto_nesterov}.

\begin{algorithm}[h!]
\caption{Stochastic Nesterov's Accelerated Method}
\label{alg:sto_nesterov}
\begin{algorithmic}[1]
\REQUIRE Initial parameters $\theta_0$, 1st moment $m_0 = 0$, learning rate $\gamma > 0$, , number of iterations $T$
\REQUIRE \textcolor{red}{\textbf{(Specific to Stochastic Nesterov)}} $\delta$ a large constant.
\STATE \textbf{Define schedules:}
\STATE \quad $\beta_1(t) \equiv 1 - \frac{\delta}{\delta+t}$, \COMMENT{Log-time schedule}
\STATE \quad $\alpha(t) \equiv \frac{\delta+t}{\delta}$, \COMMENT{Momentum schedule}
\FOR{$t = 0, 1, 2, \ldots, T-1$}
    \STATE Sample minibatch, compute stochastic gradient $g_{t+1}$
    \STATE $m_{t+1} = \beta_1(t) \cdot m_{t} + (1 - \beta_1(t)) \cdot g_{t+1}$ \COMMENT{1st moment}
    \STATE $\theta_{t+1} = \theta_t - \gamma \left (  g_{t+1} + \alpha(t) \cdot m_{t+1} \right )$ 
\ENDFOR
\end{algorithmic}
\end{algorithm}

\subsection{\Dana Algorithm/Generalized Nesterov's accelerated method}
\label{sec:gen_dana_alg}

To fix the divergence of the naive implementation of stochastic Nesterov's accelerated method, \cite{ferbach2025dimension} proposed a \textbf{\textit{Generalized Nesterov's accelerated method}} (see also \cite{defazio2024road,pagliardini2024ademamix,defazio2025smoothing,varre2022accelerated,flammarion2015from}) which updates
\begin{gather}
\label{eq:general_momentum_update}
m_{t+1} = \beta_1(t) \cdot m_t + (1-\beta_1(t)) \cdot g_{t+1} \quad \text{and} \quad \theta_{t+1} = \theta_t - \gamma ( g_{t+1} + \textcolor{KPPcolor}{\bm{\alpha(t)}} \cdot m_{t+1} )
\\
\text{where $g_{t+1} = \nabla \mathcal{L}(\theta_t)$ \quad and \quad $\beta_1(t) = 1-\frac{\delta}{t+\delta}$.}
\end{gather}
Here, $\textcolor{KPPcolor}{\bm{\alpha(t)}}$ is a damping coefficient that must be tuned. When $\alpha(t) \asymp t$, we recover Stochastic Nesterov's accelerated method (Alg.~\ref{alg:sto_nesterov}). Choosing $\alpha(t) \gtrsim t$, amplifies the momentum term; further increasing the divergence problems from the noise of the stochastic gradients. When $\alpha(t) \lesssim t$, the momentum term is dampened which helps with stability. At the same time, excessive damping causes the method to behave like standard SGD, forfeiting acceleration and scaling benefits. This raises the central question:
\begin{center}
\textit{Is there a choice of \textcolor{KPPcolor}{$\bm{\alpha(t)}$} that achieves both \textbf{stability} and \textbf{acceleration} at scale?}
\end{center}

The \Dana\footnote{The \Dana algorithm presented in \cite{ferbach2025dimension} is of slightly modified form than the one presented here. We show a near equivalence to the generalized Nesterov's accelerated method presented in this section in Section~\ref{sec:equivalence_dana_convex_combination}.} algorithm in \cite{ferbach2025dimension} studied specific $\alpha(t)$ schedules and concluded that such damping schedules exist for the power-law random features (PLRF) model. In particular, the \Dana algorithm studied damping schedules of the form:
\begin{equation} \label{eq:dana_schedule_appendix}
\alpha(t) = \bar{\gamma}(d) \cdot (1+t)^{1-\kappa},
\end{equation}
where $\bar{\gamma}$ may depend on the parameter dimension $d$ and $\kappa > 0$. The authors in \cite{ferbach2025dimension} noted two particular schedules
\begin{itemize}
\item \Danadecaying, $\alpha(t) \asymp (1+t)^{1-\kappa}$: For a specific choice of $\kappa$ related to the spectral properties of the data, this schedule yielded the most acceleration over the class of schedules in \eqref{eq:dana_schedule_appendix} and its performance did not dimension as the model sized increased. The authors denoted generalized stochastic Nesterov accelerated method with this schedule as \Danadecaying (see also \cite{yarotsky2025sgd} for related algorithm). When $\kappa \ge 1$, \Danadecaying did not accelerate over SGD on the PLRF.
\item \Danaconstant, $\alpha(t) \asymp \tfrac{t}{d}$, where $d$ is the model size. Since model size $d$ is ambiguous outside the PLRF on multi-layer transformers, it is possible to use  $\alpha(t) \asymp \tfrac{t}{T^{\kappa}}$ where $T$ is the training horizon.  While this schedule achieved some amount of acceleration, it did not accelerate as favorably as \Danadecaying and can be unstable beyond the training horizon when using $\alpha(t) \asymp \tfrac{t}{T^{\kappa}}$. The authors denoted this algorithm as \Danaconstant (see also \cite{varre2022accelerated} for related algorithm).
\end{itemize}
For more details, see Section~\ref{sec:appendix_building_adana}.

\paragraph{Choosing $\kappa$ in the damping schedule: Practical guidance.}
On the PLRF model, the optimal value of $\kappa$ depends on spectral properties of the data. Consequently, theory predicts that $\kappa$ should remain unchanged as the problem dimension grows, making it a transferable hyperparameter across scales. In practice, the specific spectral properties of data are unknown and $\kappa$ must be tuned. Empirically (see Figure~\ref{fig:optimal_kappa} \& Figure~\ref{fig:alpha_scaling_chinchilla}), we find that $\kappa$ is indeed transferable across model sizes and relatively easy to tune. In our experiments on transformer architectures trained on FineWeb data \cite{penedo2024fineweb}, effective values of $\kappa$ lie in the range \textcolor{Mutedred}{\textbf{0.75-0.9}}, with \textcolor{Mutedred}{\textbf{0.85}} performing best across all scales considered in this work and the performance across this range of $\kappa$ was quite similar. Moreover, $\kappa$ appears to transfer well across architectures, performing well with minimal tuning when moving from Enoki models to Qwen.

\begin{algorithm}[h!]
\caption{\Dana algorithm (Variants \Danaconstant, \Danadecaying)}
\label{alg:dana}
\begin{algorithmic}[1]
\REQUIRE Initial parameters $\theta_0$, 1st moments $m_0 = 0$, learning rate $\gamma > 0$, number of iterations $T$
\REQUIRE \textcolor{red}{\textbf{(Specific to \Dana)}} $\delta=8$ a large constant, $\kappa > 0$ (spectral dimension), $\hat{\gamma}$ > 0 damping constant factor.
\STATE \textbf{Define schedules:}
\STATE \quad $\beta_1(t) \equiv 1 - \frac{\delta}{\delta+t}$ \COMMENT{Equal log-time schedule}
\STATE \quad $\alpha(t) = \hat{\gamma}\times (1+t)^{1-\kappa}$ \COMMENT{\Danadecaying damping schedule}
\STATE \quad $\alpha(t) = \hat{\gamma}T^{-\kappa} \times (1+t)$ \COMMENT{\Danaconstant damping schedule}
\FOR{$t = 0, 1, 2, \ldots, T-1$}
    \STATE Sample minibatch, compute stochastic gradient $g_{t+1}$
    \STATE $m_{t+1} = \beta_1(t) \cdot m_{t} + (1 - \beta_1(t)) \cdot g_{t+1}$ \COMMENT{1st moment}
    \STATE $\theta_{t+1} = \theta_t - \gamma \left (g_{t+1} + \alpha(t) \cdot m_{t+1}\right )$ 
\ENDFOR
\end{algorithmic}
\end{algorithm}

\subsection{\ADana Variants}

In this section, we described the exact algorithms for the \ADana variant, \DanaMKfour, \Danastar, and \DanastarMKfour, described in Sec.~\ref{subsec:hardening} with full details in Sec.~\ref{sec:appendix_building_adana} and Sec.~\ref{sec:appendix_failure_modes}. In Table~\ref{table:adana_variants_hyperparameters}, we provide the hyperparameters we typically used in the LLM experiments.

\paragraph{\Danastar.} In Section~\ref{subsec:hardening} (with additional details in Section~\ref{sec:appendix_failure_modes}), we build a variant of \ADana to handle sparse gradients. We denote this variant as \Danastar. We highlight in \textcolor{Mutedred}{(red)} the differences between \Danastar and \ADana in Algorithm~\ref{alg:danastar}. The main difference between \ADana and \Danastar is the addition of $\tau$ which attempts to estimate the probability of an update. While this does add an additional memory cost, it is not substantial.

\begin{table}[t]
\centering
\caption{\textbf{Hyperparameters for \ADana variants.} Hyperpameters used for Qwen3 and Enoki scaling experiments (See Sec.~\ref{sec:enoki_implementation} for architectural details and Sec.~\ref{sec:baselining} learning rate scaling schedules) }
\label{table:adana_variants_hyperparameters}
\begin{tabular}{c|c c c c}
\toprule
\textbf{Parameter} & \textbf{\ADana} & \textbf{\Danastar} & \textbf{\DanaMKfour} & \textbf{\DanastarMKfour} \\
\midrule
$\delta$ & 8 & 8 & 8 & 8 \\
$\kappa$ & 0.85 & 0.85 & 0.85 & 0.85\\
$\omega$ & 4 & 4 & 4 & 4\\
\texttt{clipSNR} & - & - & 2.0 & 2.0\\
$\tau_0$ & - & 0 & - & 0\\
\bottomrule
\end{tabular}
\end{table}

\begin{algorithm}[t]
\caption{\Danastar}
\label{alg:danastar}
\begin{algorithmic}[1]
\REQUIRE Initial parameters $\theta_0$, 1st/2nd moments $m_0 = 0$, $v_0 = 0$, probability initialization $\tau = 0$, peak LR $\gamma^* > 0$, LR schedule $\gamma(t)$, stability constant $\epsilon > 0$, number of iterations $T$
\REQUIRE \textcolor{red}{\textbf{(Specific to \Danastar)}} $\delta > 0$ (typically $\delta = 8$), $\kappa \in (0,1)$ (spectral dimension), $\omega > 0$ (weight-decay), \textcolor{Mutedred}{$\tau_0 = 0$ (probability estimator)}
\STATE \textbf{Define schedules:}
\STATE \quad $\beta_1(t) = \beta_2(t) = 1 - \frac{\delta}{\delta + t}$ \COMMENT{Logarithmic time}
\STATE \quad $\lambda(t) = \frac{\omega}{\nicefrac{T}{10} + t}$ \COMMENT{Decaying weight-decay}
\STATE \quad $\alpha(t) = (1+t)^{1-\kappa}$ \COMMENT{Damping schedule}
\FOR{$t = 0, 1, 2, \ldots, T-1$}
    \STATE Sample minibatch, compute stochastic gradient $g_{t+1}$
    \STATE $m_{t+1} = \beta_1(t) \cdot m_{t} +(1 - \beta_1(t)) \cdot g_{t+1}$ \COMMENT{1st moment}
    \STATE $v_{t+1} = \beta_2(t) \cdot v_{t} + (1 - \beta_2(t)) \cdot g_{t+1}^2$ \COMMENT{2nd moment}
    \STATE \textcolor{Mutedred}{$\tau_{\texttt{update}} = \frac{|g_{t+1}|}{|g_{t+1}| + \sqrt{v_{t+1}} + \epsilon}$}
    \STATE \textcolor{Mutedred}{$\tau_{t+1} = (1 - \tfrac{\delta}{t+\delta}) \tau_t + \tfrac{\delta}{t+\delta} \tau_{\texttt{update}}$ \COMMENT{Update class probabilities}}
\STATE \textcolor{Mutedred}{$t_{\texttt{eff}} = \max\{t \cdot \tau_{t+1}, 1\}$, $\tau_{\texttt{clip}} = \min\{\tau_{t+1}, 0.5\}$}
\STATE \textcolor{Mutedred}{$\tilde{\tau}_{t+1} = \max\{ \tau_{\texttt{clip}} / (1-\tau_{\texttt{clip}}), \tfrac{1}{1+t} \}$}
   \STATE $\theta_{t+1} = \theta_t - \gamma(t) \left (\gamma^* \frac{\sqrt{\textcolor{Mutedred}{\tilde{\tau}_{t+1}}} \odot g_{t+1}}{\sqrt{v_{t+1}} + \epsilon} + \gamma^*\alpha(\textcolor{Mutedred}{t_{\texttt{eff}}}) \cdot \frac{\sqrt{\textcolor{Mutedred}{\tilde{\tau}_{t+1}}} \odot m_{t+1}}{\sqrt{v_{t+1}} + \epsilon}  + \lambda(t) \cdot \theta_t \right )$
\ENDFOR
\end{algorithmic}
\end{algorithm}

\paragraph{\DanaMKfour.} In Section~\ref{subsec:hardening} (with additional details in Section~\ref{sec:appendix_failure_modes}), we build a variant of \ADana to handle in homogeneous spectral dimensions, particularly making less sensitive to the choice of $\kappa$ in the damping scheduling $\alpha(t) = (1+t)^{1-\kappa}$. We denote this variant as \DanaMKfour. We highlight in \textcolor{Mutedred}{(red)} the differences between \DanaMKfour and \ADana in Algorithm~\ref{alg:dana_mk4}. We note that $\text{sign}(\cdot)$ is applied entry-wise in Algorithm~\ref{alg:dana_mk4}.

\begin{algorithm}[t]
\caption{\DanaMKfour Algorithm}
\label{alg:dana_mk4}
\begin{algorithmic}[1]
\REQUIRE Initial parameters $\theta_0$, 1st/2nd moments $m_0 = 0$, $v_0 = 0$, peak LR $\gamma^* > 0$, LR schedule $\gamma(t)$, stability constant $\epsilon > 0$, number of iterations $T$
\REQUIRE \textcolor{red}{\textbf{(Specific to \DanaMKfour)}} $\delta > 0$ (typically $\delta = 8$), $\kappa \in (0,1)$ (spectral dimension), $\omega > 0$ (weight-decay), \textcolor{Mutedred}{$\texttt{clipsnr} > 0$ (SNR clipping threshold)}
\STATE \textbf{Define schedules:}
\STATE \quad $\beta_1(t) = \beta_2(t) = 1 - \frac{\delta}{\delta + t}$ \COMMENT{Logarithmic time}
\STATE \quad $\lambda(t) = \frac{\omega}{\nicefrac{T}{10} + t}$ \COMMENT{Decaying weight-decay}
\FOR{$t = 0, 1, 2, \ldots, T-1$}
    \STATE Sample minibatch, compute stochastic gradient $g_{t+1}$
        \STATE $m_{t+1} = \beta_1(t) \cdot m_{t} + (1-\beta_1(t)) \cdot g_{t+1}$ \COMMENT{1st moment update}
    \STATE $v_{t+1} = \beta_2(t) \cdot v_{t} + (1-\beta_2(t)) \cdot g_{t+1}^2$ \COMMENT{2nd moment update}
    \STATE $\texttt{norm} = \frac{1}{\sqrt{v_{t+1}} + \epsilon}$, $\texttt{mfac} = |m_{t+1}| \cdot \texttt{norm}$
    \STATE \textcolor{Mutedred}{$\alpha_{\texttt{fac}} = \min\{ t^{1-\kappa} \cdot \texttt{mfac}, \, \texttt{clipsnr} \}$ \COMMENT{Clipped scaling}}
    \STATE $\theta_{t+1} = \theta_t - \gamma(t) \left ( \gamma^*g_{t+1} \odot \texttt{norm}+ \gamma^*\textcolor{Mutedred}{\text{sign}(m_{t+1}) \odot ( \alpha_{\texttt{fac}} + |m_{t+1}| \odot \texttt{norm})} + \lambda(t) \cdot \theta_t \right )$
\ENDFOR
\end{algorithmic}
\end{algorithm}

\paragraph{\DanastarMKfour.} We combine both \DanaMKfour and \Danastar to create \DanastarMKfour which is both less sensitive to choices of $\kappa$ and less sensitive to sparse gradients. We give a complete description in Algorithm~\ref{alg:dana_star_mk4} and highlight any differences between \ADana in \textcolor{Mutedred}{(red)}. The specific hyperparameters are the same as \DanaMKfour and \Danastar.

\begin{algorithm}[t]
\caption{\DanastarMKfour Algorithm}
\label{alg:dana_star_mk4}
\begin{algorithmic}[1]
\REQUIRE Initial parameters $\theta_0$, 1st/2nd moments $m_0 = 0$, $v_0 = 0$, probability initialization $\tau = 0$, peak LR $\gamma^* > 0$, LR schedule $\gamma(t)$, stability constant $\epsilon > 0$;
\REQUIRE \textcolor{red}{\textbf{(Specific to \DanastarMKfour)}} $\delta > 0$ (typically $\delta = 8$), $\kappa \in (0,1)$ (spectral dimension), $\omega > 0$ (weight-decay), \textcolor{Mutedred}{$\tau_0 = 0$ (probability estimator)}, \textcolor{Mutedred}{$\texttt{clipsnr} > 0$ (SNR clipping threshold)}, number of iterations $T$
\STATE \textbf{Define schedules:}
\STATE \quad $\beta_1(t) = \beta_2(t) = 1 - \frac{\delta}{\delta + t}$ \COMMENT{Logarithmic time}
\STATE \quad $\lambda(t) = \frac{\omega}{\nicefrac{T}{10} + t}$ \COMMENT{Decaying weight-decay}
\FOR{$t = 0, 1, 2, \ldots, T-1$}
    \STATE Sample minibatch, compute stochastic gradient $g_{t+1}$
            \STATE $m_{t+1} = \beta_1(t) \cdot m_{t} + (1-\beta_1(t)) \cdot g_{t+1}$ \COMMENT{1st moment update}
    \STATE $v_{t+1} = \beta_2(t) \cdot v_{t} + (1-\beta_2(t)) \cdot g_{t+1}^2$ \COMMENT{2nd moment update}
\STATE \textcolor{Mutedred}{$\tau_{\texttt{update}} = \frac{|g_{t+1}|}{|g_{t+1}| + \sqrt{v_{t+1}} + \epsilon}$}
    \STATE \textcolor{Mutedred}{$\tau_{t+1} = (1 - \tfrac{\delta}{t+\delta}) \tau_t + \tfrac{\delta}{t+\delta} \tau_{\texttt{update}}$ \COMMENT{Update class probabilities}}
\STATE \textcolor{Mutedred}{$t_{\texttt{eff}} = \max\{t \cdot \tau_{t+1}, 1\}$, $\tau_{\texttt{clip}} = \min\{\tau_{t+1}, 0.5\}$}
\STATE \textcolor{Mutedred}{$\tilde{\tau}_{t+1} = \max\{ \tau_{\texttt{clip}} / (1-\tau_{\texttt{clip}}), \tfrac{1}{1+t} \}$}
    \STATE \textcolor{Mutedred}{$\texttt{norm} = \frac{\sqrt{\tilde{\tau}_{t+1}}}{\sqrt{v_{t+1}} + \epsilon}$, $\texttt{mfac} = \frac{|m_{t+1}| \cdot \texttt{norm}}{\tilde{\tau}_{t+1}}$}
    \STATE \textcolor{Mutedred}{$\alpha_{\texttt{fac}} = \min\{ t_{\texttt{eff}}^{1-\kappa} \cdot \texttt{mfac}, \, \texttt{clipsnr} \}$ \COMMENT{Clipped scaling}}
    \STATE $\theta_{t+1} = \theta_t - \gamma(t) \left (  \gamma^*g_{t+1} \odot \texttt{norm} + \gamma^* \textcolor{Mutedred}{\text{sign}(m_{t+1}) \odot (\tilde{\tau}_{t+1} \cdot \alpha_{\texttt{fac}} + |m_{t+1}| \odot \texttt{norm})}  + \lambda(t) \cdot \theta_t \right )$
\ENDFOR
\end{algorithmic}
\end{algorithm}

\subsection{Ablation on \ADana and Comparison with \Ademamix}
\label{sec:ablation_adana}

\begin{algorithm}[h]
\caption{\ADana with short momentum EMA}
\label{alg:adana_variants_description}
\begin{algorithmic}[1]
\REQUIRE Initial parameters $\theta_0$, 1st/2nd moments $m_0 = 0$, $v_0 = 0$, peak LR $\gamma^* > 0$, LR schedule $\gamma(t)$, stability constant $\epsilon > 0$, number of iterations $T$
\REQUIRE \textcolor{red}{\textbf{(Specific to \ADana with short momentum EMA)}} $\delta > 0$ (typically $\delta = 8$), $\kappa \in (0,1)$ (spectral dimension), $\omega > 0$ (weight-decay), $\beta_3\in (0,1)$ (short momentum EMA)
\STATE \textbf{Define schedules:}
\STATE \quad $\beta_1(t) = \beta_2(t) = 1 - \frac{\delta}{\delta + t}$ \COMMENT{Logarithmic time}
\STATE \quad $\lambda(t) = \frac{\omega}{\nicefrac{T}{10} + t}$
\COMMENT{Decaying weight-decay}
\STATE \quad $\alpha(t) = (1+t)^{1-\kappa}$ \COMMENT{Damped Nesterov scaling}
\FOR{$t = 0, 1, 2, \ldots, T-1$}
    \STATE Sample minibatch, compute stochastic gradient $g_{t+1}$
    \STATE $m^{(3)}_{t+1} = \beta_3 \cdot m^{(3)}_{t} + (1 - \beta_3) \cdot g_{t+1}$ \COMMENT{Short 1st moment}
    \STATE $m_{t+1} = \beta_1(t) \cdot m_{t} + (1 - \beta_1(t)) \cdot g_{t+1}$ \COMMENT{1st moment}
    \STATE $v_{t+1} = \beta_2(t) \cdot v_{t} + (1 - \beta_2(t)) \cdot g_{t+1}^2$ \COMMENT{2nd moment}
    \STATE $\hat{m}^{(3)}_{t+1} = \frac{m^{(3)}_{t+1}}{1-\beta_3^{t+1}}$ \COMMENT{Bias Correction}
    \STATE $\theta_{t+1} = \theta_t - \gamma(t) \left ( \gamma^*\frac{\hat{m}^{(3)}_{t+1} + \alpha(t) \cdot m_{t+1}}{\sqrt{v_{t+1}} + \epsilon} + \lambda(t) \cdot \theta_t \right )$ 
\ENDFOR
\end{algorithmic}
\end{algorithm}

We now describe in more details the different algorithms used in the ablation study in \Cref{fig:ablation_small} and extend this ablation in \Cref{fig:comparison_adana_variants}. The general algorithm is shown in \Cref{alg:adana_variants_description} and uses an additional Exponential Moving Average (EMA) $\beta_3\in (0,1)$ on the gradient term $m^{(3)}_t$ compared to \Cref{alg:adana}). This EMA is introduced to study its impact on performance and compare \ADana with \Ademamix (DW) \Cref{alg:ademamix} where such an EMA is used. In \Cref{tab:abl_study} we describe the different schedules and hyper-parameters used for each algorithm in \Cref{fig:comparison_adana_variants}. All algorithms in this section use decaying weight-decay.

\paragraph{Algorithms Details.} We summarize the specificity of each algorithm used in \Cref{fig:comparison_adana_variants} in \Cref{tab:abl_study}. All algorithms use $\kappa=0.75$ and log-time weight-decay. Note that algorithms using constant EMAs such as $\beta_2$ or $\beta_3$ constant have an additional bias correction update, as done in \AdamW, although we only explicitly wrote this bias correction for $\beta_3$ in \Cref{alg:adana_variants_description} for clarity. For example for constant $\beta_2$, the update uses $\theta_{t+1} = \theta_t - \gamma(t) \left ( \gamma^*\frac{\hat{m}^{3)}_{t+1} + \alpha(t) \cdot m_{t+1}}{\sqrt{\hat{v}_{t+1}} + \epsilon} + \lambda(t) \cdot \theta_t \right )$ where $\hat{v}_{t+1} = \frac{v_{t+1}}{1-\beta_2^{t+1}}$. Note that for \Ademamix (DW) we set $\beta_1 = 1-\frac{\delta}{T}$.

\begin{table}[h]
\centering
\caption{\textbf{Optimizers used in the Ablation Study of \ADana.} \ADana variants use \Cref{alg:adana_variants_description} and \Ademamix uses \Cref{alg:ademamix}.}
\label{tab:abl_study}
\vspace{0.5em}
\begin{minipage}{0.9\textwidth}
\centering
\resizebox{\textwidth}{!}{%
\begin{tabular}{l|l|ccc|cc}
\toprule
& & \multicolumn{3}{c|}{\textbf{Main Changes}} & \multicolumn{2}{c}{\textbf{Additional Modifications}} \\
\textbf{Variant} & \textbf{Optimizer} & $\beta_3$ & $\beta_2$ & $\alpha(t)$ & $\beta_3$ Bias Correction & $\beta_2$ Bias Correction \\
\midrule
\ADana & Original \ADana & $0.0$ & $\frac{\delta}{\delta+t}$ & $(1+t)^{1-\kappa}$ & $\backslash$ & $\backslash$ \\
\ADana 2 & Dana-constant $\alpha(t)$ & $0.0$ & $\frac{\delta}{\delta+t}$ & $T^{-\kappa} \times t $ & $\backslash$ & $\backslash$  \\
\ADana 3 & Constant $\beta_2$ & $0.0$ & $0.999$ & $(1+t)^{1-\kappa}$ & $\backslash$ & $\hat{v}_{t+1} = \frac{v_{t+1}}{1-\beta_2^{t+1}}$ \\
\ADana 4 & Dana-constant $\alpha(t)$ and Constant $\beta_2$ & $0.0$ & $0.999$ & $T^{-\kappa} \times t $ & $\backslash$ & $\hat{v}_{t+1} = \frac{v_{t+1}}{1-\beta_2^{t+1}}$  \\
\ADana 5 & Dana-constant $\alpha(t)$, Short EMA & $0.9$ & $\frac{\delta}{\delta+t}$ & $T^{-\kappa} \times t $ & $\hat{m}^{(3)}_{t+1} = \frac{m^{(3)}_{t+1}}{1-\beta_3^{t+1}}$ & $\backslash$\\
\ADana 6 & Dana-constant $\alpha(t)$, Short EMA and Constant $\beta_2$ & $0.9$ & $0.999$ & $T^{-\kappa} \times t $ & $\hat{m}^{(3)}_{t+1} = \frac{m^{(3)}_{t+1}}{1-\beta_3^{t+1}}$ & $\hat{v}_{t+1} = \frac{v_{t+1}}{1-\beta_2^{t+1}}$\\
\ADana 7 & Short EMA & $0.9$ & $\frac{\delta}{\delta+t}$ & $(1+t)^{1-\kappa}$ & $\hat{m}^{(3)}_{t+1} = \frac{m^{(3)}_{t+1}}{1-\beta_3^{t+1}}$ & $\backslash$ \\
\Ademamix (DW) & \Ademamix (DW) & $0.9$ & $0.999$ & $T^{-\kappa} \times t $ & $\hat{m}^{(3)}_{t+1} = \frac{m^{(3)}_{t+1}}{1-\beta_3^{t+1}}$ & $\hat{v}_{t+1} = \frac{v_{t+1}}{1-\beta_2^{t+1}}$ \\
\bottomrule
\end{tabular}%
}
\end{minipage}
\end{table}


\paragraph{Ablation Results.} A first observation is that \Ademamix (DW) significantly under-performs \ADana, especially at larger scales. To understand the cause of this performance gap, we ran \ADana  with $\alpha(t) = T^{-\kappa}\times t$ schedule (\ADana 2), showing a strong decrease in performance. However, combining with a short-momentum EMA $\beta_3=0.9$ (\ADana $5$) almost recovers \Ademamix (DW) performance. This proves that the \Danaconstant-type schedule $\alpha(t)=T^{-\kappa}\times t$ appears to be less stable than our \ADana schedule $\alpha(t) = (1+t)^{1-\kappa}$. On the other hand a short momentum average (as is used in \Ademamix) can recover competitive performance and partly close the gap (while still showing decreased performance compared to \ADana). Similarly, adding a short-momentum EMA $\beta_3=0.9$ to original \ADana shows increased performance, especially at small scale, but remains negligible at larger scale. Note that it additionally increases memory by storing an additional momentum copy. Finally, constant $\beta_2$ has little effect at small scale but shows some instabilities at larger scales and especially diverges for one run at about $10^2$ PFH of compute.

\begin{figure}[h]
    \centering
    \includegraphics[width=0.9\linewidth]{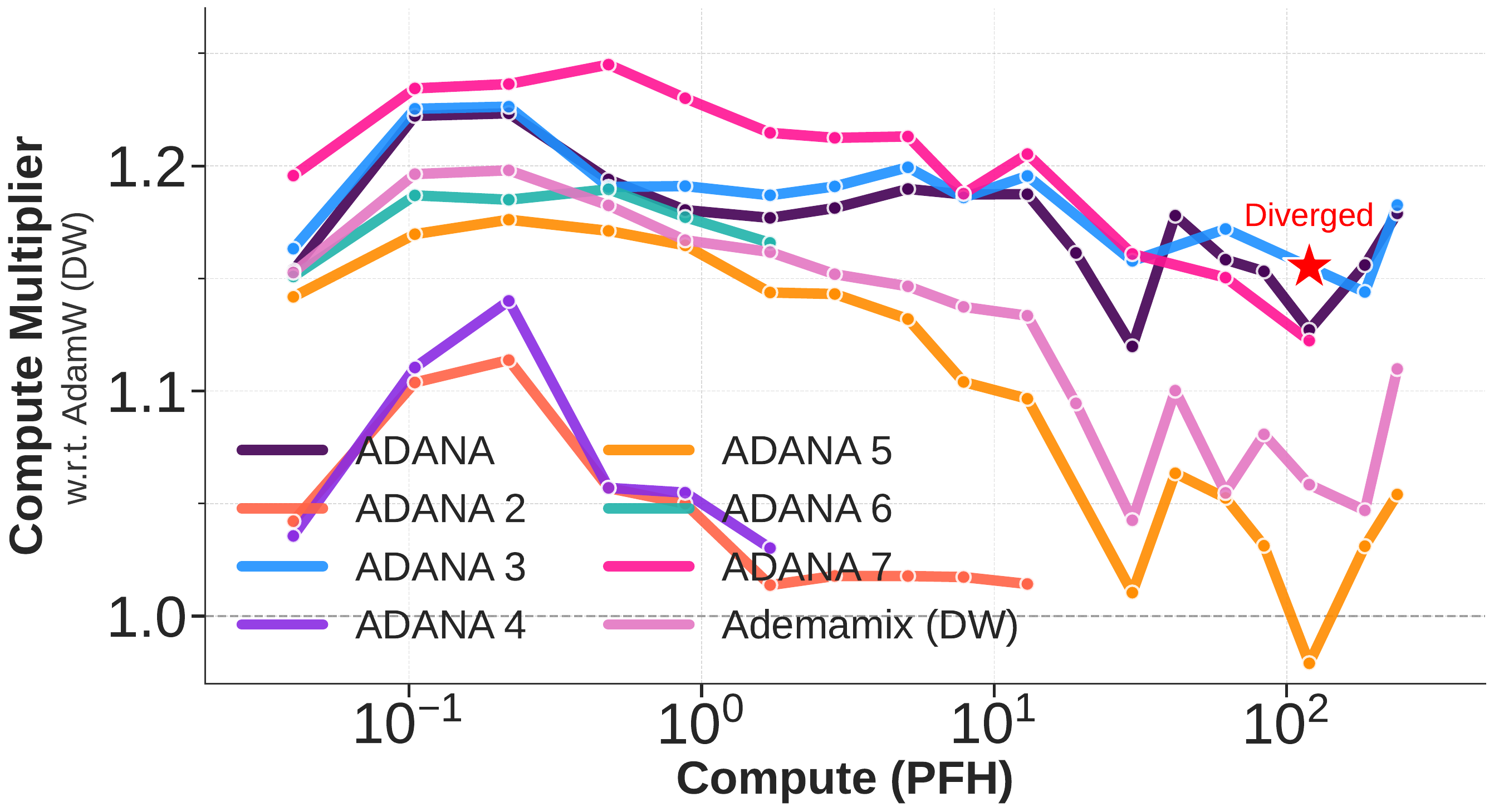}
    \caption{\textbf{Ablation Study.} Compute multipliers of variants of \ADana and \Ademamix (DW) relative to \AdamW (DW). All variants and the baseline \AdamW use log-time weight-decay and are summarized in \Cref{tab:adana_variants_small,tab:abl_study} and \Cref{alg:adana_variants_description}.}
    \label{fig:comparison_adana_variants}
\end{figure}


\clearpage
\section{Scaling Laws}
\label{sec:scaling_laws}


\subsection{Single vs Broken Power-Law Fits}
\label{subsec:single_vs_broken}

As shown in Figures~\ref{fig:scaling_main} and \ref{fig:single_vs_broken}, we fit loss scaling curves to the functional form
\begin{equation}
\label{eq:broken_power_law}
L(C) = a + b \cdot C^{-c} + e \cdot C^{-f},
\end{equation}
where $C$ denotes compute in PetaFlop-Hours (PFH)\footnote{One PetaFlop-Hour equals $3.6 \times 10^{18}$ floating-point operations ($10^{15} \times 3600$). Our experiments span approximately $10^{-2}$ to $2 \times 10^{2}$ PetaFlop-Hours. This unit is a convenient choice because a single H100 GPU delivers roughly one PetaFLOP of effective throughput for mixed-precision training, so PetaFlop-Hours correspond in order of magnitude to GPU-hours.}, $a$ is a shared saturation level across all optimizers, and $(b, c, e, f)$ are optimizer-specific parameters. This ``broken power-law'' form extends the standard single power-law model $L(C) = a + b \cdot C^{-c}$ used in Chinchilla-style scaling analyses \citep{hoffmann2022chinchilla}.

The need for broken power laws has been identified in prior work. \citet{caballero2023broken} demonstrate that neural network performance often exhibits multiple distinct power-law regimes with smooth transitions between them, rather than following a single power law. Similarly, \citet{paquette20244+} show that even simplified models of compute-optimal scaling (the PLRF) necessitates functional forms like \eqref{eq:broken_power_law}. These findings motivate the use of a richer functional form that can capture transitions between scaling regimes.

\subsection{Why Single Power Laws Are Insufficient}
\label{subsec:single_power_law_inadequate}


A standard single power law fit of the form $L(C) = a + b \cdot C^{-c}$ fails to capture the curvature present in our data at both small and large compute scales. Figure~\ref{fig:single_vs_broken} compares single and broken power law fits for the same data. The single power law systematically underestimates loss at small scales and overestimates at large scales, leading to poor extrapolation properties.

\begin{figure}[t]
\centering
\begin{subfigure}[t]{0.48\textwidth}
\centering
\includegraphics[width=\textwidth]{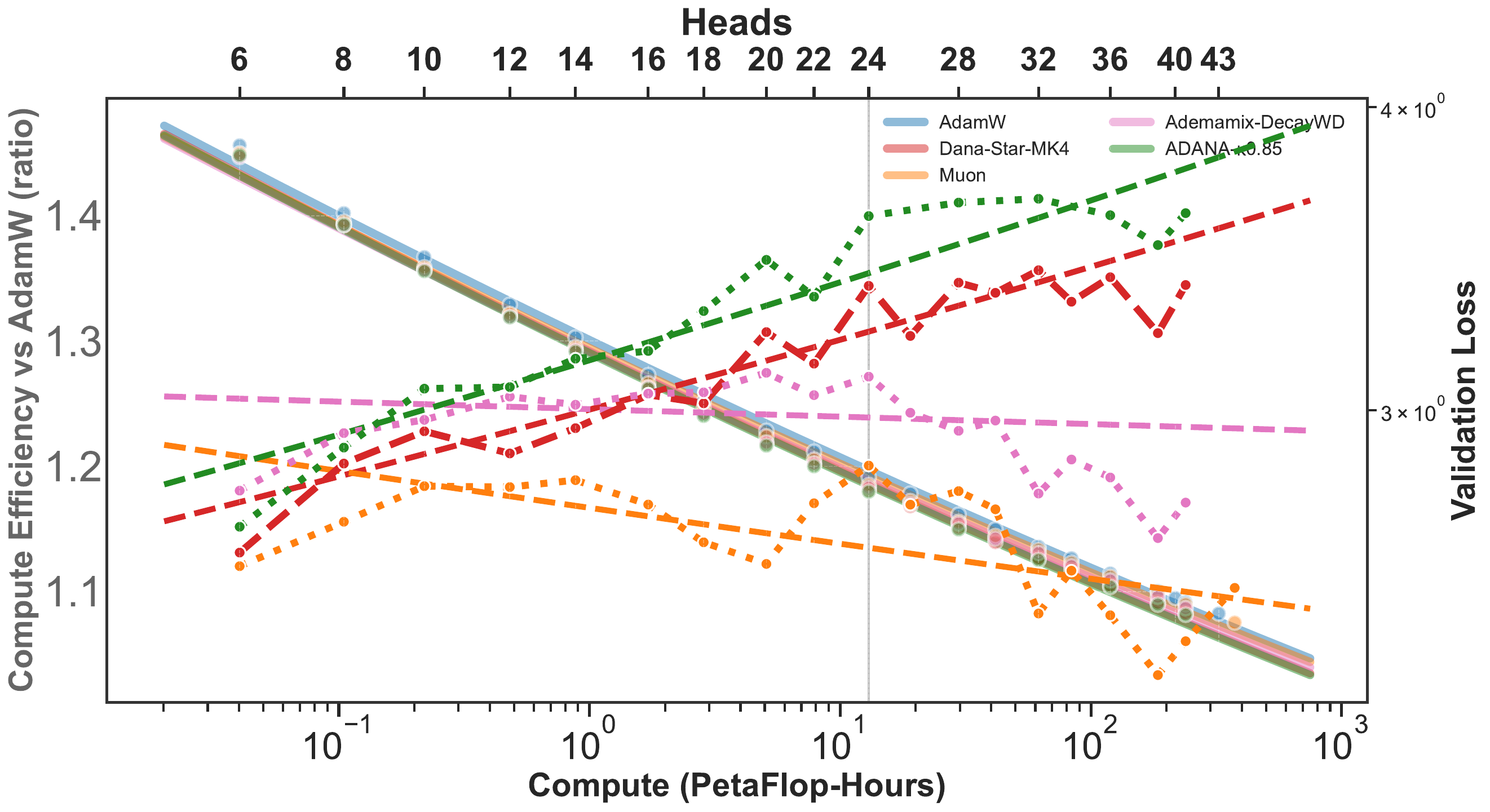}
\caption{Single power-law: $L = a + eC^{-f}$. Fitted $a = 1.045$; for \AdamW: $e = 2.16$, $f = 0.074$ ($R^2 = 0.997$).}
\label{fig:single_power_law}
\end{subfigure}
\hfill
\begin{subfigure}[t]{0.48\textwidth}
\centering
\includegraphics[width=\textwidth]{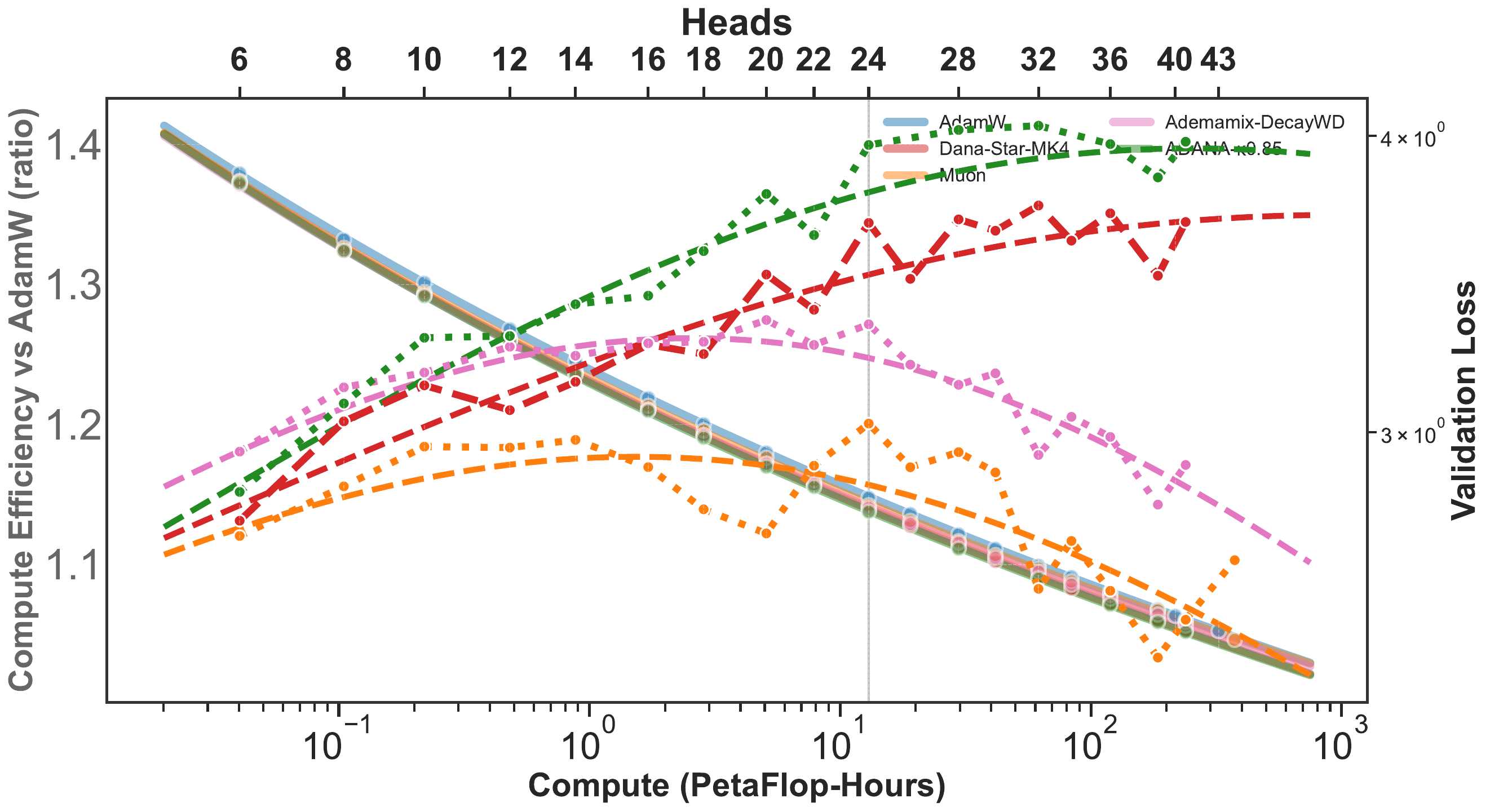}
\caption{Broken power-law: $L = a + bC^{-c} + eC^{-f}$. Fitted $a = 0.106$; for \AdamW: $b = 0.506$, $c = 0.191$, $e = 2.58$, $f = 0.027$ ($R^2 = 0.9999$).}
\label{fig:broken_power_law}
\end{subfigure}
\caption{\textbf{Single vs broken power-law fits} using Chinchilla compute ($C = 6ND$). (a) The single power law forces a compromise between fitting small and large compute scales, yielding a saturation $a \approx 1.0$ much higher than the broken power law's $a \approx 0.11$. (b) The broken power law captures two distinct scaling regimes with different exponents ($c \approx 0.19$ vs $f \approx 0.03$), enabling accurate extrapolation.}
\label{fig:single_vs_broken}
\end{figure}

A deeper issue with the single power law is that the exponent $f$ is not uniquely determined by the data. Because $a$ represents intrinsic model capacity---an unknown quantity that we assume all optimizers can achieve---different assumptions about $a$ lead to different fitted exponents. Table~\ref{tab:single_pl_sensitivity} demonstrates this sensitivity: constraining $a$ to small values (as implied by the broken power law) yields $f \approx 0.048$, while allowing $a$ to grow large yields $f \approx 0.074$.

\begin{table}[t]
\centering
\caption{\textbf{Single power-law exponent depends on assumed saturation (Chinchilla $6N$ compute).} Fitting $L = a + eC^{-f}$ to \AdamW data with different upper bounds on the saturation parameter $a$. The exponent $f$ varies by 54\% (from 0.048 to 0.074) depending on the bounds for the assumed saturation level.}
\label{tab:single_pl_sensitivity}
\vspace{0.5em}
\begin{tabular}{cccc}
\toprule
\textbf{Upper bound on $a$} & \textbf{Fitted $a$} & \textbf{$e$} & \textbf{$f$} \\
\midrule
0.11 & 0.013 & 3.20 & 0.048 \\
0.15 & 0.019 & 3.20 & 0.048 \\
0.50 & 0.118 & 3.10 & 0.050 \\
1.00 & 0.372 & 2.84 & 0.055 \\
2.00 & 0.885 & 2.32 & 0.068 \\
$0.95 \times L_{\min}$ (default)\footnote{$L_{\min}$ denotes the lowest observed validation loss across all runs, providing a data-driven upper bound on the saturation parameter.}  & 1.046 & 2.16 & 0.074 \\
\bottomrule
\end{tabular}
\end{table}

This sensitivity arises because the single power law conflates two distinct scaling regimes into a single exponent. When $a$ is constrained to be small, the fit captures the gradual long-term regime ($f \approx 0.05$); when $a$ is allowed to be large, the fit shifts toward the steep initial regime ($f \approx 0.08$). The broken power law resolves this ambiguity by fitting both regimes explicitly, yielding $c \approx 0.21$ for the steep regime and $f \approx 0.03$ for the gradual regime, with $a \approx 0.106$ determined by the data rather than assumed.  Table~\ref{tab:scaling_law_params_chinchilla} summarizes the fitted parameters for all optimizers under both functional forms using the Chinchilla compute formula. 

\begin{table}[t]
\centering
\caption{\textbf{Fitted scaling law parameters (Chinchilla $6N$ compute).} All fits share a common saturation level $a$ across optimizers. The single power law uses form $L = a + eC^{-f}$; the broken power law uses $L = a + bC^{-c} + eC^{-f}$ where the first term ($b, c$) captures the steep initial regime and the second term ($e, f$) captures the gradual long-term regime. For Kaplan $6P$ compute, see Table~\ref{tab:scaling_law_params_kaplan}.}
\label{tab:scaling_law_params_chinchilla}
\vspace{0.5em}
\begin{tabular}{l|ccc|ccccc}
\toprule
& \multicolumn{3}{c|}{\textbf{Single Power Law} ($a = 1.042$)} & \multicolumn{5}{c}{\textbf{Broken Power Law} ($a = 0.106$)} \\
\textbf{Optimizer} & $e$ & $f$ & $R^2$ & $b$ & $c$ & $e$ & $f$ & $R^2$ \\
\midrule
\AdamW             & 2.164 & 0.074 & 0.9971 & 0.405 & 0.211 & 2.68 & 0.030 & 0.99996 \\
\DanastarMKfour ($\kappa\!=\!0.85$)    & 2.132 & 0.077 & 0.9970 & 0.404 & 0.217 & 2.64 & 0.030 & 0.99991 \\
\ADana ($\kappa\!=\!0.85$) & 2.124 & 0.075 & 0.9963 & 0.404 & 0.217 & 2.64 & 0.030 & 0.99995 \\
\Muon            & 2.139 & 0.073 & 0.9960 & 0.404 & 0.217 & 2.65 & 0.029 & 0.99990 \\
\Ademamix (dec.~WD) & 2.129 & 0.074 & 0.9962 & 0.404 & 0.218 & 2.64 & 0.029 & 0.99995 \\
\bottomrule
\end{tabular}
\end{table}

\begin{table}[t]
\centering
\caption{\textbf{Fitted scaling law parameters (Kaplan $6P$ compute).} Same format as Table~\ref{tab:scaling_law_params_chinchilla} but using the Kaplan compute formula $6P$ which excludes embedding parameters. The single power-law exponents are smaller ($f \approx 0.064$--$0.066$) than with Chinchilla compute, reflecting the different compute scale.}
\label{tab:scaling_law_params_kaplan}
\vspace{0.5em}
\begin{tabular}{l|ccc|ccccc}
\toprule
& \multicolumn{3}{c|}{\textbf{Single Power Law} ($a = 0.987$)} & \multicolumn{5}{c}{\textbf{Broken Power Law} ($a = 0.105$)} \\
\textbf{Optimizer} & $e$ & $f$ & $R^2$ & $b$ & $c$ & $e$ & $f$ & $R^2$ \\
\midrule
\AdamW             & 2.146 & 0.065 & 0.9998 & 0.402 & 0.125 & 2.63 & 0.034 & 0.99977 \\
\DanastarMKfour ($\kappa\!=\!0.85$)    & 2.114 & 0.066 & 0.9998 & 0.402 & 0.130 & 2.59 & 0.033 & 0.99975 \\
\ADana ($\kappa\!=\!0.85$) & 2.107 & 0.066 & 0.9998 & 0.401 & 0.129 & 2.59 & 0.033 & 0.99976 \\
\Muon            & 2.124 & 0.064 & 0.9997 & 0.402 & 0.133 & 2.60 & 0.032 & 0.99970 \\
\Ademamix (dec.~WD) & 2.113 & 0.064 & 0.9998 & 0.402 & 0.132 & 2.59 & 0.032 & 0.99973 \\
\bottomrule
\end{tabular}
\end{table}


Several observations emerge from these fits. The steep-regime exponent $c$ in the broken power law shows meaningful variation across optimizers (0.110--0.123), with \DanastarMKfour ($\kappa = 0.85$) having the smallest value and \Muon the largest. In contrast, the gradual-regime parameters ($e, f$) are nearly identical across optimizers, indicating the efficiency gains do not appear to continue scaling.

\subsection{Comparison of compute formulas}
\label{subsec:compute_formulas}


The relationship between compute budget $C$ and model/data scale depends on how we measure the model's computational cost. In prior work, several conventions have been proposed \citep{kaplan2020scaling,hoffmann2022chinchilla,deepseek2024deepseekllm}. We compare three formulations of FLOPs per token:\footnote{The notation $6N_1$, $6N_2$, and $M$ follows the DeepSeek scaling analysis \citep{deepseek2024deepseekllm}. We use $6P$ and $6N$ to emphasize the connection to model parameters.}
\begin{align}
6P &= 72 \, n_{\text{layer}} \, n_{\text{embd}}^2 & &\text{(Kaplan)} \label{eq:6P} \\
6N &= 72 \, n_{\text{layer}} \, n_{\text{embd}}^2 + 6 \, n_{\text{vocab}} \, n_{\text{embd}} & &\text{(Chinchilla)} \label{eq:6N} \\
M &= 72 \, n_{\text{layer}} \, n_{\text{embd}}^2 + 12 \, n_{\text{layer}} \, n_{\text{embd}} \, l_{\text{seq}} & &\text{(DeepSeek)} \label{eq:M}
\end{align}
where $n_{\text{layer}}$ is the number of transformer layers, $n_{\text{embd}}$ is the model width, $n_{\text{vocab}}$ is the vocabulary size, and $l_{\text{seq}}$ is the sequence length. Here $P$ denotes the non-embedding parameter count and $N$ the total parameter count (including embeddings).

The formula $6P$ counts only non-embedding matrix multiplications and corresponds to the approximation $C \approx 6PD$ used by \citet{kaplan2020scaling}, which we use throughout this paper. The formula $6N$ additionally includes the embedding and unembedding operations, corresponding to the Chinchilla \citep{hoffmann2022chinchilla} approach. The DeepSeek formula $M$ includes the attention computation overhead (the $12 \, n_{\text{layer}} \, n_{\text{embd}} \, l_{\text{seq}}$ term accounts for the QK and attention-value products). \citet{deepseek2024deepseekllm} report that the $M$ formula produced better scaling law extrapolations for their models.


Table~\ref{tab:compute_formula_comparison} compares single power law fits $L = a + e C^{-f}$ using each compute formula on \AdamW training data. The choice of compute formula affects both the fitted saturation level $a$ and the scaling exponent $f$: formulas that account for additional compute (embedding in $6N$, attention in $M$) yield larger effective compute values and hence different exponents.

\begin{table}[b]
\centering
\caption{\textbf{Sensitivity of scaling exponent to compute formula.} Single power law fits $L = a + eC^{-f}$ for \AdamW on Enoki architecture using different compute formulas. The $6P$ formula (non-embedding only) is our default. Including embeddings ($6N$) or attention overhead ($M$) yields different exponents. The right columns show fits with $a = 0$ (no saturation), which yields smaller exponents and worse fits. Fit residuals are shown in Figure~\ref{fig:compute_formula_residuals}.}
\label{tab:compute_formula_comparison}
\vspace{0.5em}
\begin{tabular}{l|cccc|ccc}
\toprule
& \multicolumn{4}{c|}{\textbf{Fitted $a$}} & \multicolumn{3}{c}{\textbf{Forced $a = 0$}} \\
\textbf{Compute Formula} & $a$ & $e$ & $f$ & $R^2$ & $e$ & $f$ & $R^2$ \\
\midrule
$6P$ (Kaplan, default) & 0.977 & 2.16 & 0.065 & 0.9998 & 3.15 & 0.043 & 0.9988 \\
$6N$ (Chinchilla) & 1.046 & 2.16 & 0.074 & 0.9971 & 3.22 & 0.048 & 0.9928 \\
$M$ (DeepSeek) & 1.010 & 2.17 & 0.069 & 0.9997 & 3.19 & 0.045 & 0.9979 \\
\bottomrule
\end{tabular}
\end{table}

\newpage

\begin{figure}[t]
\centering
\includegraphics[width=\textwidth]{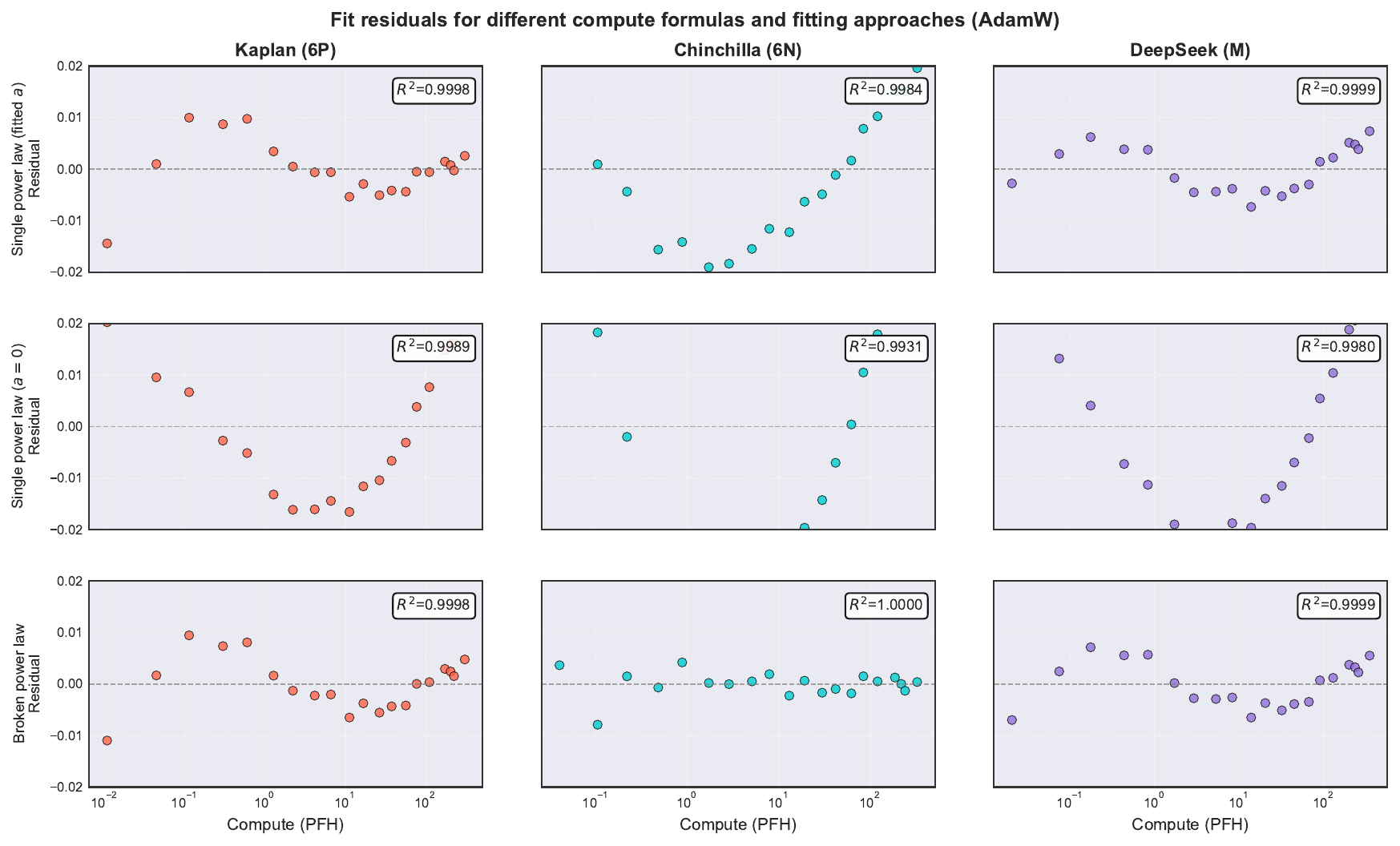}
\caption{\textbf{Fit residuals for different compute formulas and fitting approaches (AdamW).} Columns: Kaplan ($6P$), Chinchilla ($6N$), and DeepSeek ($M$) compute formulas. Rows: single power law $L = a + eC^{-f}$ with fitted saturation (top), single power law with forced $a=0$ (middle), broken power law $L = a + bC^{-c} + eC^{-f}$ (bottom). The forced $a=0$ fits show systematic positive residuals at small compute scales. The Chinchilla formula with broken power law achieves the tightest fit ($R^2 = 0.9999$) by capturing both large-scale and small-scale behavior.}
\label{fig:compute_formula_residuals}
\end{figure}

The variation in exponent $f$ (from 0.065 to 0.074) across compute formulas illustrates a methodological concern: reported scaling exponents depend on conventions for measuring compute. The $6N$ formula yields the largest exponent because embedding operations contribute a larger fraction of compute at small scales than at large scales, creating an apparent steeper scaling. The DeepSeek $M$ formula partially compensates by including attention overhead; as we do not scale the sequence length, this disappears with model scale.

The right columns of Table~\ref{tab:compute_formula_comparison} show fits with $a = 0$ forced, corresponding to pure power laws $L = eC^{-f}$ without saturation. Forcing $a = 0$ systematically reduces the scaling exponent (from $f \approx 0.065$--$0.074$ to $f \approx 0.043$--$0.048$) and degrades fit quality. The residual plots in Figure~\ref{fig:compute_formula_residuals} show that the forced $a=0$ fits have systematic positive residuals at small compute scales, while the free $a$ fits have more uniformly distributed residuals. 


\subsection{Does \ADana Outscale \AdamW?}
\label{subsec:outscaling}

A natural question is whether \ADana (or other optimizers) can ``outscale'' \AdamW in the sense of \cite{ferbach2025dimension}: an algorithm \emph{outscales} another if it achieves a larger loss exponent. On the power-law random features (PLRF) model, \Dana provably outscales \SGD when the spectral dimension satisfies $2\rho > 1$.

For the Chinchilla problem---compute-optimal training of transformers on natural language---this question is more subtle. The broken power law
\begin{equation}
L = a + bC^{-c} + eC^{-f}
\label{eq:broken_pl_outscaling}
\end{equation}
captures two distinct scaling regimes: a small-scale regime dominated by $bC^{-c}$ (steep decay) and a large-scale regime dominated by $eC^{-f}$ (gradual decay). If the exponents $c$ and $f$ differ across optimizers, this would suggest different asymptotic scaling behavior.

\begin{figure}[t]
\centering
\includegraphics[width=\textwidth]{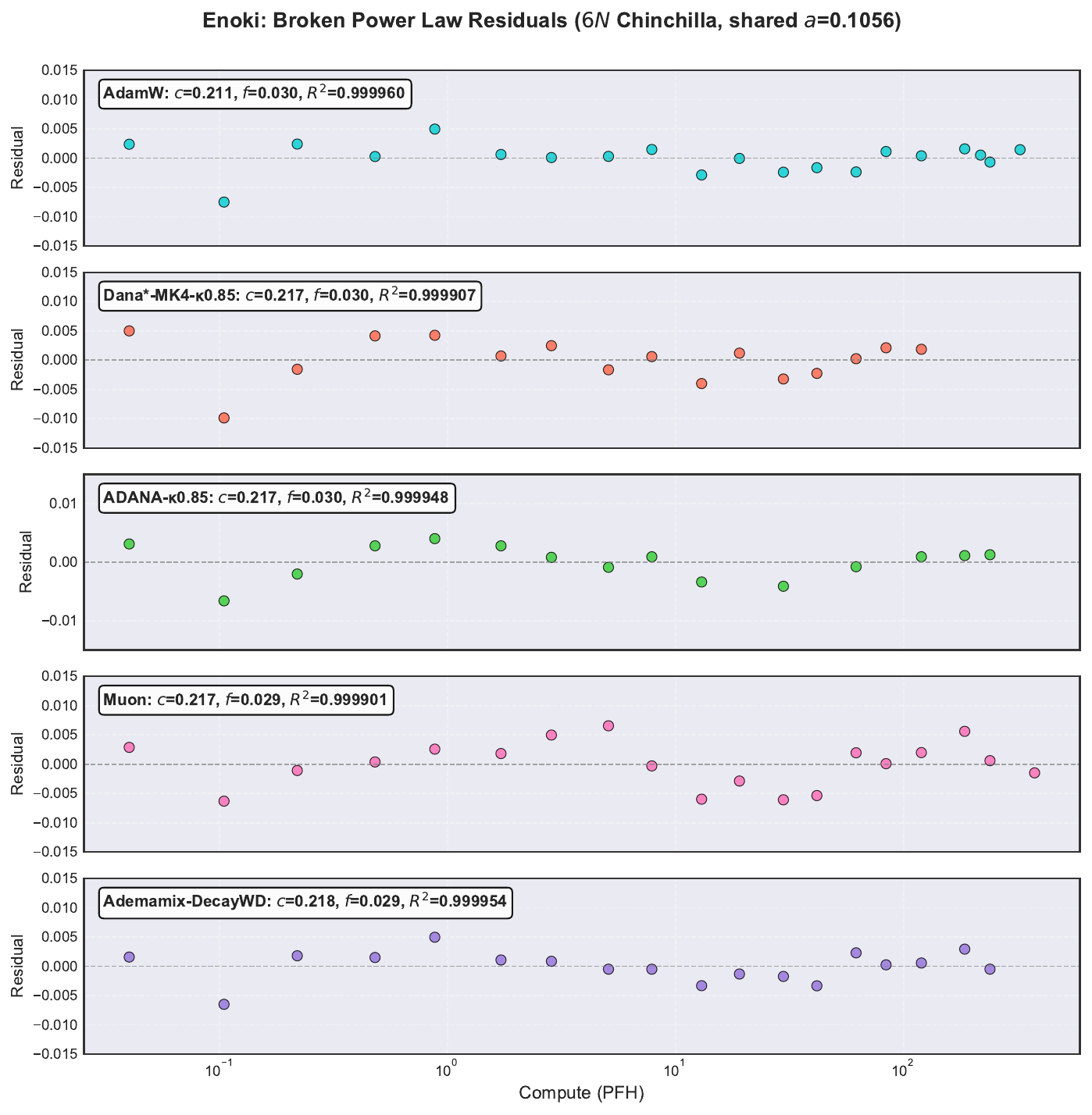}
\caption{\textbf{Broken power law residuals across optimizers (6N Chinchilla, shared $a = 0.106$).} All optimizers achieve $R^2 \geq 0.9999$, indicating that the broken power-law functional form captures the scaling behavior consistently across optimization algorithms. The uniformly tight fits suggest that scaling law predictions can be reliably extrapolated from any optimizer.}
\label{fig:optimizer_residuals}
\end{figure}

\begin{table}[t]
\centering
\caption{\textbf{Broken power-law exponents by optimizer.} Fits of~\eqref{eq:broken_pl_outscaling} with shared saturation $a = 0.106$. The small-scale exponent $c$ (steep decay) shows modest variation across optimizers, while the large-scale exponent $f$ (gradual decay) is nearly identical ($f \approx 0.029$--$0.030$).}
\label{tab:outscaling_exponents}
\vspace{0.5em}
\begin{tabular}{l|cccc|c}
\toprule
\textbf{Optimizer} & $b$ & $c$ & $e$ & $f$ & $R^2$ \\
\midrule
\AdamW           & 0.405 & 0.211 & 2.68 & 0.030 & 0.99996 \\
\DanastarMKfour{} ($\kappa$=0.85) & 0.404 & 0.217 & 2.64 & 0.030 & 0.99991 \\
\ADana{} ($\kappa$=0.85)    & 0.404 & 0.217 & 2.64 & 0.030 & 0.99995 \\
\Muon             & 0.404 & 0.217 & 2.65 & 0.029 & 0.99990 \\
\Ademamix{} (dec.~WD) & 0.404 & 0.218 & 2.64 & 0.029 & 0.99995 \\
\bottomrule
\end{tabular}
\end{table}

Table~\ref{tab:outscaling_exponents} and Figure~\ref{fig:optimizer_residuals} show the fitted exponents for each optimizer using the $6N$ Chinchilla compute formula. Two observations emerge:

\textbf{Small-scale behavior.} The exponent $c$ is nearly identical across all optimizers ($c \approx 0.211$--$0.218$), with only a 3\% variation between \AdamW ($c = 0.211$) and the \Dana variants ($c \approx 0.217$).

\textbf{Large-scale behavior.} The exponent $f$ governing large-scale behavior is also nearly identical across all optimizers ($f \approx 0.029$--$0.030$). This indicates that the compute savings achieved by \Dana variants persist at scale: the performance gap between optimizers appears as a consistent vertical shift in log-log space rather than a change in scaling exponent.

\emph{We emphasize that these changes are modest, and that further scaling experiments are needed to validate the scaling gap between the optimizers.}

\subsection{Measured exponents from the literature}
\label{subsec:measured_exponents_literature}


To contextualize our measured scaling exponents, we survey the compute scaling law literature. Two distinct exponents appear in these discussions: the \emph{loss scaling exponent} $f$ describing how loss decreases with compute ($L \propto C^{-f}$), and the \emph{allocation exponent} $a$ describing how optimal model size scales with compute ($N^* \propto C^a$). We focus primarily on the loss scaling exponent, which is directly comparable to our measurements.

\begin{table}[t]
\centering
\caption{\textbf{Compute scaling law exponents from the literature.} Loss exponent $f$ from fits of the form $L = eC^{-f}$ (pure power law) or $L = a + eC^{-f}$ (with irreducible loss). The ``Formula'' column indicates the compute calculation: $6P$ = Kaplan (non-embedding FLOPs), $6N$ = Chinchilla (includes embeddings), $M$ = DeepSeek (includes attention overhead). Dagger ($\dagger$) indicates exponent was extracted from figures rather than reported directly; question mark indicates formula unclear from paper.  All optimizers are \AdamW, except for \cite{qiu2025hyperparameter} which uses \Muon and OpenAI which are unknown.}
\label{tab:literature_exponents}
\vspace{0.5em}
\small
\begin{tabular}{l|c|c|c|c|l}
\toprule
\textbf{Paper} & \textbf{Formula} & \textbf{Loss Exp.\ $f$} & \textbf{Irred.?} & \textbf{Max Scale (PfH)} & \textbf{Dataset} \\
\midrule
Kaplan et al.\ (2020) & $6P$ & 0.050 & No & $\sim$550 & WebText2 \\
Brown et al.\ (2020) [GPT-3] & $6P$ & 0.048 & No & $\sim$87,000 & Common Crawl \\
\multirow{2}{*}{OpenAI (2023) [GPT-4]} & \multirow{2}{*}{?} & 0.067$^\dagger$ & No & \multirow{2}{*}{$>10^6$} & \multirow{2}{*}{Internal} \\
 &  & 0.120$^\dagger$ & Yes &  &  \\
\multirow{2}{*}{Hoffmann et al.\ (2022)} & \multirow{2}{*}{$6N$} & 0.05 & No & \multirow{2}{*}{$\sim$833} & \multirow{2}{*}{MassiveText} \\
 &  & 0.154 & Yes &  &  \\
DeepSeek-AI (2024) & $M$ & 0.048$^\dagger$ & No & $\sim$125,000 & In-house \\
Charles et al.\ (2025) [DiLoCo] & $6P$ & 0.048 & No & $\sim$33,000 & C4, Dolma \\
Dey et al.\ (2025) [CompleteP] & ? & 0.051--0.056 & No & $\sim$278 & SlimPajama \\
Gadre et al.\ (2025) [Gemstones] & $6N$ & 0.052$^\dagger$ & Yes & $\sim$5,600 & Dolma 1.7 \\
Porian et al.\ (2024) & $6P$ & 0.050$^\dagger$ & No & $\sim$28 & OpenWebText2 \\
Qiu et al.\ (2025) [Muon $\mu$P] & $6N$ & 0.049--0.051$^\dagger$ & No & $\sim$47 & FineWeb \\
\midrule
\multirow{6}{*}{\textbf{This work}} & \multirow{2}{*}{$6P$} & \textbf{0.043} & No & \multirow{6}{*}{$\sim$370} & \multirow{6}{*}{FineWeb} \\
 &  & \textbf{0.065} & Yes &  &  \\
 & \multirow{2}{*}{$6N$} & \textbf{0.048} & No &  &  \\
 &  & \textbf{0.074} & Yes &  &  \\
 & \multirow{2}{*}{$M$} & \textbf{0.045} & No &  &  \\
 &  & \textbf{0.069} & Yes &  &  \\
\bottomrule
\end{tabular}
\end{table}

Table~\ref{tab:literature_exponents} summarizes scaling law exponents across the literature. A striking pattern emerges: nearly all studies report $f \approx 0.05$ when fitting pure power laws without an irreducible loss term. The apparent discrepancy with Chinchilla's $f \approx 0.15$ arises because their exponent applies to the reducible portion $L - E$, not the total loss.

\paragraph{Kaplan et al.\ (2020) \citep{kaplan2020scaling}.} The foundational scaling laws paper fit $L(C) = (C_c/C)^{f_C}$ with $f_C = 0.050$ and $C_c = 3.1 \times 10^8$ PF-days. They advocated training larger models on less data than subsequent work would recommend, finding optimal allocation $N^* \propto C^{0.73}$.

\paragraph{Brown et al.\ (2020) [GPT-3] \citep{brown2020language}.} Validated Kaplan scaling at unprecedented scale, fitting $L(C) = 2.57 \cdot C^{-0.048}$ with $C$ in PF-days. GPT-3 175B required approximately 87,000 PfH of compute.

\paragraph{OpenAI (2023) [GPT-4] \citep{openai2023gpt4}.} The GPT-4 technical report demonstrates that scaling laws enable accurate capability prediction: they fit $L(C) = aC^{-b} + c$ (with irreducible loss term $c$) using models trained with at most $1/10{,}000$th the compute of GPT-4, and accurately predicted GPT-4's final loss on an internal codebase. Although the specific parameters were not disclosed, we extracted data from their Figure~1 using automatic circle detection and fit scaling laws. A pure power law yields $f \approx 0.067$ ($R^2 = 0.96$), while including an irreducible loss term $E \approx 0.92$ bits/word yields $f \approx 0.120$ ($R^2 = 0.9999$). The latter is consistent with Chinchilla's exponent on reducible loss, while the former is substantially larger than other claimed data in the literature. The compute range spans approximately 10 orders of magnitude (from $10^{-10}$ to $1$ relative to GPT-4).

\paragraph{Hoffmann et al.\ (2022) [Chinchilla] \citep{hoffmann2022chinchilla}.} Introduced the two-term loss decomposition $L(N,D) = E + A/N^\alpha + B/D^\beta$ with irreducible loss $E = 1.69$. At compute-optimality, the reducible loss scales as $(L - E) \propto C^{-\gamma}$ with $\gamma = \alpha\beta/(\alpha+\beta) \approx 0.154$. Crucially, this exponent applies only to the reducible portion; fitting a pure power law to their data yields $f \approx 0.05$, consistent with other work. Chinchilla established the ``20 tokens per parameter'' guideline ($N^* \propto C^{0.5}$).

\paragraph{DeepSeek-AI (2024) \citep{deepseek2024deepseekllm}.} Proposed using non-embedding FLOPs per token ($M$) rather than parameter count ($N$) as the model scale metric, arguing this gives better extrapolation from small to large models. They find optimal allocation $M_{\text{opt}} \propto C^{0.524}$. The loss scaling exponent $f \approx 0.048$ was extracted from their Figure~5 by digitizing pixel coordinates and fitting a power law ($R^2 = 0.983$).

\paragraph{Charles et al.\ (2025) [DiLoCo] \citep{charles2025communication}.} Studied scaling laws for distributed training with local optimization. At Chinchilla-optimal allocation ($D = 20N$, hence $C = 120N^2$), their parameter scaling $L \propto N^{-0.095}$ converts to $L \propto C^{-0.048}$ on compute.

\paragraph{Dey et al.\ (2025) [CompleteP] \citep{dey2025don}.} Compared standard $\mu$P initialization to their proposed CompleteP, finding $f = 0.051$ for $\mu$P and $f = 0.056$ for CompleteP---a modest improvement in scaling exponent from better initialization.

\paragraph{Gadre et al.\ (2025) [Gemstones] \citep{mcleish2025gemstones}.} Systematically varied width and depth to study scaling in a 2D architecture space. A key finding is that the fitted allocation exponent is sensitive to methodology, ranging from $a = 0.46$ to $a = 0.80$ depending on fitting choices. The loss exponent $f \approx 0.052$ was extracted from their Figure~21 (convex hull points) and is consistent across compute scales ($R^2 = 0.9992$).

\paragraph{Porian et al.\ (2024) \citep{porian2024resolving}.} Experimentally reconciled the Kaplan--Chinchilla discrepancy by identifying three factors: (1) FLOP counting conventions, (2) learning rate warmup scaling, and (3) optimizer hyperparameter tuning. Applying all corrections shifts the allocation exponent from Kaplan's $a \approx 0.88$ to Chinchilla's $a \approx 0.50$.

\paragraph{Qiu et al.\ (2025) [Muon $\mu$P] \citep{qiu2025hyperparameter}.} Demonstrated that matrix-preconditioned optimizers (\Muon, \Shampoo) achieve consistent $1.3$--$1.4\times$ compute speedups over \AdamW when using proper hyperparameter transfer via $\mu$P. The loss exponents $f \approx 0.049$--$0.051$ (extracted from their Table~5) are consistent across optimizers---the speedup manifests as a vertical shift in log-log space rather than a change in slope.

\paragraph{Comparison with this work.} When fitting a single power law $L = a + eC^{-f}$ to our \AdamW data, we obtain $f = 0.044$ without saturation ($a = 0$) or $f = 0.065$ with fitted saturation ($a \approx 0.99$). Both values lie within the range of the literature. The sensitivity to the saturation parameter (Table~\ref{tab:single_pl_sensitivity}) underscores the difficulty of comparing exponents across studies with different fitting methodologies. Our broken power law analysis reveals that a single exponent inadequately captures the loss-compute relationship: the small-scale exponent $c \approx 0.115$ governs initial rapid improvement, while the large-scale exponent $f \approx 0.032$ reflects slower asymptotic gains. This two-regime structure suggests that literature exponents $f \approx 0.05$ may represent an average over distinct scaling phases.

\clearpage
\section{Baselining Procedure}
\label{sec:baselining}

We describe our systematic hyperparameter search procedure for fair comparison across optimizers. A rigorous baselining methodology is essential when comparing optimizers, as performance differences can easily be confounded by suboptimal hyperparameter choices for one or more methods.

\subsection{Search Strategy}
\label{sec:search_strategy}

For each optimizer and model size, we perform a two-stage search over learning rate and weight-decay. The first stage uses a coarse grid to identify the region of good performance, followed by a finer search to locate the optimum.

\paragraph{Coarse learning rate search.}
We first search over learning rates on a log scale with factor-of-2 spacing:
\begin{equation}
\gamma^* \in \{\cdots, 2^{-14}, 2^{-13}, 2^{-12}, \cdots, 2^{-4}, 2^{-3}\}.
\end{equation}
The upper and lower bounds of the search are truncated when the loss behavior begins to consistently grow or become unstable, and it contains a local minimum with at least 2 data points on either side.

\paragraph{Fine learning rate search.}
Once the coarse search identifies the best learning rate $\bar{\gamma}$, we perform a finer search using factors of $1.25$:
\begin{equation}
\gamma^* \in \{\bar{\gamma}/1.25^2, \bar{\gamma}/1.25,\bar{\gamma}, \bar{\gamma} \times 1.25, \bar{\gamma} \times 1.25^2\}.
\end{equation}
We continue this refinement until we have at least two data points on either side of the local minimum, ensuring robust identification of the optimal learning rate.

\paragraph{Weight-decay grid.}
For \AdamW, \Ademamix, and other fixed-weight-decay optimizers, we search:
\begin{equation}
\lambda \in \{0, 0.01, 0.03, 0.1, 0.3, 1.0\}.
\end{equation}
For \ADana variants with decaying weight-decay, $\lambda(t) = \omega / (t + t_{\text{wd}})$ where $t_{\text{wd}} = T/10$, we search $\omega$:
\begin{equation}
\omega \in \{1.0, 2.0, 4.0, 8.0\}.
\end{equation}
In practice, we find $\omega = 4.0$ to be robust across model sizes and optimizers.

\subsection{Other training details}
\label{sec:additional_training_details}

\paragraph{Learning Rate Schedule and Batch.} Given a peak learning rate $\gamma^*$, we use linear warmup for $T/50$ steps starting $0.01\times \eta$ at the start of training followed by cosine decay \citep{loshchilov2017sgdr} to $0.1\times \gamma^*$ of the maximum learning rate at the end of training. The batch size is fixed at 32 sequences $\times$ 2048 tokens = 65,536 tokens per step.

\paragraph{Gradient Accumulation.} We employ gradient accumulation over multiple microbatches to simulate larger effective batch sizes, following the implementation of~\citet{semenov2025benchmarking}. For each training iteration, we accumulate gradients over $K$ microbatches (default $K=1$) by dividing the loss by $K$ before the backward pass, ensuring the accumulated gradient represents the average rather than the sum. The cross-entropy loss excludes padding tokens via \texttt{ignore\_index}, with automatic normalization by the count of non-padding tokens only. The forward pass uses \texttt{bfloat16} precision, but model parameters, optimizer state, and accumulated gradients are maintained in \texttt{float32}. Because accumulation occurs in \texttt{float32}, no specialized numerical precision techniques (e.g., Kahan summation) are employed.

\paragraph{Weight-Decay.} All algorithms use independent weight decay \citep{loshchilov2017decoupled}, and we do not apply weight decay to embeddings and layer-norm parameters. We use and compare two different types of weight-decay, either constant $\lambda(t) = \omega / T $ or log-time decaying $\lambda(t) = \omega / (T/10 + t)$  where $\omega$ is a hyperparameter. See Section~\ref{sec:weight_decay} for more details.

\begin{table}[h]
\centering
\caption{\textbf{Training hyperparameters.} Fixed values across all experiments unless otherwise noted.}
\label{table:training_config}
\begin{tabular}{l|c}
\toprule
\textbf{Parameter} & \textbf{Value} \\
\midrule
Sequence length & 2048 \\
Batch size (sequences) & 32 \\
Tokens per batch & 65,536 \\
Vocabulary size & 50,304 \\
Warmup fraction & 0.02 (2\% of total steps) \\
Final LR fraction & 0.10 \\
Precision & bfloat16 \\
Optimizer state precision & float32 \\
Gradient clipping & 0.5 (global norm) \\
\bottomrule
\end{tabular}
\end{table}

\paragraph{Chinchilla-optimal training.}
Following the Chinchilla scaling laws~\citep{hoffmann2022chinchilla}, we train each model for a number of steps proportional to its parameter count, ensuring approximately 20 tokens per parameter. Specifically, for a model with $N$ total parameters:
\begin{equation}
\text{iterations} = \left\lfloor \frac{20 \times N}{65536} \right\rfloor
\end{equation}
where 65,536 is the number of tokens per batch. This ensures that all models in our scaling study are trained in the compute-optimal regime.

\paragraph{Weight-Decay Exclusions}
The following parameter groups are excluded from weight-decay regularization:
\begin{itemize}
    \item Token embeddings
    \item All LayerNorm/RMSNorm scale parameters
    \item All bias parameters
\end{itemize}
This follows standard practice~\citep{loshchilov2017decoupled} and prevents weight-decay from interfering with normalization dynamics. Applying weight-decay to normalization parameters can destabilize training by preventing the model from learning appropriate feature scales.

\paragraph{Compute Budget.}

\begin{table}[h]
\centering
\caption{\textbf{Per-model run compute.} GPU hours for AdamW per model size.}
\label{table:compute_budget}
\begin{tabular}{l|c|c}
\toprule
\textbf{Model Size} & \textbf{GPU Hours} & \textbf{GPU Type} \\
\midrule
45.7M & 0.4 & A100 \\
70.4M & 0.9 & A100 \\
141.0M & 3.7 & A100 \\
254.0M & 14 & A100 \\
423.7M & 17 & H100 \\
664.1M & 41 & H100 \\
1.41B & 173 & H100 \\
2.62B & 626 & H100 \\
\bottomrule
\end{tabular}
\end{table}

\subsection{Learning Rate Scaling Laws}
\label{sec:lr_scaling}


A key component of fair comparison across model scales is determining how hyperparameters---particularly the learning rate---should scale with model size. We fit power-law scaling rules to the optimal learning rates found at each model size, enabling extrapolation to larger scales without exhaustive hyperparameter search.

\subsubsection{Fitting Methodology}

For each optimizer, we collect the top-$K$ learning rates (ranked by final validation loss) at each model size. We use $K=5$ to ensure robustness against noise in individual runs. Each learning rate receives a weight: the best learning rate receives weight $K$, the second-best receives weight $K-1$, and so on.

We fit a saturated power law of the form:
\begin{equation}
\gamma^*(P) = a \cdot (b + P)^{d},
\label{eq:lr_power_law}
\end{equation}
where $P$ denotes the number of non-embedding parameters, and $a, b, d$ are fitted coefficients with constraints $a, b > 0$. The saturation term $b$ accounts for the observation that optimal learning rates do not diverge as model size approaches zero; instead, they saturate to a finite value.

The fit is performed by minimizing a weighted mean-squared error loss in log-space:
\begin{equation}
\mathcal{L} = \frac{\sum_{i} w_i^2 \cdot P_i \cdot (\log \gamma_i^* - \log \hat{\gamma^*}(P_i))^2}{\sum_{i} w_i^2 \cdot P_i},
\end{equation}
where $w_i$ is the weight assigned to each data point and the factor of $P_i$ gives additional emphasis to larger models where extrapolation accuracy matters most. Here $\gamma_i^*$ are the observed peak LR and $\hat{\gamma^*}$ are the predicted LR. We optimize using Adagrad for 200,000 steps, enforcing positivity constraints on $a$ and $b$ via exponential reparameterization.

\subsubsection{Enoki Model Learning Rate Fits}

For the Enoki architecture (Section~\ref{sec:enoki}), we fit learning rate scaling laws using data from model sizes with 6 to 24 attention heads, corresponding to 7M to 510M non-embedding parameters. These fits were used to select learning rates for larger-scale runs (up to 41 heads):

\begin{align}
\text{\AdamW:} \quad & \gamma^* = 1.28 \times 10^1 \cdot (1.67 \times 10^4 + P)^{-0.515}, \label{eq:adamw_enoki_lr}\\
\text{\DanaMKfour:} \quad & \gamma^* = 5.60 \times 10^1 \cdot (5.71 \times 10^3 + P)^{-0.708}, \label{eq:danamk4_enoki_lr}\\
\text{\Ademamix:} \quad & \gamma^* = 4.51 \cdot (6.71 \times 10^3 + P)^{-0.503}, \label{eq:ademamix_enoki_lr}\\
\text{\Muon:} \quad & \gamma^* = 2.19 \cdot (5.64 \times 10^4 + P)^{-0.417}. \label{eq:dmuon_enoki_lr}
\end{align}

Several patterns emerge from these fits. First, the exponent $d$ varies significantly across optimizers, ranging from $-0.42$ for \Muon to $-0.71$ for \DanaMKfour. This indicates that different optimizers have fundamentally different scaling behavior: DANA variants require more aggressive learning rate reduction at large scales, while \Muon maintains relatively high learning rates. Second, the saturation parameter $b$ is typically on the order of $10^3$ to $10^4$, indicating that the power-law regime dominates for models larger than approximately 10M parameters.

\begin{figure}[t]
\centering
\begin{subfigure}[b]{0.48\textwidth}
    \includegraphics[width=\textwidth]{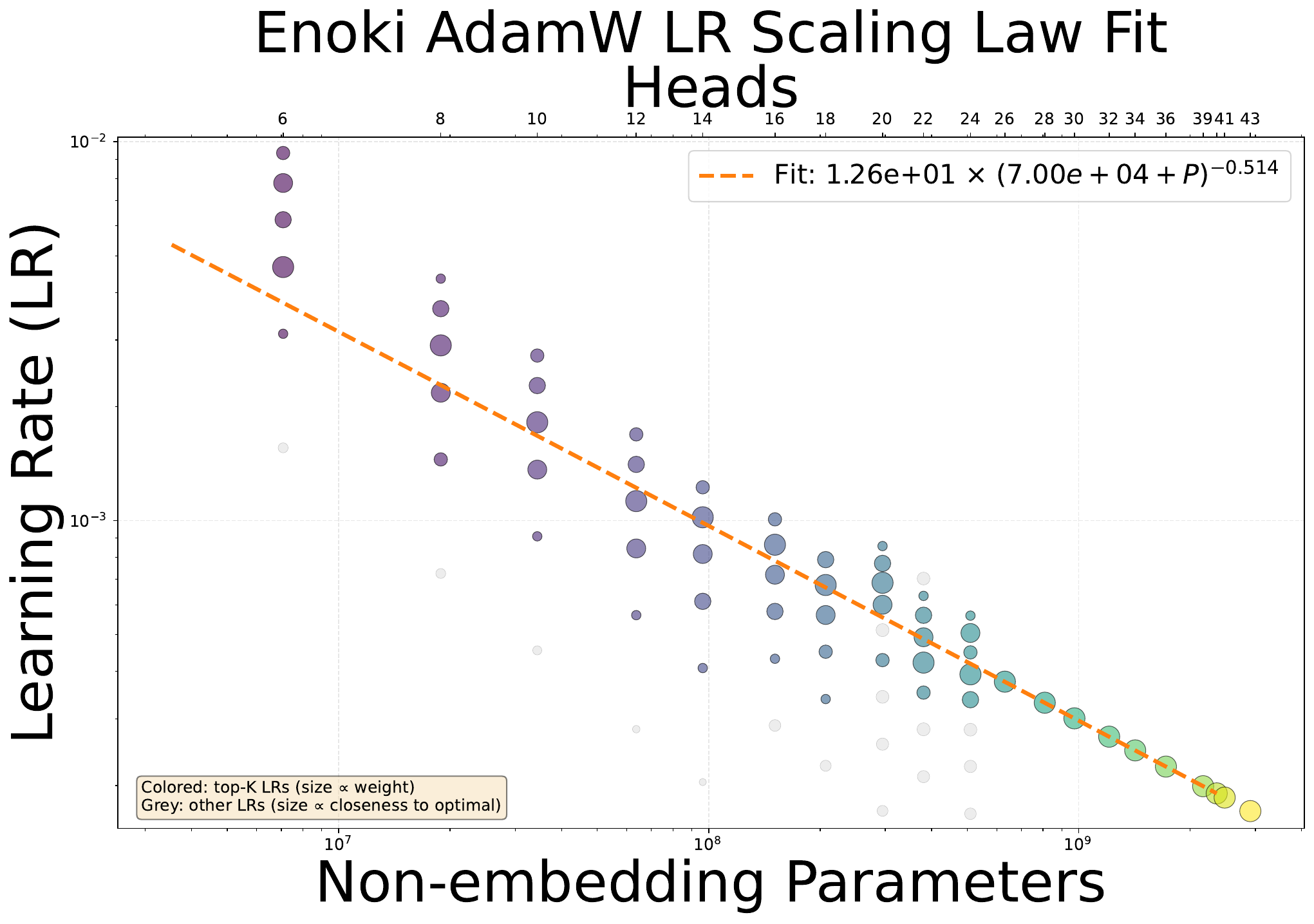}
    \caption{\AdamW}
\end{subfigure}
\hfill
\begin{subfigure}[b]{0.48\textwidth}
    \includegraphics[width=\textwidth]{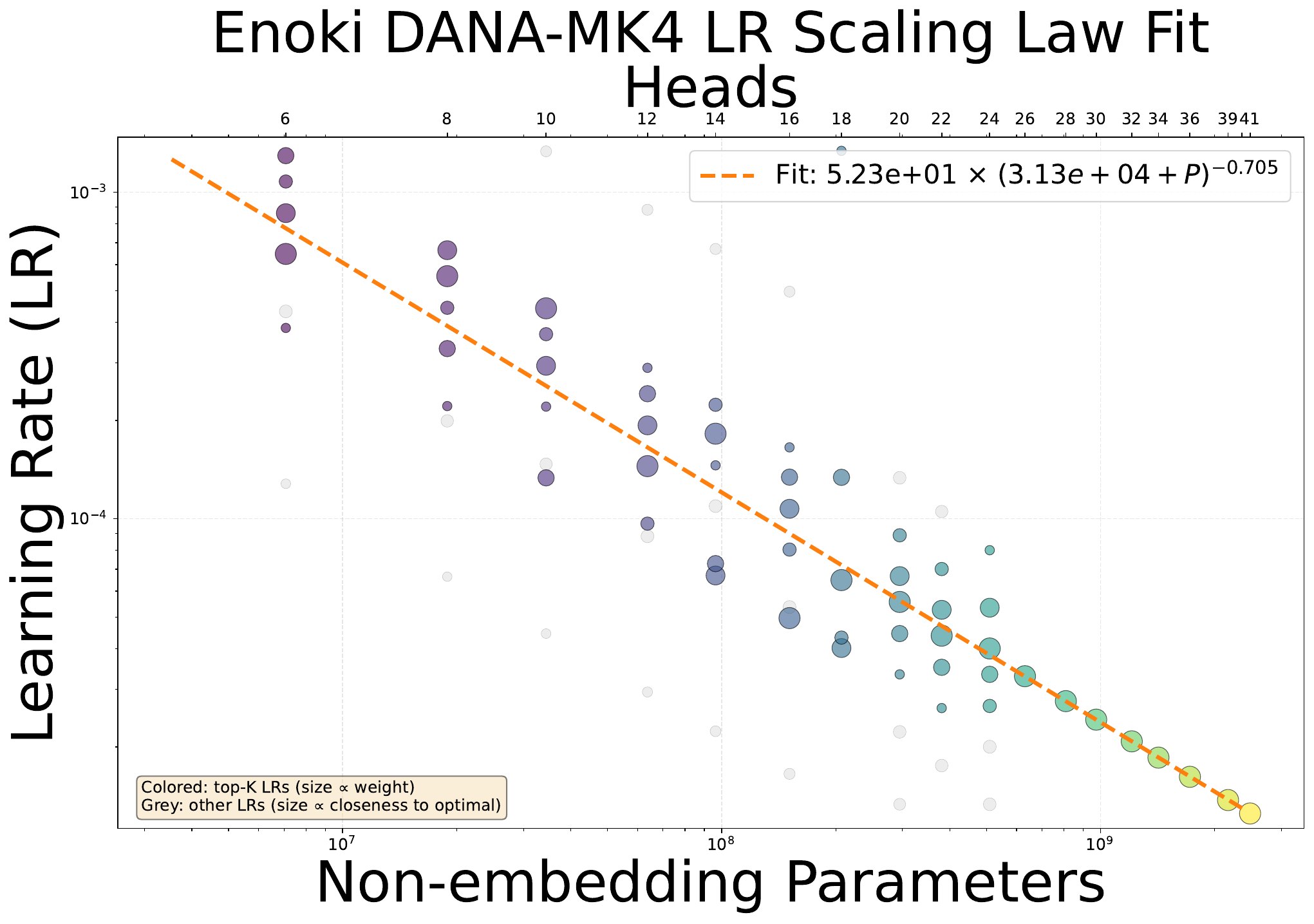}
    \caption{\DanaMKfour}
\end{subfigure}
\\[1em]
\begin{subfigure}[b]{0.48\textwidth}
    \includegraphics[width=\textwidth]{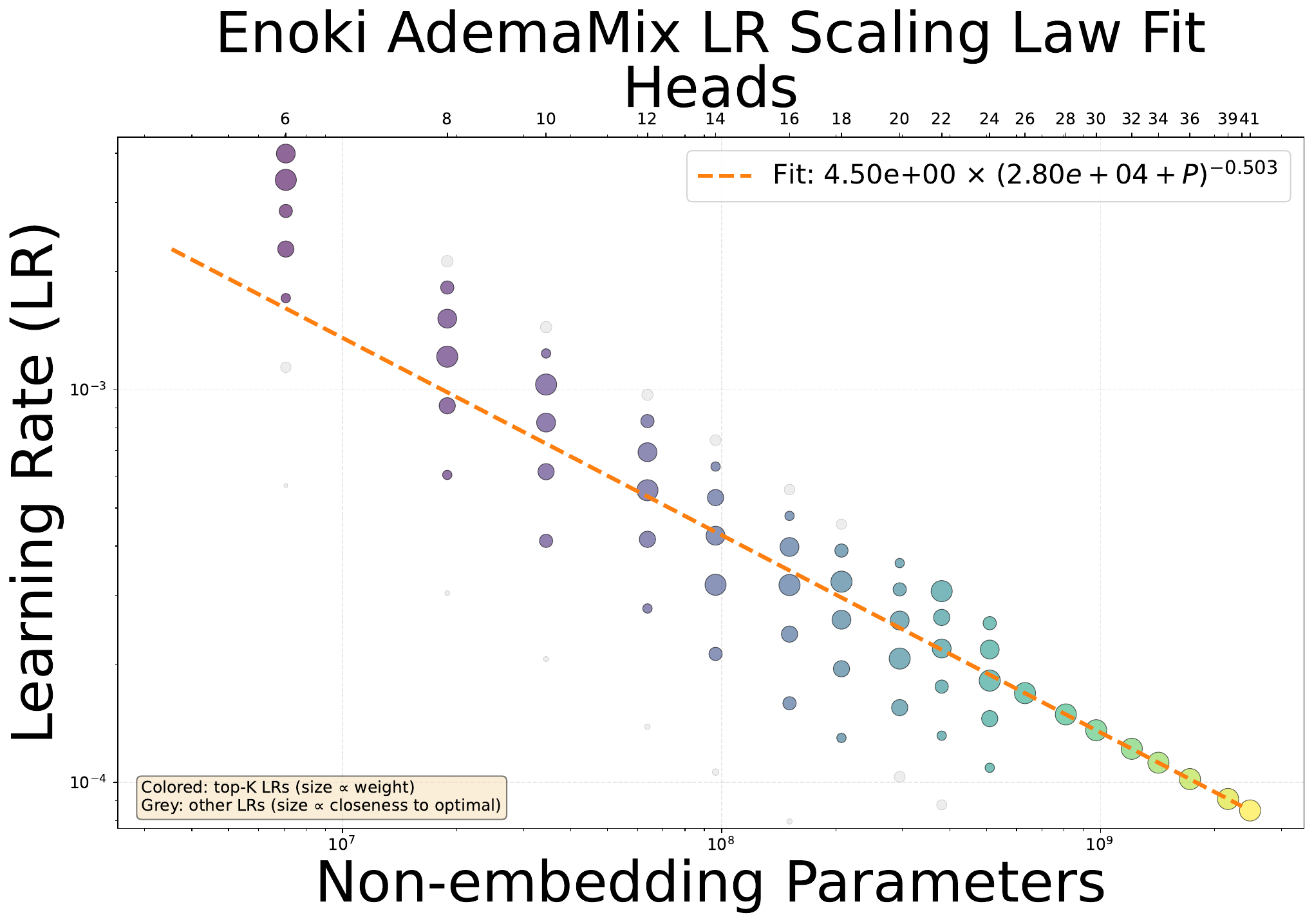}
    \caption{\Ademamix}
\end{subfigure}
\hfill
\begin{subfigure}[b]{0.48\textwidth}
    \includegraphics[width=\textwidth]{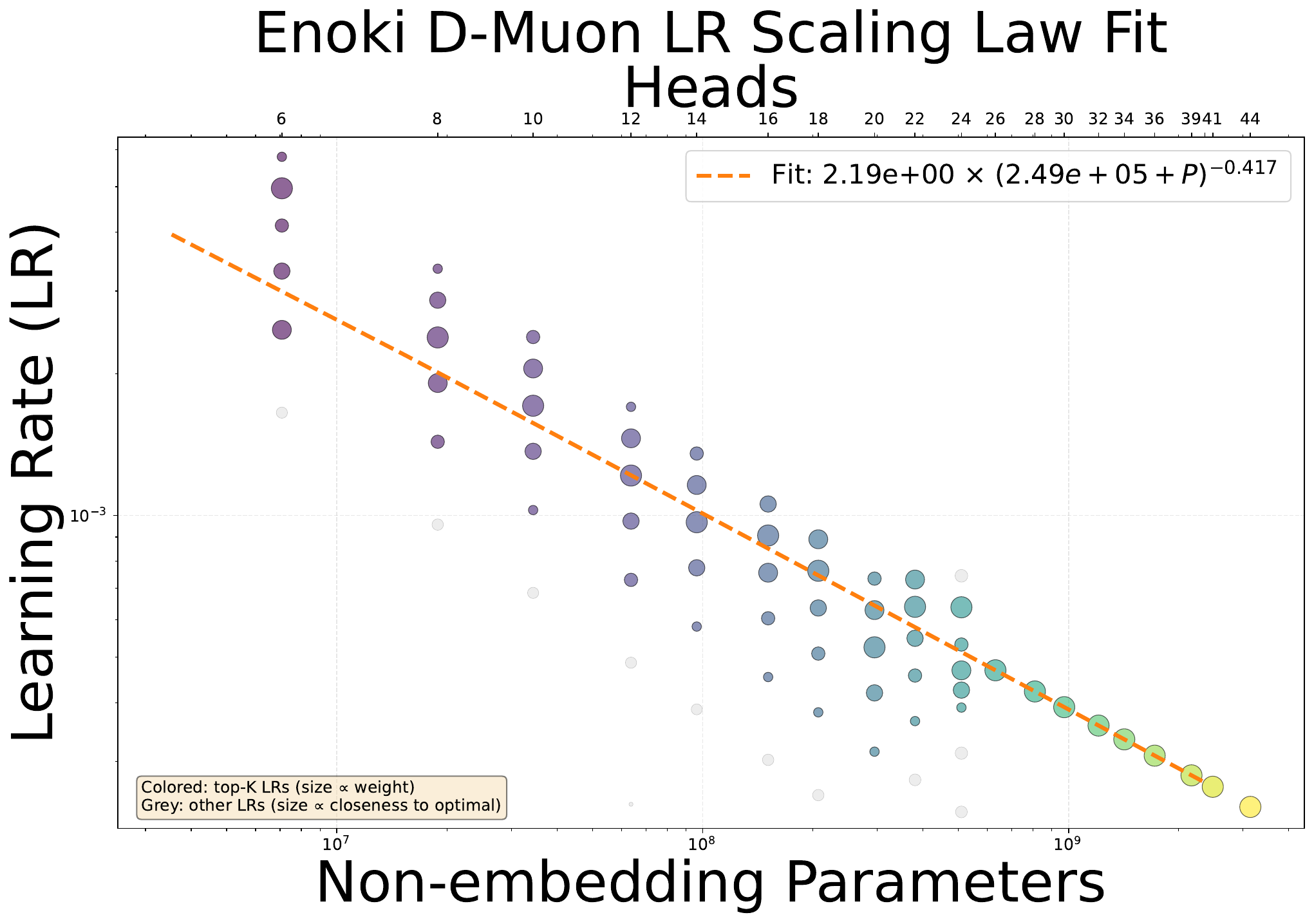}
    \caption{\Muon}
\end{subfigure}
\caption{\textbf{Peak learning rate $\gamma^*$ scaling for Enoki models.} Each panel shows the optimal learning rates (circles, sized by weight) and the fitted power-law curve (dashed line) for a single optimizer. The fits shown here include data from all scales (6--41 heads) and may differ slightly from the equations above, which were fit using only data up to 24 heads and used to extrapolate learning rates for larger runs.}
\label{fig:enoki_lr_scaling}
\end{figure}

\subsubsection{ADANA, DANA-MK4 and DANA-STAR-MK4 Learning Rate Fits by $\kappa$}

For \ADana and \DanastarMKfour, the spectral dimension parameter $\kappa$ affects the optimal learning rate scaling. We fit separate scaling laws for different $\kappa$ values to understand this dependence. The results are shown in Figures~\ref{fig:enoki_lr_scaling_kappa} and \ref{fig:enoki_lr_scaling_kappa2}.

For \DanastarMKfour:
\begin{align}
\kappa = 0.75: \quad & \gamma^* = 4.10 \times 10^1 \cdot (3.03 \times 10^4 + P)^{-0.661}, \label{eq:danastar_kappa075_lr}\\
\kappa = 0.85: \quad & \gamma^* = 2.23 \times 10^2 \cdot (2.40 \times 10^6 + P)^{-0.724}. \label{eq:danastar_kappa085_lr}
\end{align}

For \ADana :
\begin{align}
\kappa = 0.75: \quad & \gamma^* = 1.27 \times 10^1 \cdot (9.55 \times 10^3 + P)^{-0.624}, \label{eq:adana_kappa075_lr}\\
\kappa = 0.80: \quad & \gamma^* = 1.76 \times 10^1 \cdot (3.66 \times 10^4 + P)^{-0.624}, \label{eq:adana_kappa080_lr}\\
\kappa = 0.85: \quad & \gamma^* = 2.25 \times 10^1 \cdot (1.22 \times 10^5 + P)^{-0.618}, \label{eq:adana_kappa085_lr}\\
\kappa = 0.9: \quad & \gamma^* = 6.89 \times 10^1 \cdot (2.95 \times 10^4 + P)^{-0.667}. \label{eq:adana_kappa09_lr}
\end{align}

For \DanaMKfour :
\begin{align}
\kappa = 0.85: \quad & \gamma^* = 2.30 \times 10^1 \cdot (1.46 \times 10^4 +P)^{-0.627}. \label{eq:danamk4_kappa085_lr}\\
\end{align}

The higher $\kappa$ value leads to a larger saturation parameter $b$, indicating that the power-law regime begins at larger model sizes. The exponents are similar between the two $\kappa$ values for each optimizer family, suggesting that $\kappa$ primarily affects the learning rate magnitude rather than the scaling behavior.

\begin{figure}[t]
\centering
\begin{subfigure}[b]{0.48\textwidth}
    \includegraphics[width=\textwidth]{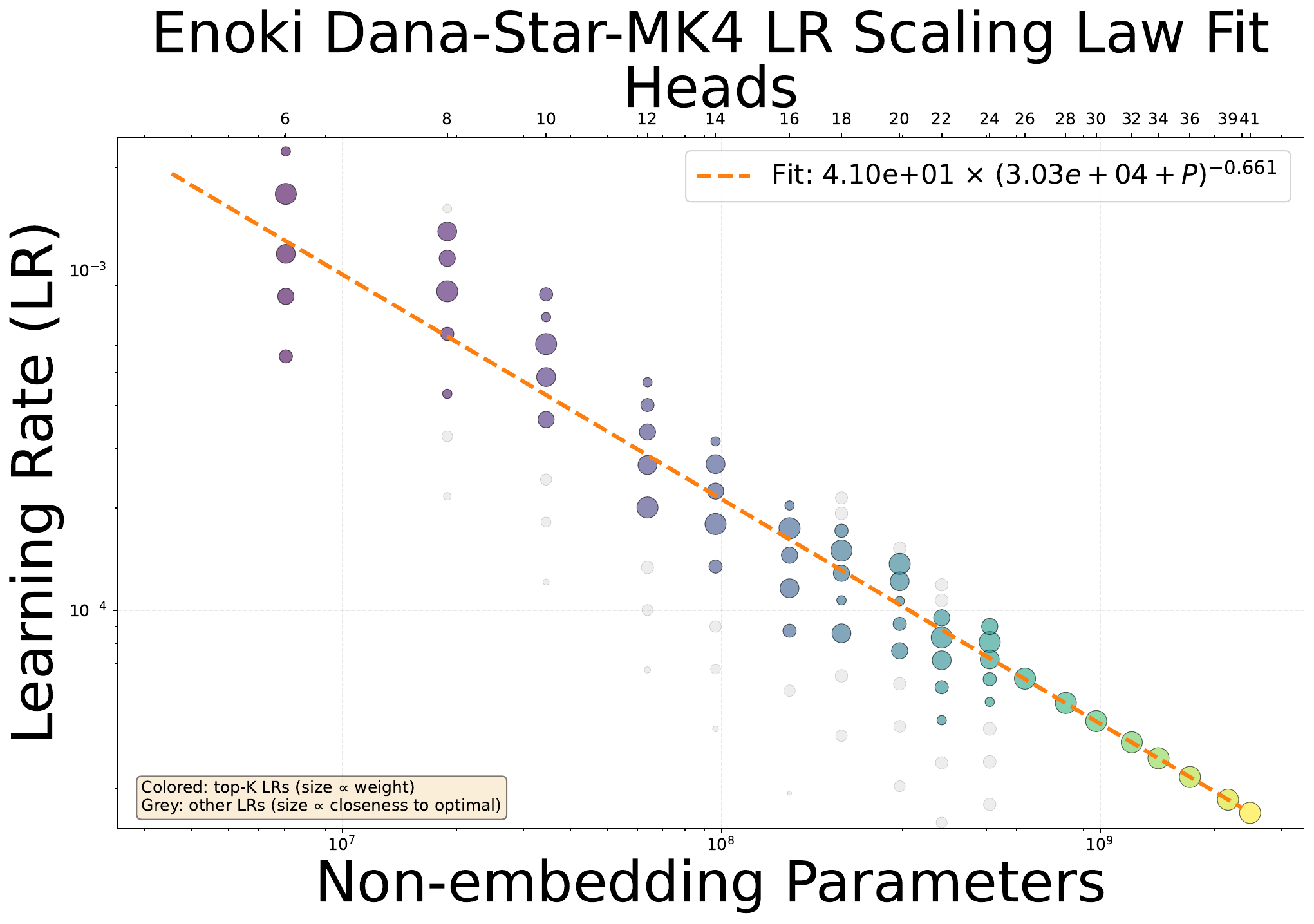}
    \caption{\DanastarMKfour ($\kappa = 0.75$)}
\end{subfigure}
\hfill
\begin{subfigure}[b]{0.48\textwidth}
    \includegraphics[width=\textwidth]{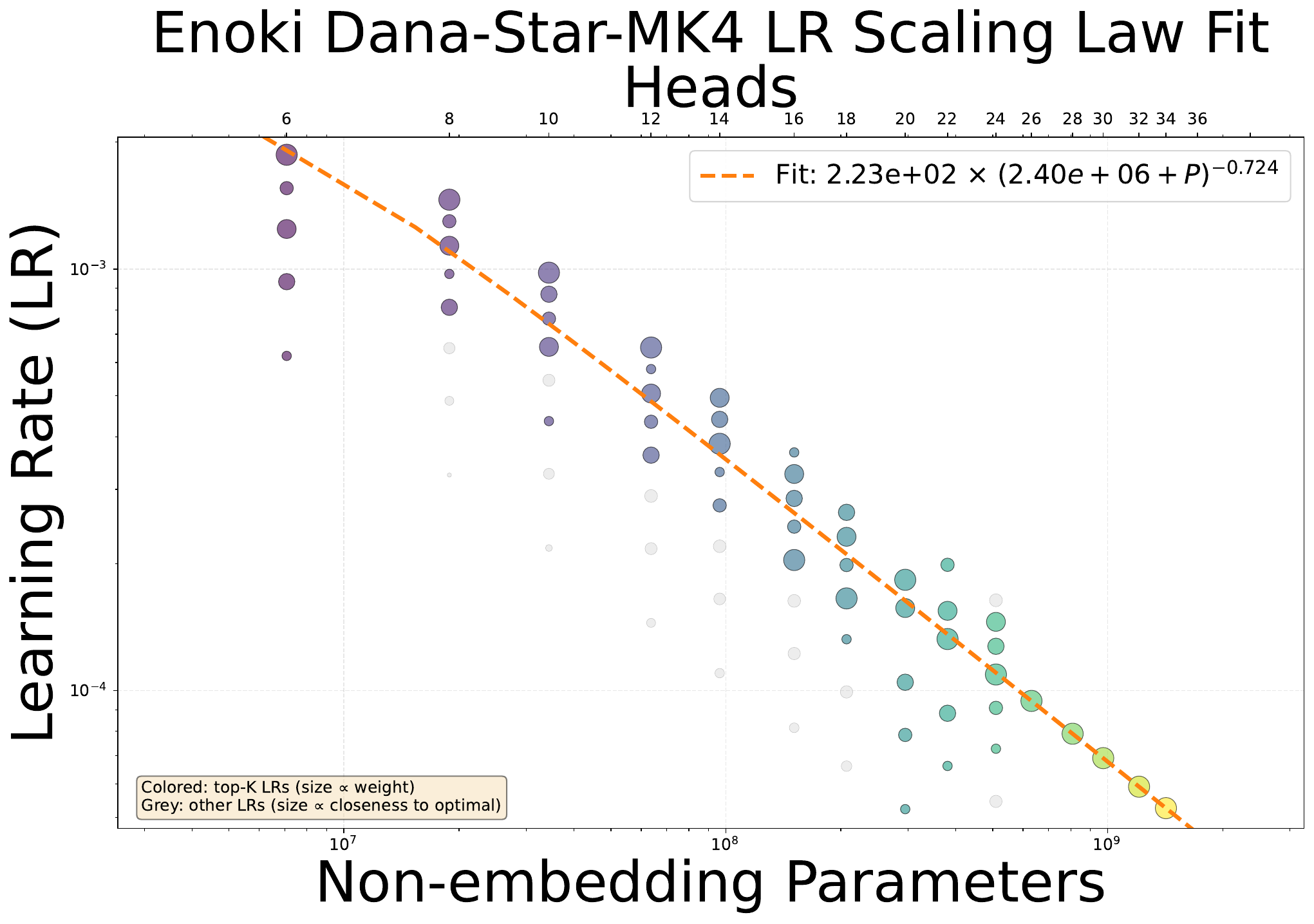}
    \caption{\DanastarMKfour ($\kappa = 0.85$)}
\end{subfigure}
\\[1em]
\begin{subfigure}[b]{0.48\textwidth}
    \includegraphics[width=\textwidth]{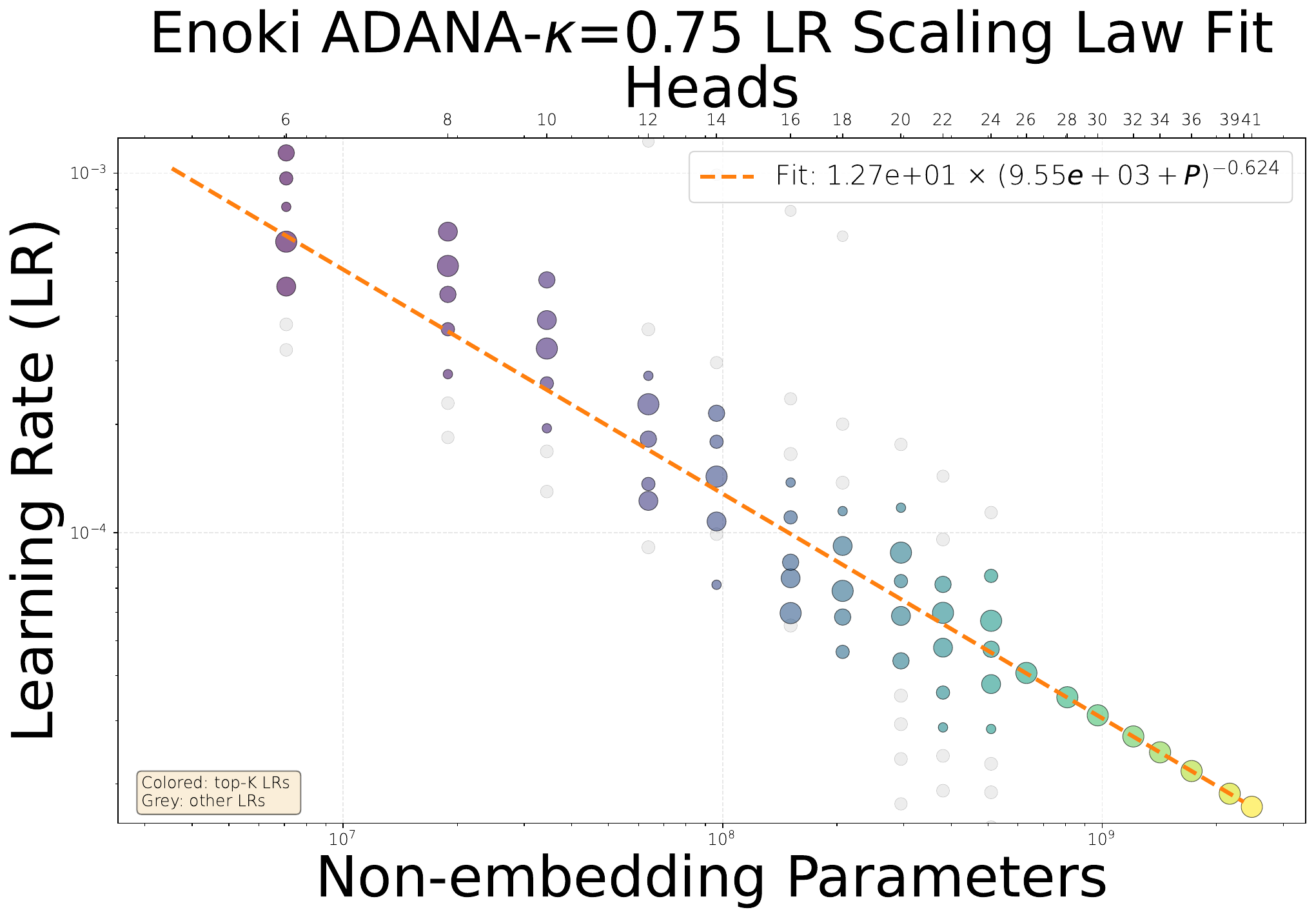}
    \caption{\ADana ($\kappa = 0.75$)}
\end{subfigure}
\hfill
\begin{subfigure}[b]{0.48\textwidth}
    \includegraphics[width=\textwidth]{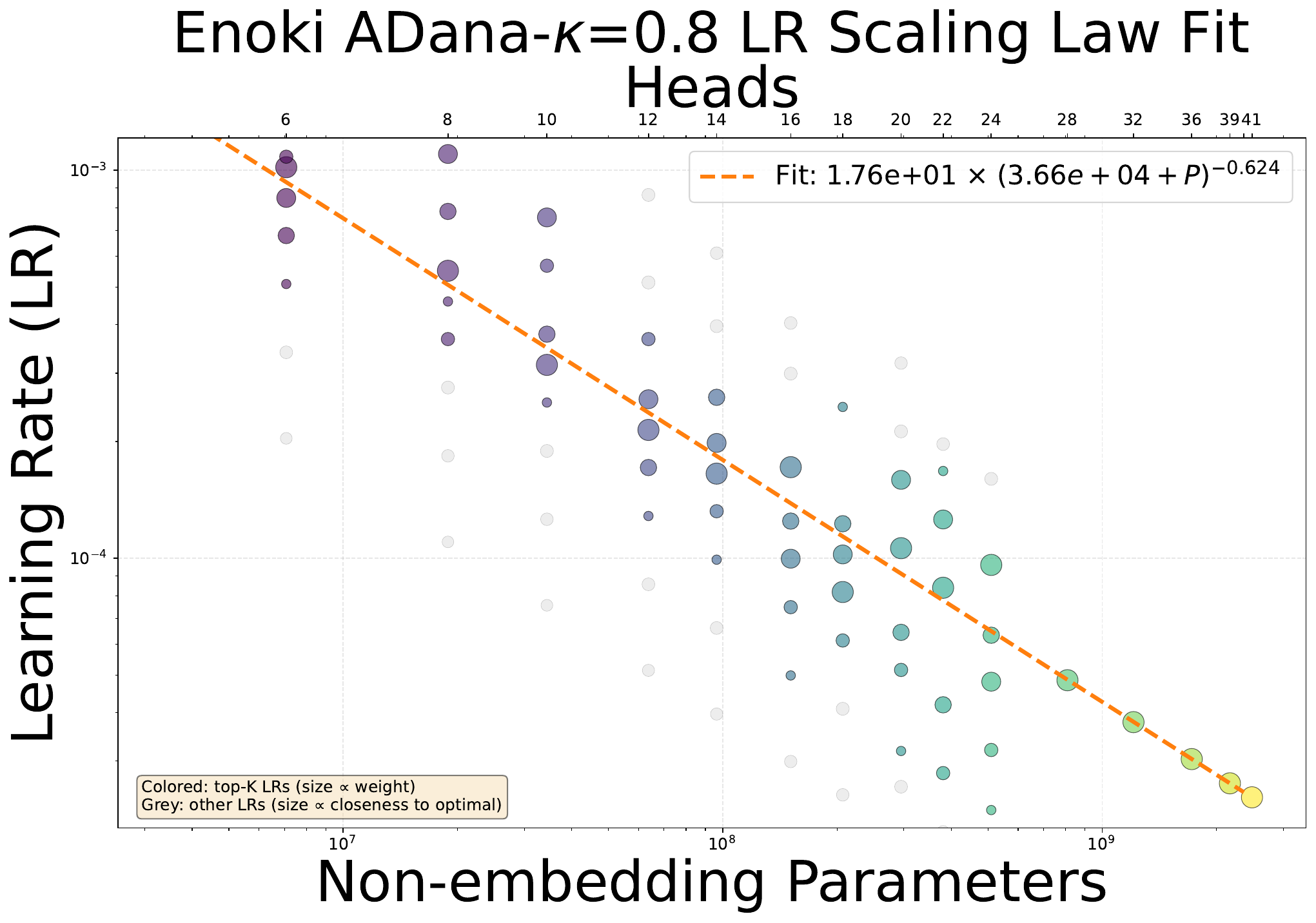}
    \caption{\ADana ($\kappa = 0.80$)}
\end{subfigure}
\\[1em]
\begin{subfigure}[b]{0.48\textwidth}
    \includegraphics[width=\textwidth]{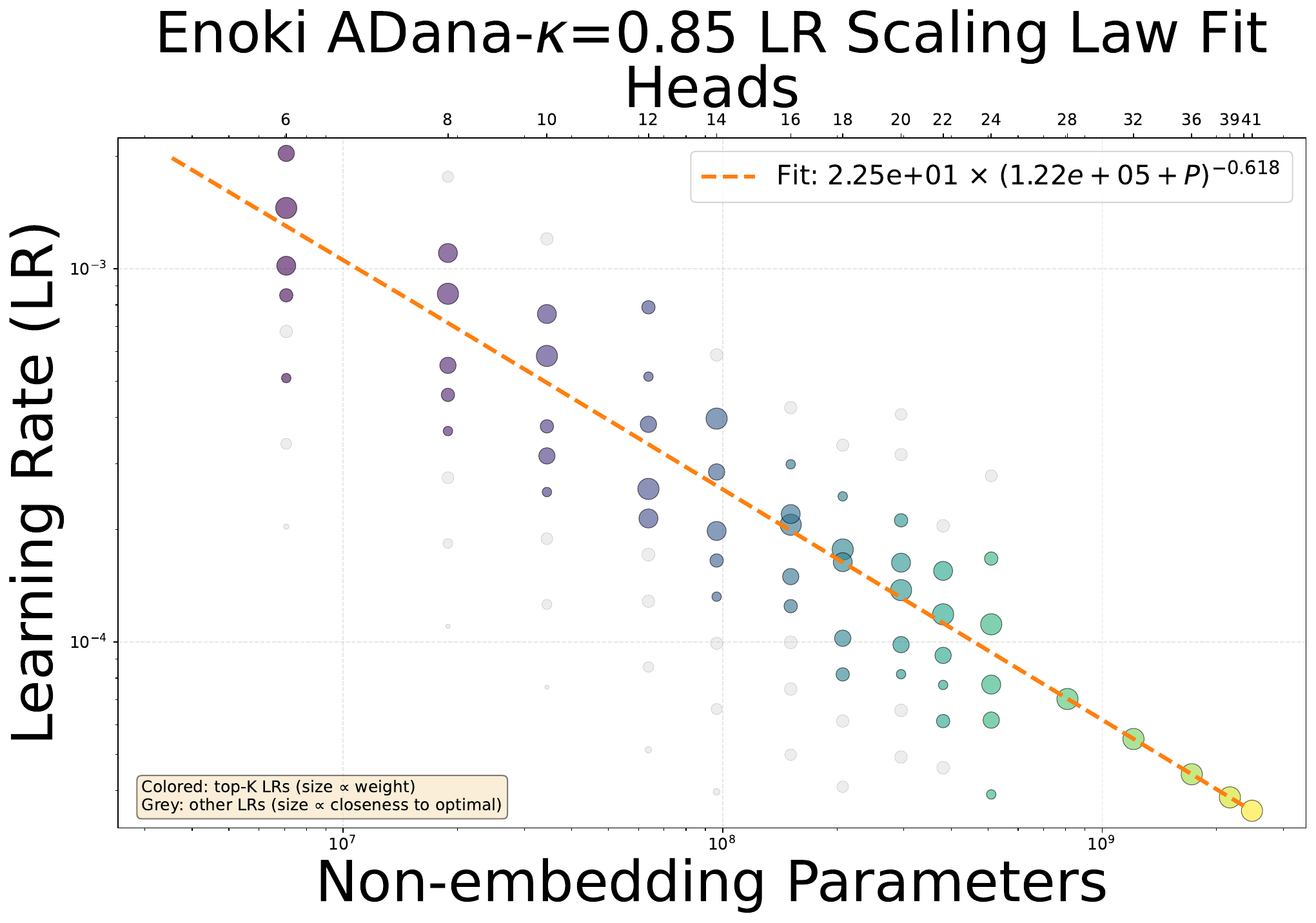}
    \caption{\ADana ($\kappa = 0.85$)}
\end{subfigure}
\hfill
\begin{subfigure}[b]{0.48\textwidth}
    \includegraphics[width=\textwidth]{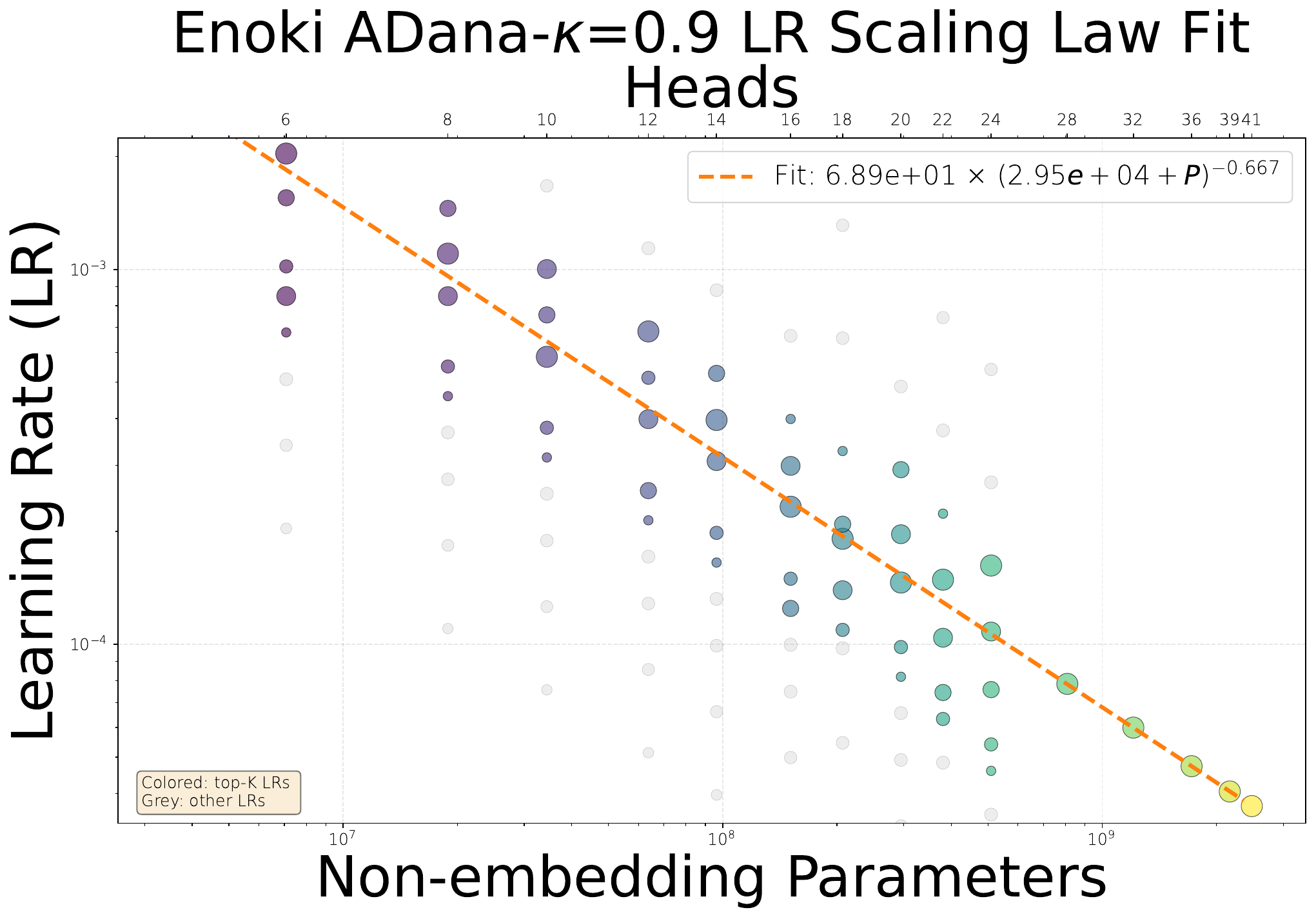}
    \caption{\ADana ($\kappa = 0.9$)}
\end{subfigure}
\caption{\textbf{Peak learning rate $\gamma^*$ scaling for \ADana and \DanastarMKfour by $\kappa$.} Top row: \DanastarMKfour with different $\kappa$ values. Middle and bottom rows: \ADana (without $\tau$ estimator) with different $\kappa$ values. Colored circles indicate top-$K$ learning rates used in the fit; grey circles show other explored learning rates with size proportional to closeness to optimal.}
\label{fig:enoki_lr_scaling_kappa}
\end{figure}

\begin{figure}[t]
\centering
\begin{subfigure}[b]{0.48\textwidth}
    \includegraphics[width=\textwidth]{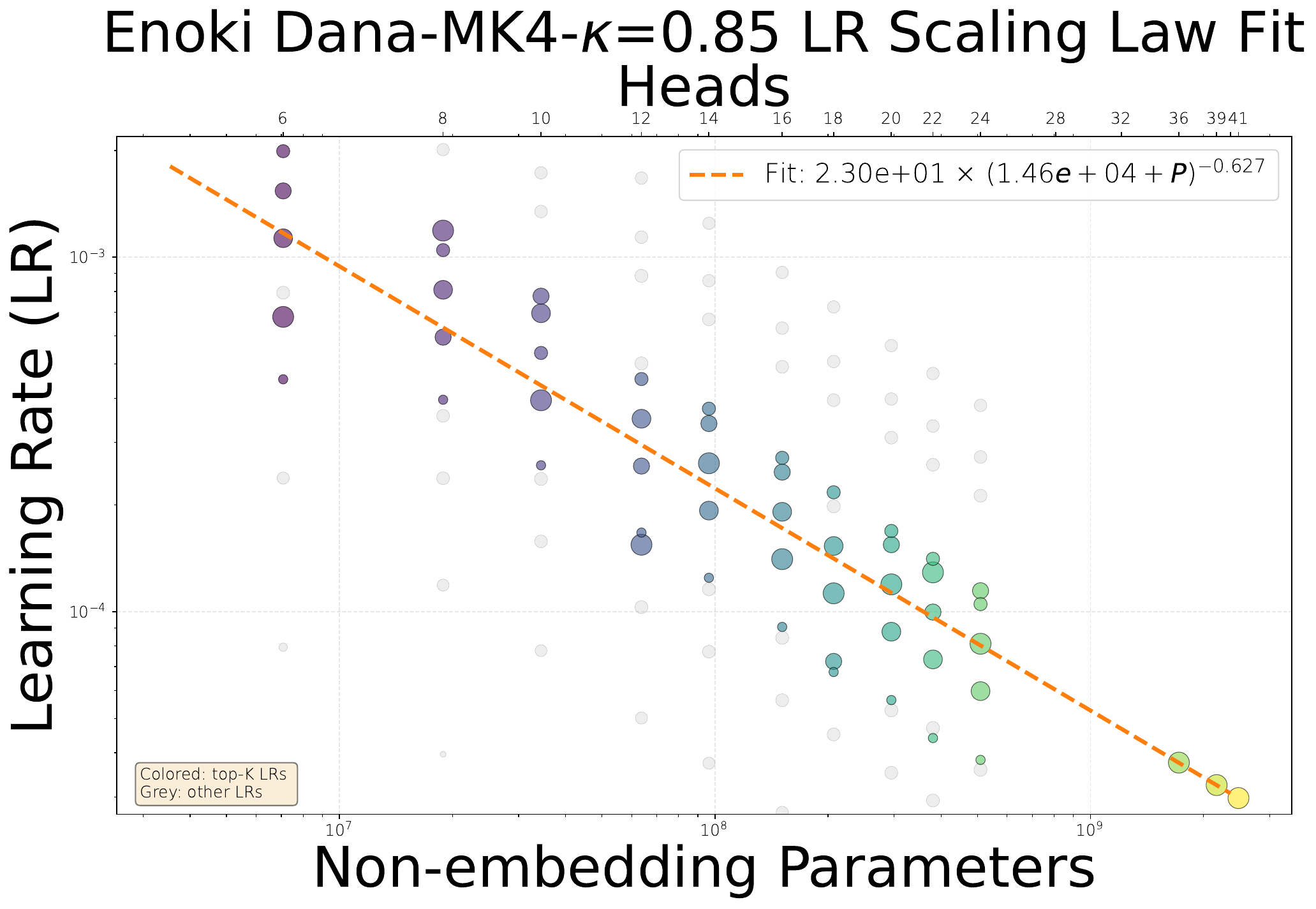}
\caption{\DanaMKfour ($\kappa = 0.85$)}
\end{subfigure}
\caption{Peak learning rate $\gamma^*$ scaling for \DanaMKfour, $\kappa=0.85$.}
\label{fig:enoki_lr_scaling_kappa2}
\end{figure}


\subsubsection{\AdamW and \Ademamix with decaying weight-decay}

With decaying logarithmic-time weight-decay, \AdamW (DW) and \Ademamix (DW) have the following LR fits:

\begin{align}
\text{\AdamW (DW):} \quad & \gamma^* = 6.57\times 10^{1} \cdot (9.83\times 10^3 + P)^{-0.608}, \label{eq:adamw_dw_enoki_lr}\\
\text{\Ademamix (DW):} \quad & \gamma^* = 7.28 \cdot (8.43\times 10^3 + P)^{-0.525}. \label{eq:ademamix_dw_enoki_lr}
\end{align}

The peak learning rate fits are shown in \Cref{fig:lr_scaling_dw}.

\begin{figure}[t]
\centering
\begin{subfigure}[b]{0.48\textwidth}
    \includegraphics[width=\textwidth]{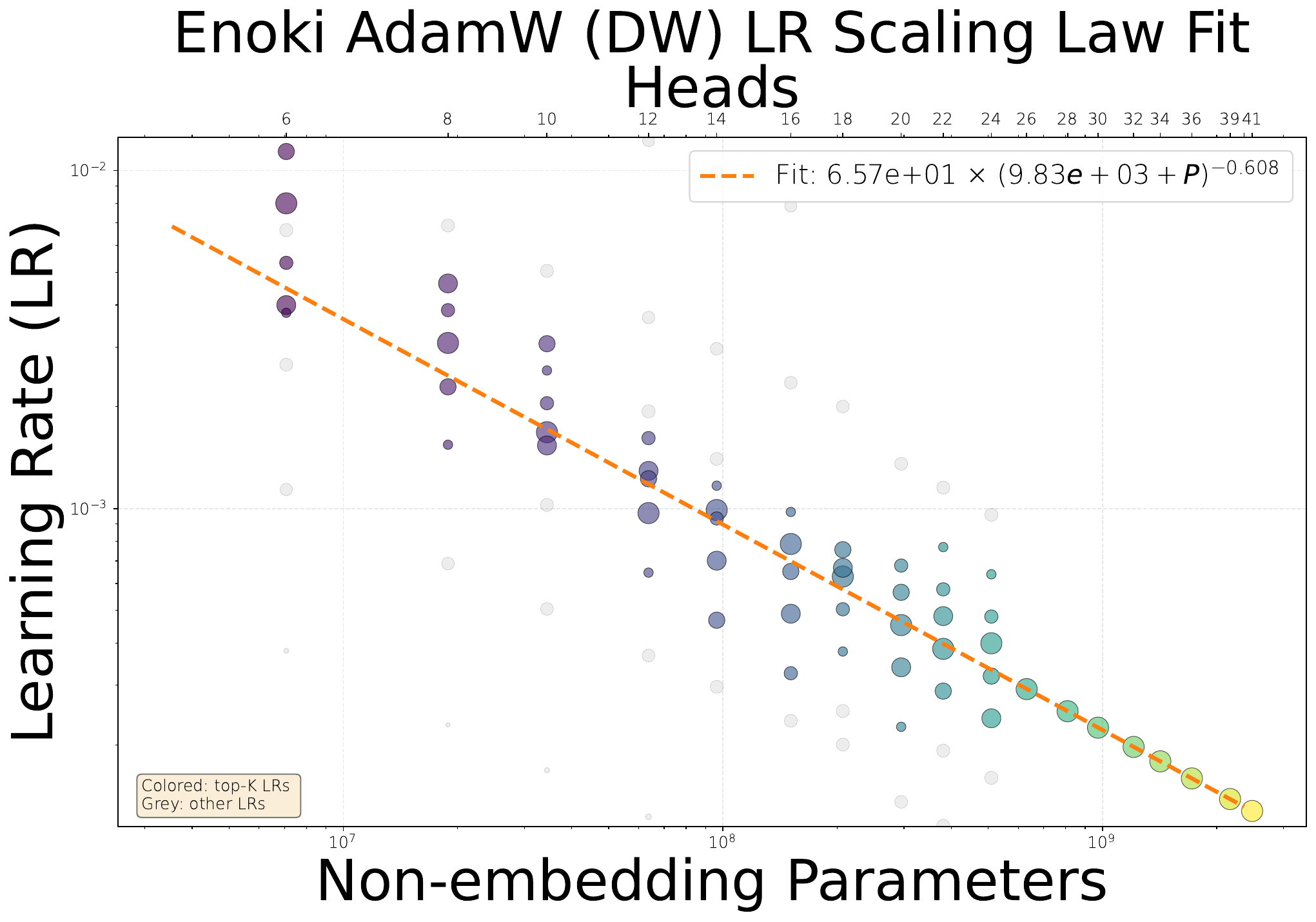}
    \caption{\AdamW (Decaying Weight-decay)}
\end{subfigure}
\hfill
\begin{subfigure}[b]{0.48\textwidth}
    \includegraphics[width=\textwidth]{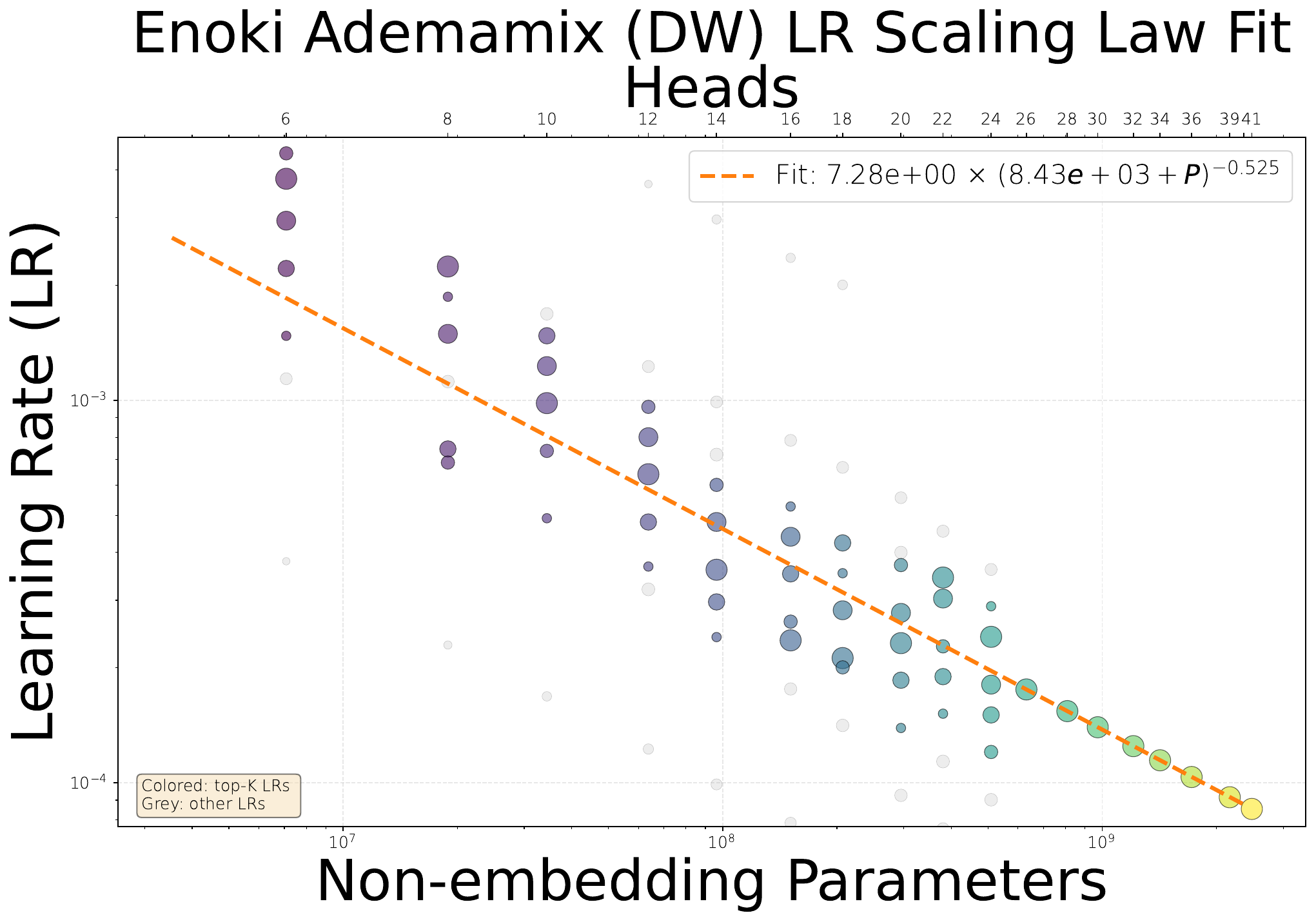}
    \caption{\Ademamix (Decaying Weight-decay)}
\end{subfigure}
\caption{Peak learning rate $\gamma^*$ scaling on Enoki models for \AdamW and \Ademamix using logarithmic-time weight-decay.}
\label{fig:lr_scaling_dw}
\end{figure}

\subsubsection{\ADana variants used in the ablation study in \Cref{sec:ablation_adana}}

Below we report the LR fits of scaled \ADana variants used in the ablation \Cref{sec:ablation_adana} and reported in \Cref{tab:abl_study}.

\begin{align}
\text{\ADana $3$:} \quad & \gamma^* = 2.38\times 10^1 \cdot (8.13\times 10^3 + P)^{-0.657}, \label{eq:adana_3_enoki_lr}\\
\text{\ADana $5$:} \quad & \gamma^* = 4.57\times 10^1 \cdot (1.47\times 10^4 + P)^{-0.627}, \label{eq:adana_5_enoki_lr}\\
\text{\ADana $7$:} \quad & \gamma^* = 2.58 \cdot (6.04\times 10^3 + P)^{-0.533}. \label{eq:adana_7_enoki_lr}
\end{align}

\begin{figure}[t]
\centering
\begin{subfigure}[b]{0.48\textwidth}
    \includegraphics[width=\textwidth]{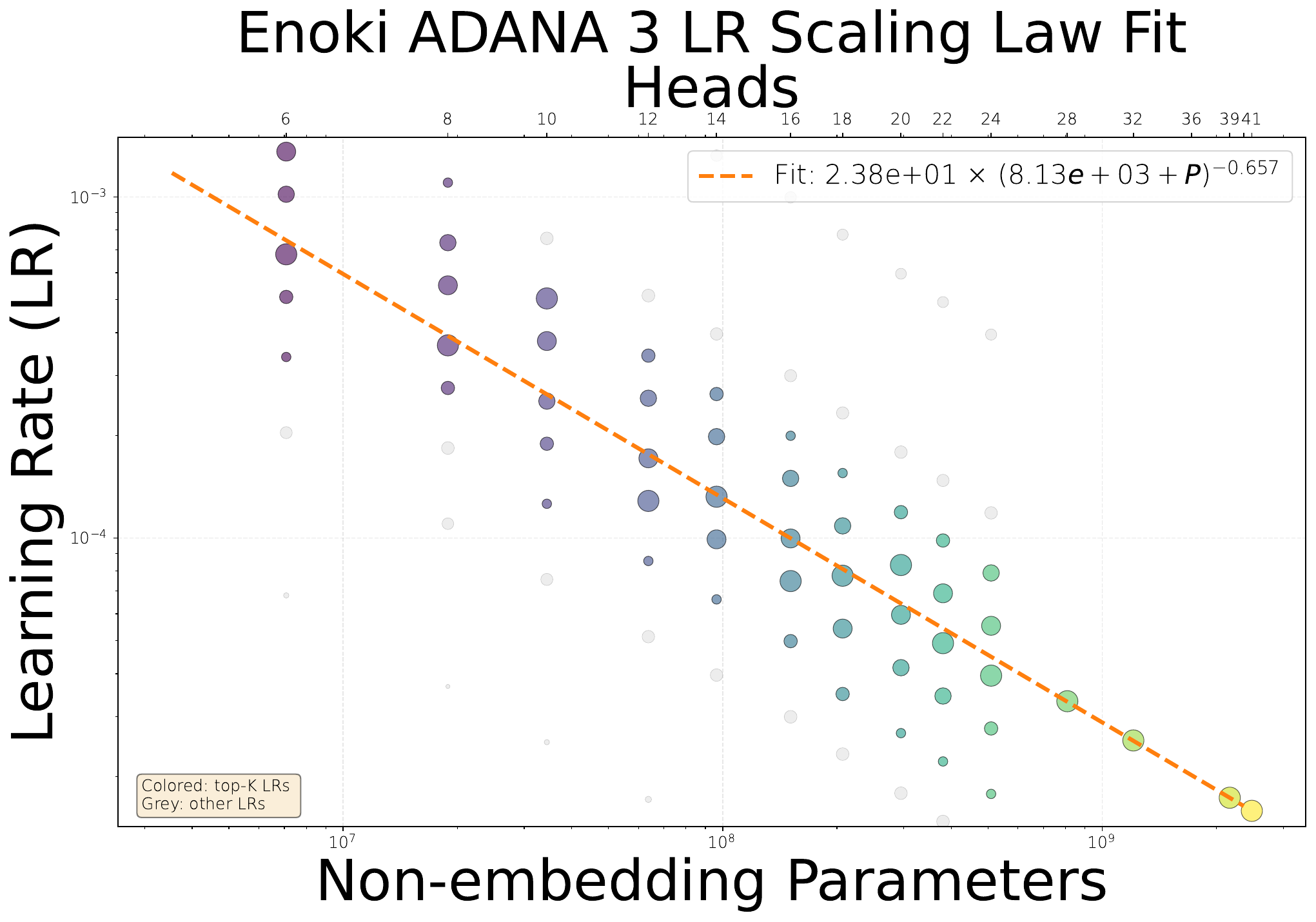}
    \caption{\ADana $3$}
\end{subfigure}
\hfill
\begin{subfigure}[b]{0.48\textwidth}
    \includegraphics[width=\textwidth]{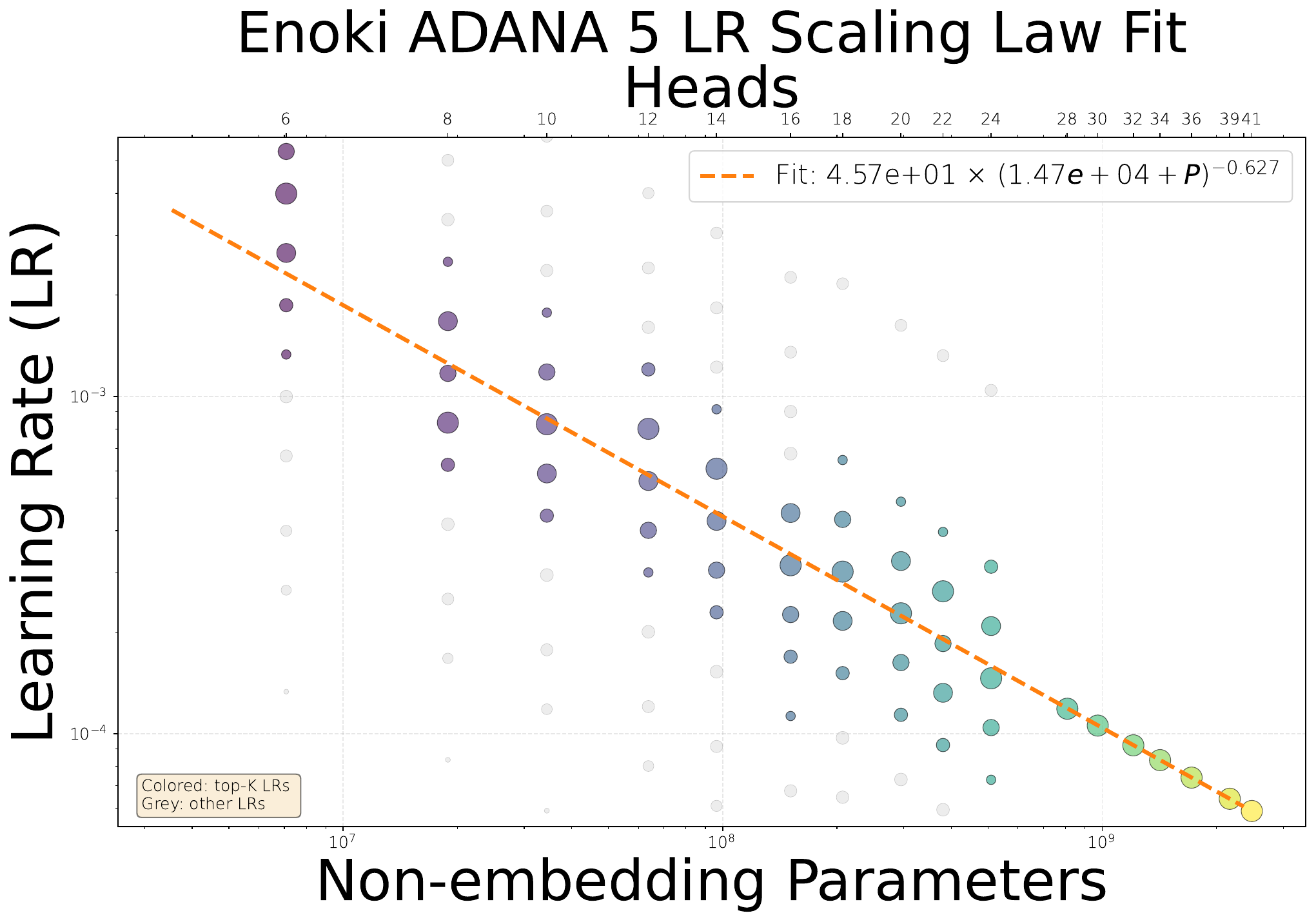}
    \caption{\ADana $5$}
\end{subfigure}
\\[1em]
\begin{subfigure}[b]{0.48\textwidth}
    \includegraphics[width=\textwidth]{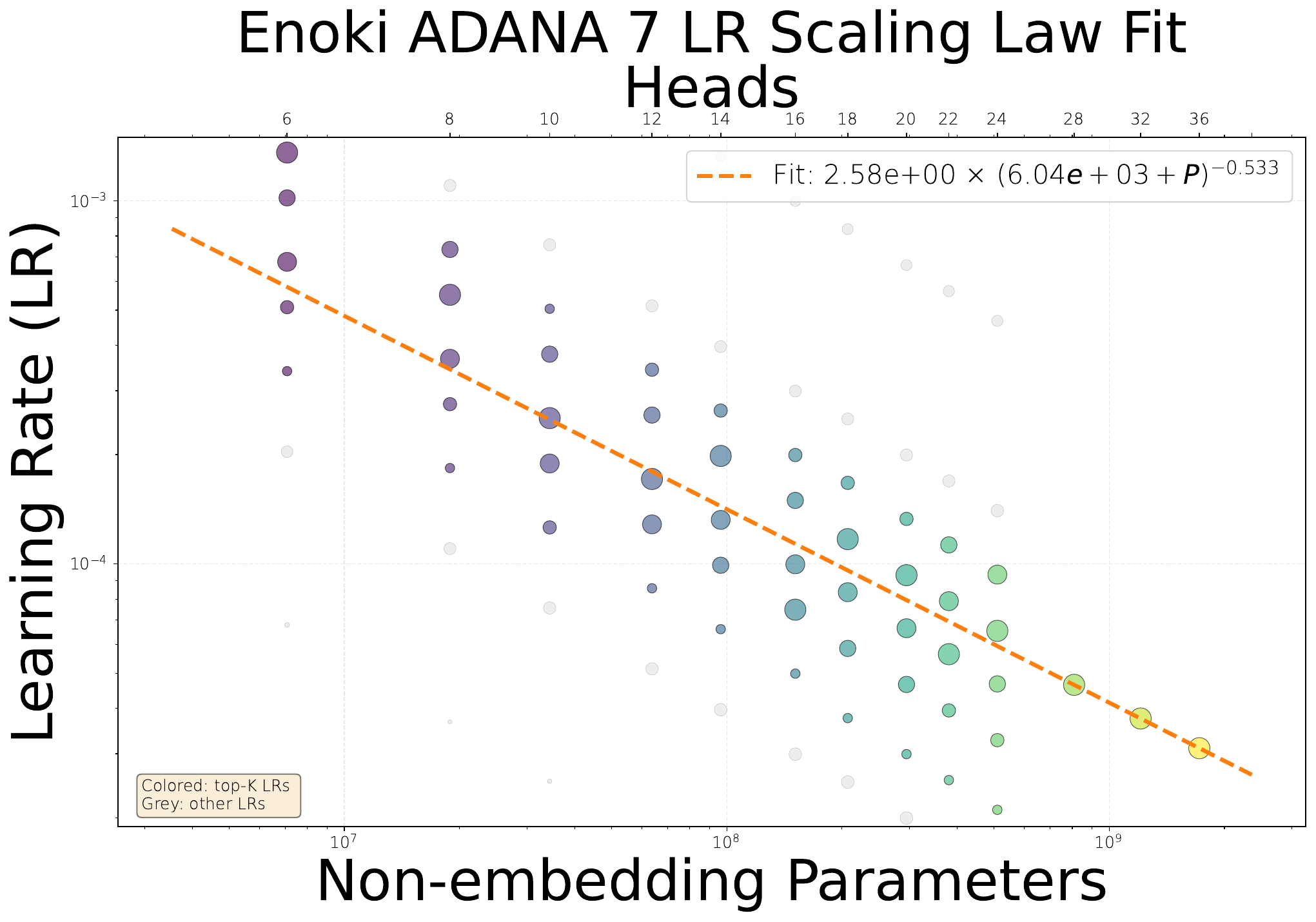}
    \caption{\ADana $7$}
\end{subfigure}
\caption{Peak learning rate $\gamma^*$ scaling for optimizers used in the ablation study \Cref{tab:abl_study} on Enoki models. We report only optimizers that were scaled beyond $24$ heads.}
\label{fig:lr_scaling_ablation}
\end{figure}

\subsubsection{Extrapolation and Sensitivity Analysis}
\label{sec:lr_sensitivity}

The learning rate scaling laws enable principled extrapolation to model sizes beyond those directly tuned. For the Enoki architecture at 2.4B parameters (40 heads), the predicted learning rates are:
\begin{center}
\begin{tabular}{lc}
\toprule
\textbf{Optimizer} & \textbf{Predicted LR $\gamma^*$} \\
\midrule
\AdamW & $1.91 \times 10^{-4}$ \\
\Ademamix & $8.70 \times 10^{-5}$ \\
\DanastarMKfour & $2.63 \times 10^{-5}$ \\
\Muon & $2.70 \times 10^{-4}$ \\
\DanaMKfour & $1.30 \times 10^{-5}$ \\
\bottomrule
\end{tabular}
\end{center}

While these point estimates allow us to run large-scale experiments with minimal hyperparameter search, practitioners need to understand the uncertainty in these predictions and their downstream impact on compute savings claims. We present a three-stage sensitivity analysis: (1) bootstrap confidence intervals on the fitted parameters, (2) loss curvature analysis, and (3) propagation of uncertainties to compute savings estimates.

\paragraph{Bootstrap Confidence Intervals.}
We quantify uncertainty in the fitted LR scaling parameters using model-size level bootstrap resampling. For each of $N=100$ bootstrap samples, we resample the set of $M$ model sizes with replacement, drawing exactly $M$ samples (e.g., if we have data at heads $\{6, 8, 10, 12, 14, 16, 18, 20, 22, 24\}$, we draw 10 sizes with replacement, so a bootstrap sample might include $\{6, 8, 8, 12, 14, 14, 14, 20, 22, 24\}$ with some sizes appearing multiple times and others omitted). For each sampled size, we retain all of its top-$K$ LRs with their associated weights. We then refit the power law $\gamma^*(P) = a \cdot (b + P)^d$ using weighted MSE, where weights are the product of the top-$K$ rank weight (best LR gets weight $K$, second-best gets $K-1$, etc.) squared and the parameter count, giving more emphasis to larger models.

This yields distributions over the fitted parameters $(a, b, d)$ and, critically, over the predicted learning rates at extrapolation scales. The 95\% confidence intervals for predicted learning rates are quite tight, even at 2.4B parameters (40 heads):
\begin{center}
\begin{tabular}{lccc}
\toprule
\textbf{Optimizer} & \textbf{Predicted LR} & \textbf{95\% CI} & \textbf{Relative Width} \\
\midrule
\AdamW & $1.91 \times 10^{-4}$ & $[1.89, 1.92] \times 10^{-4}$ & $\pm 0.7\%$ \\
\DanaMKfour & $1.30 \times 10^{-5}$ & $[1.28, 1.32] \times 10^{-5}$ & $\pm 1.5\%$ \\
\bottomrule
\end{tabular}
\end{center}

The narrow confidence intervals indicate that the power-law fits are highly constrained by the available data. The exponent $d$ is particularly well-determined: for \AdamW, $d = -0.514$ with 95\% CI $[-0.551, -0.495]$; for \DanaMKfour, $d = -0.704$ with 95\% CI $[-0.724, -0.684]$.

\paragraph{Loss Curvature.}
Learning rate errors translate to loss increases via the curvature of the loss landscape around the optimal LR. We characterize this by fitting a parabola in log-LR space at each model size:
\begin{equation}
L(\log \gamma^*) \approx L^*(P) + \zeta(P) \cdot (\log \gamma^* - \log \bar{\gamma}(P))^2,
\end{equation}
where $\zeta(P)$ is the local curvature, $\gamma^*$ is the observed peak learning rate, and $\bar{\gamma}(P)$ is the optimal predicted peak learning rate. We extract $\zeta$ at each model size by fitting to all LR sweep data (not just the top-$K$).

Figure~\ref{fig:curvature_scaling} shows the measured curvatures across model sizes. Due to systematic variations that do not follow a simple power law (visible in the scatter), we report the average curvature rather than fitting a scaling law:
\begin{center}
\begin{tabular}{lcc}
\toprule
\textbf{Optimizer} & $\bar{\zeta}$ & \textbf{Std Dev} \\
\midrule
\AdamW & $1.93 \times 10^{-2}$ & $3.5 \times 10^{-3}$ \\
\DanaMKfour & $9.3 \times 10^{-3}$ & $2.7 \times 10^{-3}$ \\
\bottomrule
\end{tabular}
\end{center}

These curvatures are modest, meaning the loss landscape around the optimal LR is relatively flat. For a relative LR error $\varepsilon$, the expected loss increase is $\Delta L = \bar{\zeta} \cdot (\log(1 + \varepsilon))^2$. Even a 50\% learning rate error results in an expected loss increase of only:
\begin{equation}
\Delta L \approx \bar{\zeta} \cdot (\log 1.5)^2 \approx 0.019 \times 0.16 \approx 0.003 \text{ nats},
\end{equation}
which is negligible compared to typical loss differences between optimizers ($\sim 0.05$--$0.1$ nats).

\begin{figure}[t]
\centering
\begin{subfigure}[b]{0.48\textwidth}
    \includegraphics[width=\textwidth]{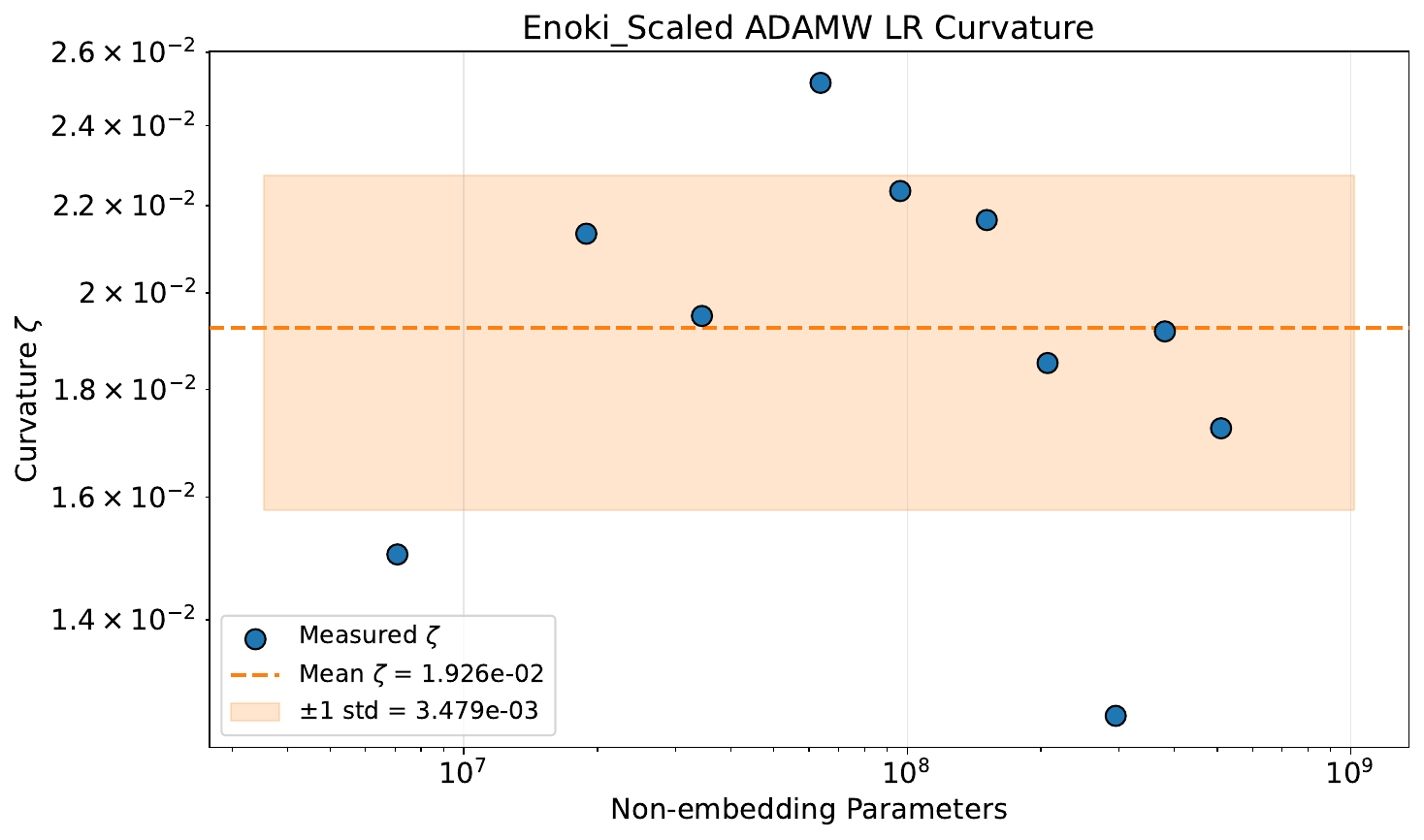}
    \caption{\AdamW}
\end{subfigure}
\hfill
\begin{subfigure}[b]{0.48\textwidth}
    \includegraphics[width=\textwidth]{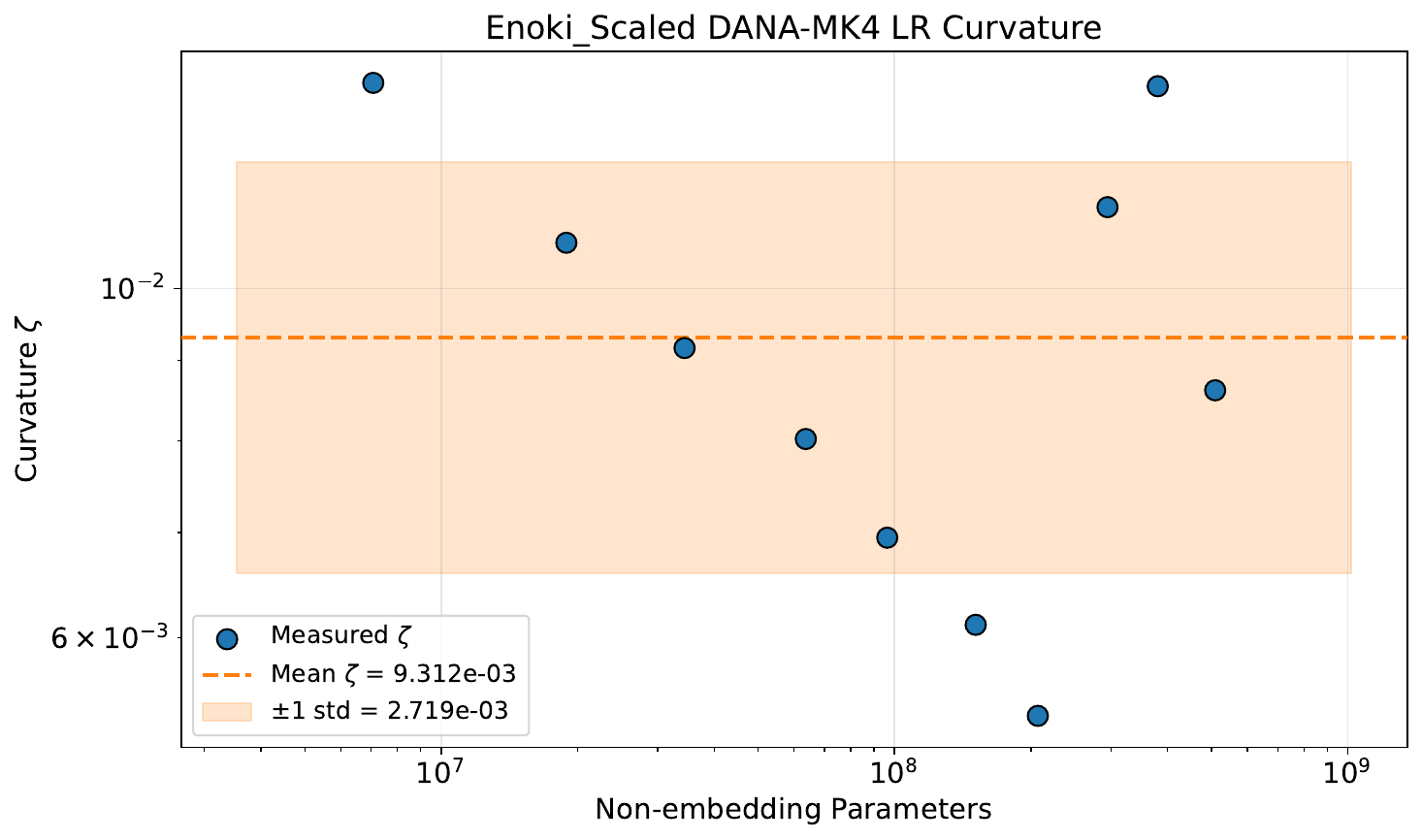}
    \caption{\DanaMKfour}
\end{subfigure}
\caption{\textbf{LR curvature across model sizes.} Local curvature $\zeta$ (measuring sensitivity to LR errors) as a function of model size. Points show measured curvatures from parabolic fits to LR sweep data; the dashed horizontal line shows the mean curvature with shaded $\pm 1$ standard deviation band. The scatter does not follow a simple power law, so we use the average curvature for sensitivity calculations.}
\label{fig:curvature_scaling}
\end{figure}

\paragraph{Propagation to Compute Savings Confidence Intervals.}
We translate loss uncertainty into compute savings uncertainty using the slope of the loss scaling law at the end of training. From $L = a + b \cdot C^{-c}$, the marginal compute cost of a loss improvement is:
\begin{equation}
\frac{dC}{dL} = -\frac{C}{c(L - a)}.
\end{equation}
Thus, a loss increase $\Delta L$ (from LR error) corresponds to a relative compute increase of:
\begin{equation}
\frac{\Delta C}{C} = \frac{\Delta L}{|c|(L - a)}.
\end{equation}

For the Enoki architecture at the 2.4B scale, with $L \approx 2.51$, $a \approx 0.20$, and $c \approx -0.043$ (from the loss scaling law fits in Section~\ref{sec:scaling_laws}), we have $|c|(L-a) \approx 0.10$. Combined with the bootstrap LR confidence intervals ($\pm 1$--$2\%$) and the curvature $\bar{\zeta} \approx 0.015$, the expected loss increase from LR uncertainty is:
\begin{equation}
\Delta L \approx \bar{\zeta} \cdot (\log 1.02)^2 \approx 0.015 \times 0.0004 \approx 6 \times 10^{-6} \text{ nats}.
\end{equation}
This translates to a relative compute uncertainty of:
\begin{equation}
\frac{\Delta C}{C} \approx \frac{6 \times 10^{-6}}{0.10} \approx 0.006\%.
\end{equation}

This negligible uncertainty confirms that the compute savings claims in the main text are robust to LR extrapolation error. Even under pessimistic assumptions (50\% LR error), the compute penalty would be approximately $0.003 / 0.10 \approx 3\%$---still far less than the compute savings achieved by the better optimizers.


\clearpage
\section{Qwen3 Experiments}
\label{sec:appendix_qwen}

To validate that our optimizer findings generalize beyond the Enoki architecture, we conduct experiments on Qwen3, a modern SwiGLU-based transformer with different design choices.

\subsection{Qwen3 Model Architecture}
\label{sec:qwen3_architecture}

The Qwen3 model features SwiGLU activation in the MLP \citep{shazeer2020glu}, RMSNorm instead of LayerNorm \citep{zhang2019root}, elementwise attention output gating \citep{qiu2025gated}, and RoPE positional embeddings. The head dimension is fixed at 128 following the Qwen3 architecture.

The Qwen3 scaling rule $(\text{\# of heads} = \text{heads})$:
\begin{align}
\text{n\_layer} = 2 \times \text{heads}, \quad
\text{n\_embd} = \text{heads} \times 128.
\end{align}

Our implementation differs from the official Qwen3 architecture in several respects: we do not employ Grouped Query Attention (GQA), as our goal is to test optimizer properties rather than parameter efficiency; we do not tie embedding weights (larger Qwen3 models also do not tie them); and we use 2048-token context length rather than the extended contexts typical of Qwen models (at least 32K tokens).

\paragraph{SwiGLU activation.}
The MLP uses the SwiGLU (Swish-Gated Linear Unit) activation function \citep{shazeer2020glu}:
\begin{equation}
\text{MLP}(x) = W_{\text{down}}\left(\text{SiLU}(xW_{\text{gate}}) \odot xW_{\text{up}}\right)
\end{equation}
with hidden dimension $3 \times n_{\text{embd}}$ to maintain comparable parameter counts. SwiGLU combines the Swish activation with a gating mechanism, using three linear projections instead of two. This architecture has become standard in modern open-weight LLMs: LLaMA 2+ \citep{touvron2023llama,touvron2023llama2}, Mistral \citep{jiang2023mistral}, OLMo \citep{groeneveld2024olmo}, and Qwen \citep{qwen3} all employ SwiGLU in their feed-forward layers. The gating mechanism provides greater expressiveness and improved capacity to model complex relationships compared to standard ReLU or GELU activations.

\paragraph{Elementwise attention output gating.}
Following \citet{qiu2025gated}, a sigmoid gate is applied to the attention output:
\begin{equation}
\text{Attn}_{\text{gated}}(X) = \text{Attn}(X) \odot \sigma(XW_g)
\end{equation}
where $W_g \in \mathbb{R}^{n_{\text{embd}} \times (n_{\text{head}} \cdot d_h)}$ are learned parameters. This mechanism provides three key benefits: (1) it introduces non-linearity between the value projection $W_V$ and output projection $W_O$, which otherwise form a low-rank linear mapping that limits expressiveness; (2) the sigmoid activation produces sparse gating scores (mean $\approx 0.12$), creating input-dependent (query-dependent) sparsity that filters irrelevant context; and (3) this sparsity eliminates the ``attention sink'' phenomenon where initial tokens disproportionately accumulate attention scores. Empirically, gated attention also improves training stability, enabling larger learning rates and better length generalization.

\paragraph{Parameter count formulas.}
For Qwen3 with elementwise gating:
\begin{align}
\text{attn} &= 5 \times n_{\text{embd}} \times (n_{\text{head}} \cdot d_h) + 2 \times d_h \\
\text{mlp} &= 9 \times n_{\text{embd}}^2 \\
\text{per\_layer} &= \text{attn} + \text{mlp} + 2 \times n_{\text{embd}} \\
\text{total} &= n_{\text{layer}} \times \text{per\_layer} + n_{\text{embd}} + 2 \times n_{\text{embd}} \times V
\end{align}
The factor of 5 in the attention term accounts for the doubled query projection (which includes the gate) plus key, value, and output projections. The MLP factor of 9 reflects the $3\times$ hidden dimension with three projection matrices (gate, up, down). The additional $2 \times d_h$ accounts for the QK-norm parameters, and the extra $n_{\text{embd}}$ term accounts for the final RMSNorm layer.

\subsection{Comparison with Enoki Scaling}
\label{sec:scaling_comparison}

Different scaling rules produce architectures with substantially different depth-to-width ratios, which can affect both optimization dynamics and final model quality. Table~\ref{table:scaling_rules_detailed} summarizes the scaling rules used in our experiments.

\begin{table}[h]
\centering
\caption{\textbf{Comparison of scaling rules.} Each rule fixes certain aspect ratios while scaling others. The Qwen3 rule produces deeper models than Enoki for equivalent parameter counts.}
\label{table:scaling_rules_detailed}
\begin{tabular}{l|cccc}
\toprule
\textbf{Rule} & \textbf{$n_{\text{layer}}$} & \textbf{$n_{\text{embd}}$} & \textbf{$d_h$} & \textbf{Depth/Width Ratio} \\
\midrule
Enoki & $3h/4$ & $64h$ & 64 (fixed) & 3/256 \\
Qwen3 & $2h$ & $128h$ & 128 (fixed) & 1/64 \\
\bottomrule
\end{tabular}
\end{table}

A key observation is that Qwen3 models are generally deeper than Enoki models at comparable parameter counts. For a given number of heads $h$, Qwen3 has $2h$ layers while Enoki has only $3h/4$ layers. Since deeper models can express more complex compositional functions but may also exhibit greater sensitivity to learning rate and more challenging optimization landscapes~\citep{wortsman2023small}, this architectural difference has implications for optimizer selection.

\subsection{Qwen3 Scaling Results}
\label{sec:qwen3_scaling_results}


\begin{table}[t]
\centering
\caption{\textbf{Qwen3 results.} Final validation loss for different optimizers on Qwen3 models. Best loss per row in bold.}
\label{table:qwen3_results_appendix}
\small
\begin{tabular}{l|ccc|cc|cc}
\toprule
& & \textbf{\ADana} & & \multicolumn{2}{c|}{\textbf{\DanastarMKfour}} & \multicolumn{2}{c}{\textbf{\DanaMKfour}} \\
\textbf{Size} & \textbf{\AdamW} & $\kappa$\textbf{0.85} & \textbf{\Ademamix} & $\kappa$\textbf{0.75} & $\kappa$\textbf{0.85} & $\kappa$\textbf{0.75} & $\kappa$\textbf{0.85} \\
\midrule
176M  & 3.228 & \textbf{3.182} & 3.214 & 3.199 & 3.202 & 3.185 & 3.183 \\
338M  & 3.017 & \textbf{2.965} & 2.999 & 2.980 & 2.977 & 2.971 & \textbf{2.965} \\
588M  & 2.871 & 2.816 & 2.868 & 2.839 & 2.820 & 2.834 & \textbf{2.815} \\
947M  & 2.740 & 2.696 & 2.730 & 2.709 & 2.699 & 2.707 & \textbf{2.692} \\
1.44B & 2.644 & 2.608 & 2.641 & 2.620 & 2.609 & 2.618 & \textbf{2.601} \\
2.09B & 2.572 & 2.531 & 2.566 & 2.543 & 2.530 & 2.542 & \textbf{2.527} \\
2.47B & 2.536 & 2.496 & 3.226$^\dagger$ & 2.512 & 2.496 & 2.501 & \textbf{2.493} \\
\bottomrule
\multicolumn{8}{l}{\small $^\dagger$Training instability (outlier).}
\end{tabular}
\end{table}

Figure~\ref{fig:qwen3_scaling} and Table~\ref{tab:qwen3_scaling} present scaling law fits for Qwen3.
\begin{figure}[t]
\centering
\includegraphics[width=0.95\textwidth]{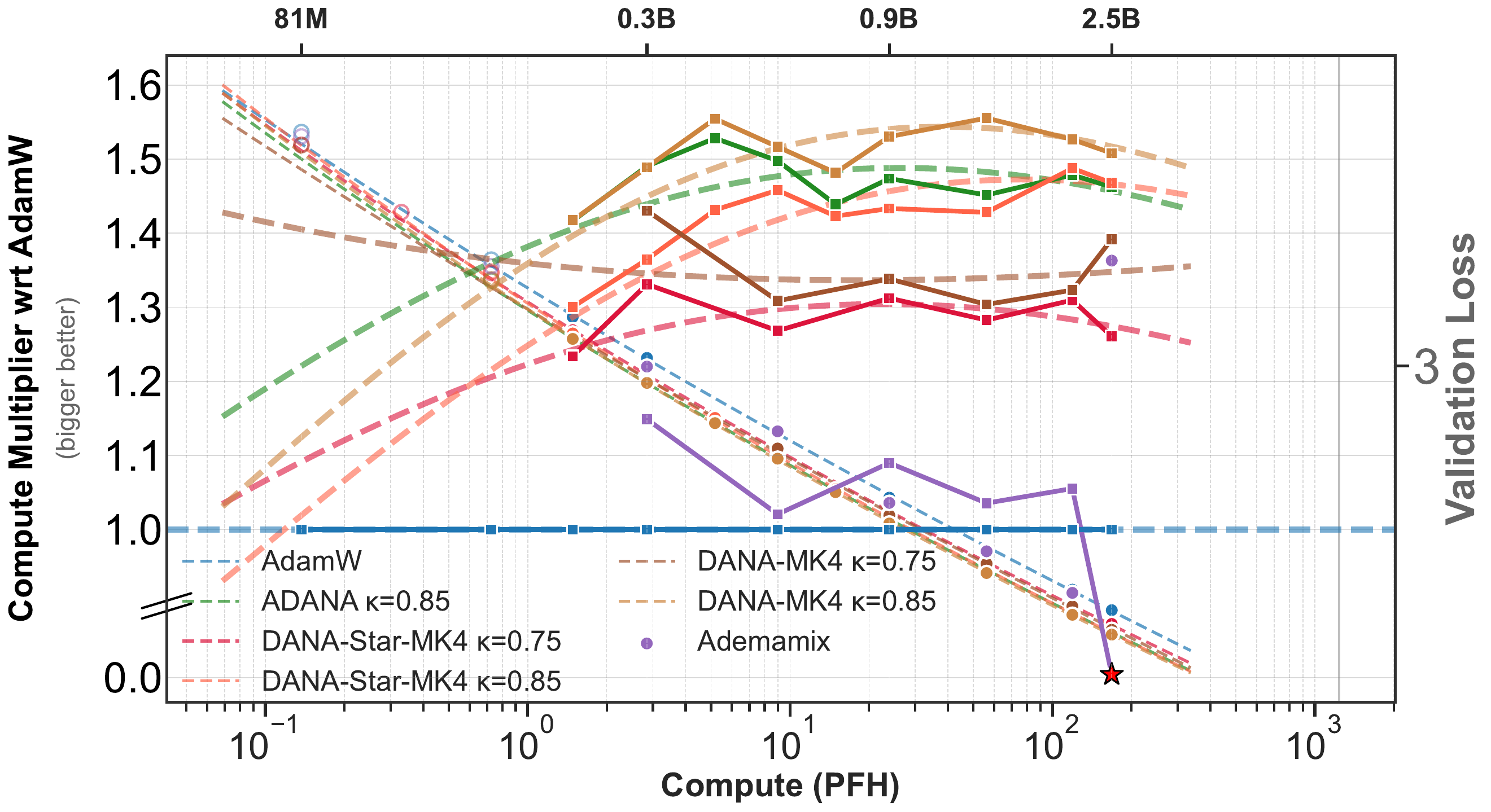}
\caption{\textbf{Qwen3 architecture scaling.} \textit{Left:} Compute savings relative to \AdamW. \textit{Right:} Validation loss vs compute with broken power law fits. \ADana{} ($\kappa = 0.85$), \DanastarMKfour{} ($\kappa=0.75$, $0.85$), and \DanaMKfour{} ($\kappa=0.75$, $0.85$) are compared against \AdamW. \Ademamix{} exhibits training instability at 17 heads (marked with star). \DanaMKfour{} $\kappa=0.85$ achieves the best loss at most scales.}
\label{fig:qwen3_scaling}
\end{figure}

\begin{table}[t]
\centering
\caption{\textbf{Qwen3 scaling parameters.} Broken power law fits $L = a + bC^{-c} + eC^{-f}$ with shared $a = 0.108$ using Chinchilla compute formula. Fits exclude heads $<7$. Because the broken power law has five free parameters and Qwen3 has fewer data points than Enoki, we initialize the optimization using values informed by the Enoki fits: $b_0 = 0.40$, $c_0 = 0.20$, $e_0 = 2.50$, $f_0 = 0.03$.}
\label{tab:qwen3_scaling}
\vspace{0.5em}
\begin{tabular}{l|ccccc}
\toprule
\textbf{Optimizer} & $b$ & $c$ & $e$ & $f$ & $R^2$ \\
\midrule
\AdamW          & 0.403 & 0.164 & 2.66 & 0.032 & 0.9995 \\
\ADana{} ($\kappa\!=\!0.85$) & 0.402 & 0.190 & 2.61 & 0.030 & 0.9999 \\
\DanastarMKfour{} ($\kappa\!=\!0.75$) & 0.403 & 0.188 & 2.62 & 0.030 & 0.9997 \\
\DanastarMKfour{} ($\kappa\!=\!0.85$) & 0.403 & 0.199 & 2.62 & 0.031 & 0.9998 \\
\DanaMKfour{} ($\kappa\!=\!0.75$) & 0.401 & 0.151 & 2.61 & 0.032 & 0.9994 \\
\DanaMKfour{} ($\kappa\!=\!0.85$) & 0.403 & 0.197 & 2.61 & 0.030 & 0.9999 \\
\Ademamix       & ---   & ---   & ---  & ---  & --- \\
\bottomrule
\end{tabular}
\end{table}

All optimizers achieve good fits ($R^2 > 0.99$) when excluding heads $<7$. \DanaMKfour{} ($\kappa=0.75$) achieves the highest $R^2$ (0.999). The $\kappa=0.85$ variants (both \DanastarMKfour{} and \DanaMKfour) consistently achieve lower absolute losses than their $\kappa=0.75$ counterparts, with \DanaMKfour{} $\kappa=0.85$ achieving the best loss at most scales (Table~\ref{table:qwen3_results_appendix}). \Ademamix{} exhibits training instability at 17 heads (loss diverges to 3.23 from expected $\sim$2.5), preventing a reliable power law fit.

Figure~\ref{fig:qwen3_compute_formula_residuals} shows the fit residuals for Qwen3 using the Chinchilla compute formula ($6ND$). All DANA variants achieve excellent fits with residuals well within $\pm 0.01$.

\begin{figure}[t]
\centering
\includegraphics[scale=0.6]{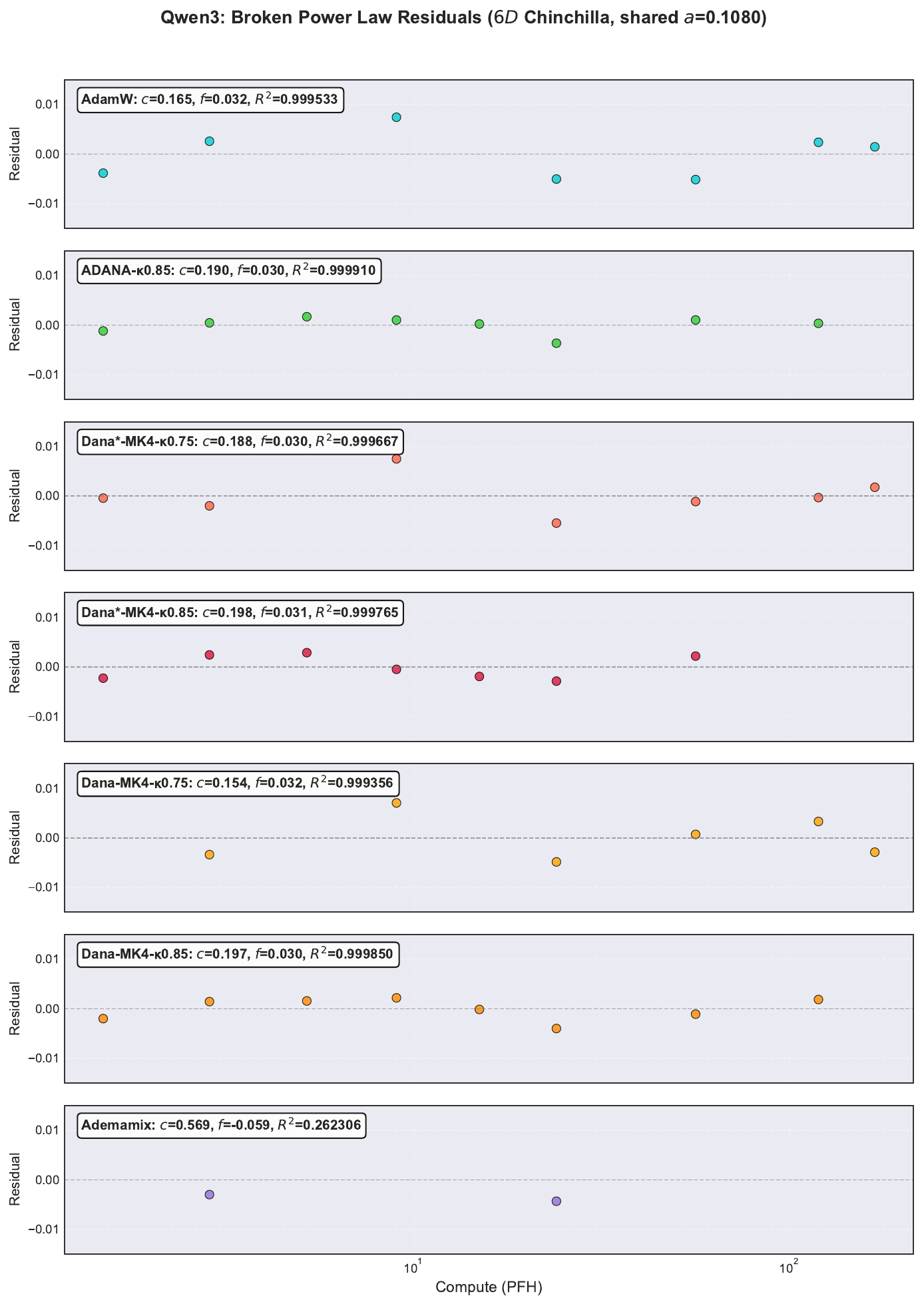}
\caption{\textbf{Qwen3 broken power law fit residuals.} Residuals from fitting $L = a + bC^{-c} + eC^{-f}$ with shared $a = 0.108$ using Chinchilla compute formula ($6ND$). Fits use heads $\geq 7$. All DANA variants achieve excellent fits ($R^2 > 0.999$), while \Ademamix{} shows poor fit quality due to training instability.}
\label{fig:qwen3_compute_formula_residuals}
\end{figure}

\paragraph{Comparison with Enoki scaling exponents.}
The broken power law structure is broadly similar between architectures. For Enoki (Table~\ref{tab:outscaling_exponents}, shared $a = 0.106$), the large-scale exponent $f \approx 0.029$--$0.030$ is nearly identical across optimizers; for Qwen3, $f \approx 0.030$--$0.032$ is likewise consistent. The small-scale exponent $c$ shows more variation: Enoki achieves $c \approx 0.211$--$0.218$, while Qwen3 ranges from $c = 0.151$ (\DanaMKfour{} $\kappa=0.75$) to $c = 0.199$ (\DanastarMKfour{} $\kappa=0.85$). This suggests Enoki exhibits slightly steeper initial scaling improvement. Despite these differences, the relative ordering of optimizers is preserved: DANA variants consistently outperform \AdamW{} on both architectures, with the compute savings persisting at scale.

\subsection{Qwen3 Learning Rate Scaling}
\label{sec:qwen3_lr_scaling}


For Qwen3, the smallest model sizes (below 100M non-embedding parameters) deviate from power-law scaling and are excluded from the fit. The resulting fits are:

\begin{align}
\text{\AdamW:} \quad & \gamma^* = 2.16 \times 10^4 \cdot (7.20 \times 10^4 + P)^{-0.873}, \\
\text{\DanastarMKfour{} ($\kappa\!=\!0.85$):} \quad & \gamma^* = 9.39 \times 10^1 \cdot (1.71 \times 10^4 + P)^{-0.701}, \\
\text{\ADana{} ($\kappa\!=\!0.85$):} \quad & \gamma^* = 3.17 \cdot (3.99 \times 10^4 + P)^{-0.553}, \\
\text{\DanaMKfour{} ($\kappa\!=\!0.85$):} \quad & \gamma^* = 1.20 \cdot (3.44 \times 10^4 + P)^{-0.501}.
\end{align}

The \AdamW{} fit exhibits the largest exponent ($d \approx -0.87$), while the DANA variants show more moderate exponents ($d \approx -0.50$ to $-0.70$).

\begin{figure}[t]
\centering
\begin{subfigure}[b]{0.48\textwidth}
    \includegraphics[width=\textwidth]{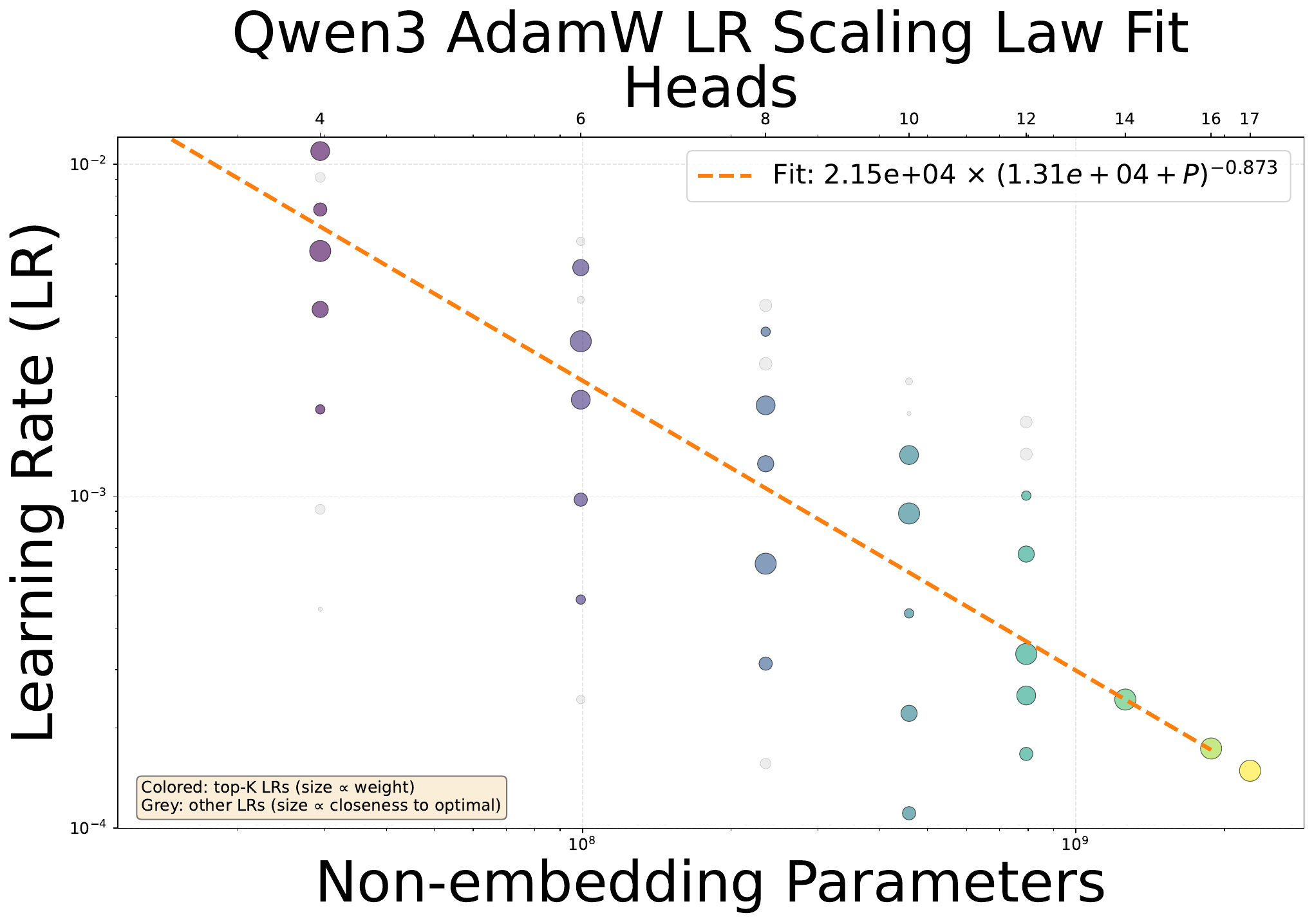}
    \caption{\AdamW}
\end{subfigure}
\hfill
\begin{subfigure}[b]{0.48\textwidth}
    \includegraphics[width=\textwidth]{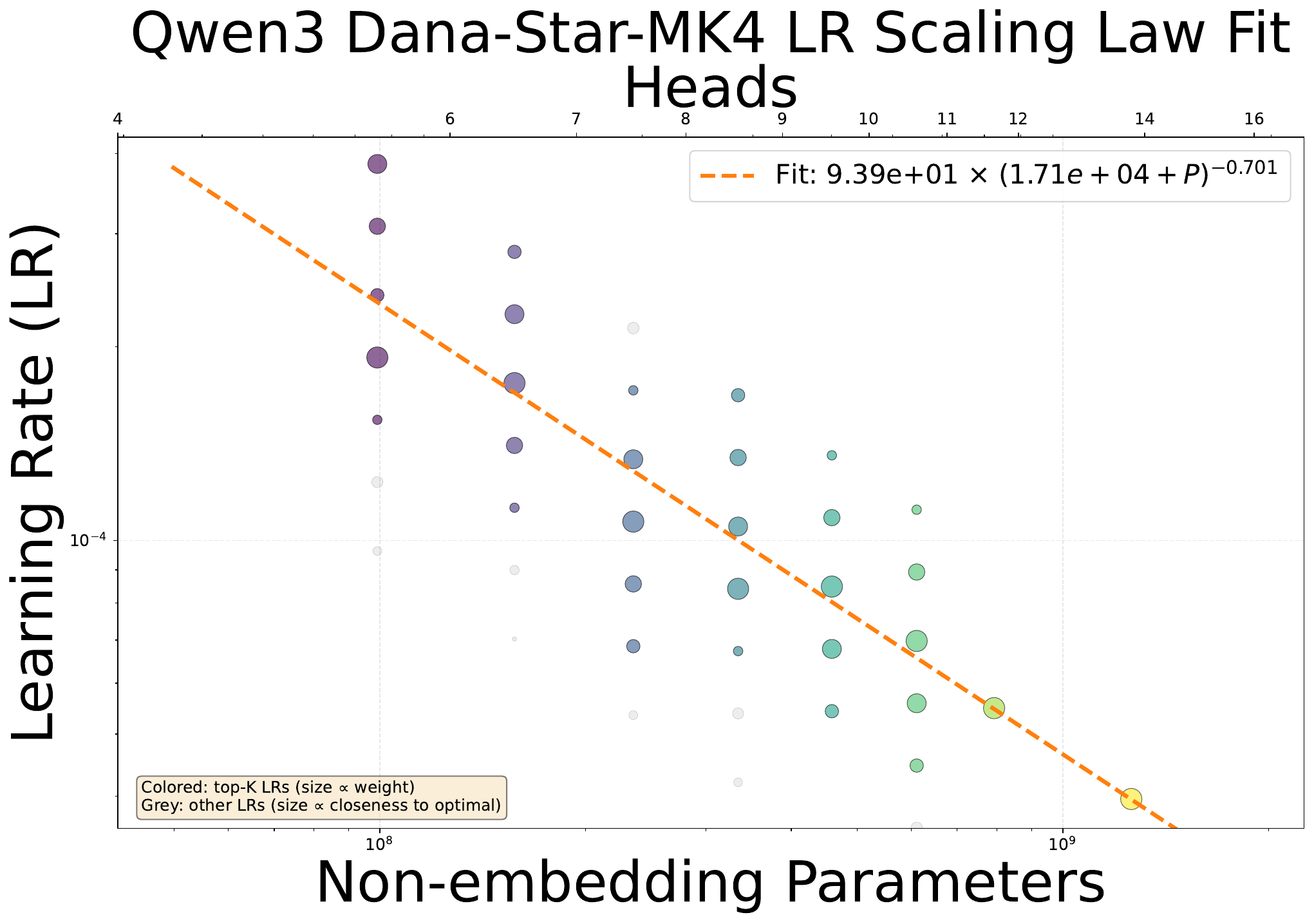}
    \caption{\DanastarMKfour $\kappa = 0.85$}
\end{subfigure}
\\[1em]
\begin{subfigure}[b]{0.48\textwidth}
    \includegraphics[width=\textwidth]{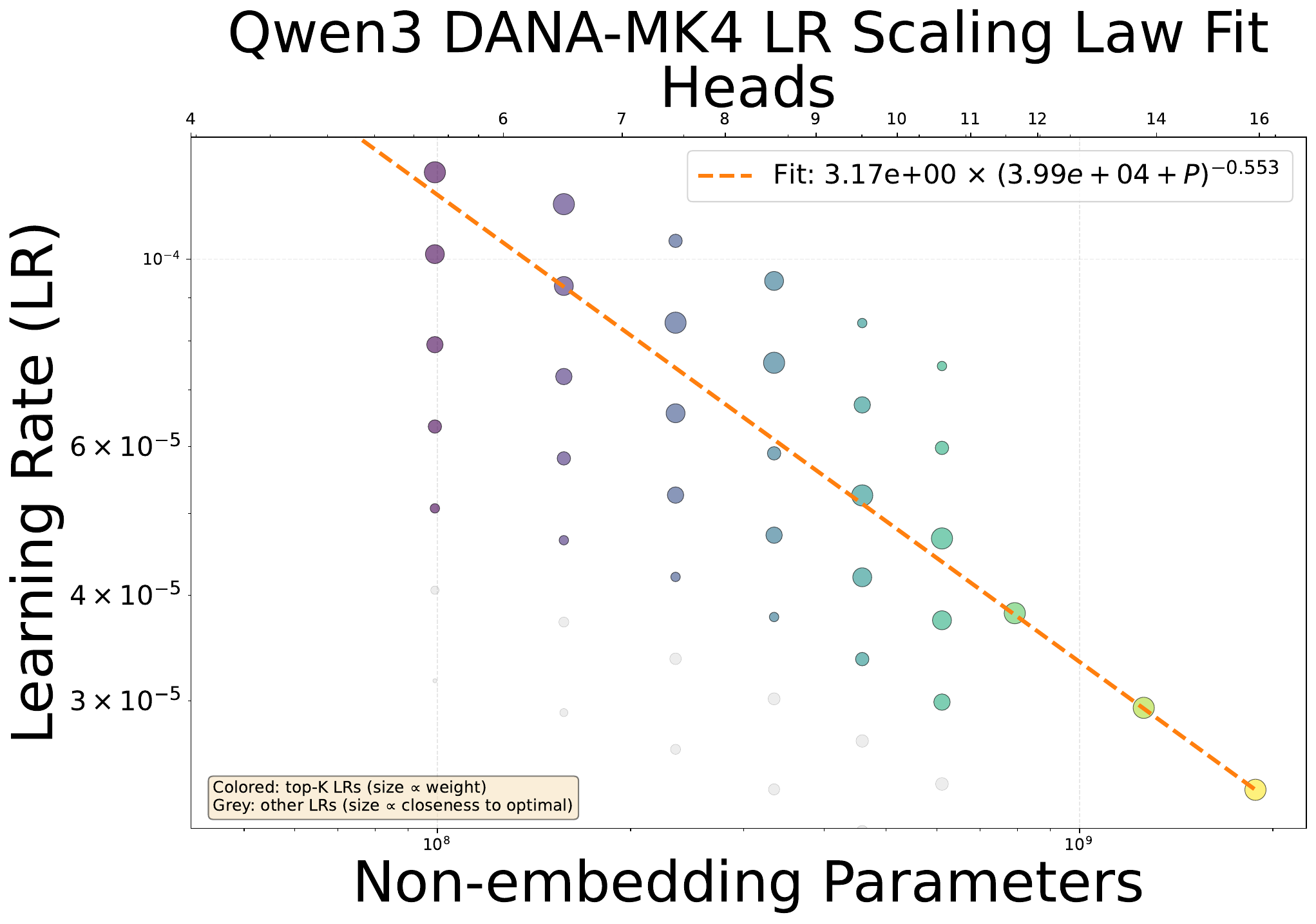}
    \caption{\ADana $\kappa = 0.85$}
\end{subfigure}
\hfill
\begin{subfigure}[b]{0.48\textwidth}
    \includegraphics[width=\textwidth]{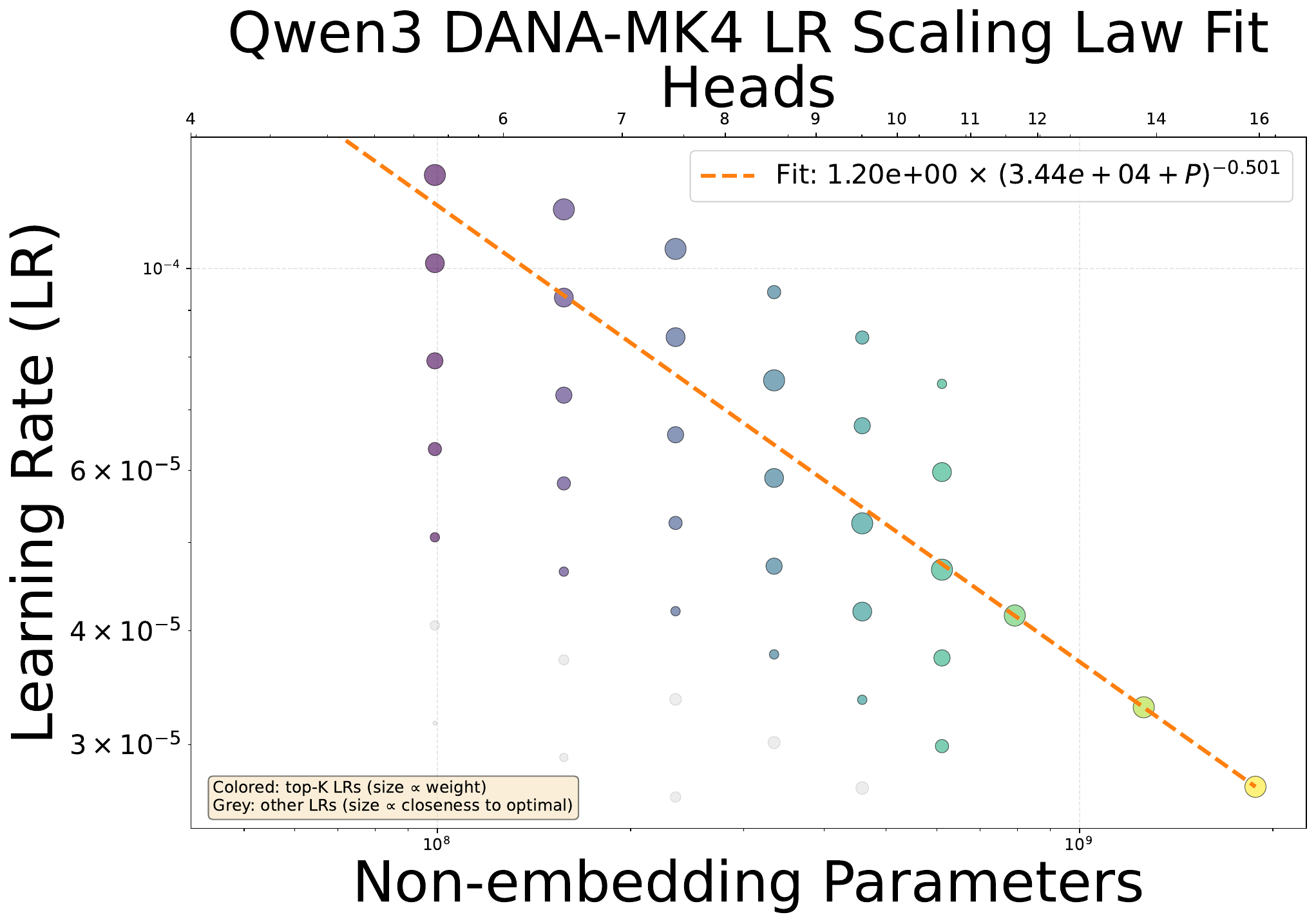}
    \caption{\DanaMKfour $\kappa = 0.85$}
\end{subfigure}
\caption{\textbf{Peak learning rate scaling for Qwen3 models.} Scales below 100M non-embedding parameters are excluded from the fit.}
\label{fig:qwen3_lr_scaling}
\end{figure}

\subsection{Key Findings}

The Qwen3 experiments confirm that optimizer improvements transfer across architectures---\ADana, \DanastarMKfour, and \DanaMKfour{} variants consistently outperform \AdamW{} on both Enoki and Qwen3. Several architecture-specific patterns emerge:

\begin{enumerate}
    \item \textbf{Hardened variant outperforms on Qwen3.} Unlike Enoki where \ADana{} achieves the best scaling exponent, \DanaMKfour{} ($\kappa=0.85$) achieves the best final validation loss at most scales on Qwen3 (Table~\ref{table:qwen3_results_appendix}). This may reflect greater sensitivity to $\kappa$ in the deeper Qwen3 architecture, where the hardened variant's reduced $\kappa$-sensitivity provides an advantage.

    \item \textbf{Star variant ($\tau$ estimator) has limited impact at this scale.} The \DanastarMKfour{} variants do not substantially outperform \DanaMKfour{} on Qwen3, consistent with Enoki results. The $\tau$ probability estimator appears most beneficial at larger scales or with sparser gradient distributions.

    \item \textbf{\Ademamix{} exhibits training instability.} At 17 heads ($\sim$2.5B parameters), \Ademamix{} diverges (loss 3.23 vs.\ expected $\sim$2.5), preventing reliable power law fits. All DANA variants remain stable across the full scale range.

    \item \textbf{Learning rate scaling differs between architectures.} The LR exponent for \AdamW{} is $-0.87$ on Qwen3 versus $-0.49$ on Enoki (Section~\ref{sec:lr_scaling}), while DANA variants show more moderate exponents ($-0.50$ to $-0.70$) on both. This suggests architecture-specific LR tuning is necessary, though DANA variants are less sensitive to this choice.

    \item \textbf{Scaling law structure differs but relative rankings persist.} Qwen3 exhibits larger small-scale exponents $c$ and smaller large-scale exponents $f$ compared to Enoki, yet DANA variants consistently outperform \AdamW{} on both architectures.
\end{enumerate}

\clearpage
\section{Impact of Batch Size}
\label{sec:appendix_batch}


In this section, we discuss the role of batch size in \ADana. We describe our experimental setup, discuss the theoretical implications of batch size for logarithmic time scheduling, and present preliminary results demonstrating that \ADana maintains its scaling advantages at larger batch sizes.

Our primary experiments use a batch size of 32 sequences (approximately 65K tokens per batch with sequence length 2048). This small-batch regime offers several advantages: it is more FLOP-efficient \citep{marek2025small}, requires less memory, and allows for more frequent parameter updates. Larger batch sizes are typically motivated by hardware utilization on multi-GPU systems rather than optimization benefits.

\paragraph{$\beta_2$ and batch size stability.}
Following \citet{brown2020language, touvron2023llama}, our \AdamW baseline uses $\beta_2 = 0.95$ rather than the default $\beta_2 = 0.999$. \citet{liu2019roberta} found that $\beta_2 = 0.98$ improves stability when training with large batch sizes, and \citet{chen2020igpt} found that $\beta_2 = 0.95$ reduced loss spikes compared to $\beta_2 = 0.999$. We do not modify the \AdamW $\beta_2$ in our experiments; these findings serve as context for potential instabilities when $\beta_2 \to 1$.

Note that under our logarithmic time framework, \ADana uses $\beta_2(k) = 1 - \delta/k$, which sends $\beta_2 \to 1$ as training progresses. One motivation for the large-batch experiments in this section is to verify that this asymptotic behavior does not introduce instabilities in the large batch setting, where prior work has found $\beta_2$ close to $1$ to be problematic.

\paragraph{Batch size and gradient noise.}
Under the hypothesis that the limiting noise in language model training arises from feature learning on Zipfian-distributed data (see \Cref{sec:hilberg_hypothesis}), we expect the correct way to scale batch is through a batch-dependent rescaling of the $\gamma_3$ damping parameter rather than standard noise reduction arguments. Specifically, if gradient noise were purely due to sampling variance, it would decrease as $1/\sqrt{B}$ with batch size $B$. However, for language data, the noise from learning semantic concepts---which exhibits power-law structure characterized by Hilberg's exponent (see \Cref{sec:hilberg_hypothesis})---persists across batch sizes, so while larger batches do reduce noise, this motivates keeping the same logarithmic time schedule while adjusting $\gamma_3$ by a batch-dependent constant.

\subsection{Impact of Batch Size on the Damping Schedule $\alpha(t)$}

In \Cref{fig:alpha_scaling_chinchilla} (main text) we studied the dependence of the damping factor $\tilde{\alpha}$ on training time and model scale in the context of the \Danaconstant-type damping schedule $\alpha(t)=\tilde{\alpha} t$. In \Cref{fig:alpha_scaling_batch}, we extend this analysis by varying both the iterations, model sizes, and batch size (number of sequences seen at each iteration) between 32 and 512.

For a fixed batch size, the optimal $\tilde{\alpha}$ continues to scale as a power law in the number of iterations $T$ and remains approximately independent of model size. However, increasing batch size increases the optimal $\tilde{\alpha}$ allowed. This is consistent with the idea that increasing batch size decreases the gradient noise, hence improving stability and alleviating the need for a small damping schedule $\alpha(t)$. Note that in the limit of large batch sizes we expect Stochastic Nesterov to be stable and hence $\tilde{\alpha}\equiv 1$ to be stable.

\begin{figure}[t]
\centering

\begin{subfigure}{0.45\textwidth}
    \centering
    \includegraphics[width=\linewidth]{figures/scaling_new_figs/gamma3_factor_scaling_Enoki_Scaled_omega4.0_effbatch32_power.pdf}
    \caption{\small \textbf{Batch 32.}}
    \label{fig:batch_32}
\end{subfigure}
\hfill
\begin{subfigure}{0.45\textwidth}
    \centering
    \includegraphics[width=\linewidth]{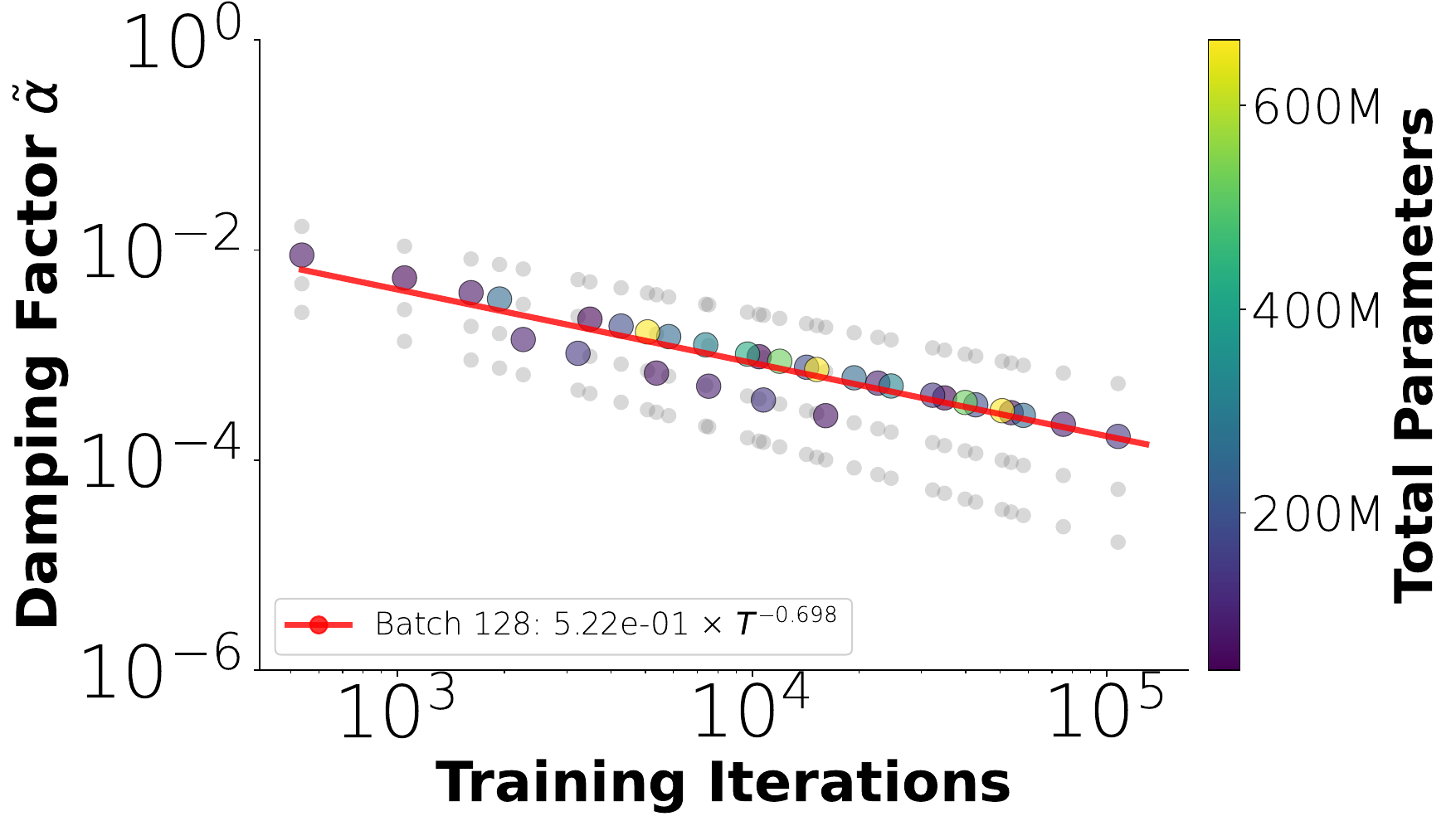}
    \caption{\small \textbf{Batch 128.}}
    \label{fig:batch_128}
\end{subfigure}

\begin{subfigure}{0.45\textwidth}
    \centering
    \includegraphics[width=\linewidth]{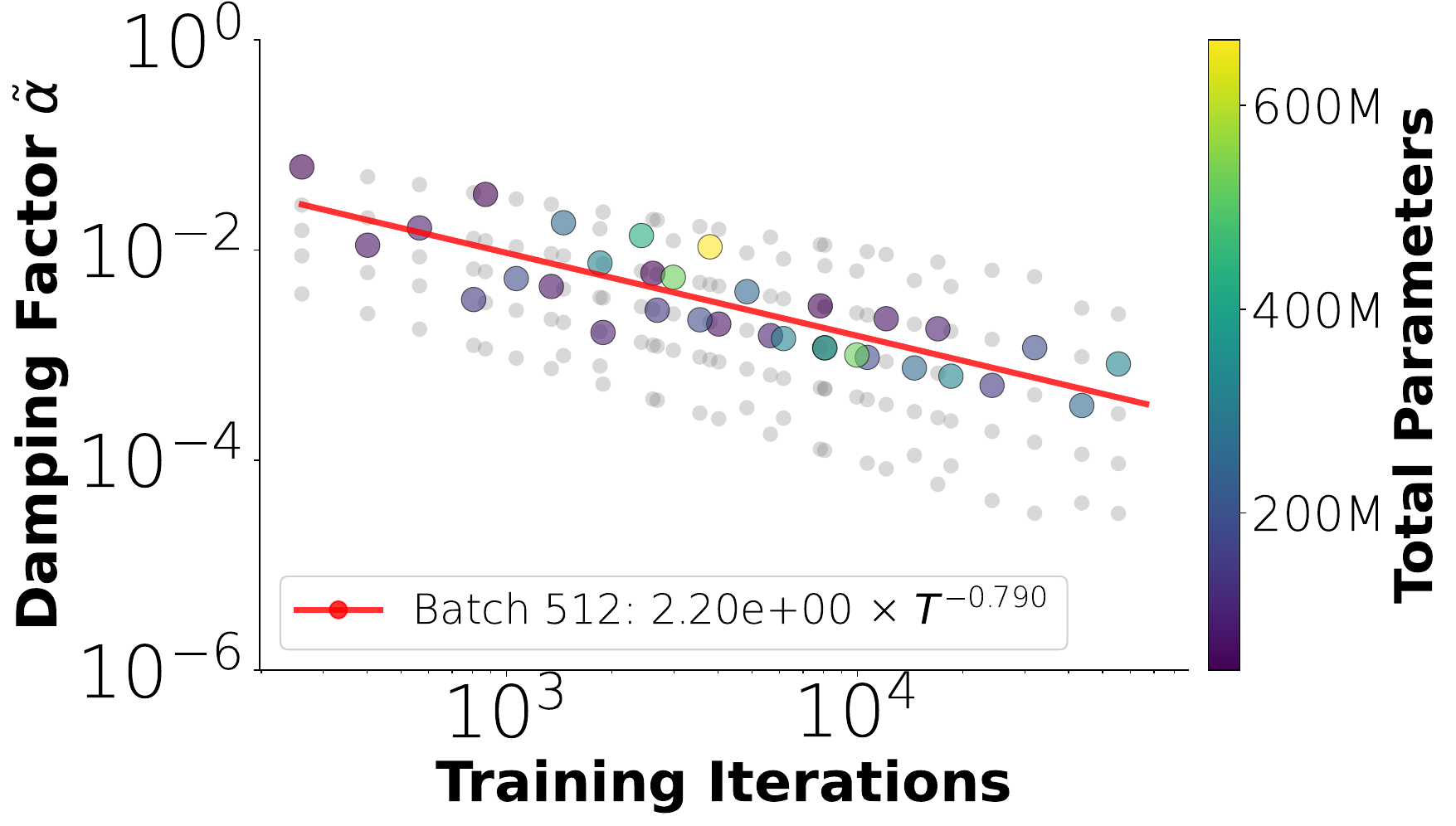}
    \caption{\small \textbf{Batch 512.}}
    \label{fig:batch_512}
\end{subfigure}
\hfill
\begin{subfigure}{0.45\textwidth}
    \centering
    \includegraphics[width=\linewidth]{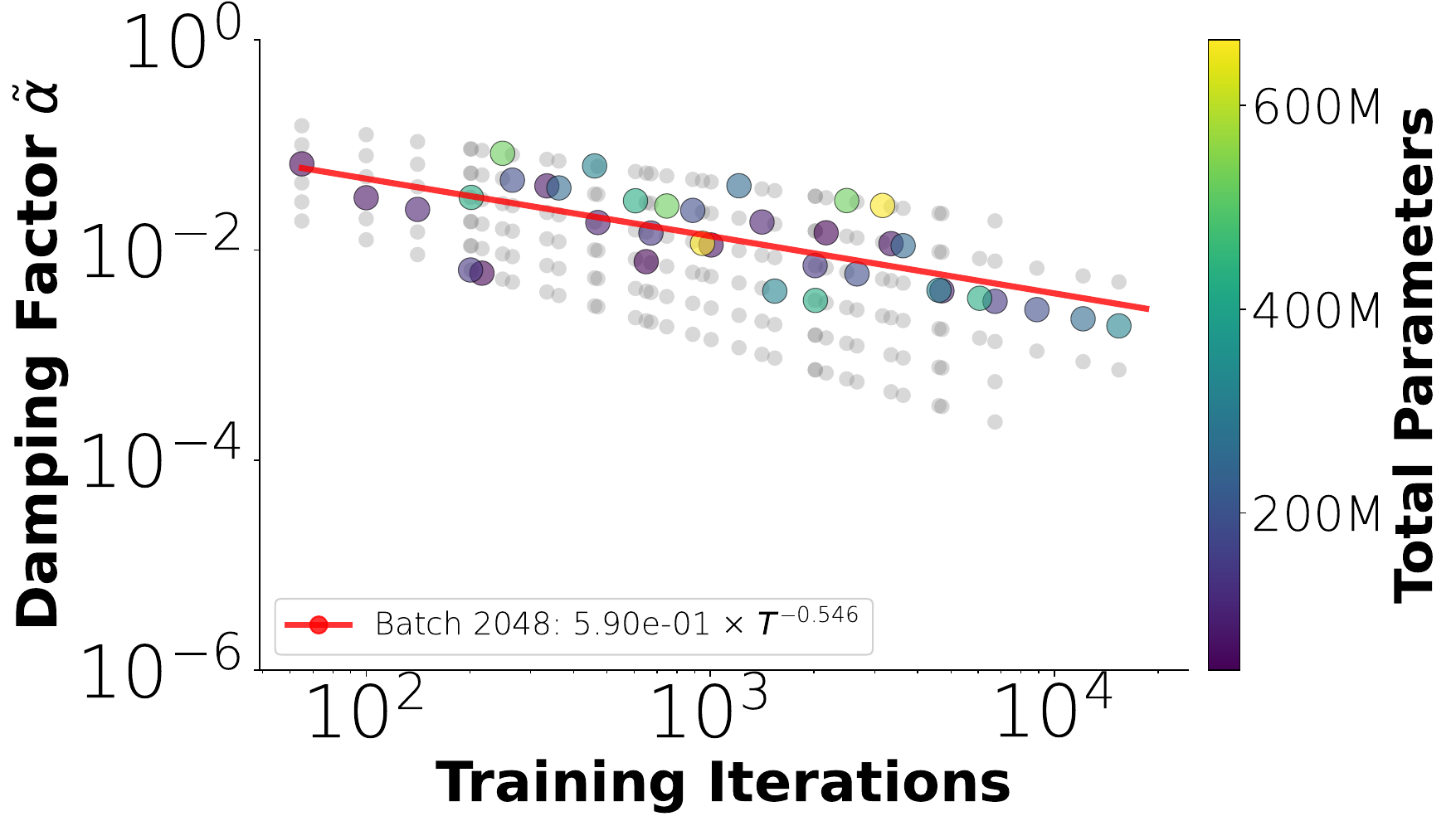}
    \caption{\small \textbf{Batch 2048.}}
    \label{fig:batch_2048}
\end{subfigure}

\caption{\textbf{Impact of batch size on $\alpha(t)$.} The optimal damping constant $\tilde{\alpha}$ in the schedule $\alpha(t)=\tilde{\alpha} t$ scales as a power law in iterations $T$. Increasing batch size from 32 to 2048 increases the optimal $\tilde{\alpha}$, consistent with reduced gradient noise at larger batches.
}
\label{fig:alpha_scaling_batch}
\end{figure}

The empirical relationship shown in \Cref{fig:alpha_scaling_batch} suggests that the optimal $\gamma_3 = \tilde{\alpha} \cdot T^{1-\kappa}$ scales as a power law in batch size, i.e., $\tilde{\alpha} \propto B^{\eta}$ for some exponent $\eta > 0$. Based on our fits, increasing batch size from 32 to 512 (a factor of 16) requires approximately a 2.5$\times$ increase in $\tilde{\alpha}$ to maintain optimal performance and implies an exponent of $\eta \approx 0.33$.

\subsection{Batch Size 256 Results}

We conducted a sweep at batch size 256 (approximately 524K tokens per batch) to verify that the scaling improvements of \ADana and \DanaMKfour transfer to larger batch regimes.  Results are presented in Figure \ref{fig:batch_256_scaling}, and show that \AdamW performance is increased substantially at small scales, that \Muon benefits even more substantially, but that \ADana catches up to this performance around $1B$ parameters.  We note that this is without incorporating batch adjustments to the momentum term. A main observation is that \ADana performs worse that \AdamW for very small compute budget but its performance increases with scale and improves over \AdamW starting medium scale compute. An interpretation is that at very small compute budget and large batch, the number of iterations to run is very small which may impact results: $1743$ iterations for batch $256$ against $13950$ for batch $32$ on $\text{Heads}=6$ Enoki model. At moderate scales $\text{Heads}=24$, batch $256$ performs a larger number of iterations of $25335$.

\begin{figure}[h]
    \centering
    \includegraphics[width=0.9\linewidth]{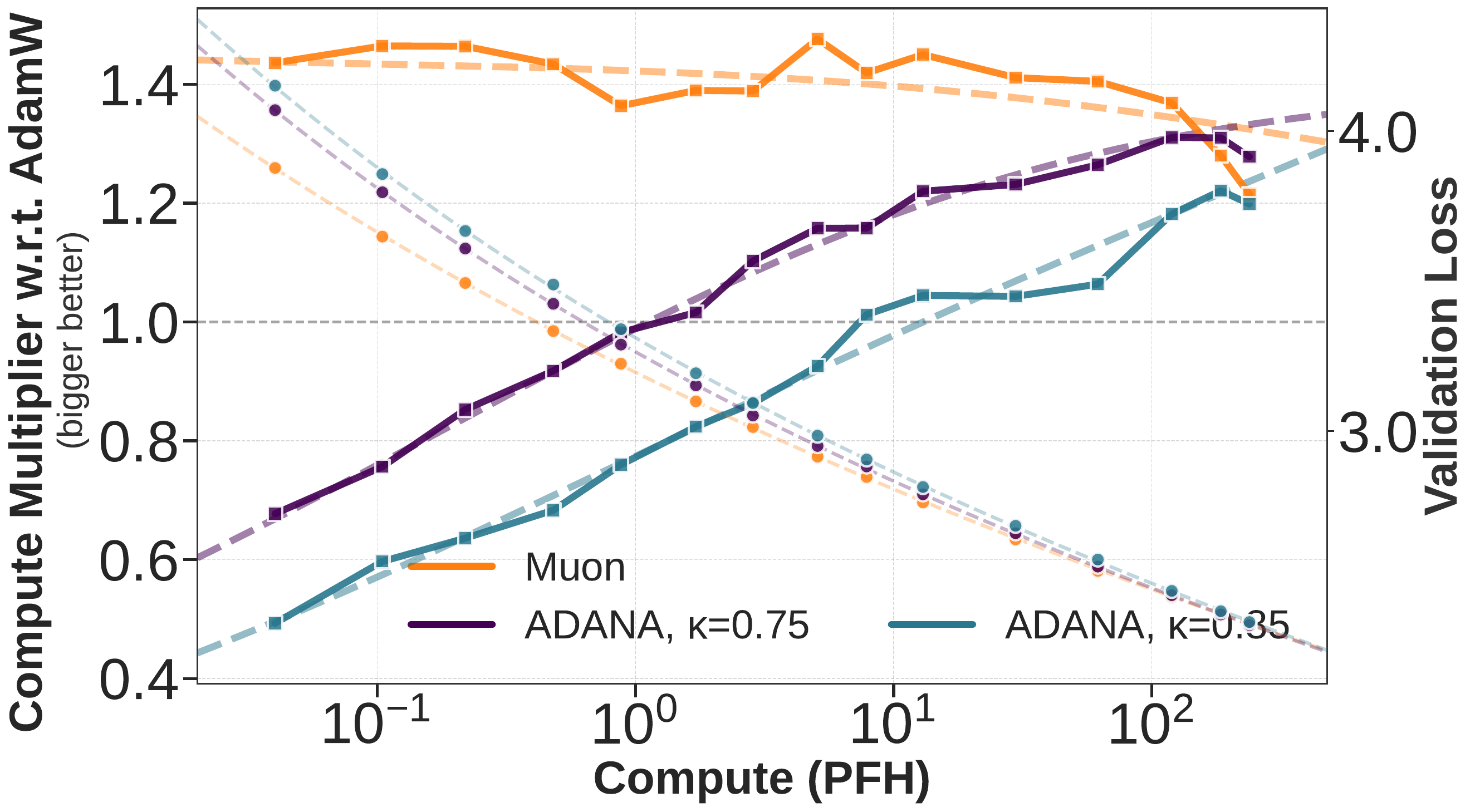}
    \caption{\textbf{Batch size 256 scaling comparison.} Validation loss and compute multipliers with respect to \AdamW (constant WD) for \Muon and \ADana with $\kappa \in \{0.75,0.85\}$ using logarithmic time weight-decay. Batch size is 256 with sequence length 2048. The compute efficiency gains of \ADana persist at this larger batch size.}
    \label{fig:batch_256_scaling}
\end{figure}

\subsection{Batch Size 512 with Rescaled $\gamma_3$}

For batch size 512 (approximately 1M tokens per batch), we rescale $\gamma_3$ by a factor of 2.5, derived from the empirical fit in \Cref{fig:alpha_scaling_batch}. We re-baseline \AdamW at this level, and re-fit the scaling rules for the learning rate.  The resulting figures are provided in Figure \ref{fig:lr_scaling_batch_512} and \ref{fig:batch_512_scaling}.

\begin{figure}[t]
\centering
\begin{subfigure}[t]{0.48\textwidth}
    \centering
    \includegraphics[width=\textwidth]{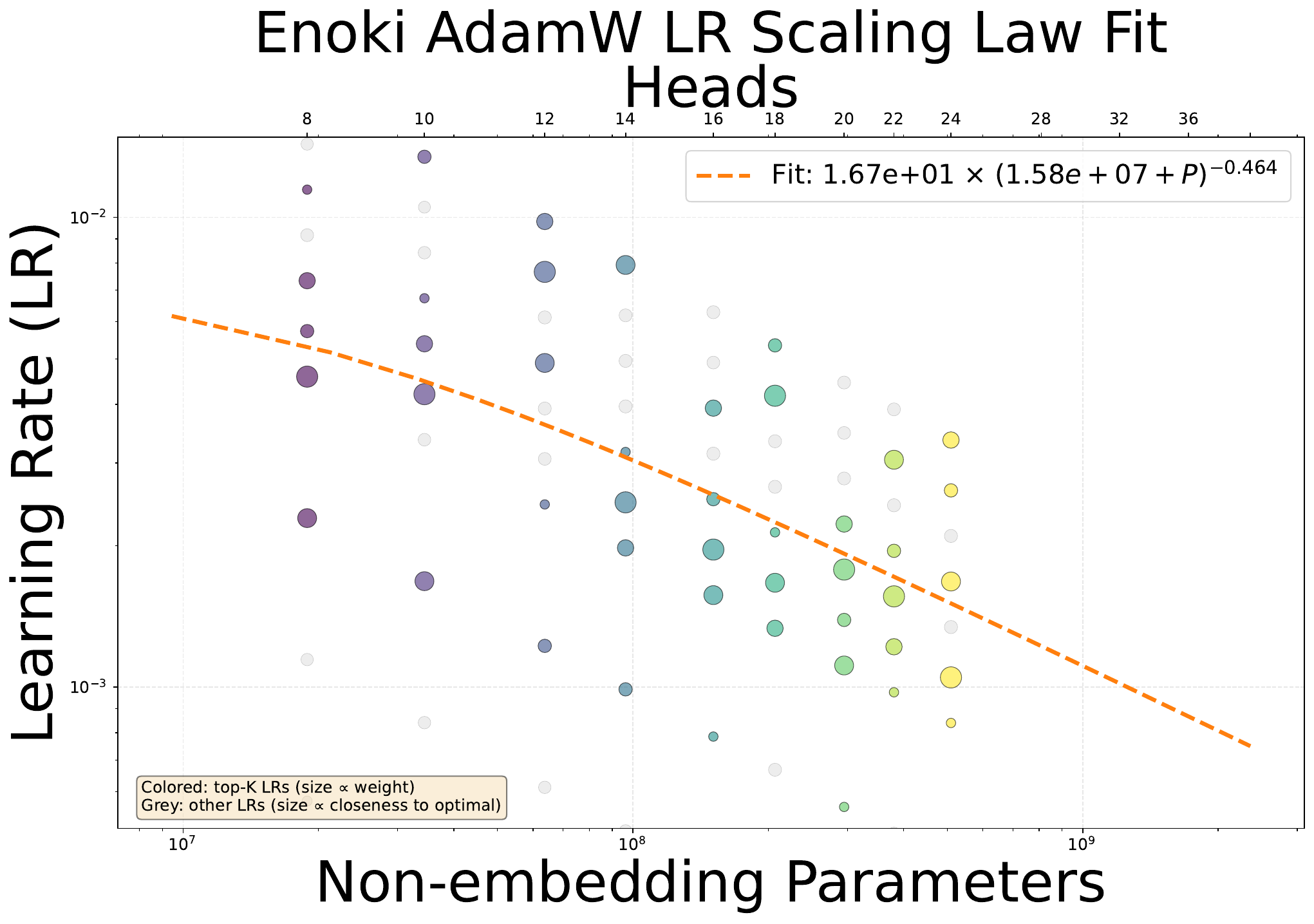}
    \caption{\AdamW}
\end{subfigure}
\hfill
\begin{subfigure}[t]{0.48\textwidth}
    \centering
    \includegraphics[width=\textwidth]{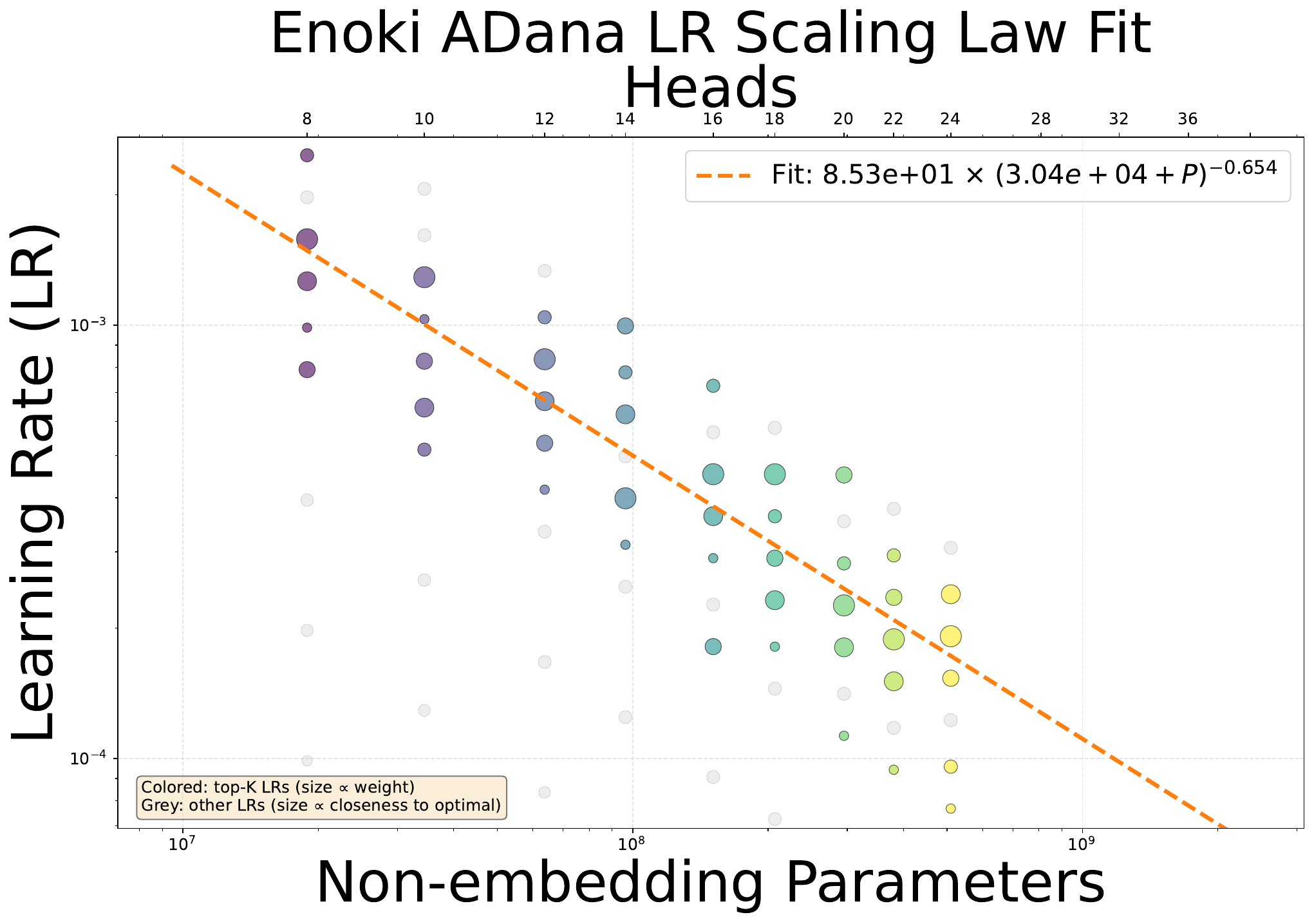}
    \caption{\ADana}
\end{subfigure}
\\[1em]
\begin{subfigure}[t]{0.48\textwidth}
    \centering
    \includegraphics[width=\textwidth]{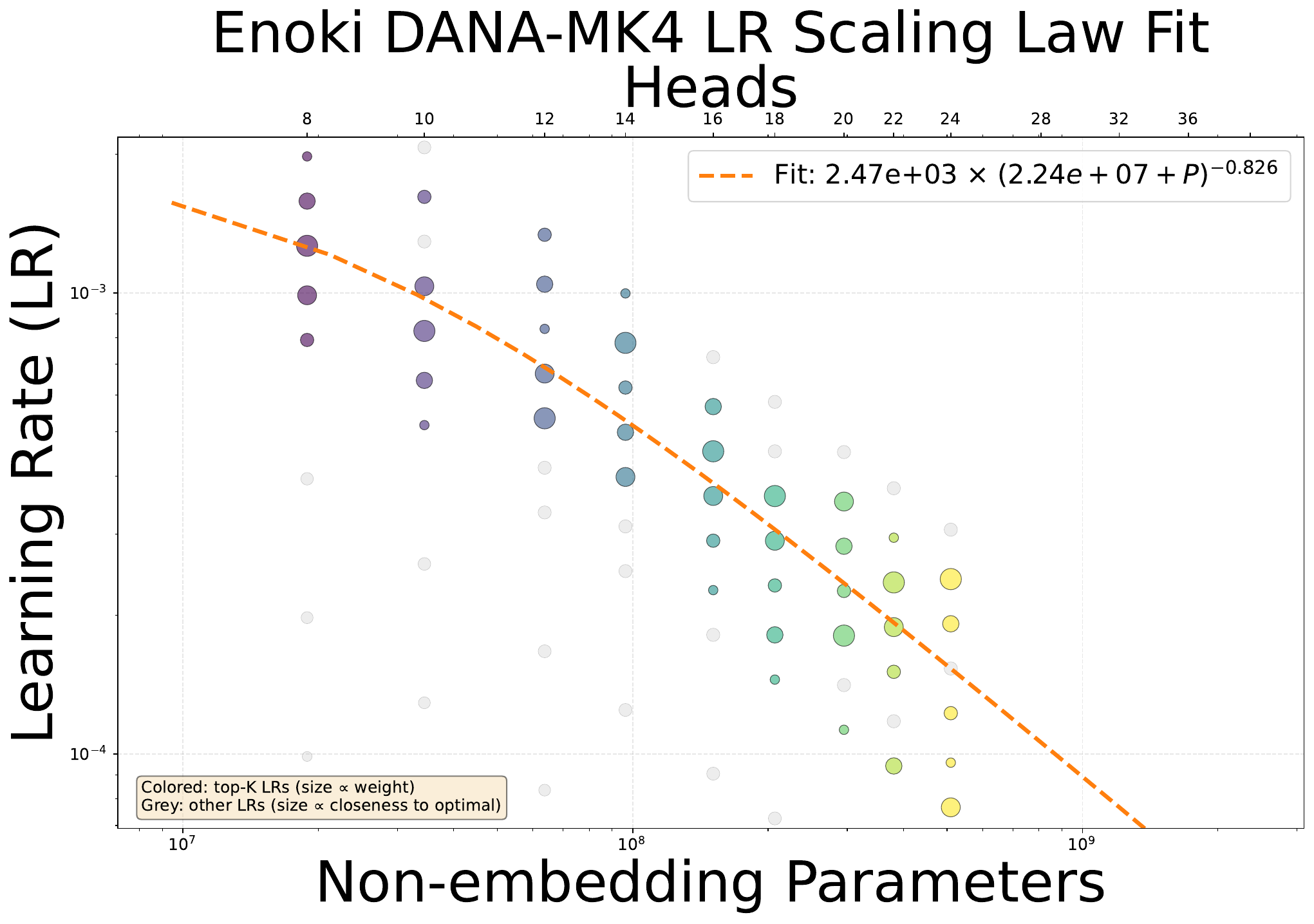}
    \caption{\DanaMKfour}
\end{subfigure}
\hfill
\begin{subfigure}[t]{0.48\textwidth}
    \centering
    \includegraphics[width=\textwidth]{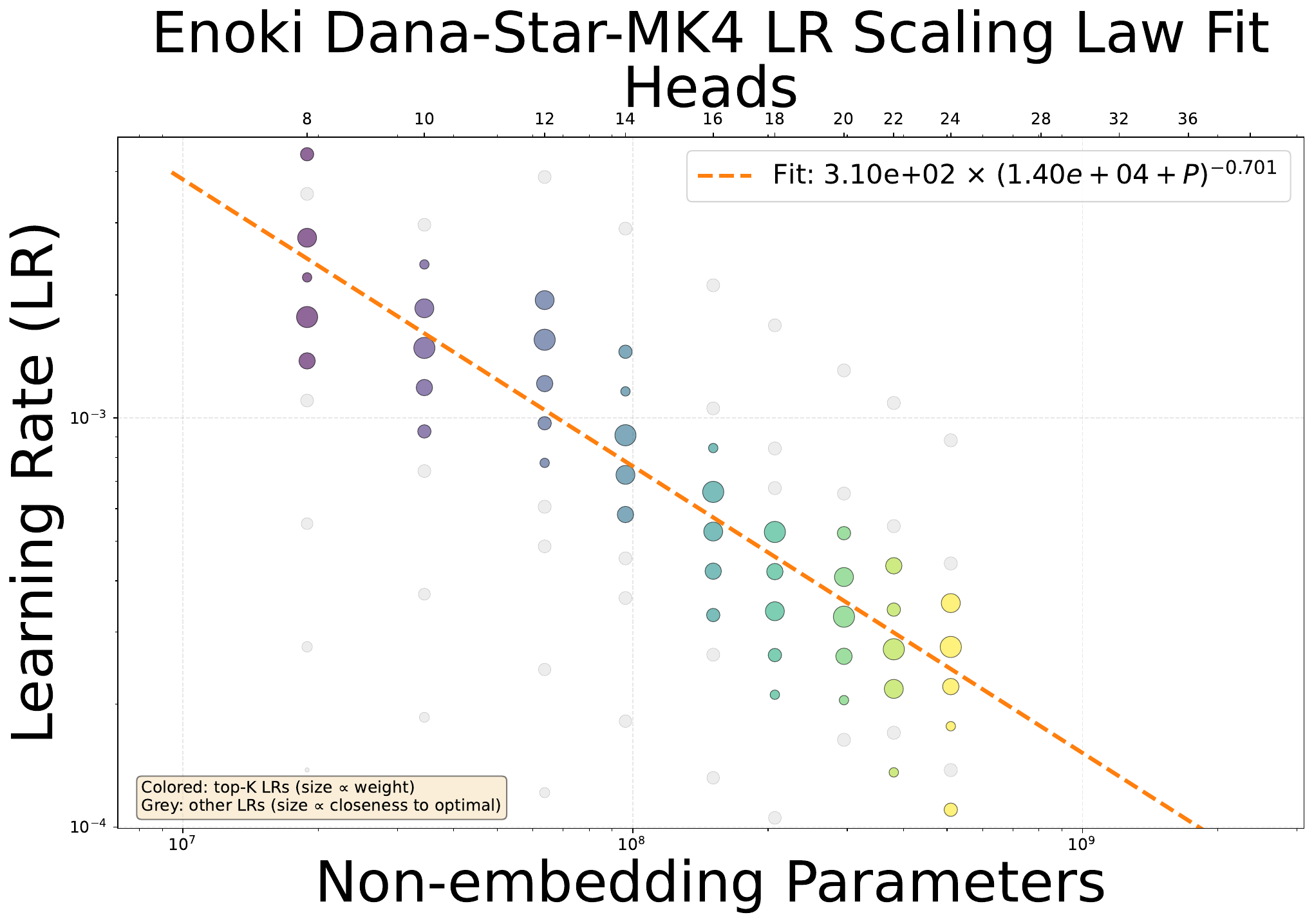}
    \caption{\DanastarMKfour}
\end{subfigure}
\caption{\textbf{LR scaling at batch size 512.} Optimal learning rate as a function of model size for \AdamW, \ADana, \DanaMKfour, and \DanastarMKfour at batch size 512. The fitted saturated power laws $\gamma^* = a(b + P_{\mathrm{non\text{-}emb}})^d$ are:
\AdamW: $\gamma^* = 16.7\,(1.58 \times 10^7 + P)^{-0.464}$;
\ADana ($\kappa=0.85$): $\gamma^* = 85.3\,(3.04 \times 10^4 + P)^{-0.654}$;
\DanaMKfour ($\kappa=0.85$): $\gamma^* = 2.47 \times 10^3\,(2.24 \times 10^7 + P)^{-0.826}$;
\DanastarMKfour ($\kappa=0.85$): $\gamma^* = 310\,(1.40 \times 10^4 + P)^{-0.701}$.}
\label{fig:lr_scaling_batch_512}
\end{figure}

\begin{figure}[h]
    \centering
    \includegraphics[width=0.9\linewidth]{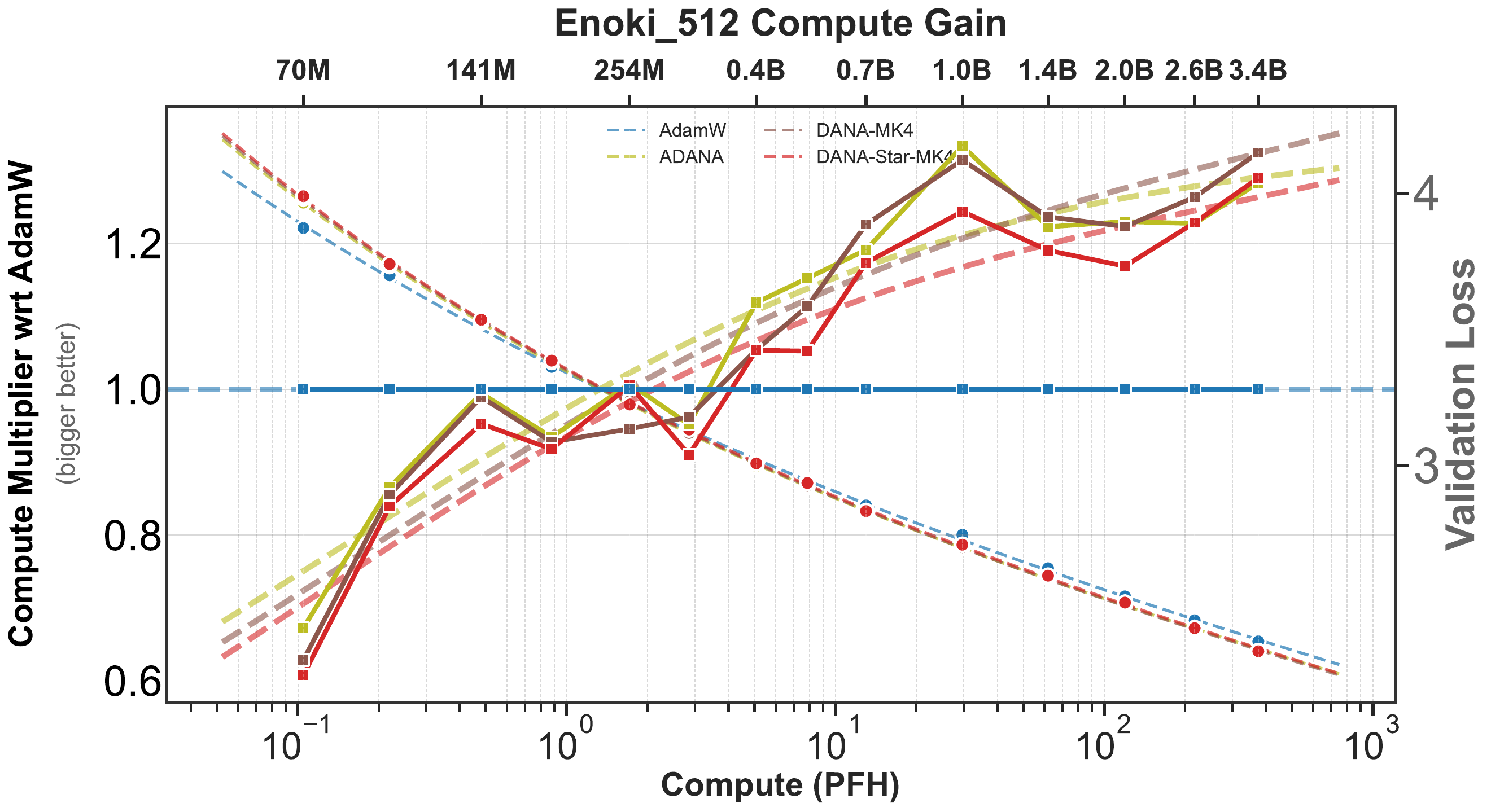}
    \caption{\textbf{Batch size 512 scaling comparison.} Validation loss comparing \AdamW, \ADana, \DanaMKfour, and \DanastarMKfour at batch size 512 with $\tilde{\alpha}=2.5$.}
    \label{fig:batch_512_scaling}
\end{figure}

\subsection{Discussion}

Preliminary results suggest that at larger batch sizes, both \ADana and \AdamW continue to scale, although at small model scales \AdamW outperforms \ADana, whereas at larger scales the slope of the scaling law fit appears even better for \ADana with large batch. This is consistent with the Hilberg hypothesis discussed in \Cref{sec:hilberg_hypothesis}.

The hardened variants (\DanaMKfour) provide additional stability margin at larger batch sizes. We have not observed instability issues at batch size 512 that would suggest $\beta_2 \approx 1$ is problematic in this regime, consistent with our choice of the log-time schedule $\beta_2(k) = 1 - \delta/k$ which naturally adapts to the training horizon.

The key practical implication is that when scaling batch size, practitioners should scale $\gamma_3$ (or equivalently $\tilde{\alpha}$) according to the power-law relationship shown in \Cref{fig:alpha_scaling_batch}, rather than keeping these hyperparameters fixed. This differs from standard batch size scaling rules for learning rate \citep{goyal2017accurate} because the damping parameter controls the stability of the long-momentum term rather than the gradient step size.

\clearpage
\section{Measuring the Data Exponent}
\label{sec:appendix_data_exponent}

In the Power-Law Random Features (PLRF) model and its Mixture-of-Experts extension (MOE-PLRF, see \Cref{sec:appendix_synthetic}), the data covariance matrix has eigenvalues decaying as $\lambda_j \propto j^{-2\rho}$, where the exponent $2\rho$ characterizes the effective dimensionality of the data. From the theoretical analysis in \citep{ferbach2025dimension}, the optimal choice of the momentum exponent is $\kappa = 1/(2\rho)$, which can be interpreted as the spectral dimension of the problem. When $2\rho > 1$, the data is approximately low-dimensional and acceleration is possible; when $2\rho < 1$, the data becomes effectively high-dimensional.

While the activation covariance spectra measured in transformers are not identical to the PLRF covariance structure, we expect a meaningful relationship to exist. To investigate this, we measure the eigenvalue distributions of activation covariances at different layers of our transformer models, both before and after training. We fit power laws of the form $\lambda_j \propto j^{-2\rho}$ to these spectra and extract the corresponding exponent.

We observe the following behaviors of the data exponent at different layers of the transformers:
\begin{itemize}
    \item The data exponent $2\rho$ increases with layer depth.
    \item The data exponent $2\rho$ decreases during training.
\end{itemize}
In particular, the largest data exponents (especially those above the high-dimensional threshold $2\rho = 1$) are reached at initialization for deep layers. After training, or for shallower layers, the exponent is often below the high-dimensional line.

Notably, the optimal observed $\kappa$ for the \ADana algorithm in our experiments ($\kappa \approx 0.75$--$0.85$) corresponds to an effective spectral dimension consistent with measured exponents $2\rho \approx 1.2$--$1.3$ in the activation covariance spectra. This provides empirical support for the connection between the theoretical PLRF analysis and practical transformer optimization.

\begin{figure}[h!]
\centering
\begin{subfigure}[t]{0.24\textwidth}
\centering
\includegraphics[width=\textwidth]{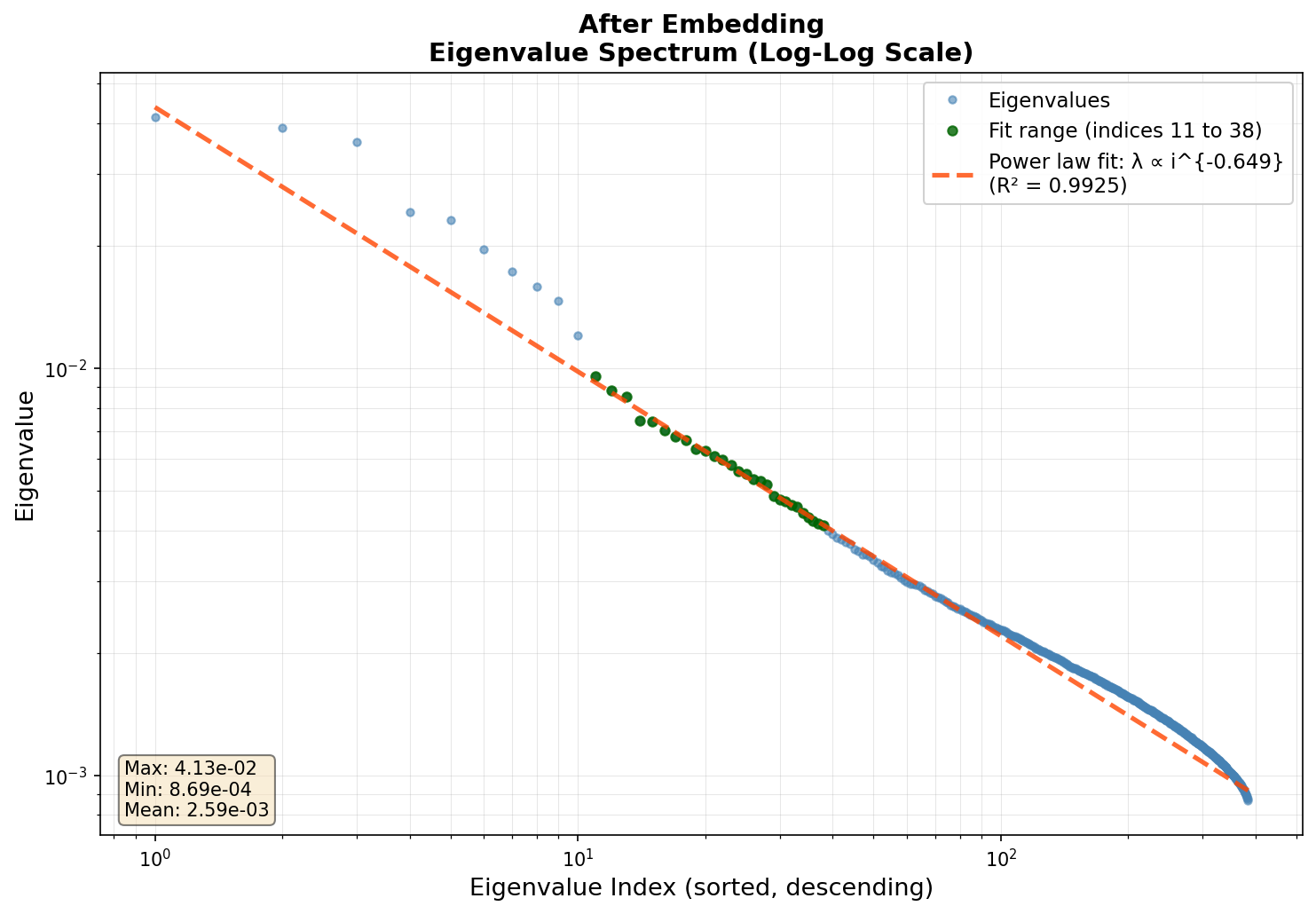}
\caption{Layer 0: After embedding}
\end{subfigure}
\hfill
\begin{subfigure}[t]{0.24\textwidth}
\centering
\includegraphics[width=\textwidth]{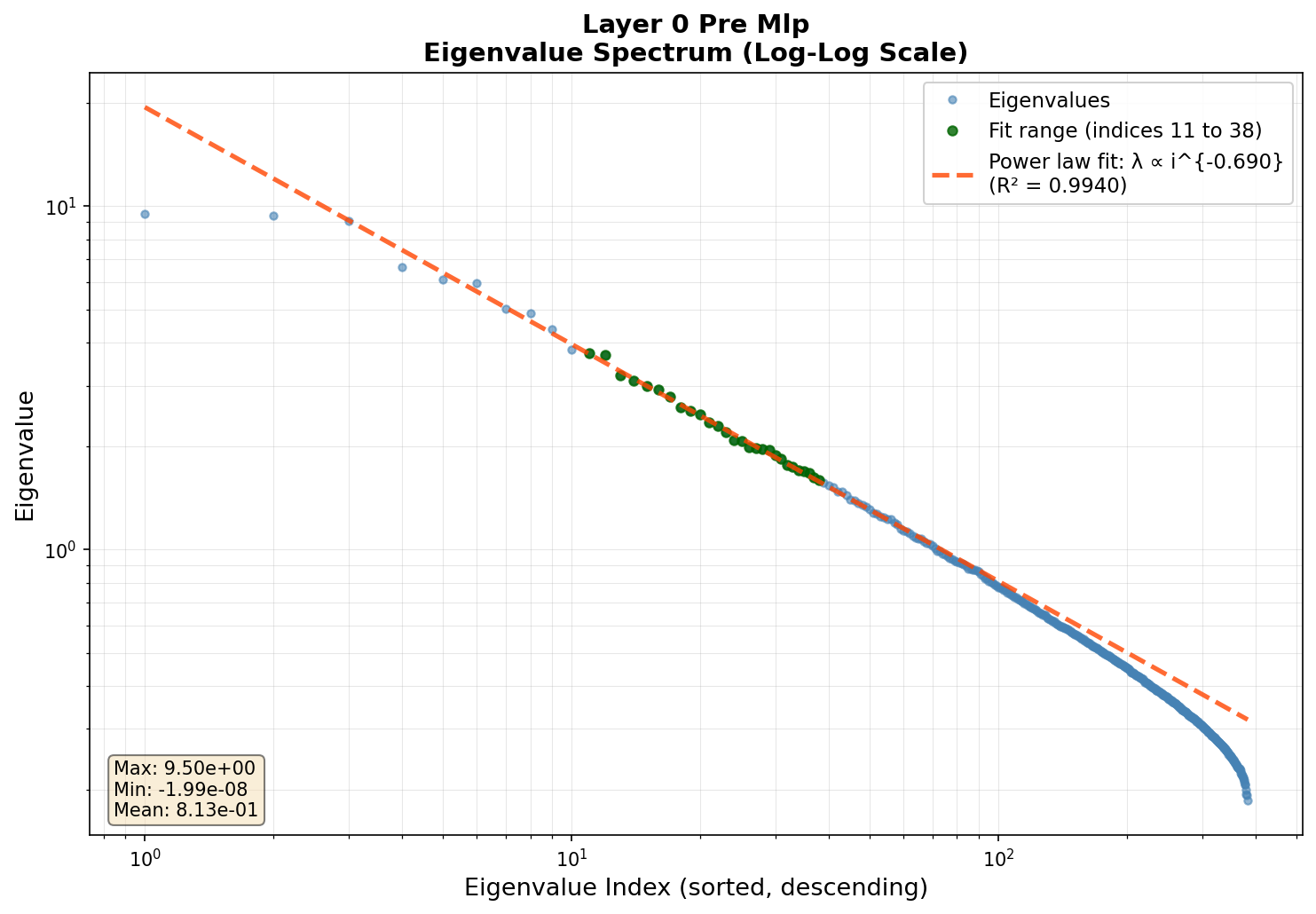}
\caption{Layer 0: Pre-MLP}
\end{subfigure}
\hfill
\begin{subfigure}[t]{0.24\textwidth}
\centering
\includegraphics[width=\textwidth]{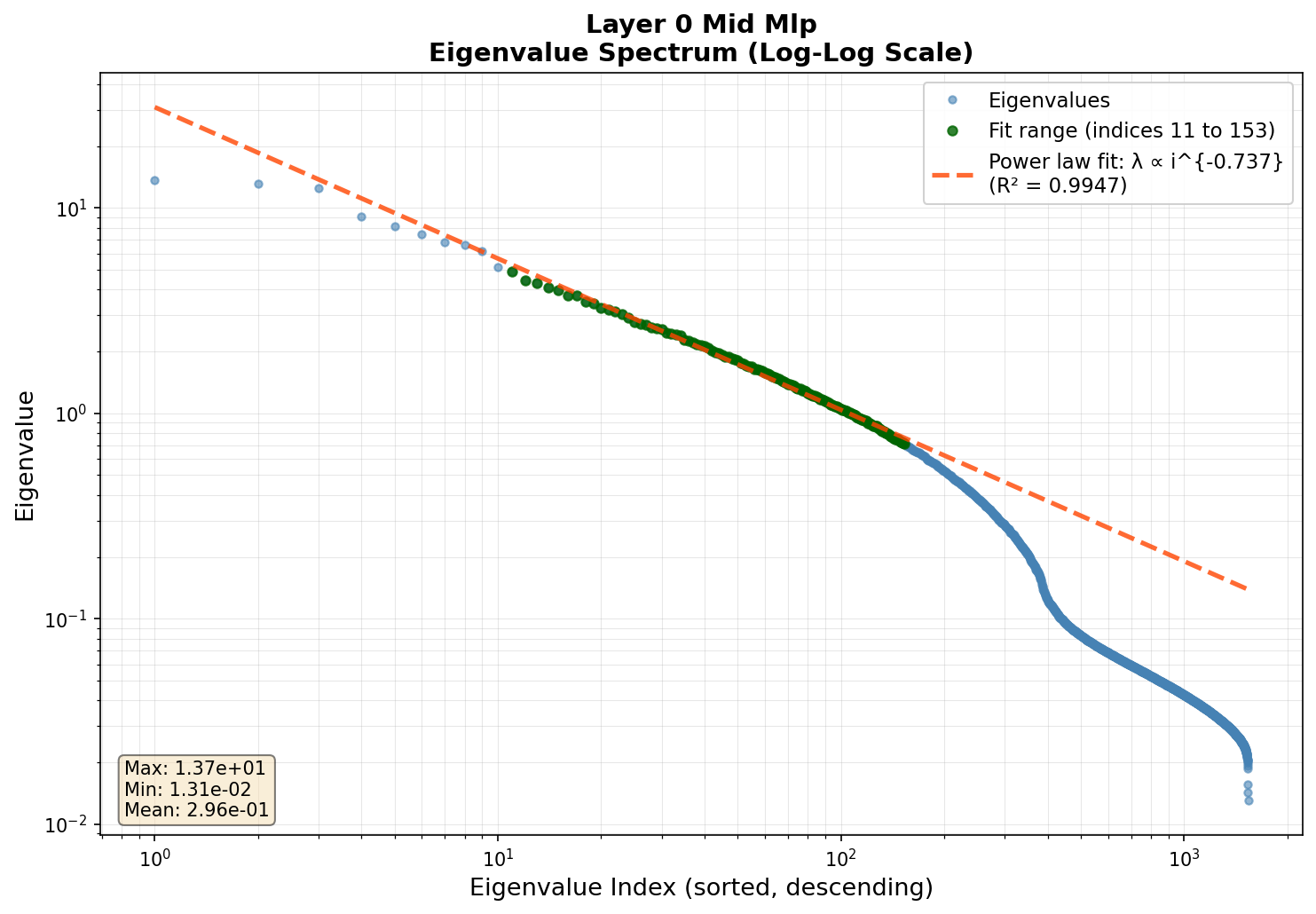}
\caption{Layer 0: Mid-MLP}
\end{subfigure}
\hfill
\begin{subfigure}[t]{0.24\textwidth}
\centering
\includegraphics[width=\textwidth]{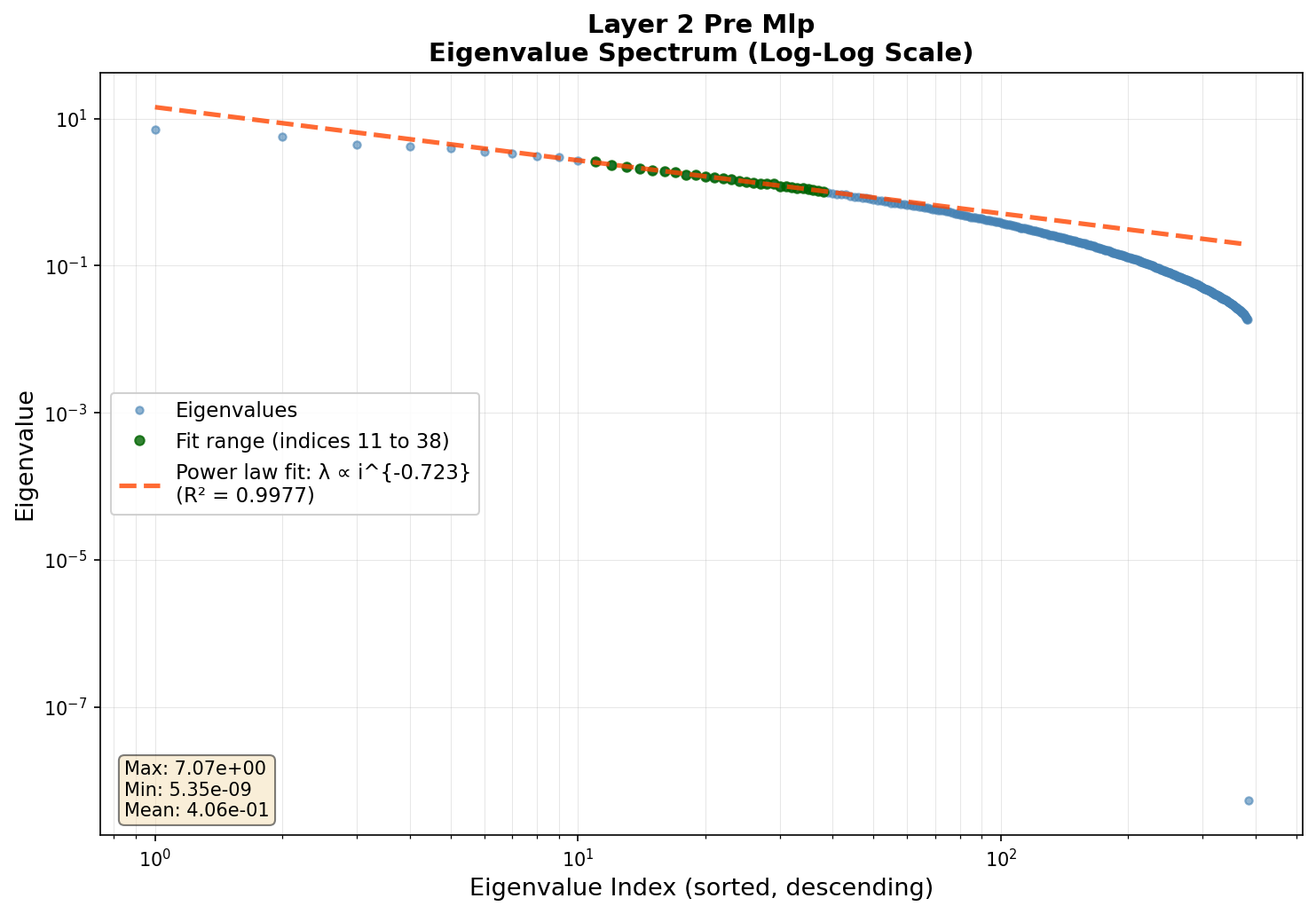}
\caption{Layer 2: Pre-MLP}
\end{subfigure}

\begin{subfigure}[t]{0.24\textwidth}
\centering
\includegraphics[width=\textwidth]{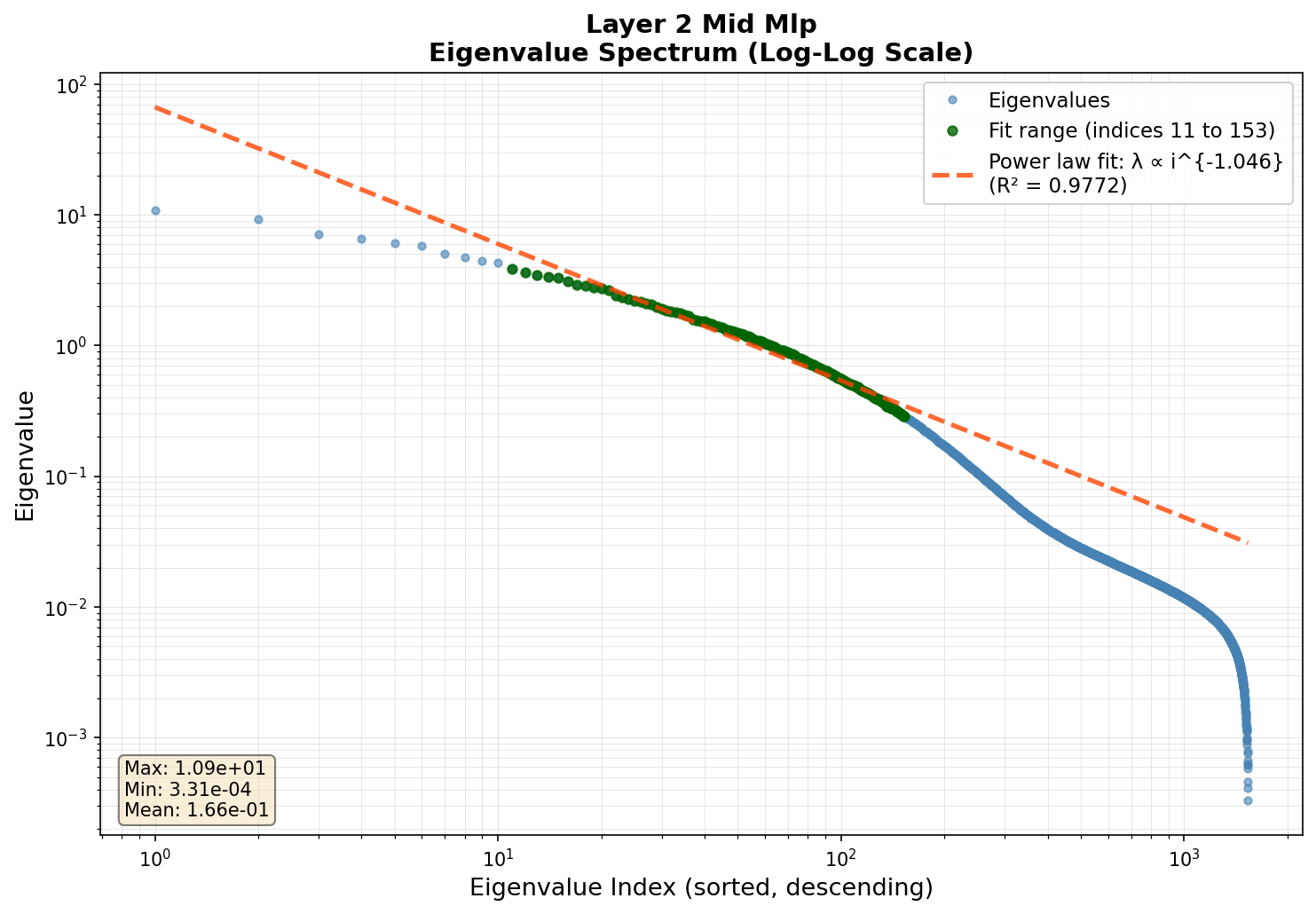}
\caption{Layer 2: Mid-MLP}
\end{subfigure}
\hfill
\begin{subfigure}[t]{0.24\textwidth}
\centering
\includegraphics[width=\textwidth]{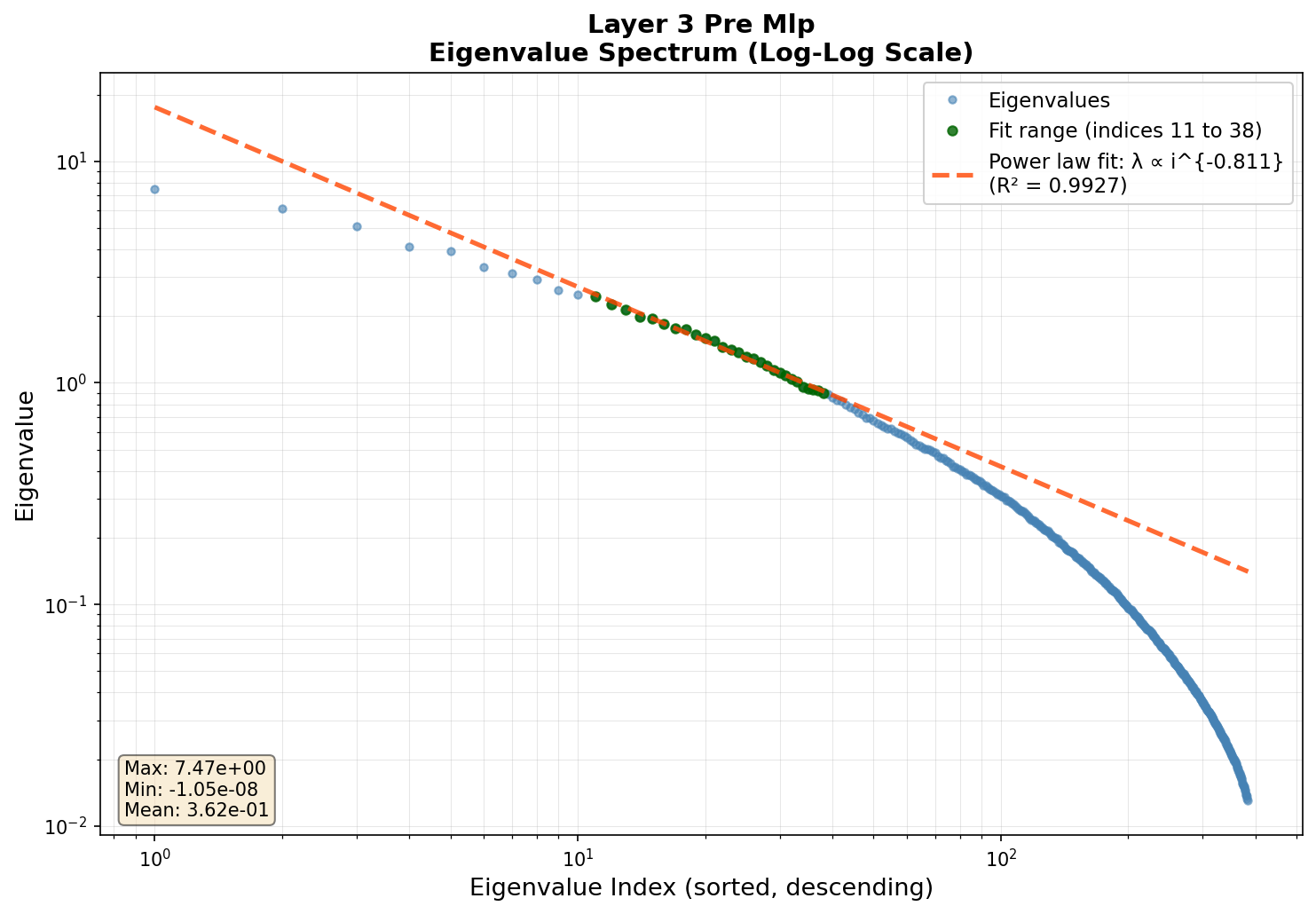}
\caption{Layer 3: Pre-MLP}
\end{subfigure}
\hfill
\begin{subfigure}[t]{0.24\textwidth}
\centering
\includegraphics[width=\textwidth]{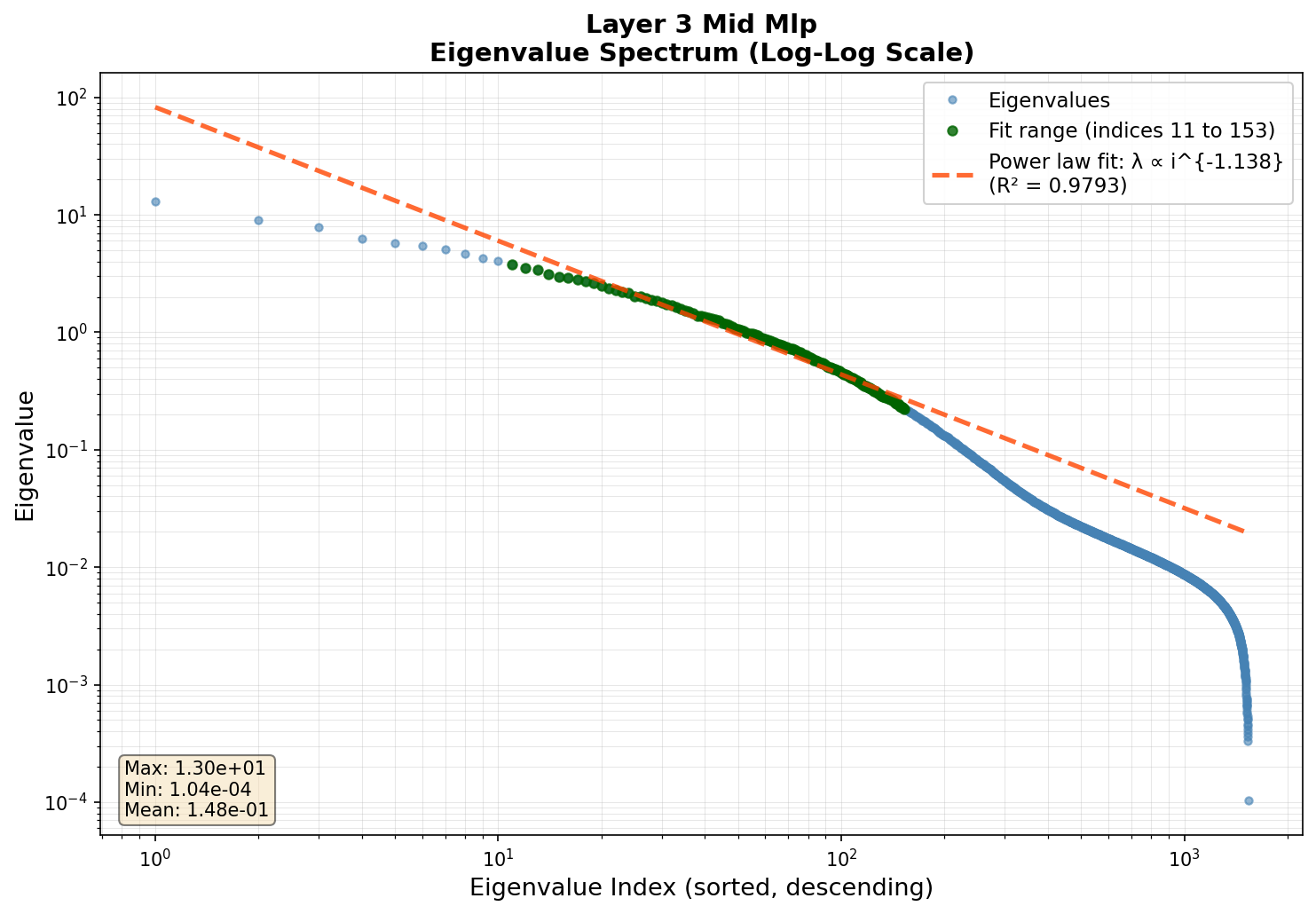}
\caption{Layer 3: Mid-MLP}
\end{subfigure}
\hfill
\begin{subfigure}[t]{0.24\textwidth}
\centering
\includegraphics[width=\textwidth]{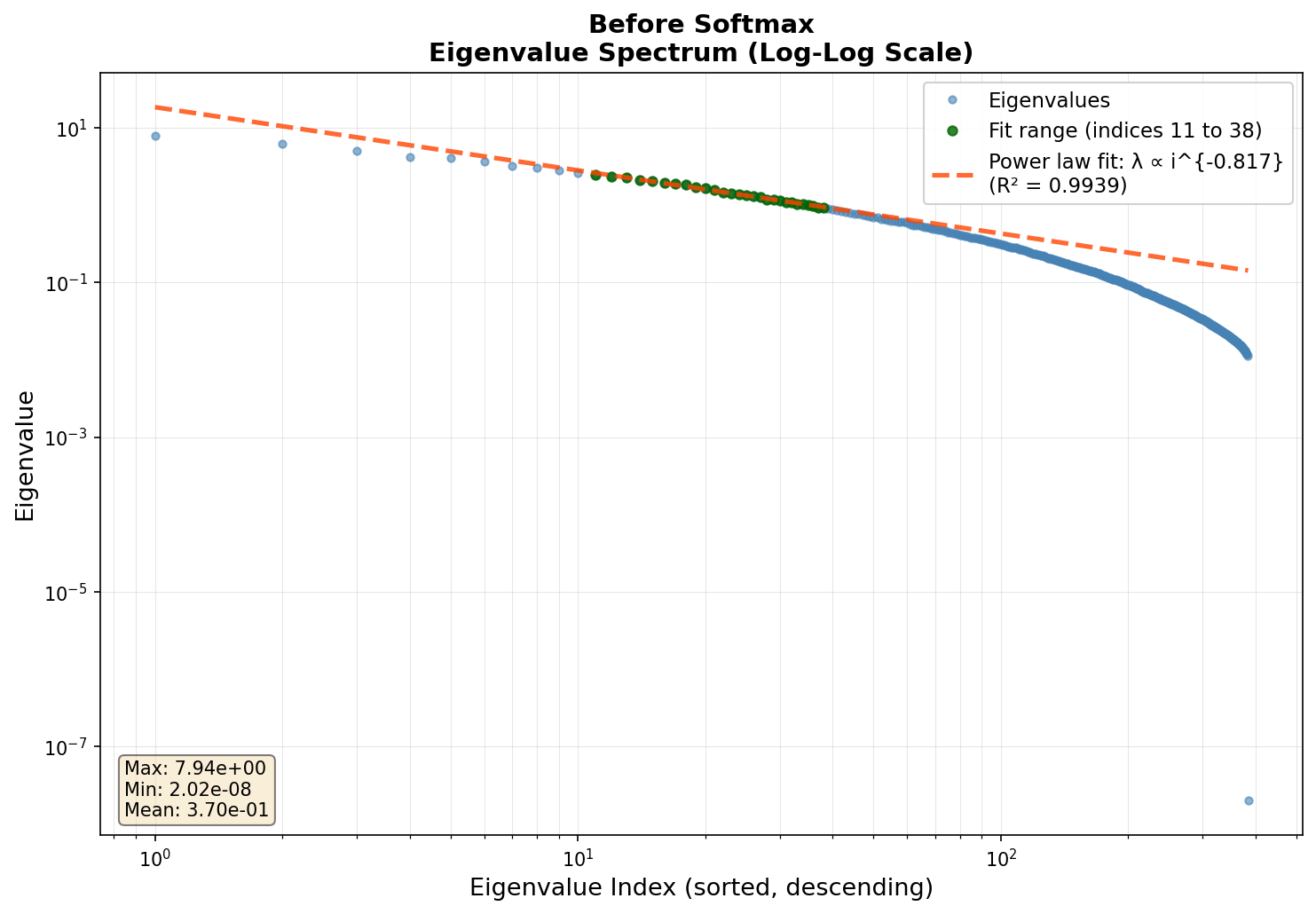}
\caption{Final: Before softmax}
\end{subfigure}
\caption{\textbf{Eigenvalue distributions at different layers before training (24-head model).} Each subplot shows the eigenvalue spectrum of the activation covariance matrix at a specific layer, with a power-law fit $\lambda_j \propto j^{-2\rho}$. The fitted exponent $2\rho$ is shown in each panel.}
\label{fig:layer_plots_before_training_head_24}
\end{figure}

\begin{figure}[h!]
\centering
\begin{subfigure}[t]{0.24\textwidth}
\centering
\includegraphics[width=\textwidth]{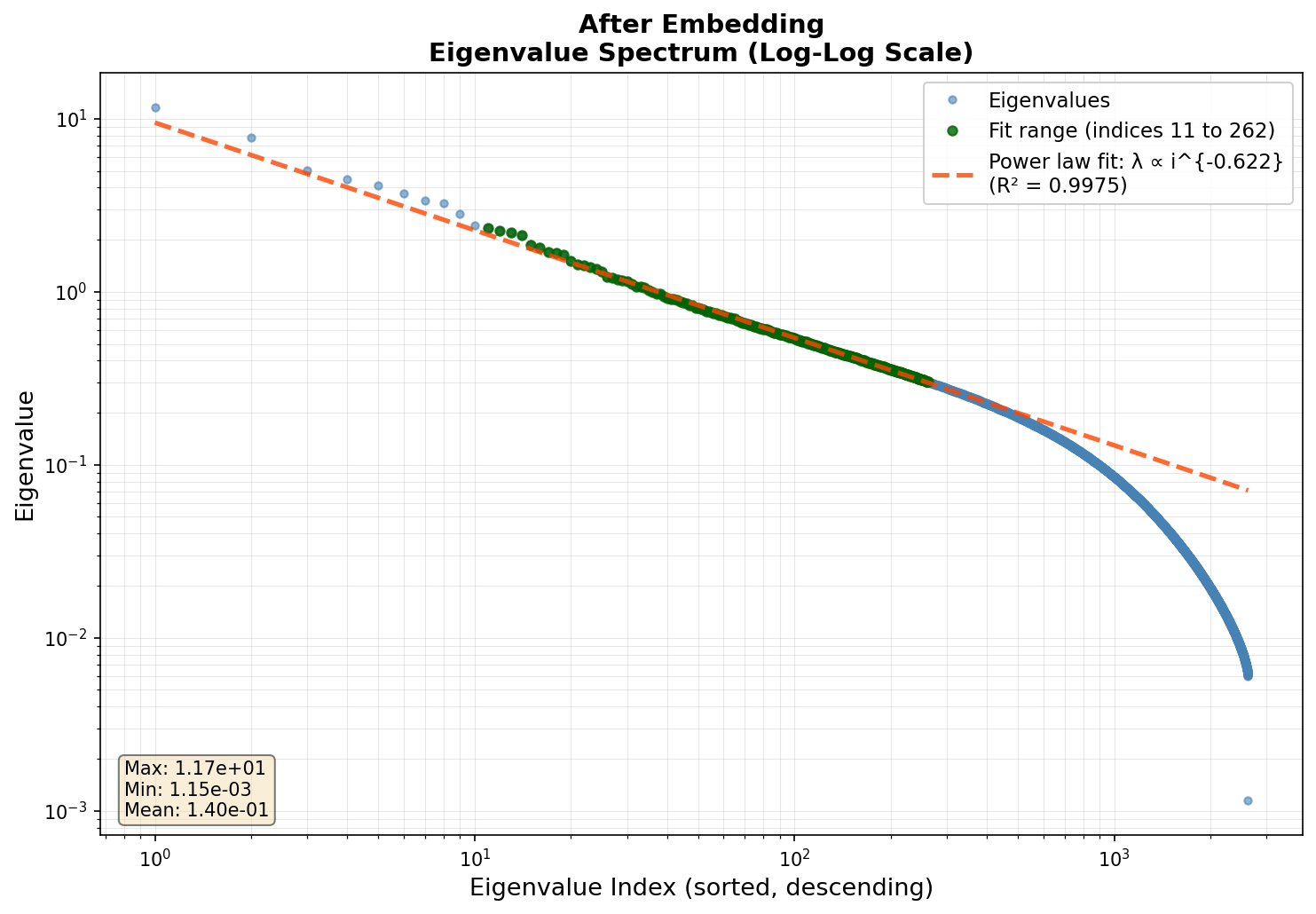}
\caption{Layer 0: After embedding}
\end{subfigure}
\hfill
\begin{subfigure}[t]{0.24\textwidth}
\centering
\includegraphics[width=\textwidth]{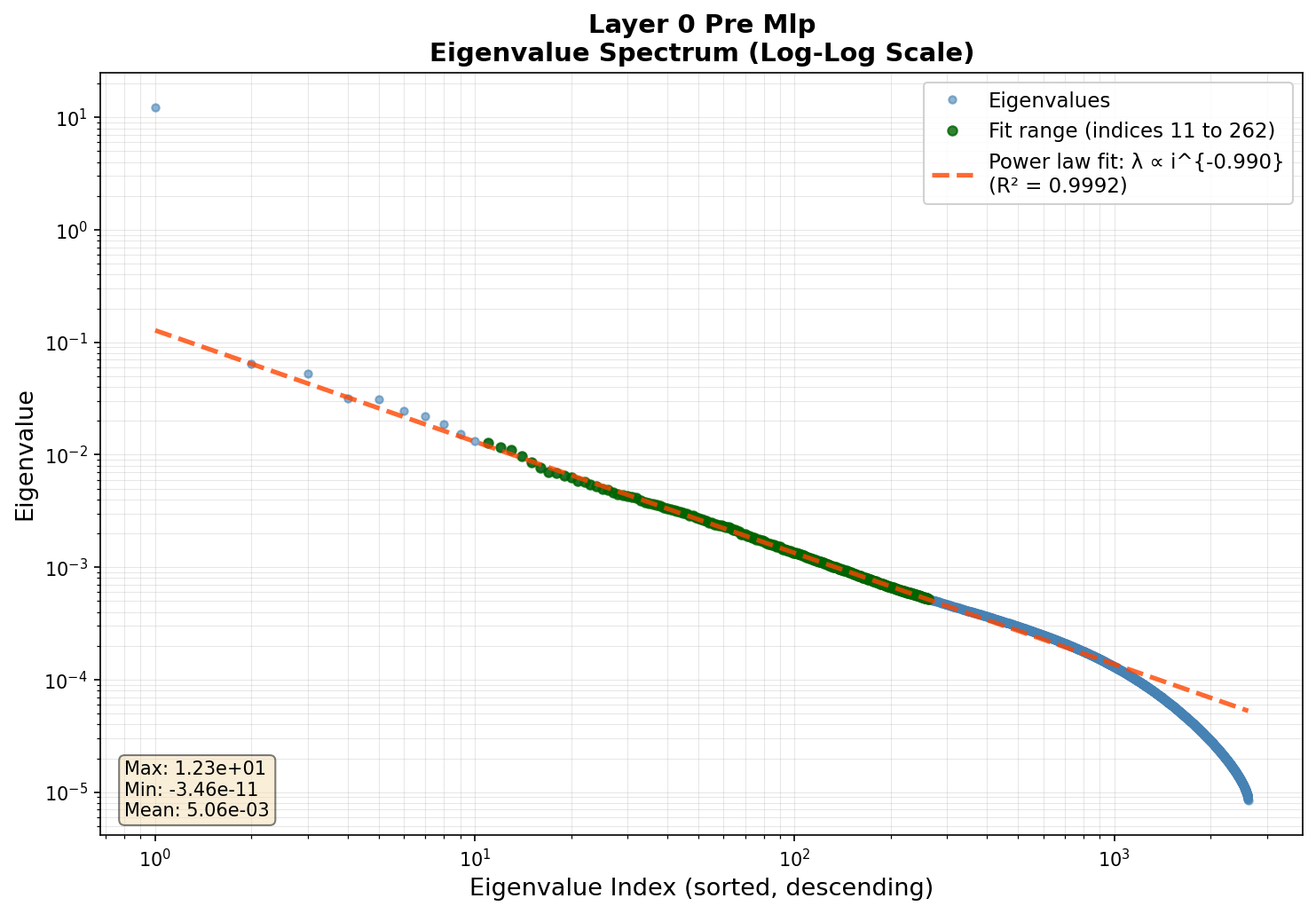}
\caption{Layer 0: Pre-MLP}
\end{subfigure}
\hfill
\begin{subfigure}[t]{0.24\textwidth}
\centering
\includegraphics[width=\textwidth]{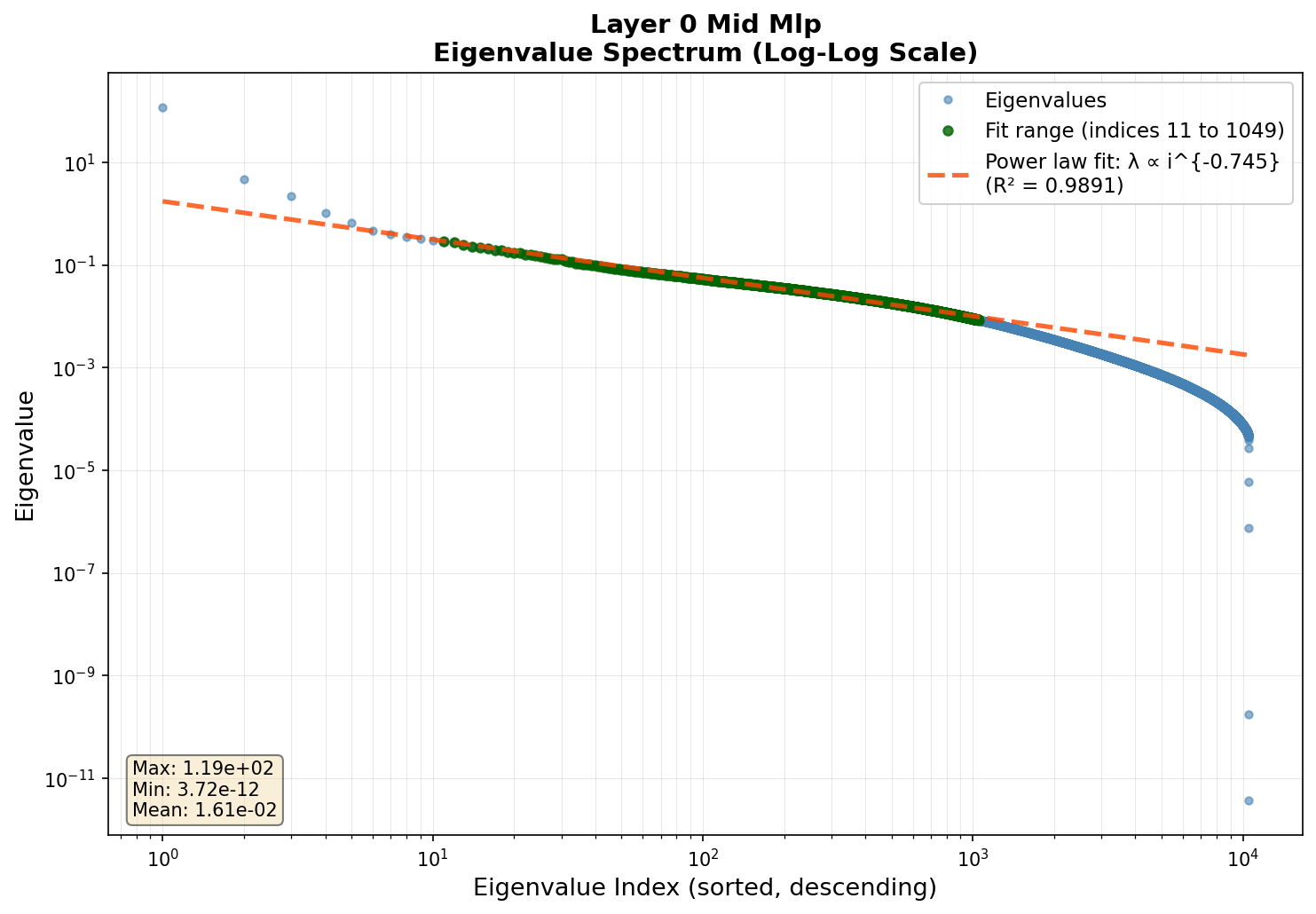}
\caption{Layer 0: Mid-MLP}
\end{subfigure}
\hfill
\begin{subfigure}[t]{0.24\textwidth}
\centering
\includegraphics[width=\textwidth]{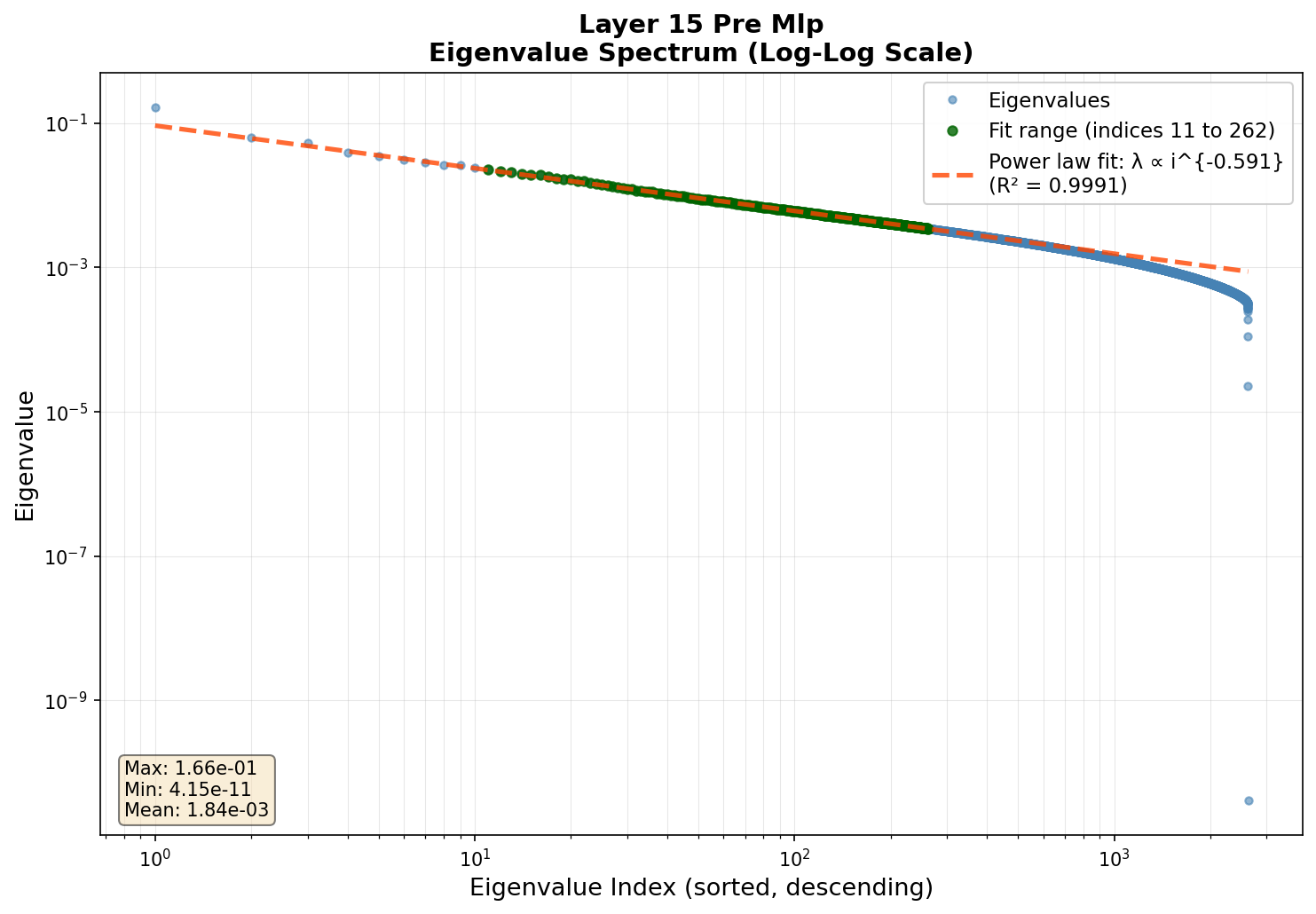}
\caption{Layer 15: Pre-MLP}
\end{subfigure}

\begin{subfigure}[t]{0.24\textwidth}
\centering
\includegraphics[width=\textwidth]{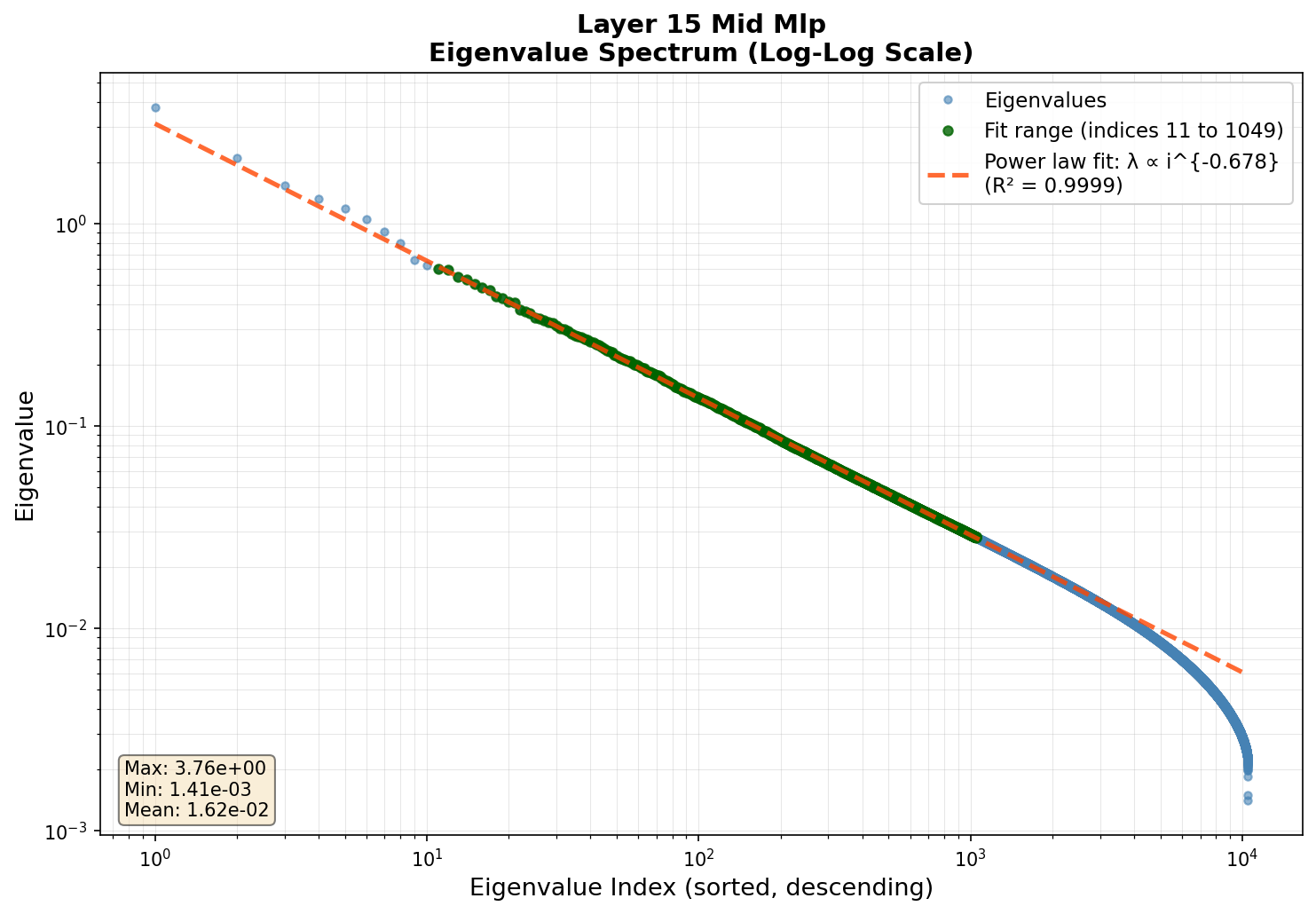}
\caption{Layer 15: Mid-MLP}
\end{subfigure}
\hfill
\begin{subfigure}[t]{0.24\textwidth}
\centering
\includegraphics[width=\textwidth]{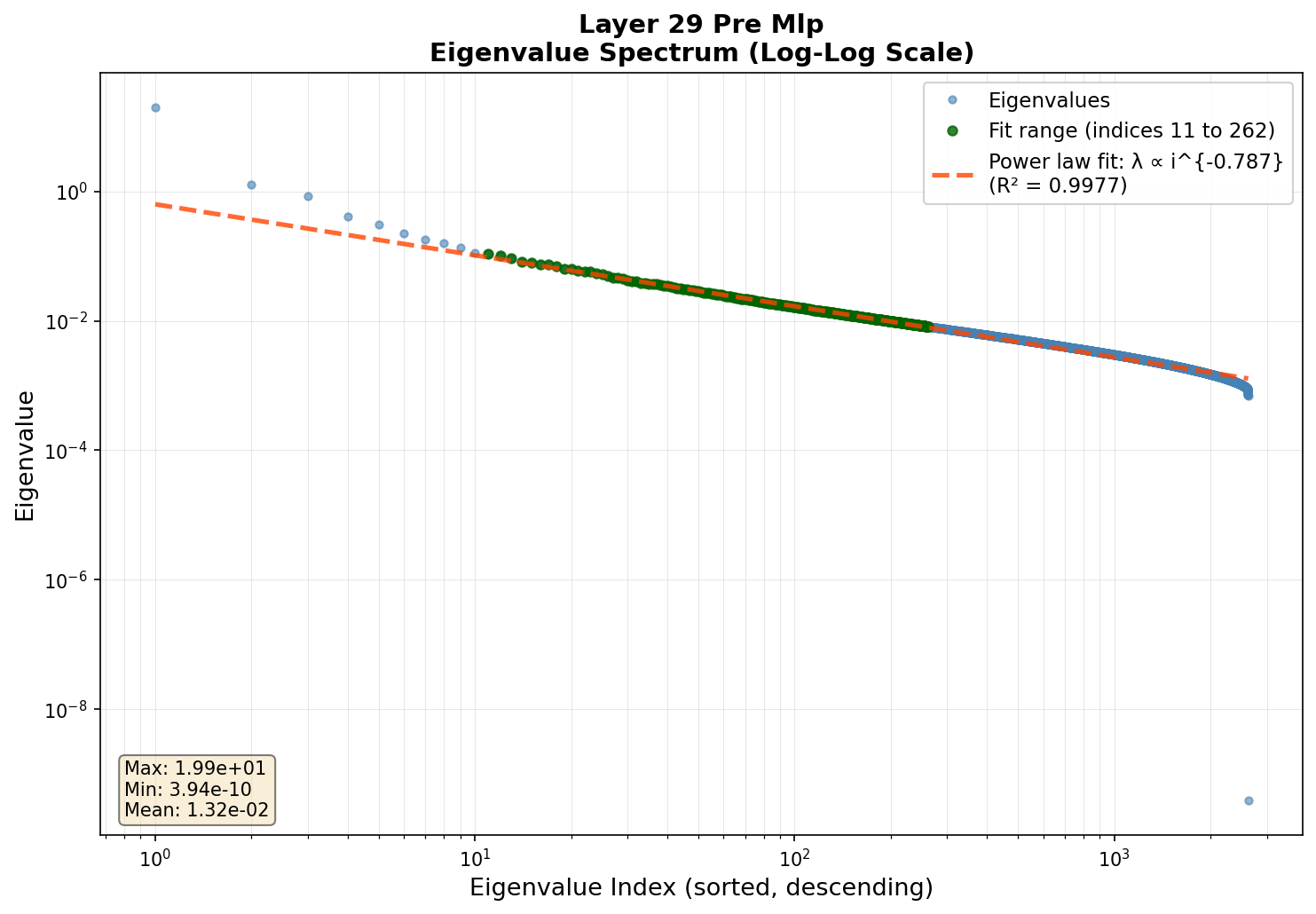}
\caption{Layer 29: Pre-MLP}
\end{subfigure}
\hfill
\begin{subfigure}[t]{0.24\textwidth}
\centering
\includegraphics[width=\textwidth]{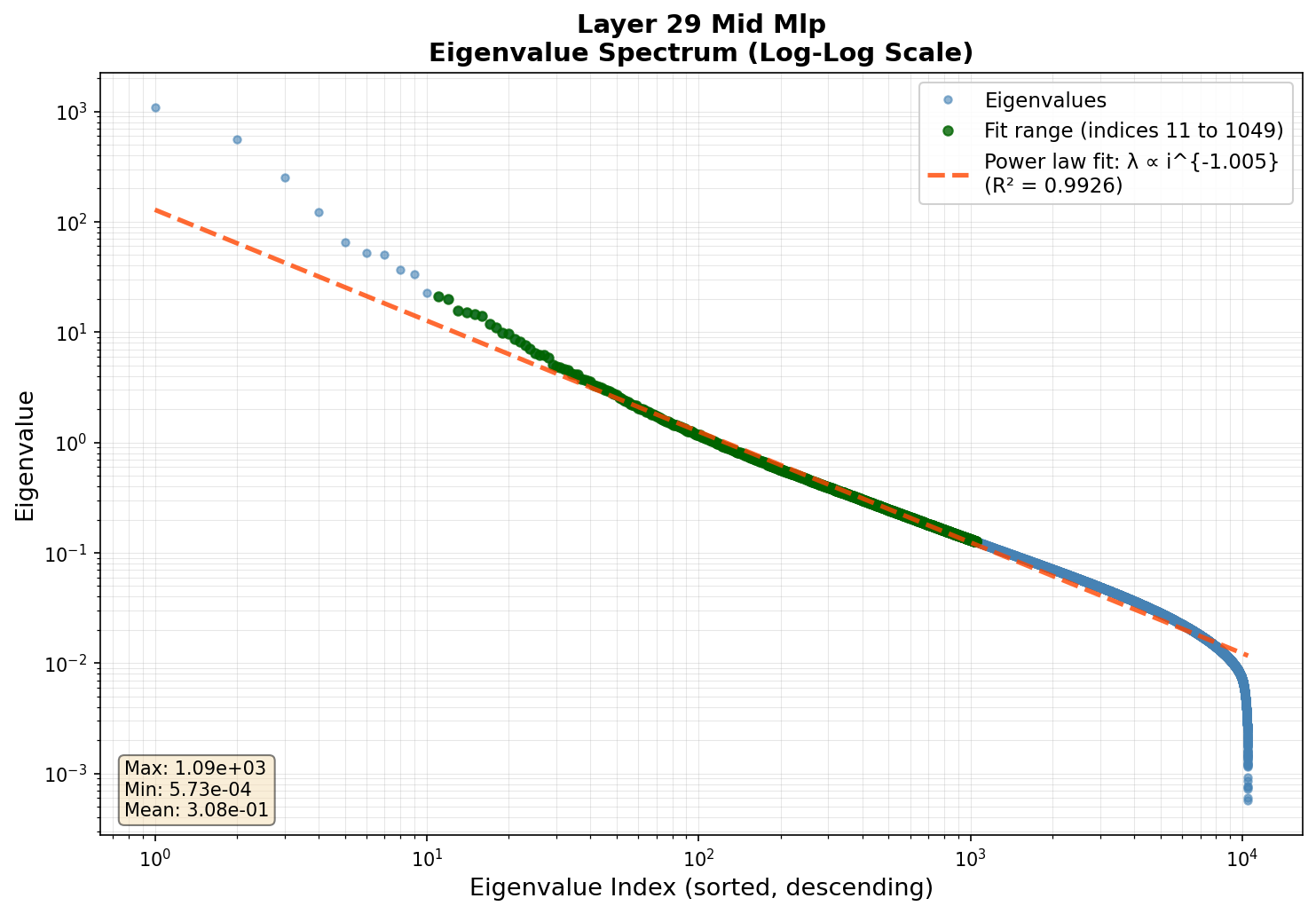}
\caption{Layer 29: Mid-MLP}
\end{subfigure}
\hfill
\begin{subfigure}[t]{0.24\textwidth}
\centering
\includegraphics[width=\textwidth]{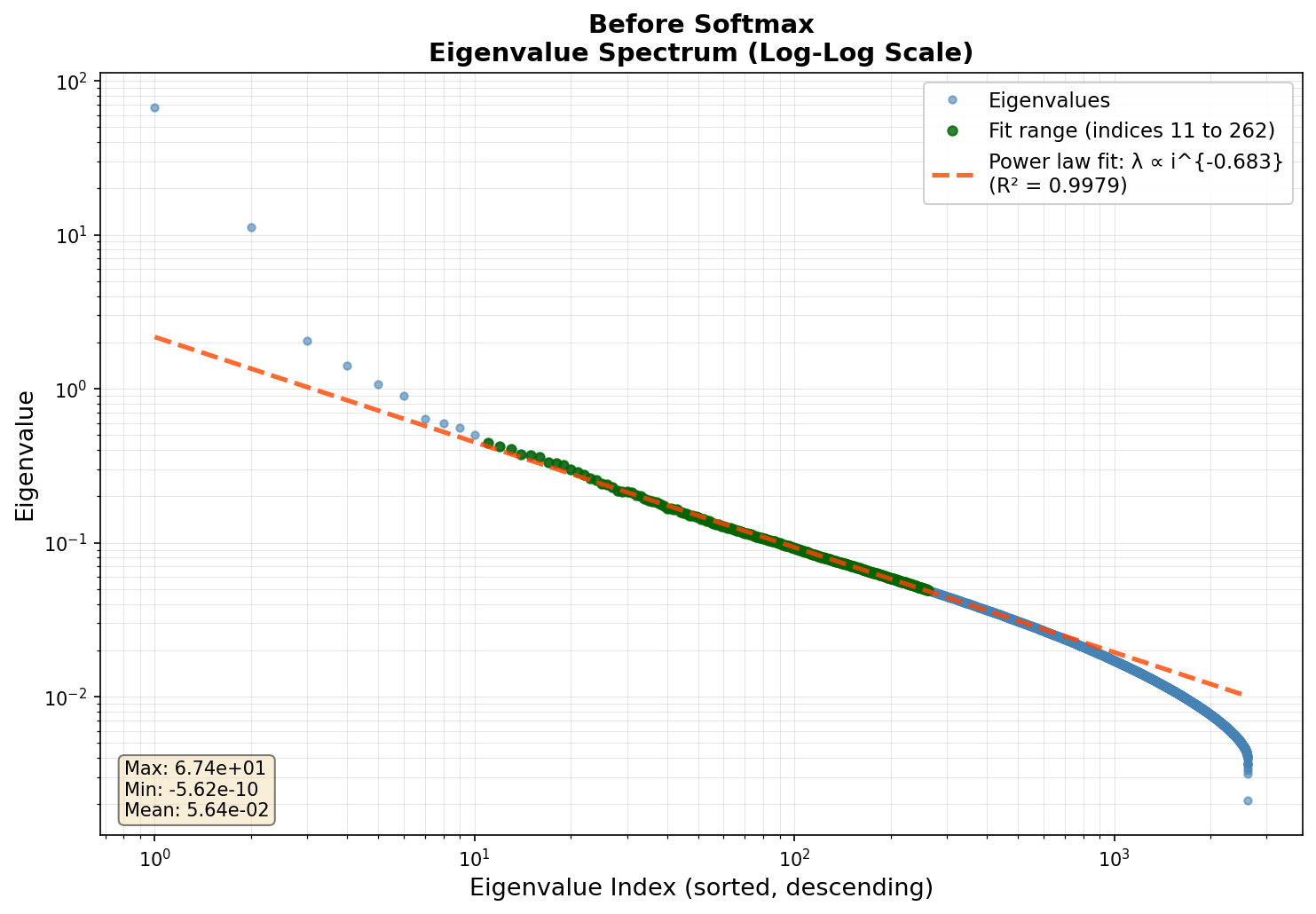}
\caption{Final: Before softmax}
\end{subfigure}
\caption{\textbf{Eigenvalue distributions at different layers after training (41-head model).} After training, the power-law exponents tend to be smaller (closer to the high-dimensional threshold) compared to initialization.}
\label{fig:layer_plots_after_training_head_41}
\end{figure}

\begin{figure}[h!]
\centering
\begin{subfigure}[t]{0.24\textwidth}
\centering
\includegraphics[width=\textwidth]{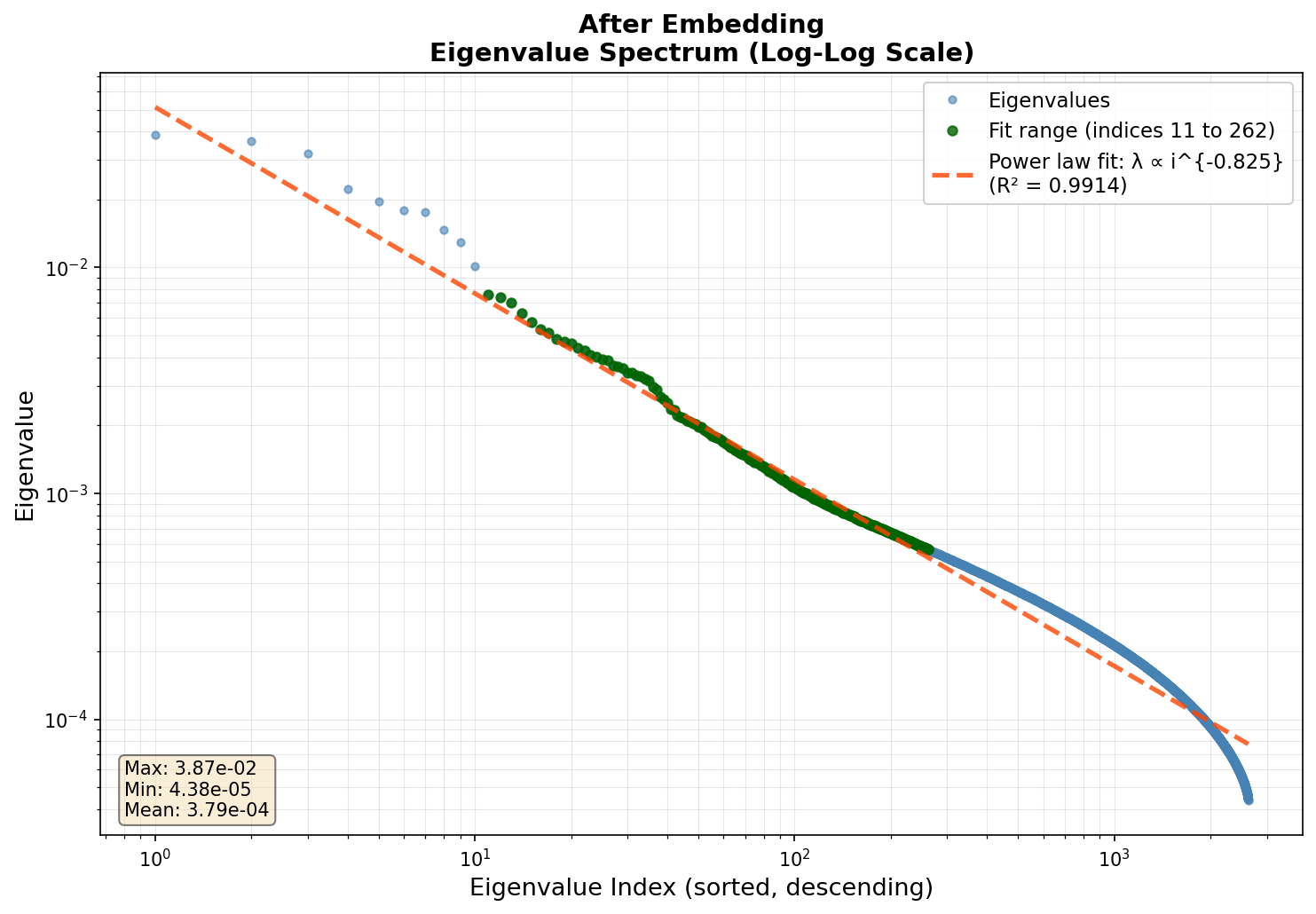}
\caption{Layer 0: After embedding}
\end{subfigure}
\hfill
\begin{subfigure}[t]{0.24\textwidth}
\centering
\includegraphics[width=\textwidth]{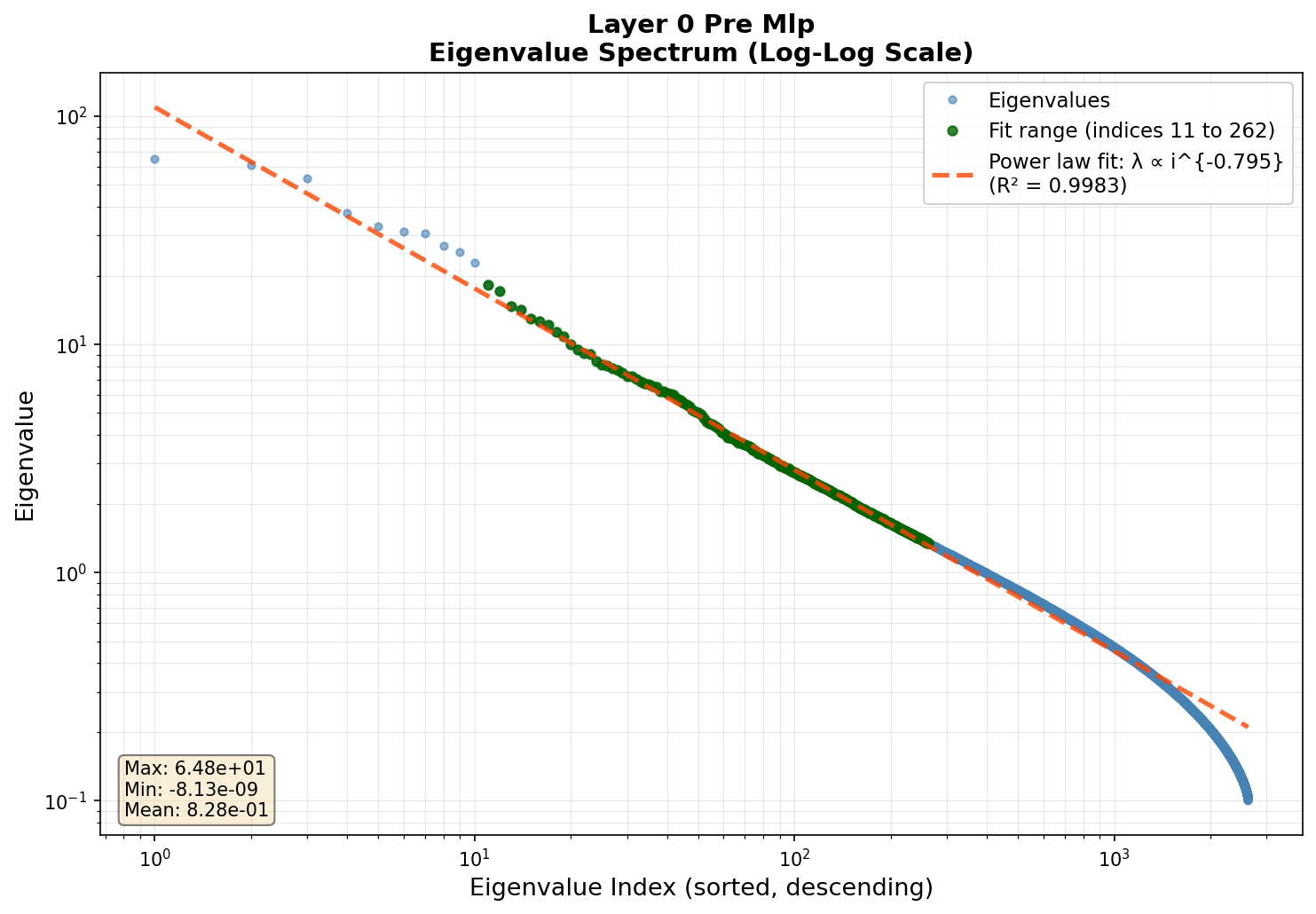}
\caption{Layer 0: Pre-MLP}
\end{subfigure}
\hfill
\begin{subfigure}[t]{0.24\textwidth}
\centering
\includegraphics[width=\textwidth]{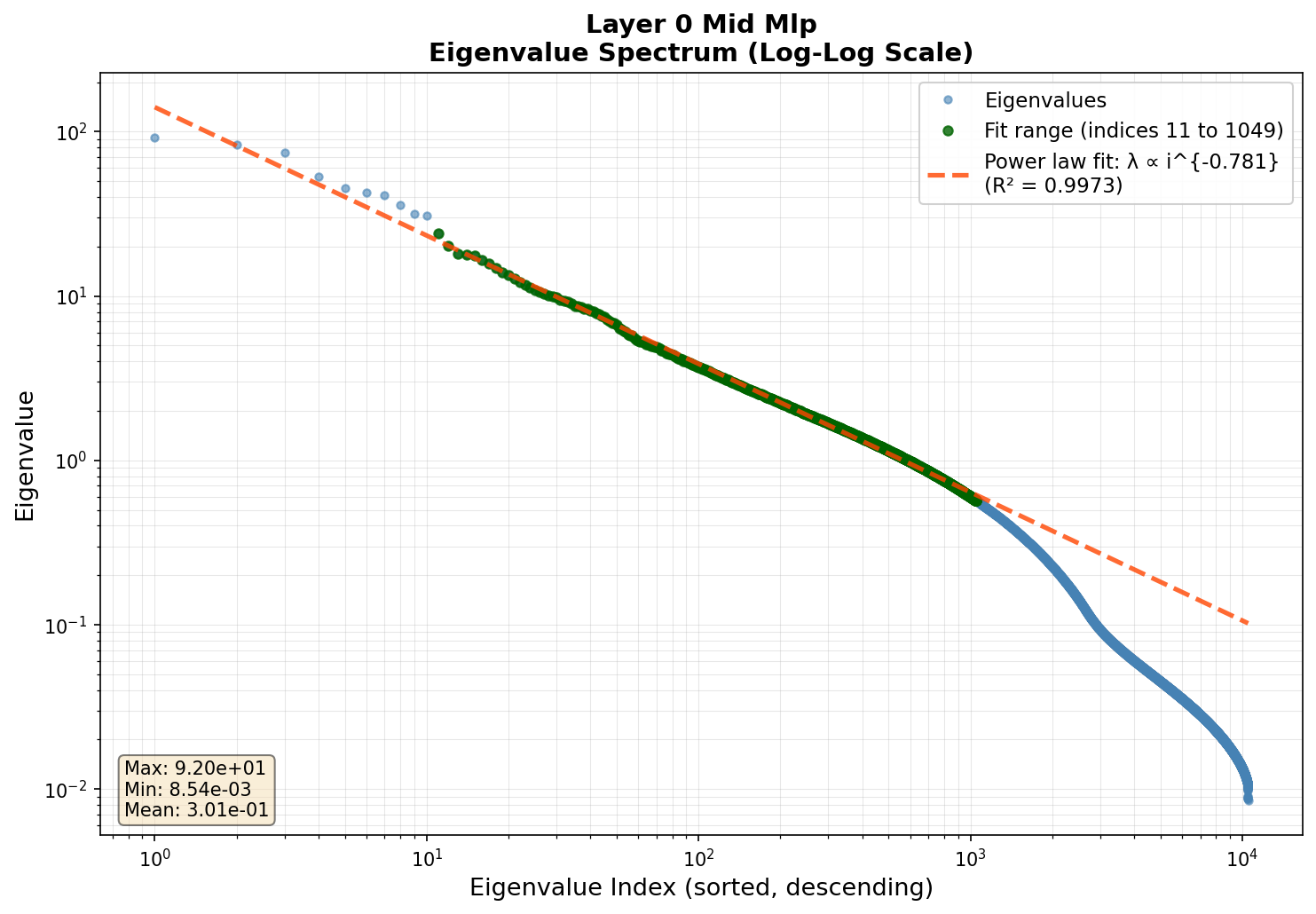}
\caption{Layer 0: Mid-MLP}
\end{subfigure}
\hfill
\begin{subfigure}[t]{0.24\textwidth}
\centering
\includegraphics[width=\textwidth]{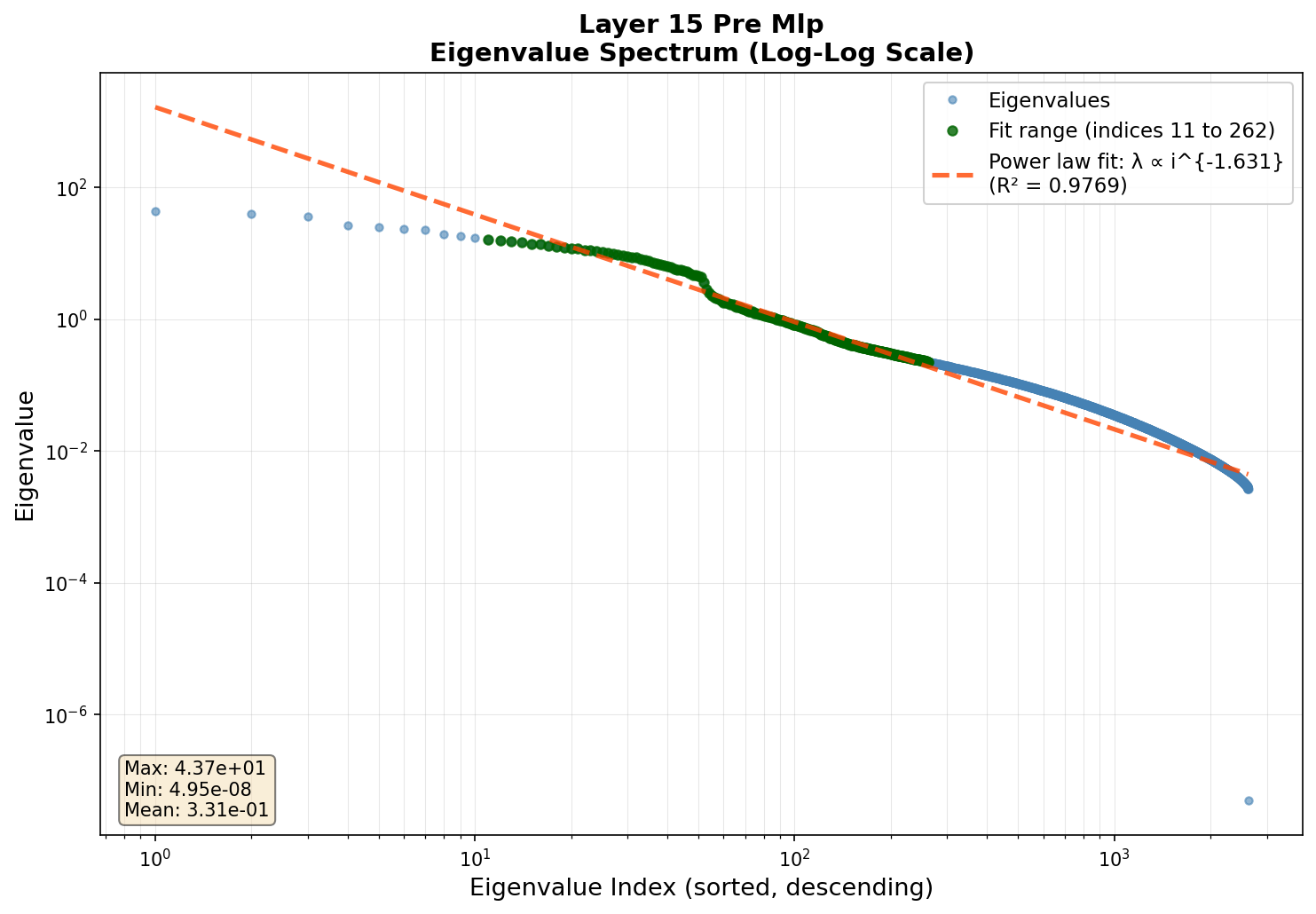}
\caption{Layer 15: Pre-MLP}
\end{subfigure}

\begin{subfigure}[t]{0.24\textwidth}
\centering
\includegraphics[width=\textwidth]{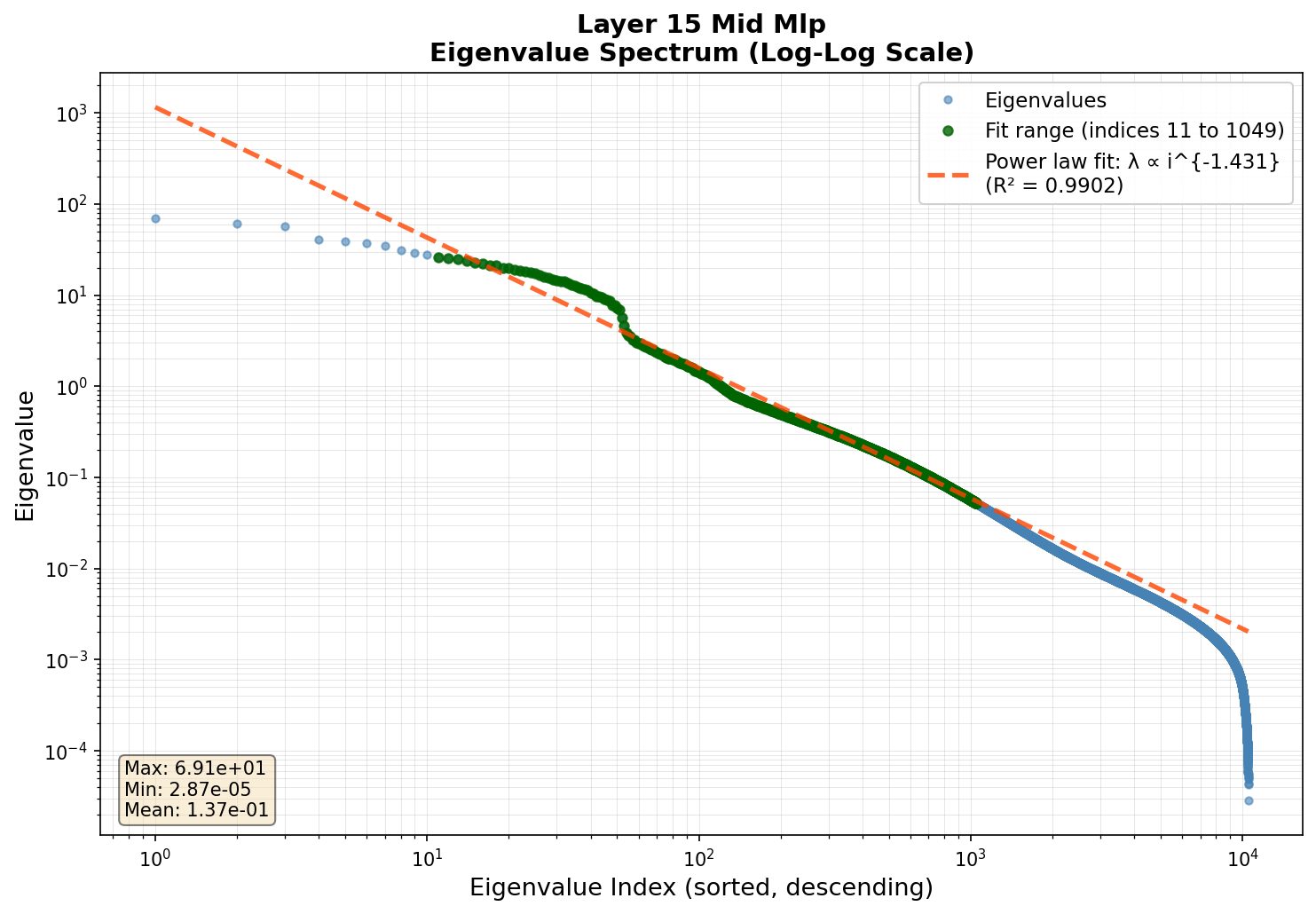}
\caption{Layer 15: Mid-MLP}
\end{subfigure}
\hfill
\begin{subfigure}[t]{0.24\textwidth}
\centering
\includegraphics[width=\textwidth]{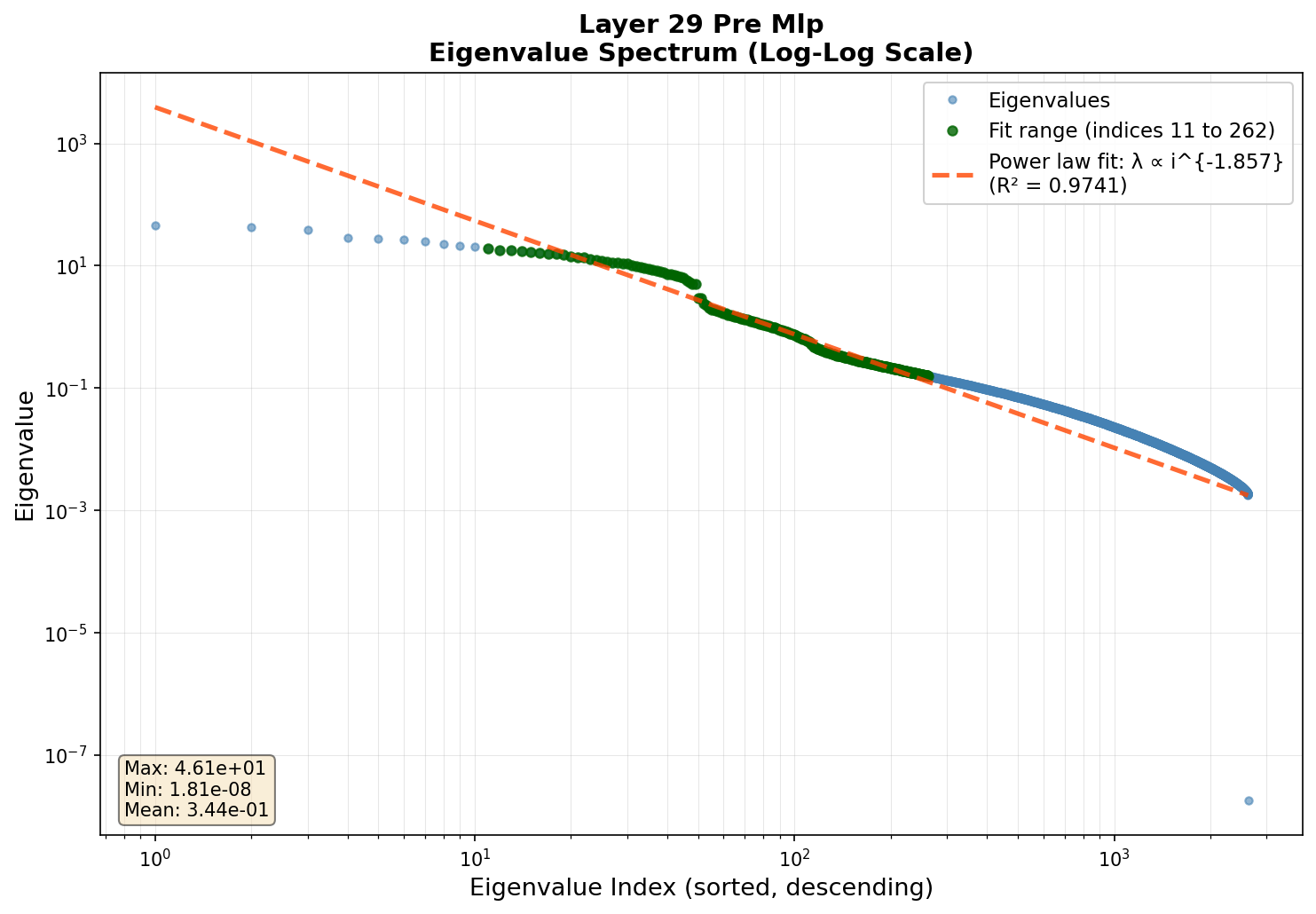}
\caption{Layer 29: Pre-MLP}
\end{subfigure}
\hfill
\begin{subfigure}[t]{0.24\textwidth}
\centering
\includegraphics[width=\textwidth]{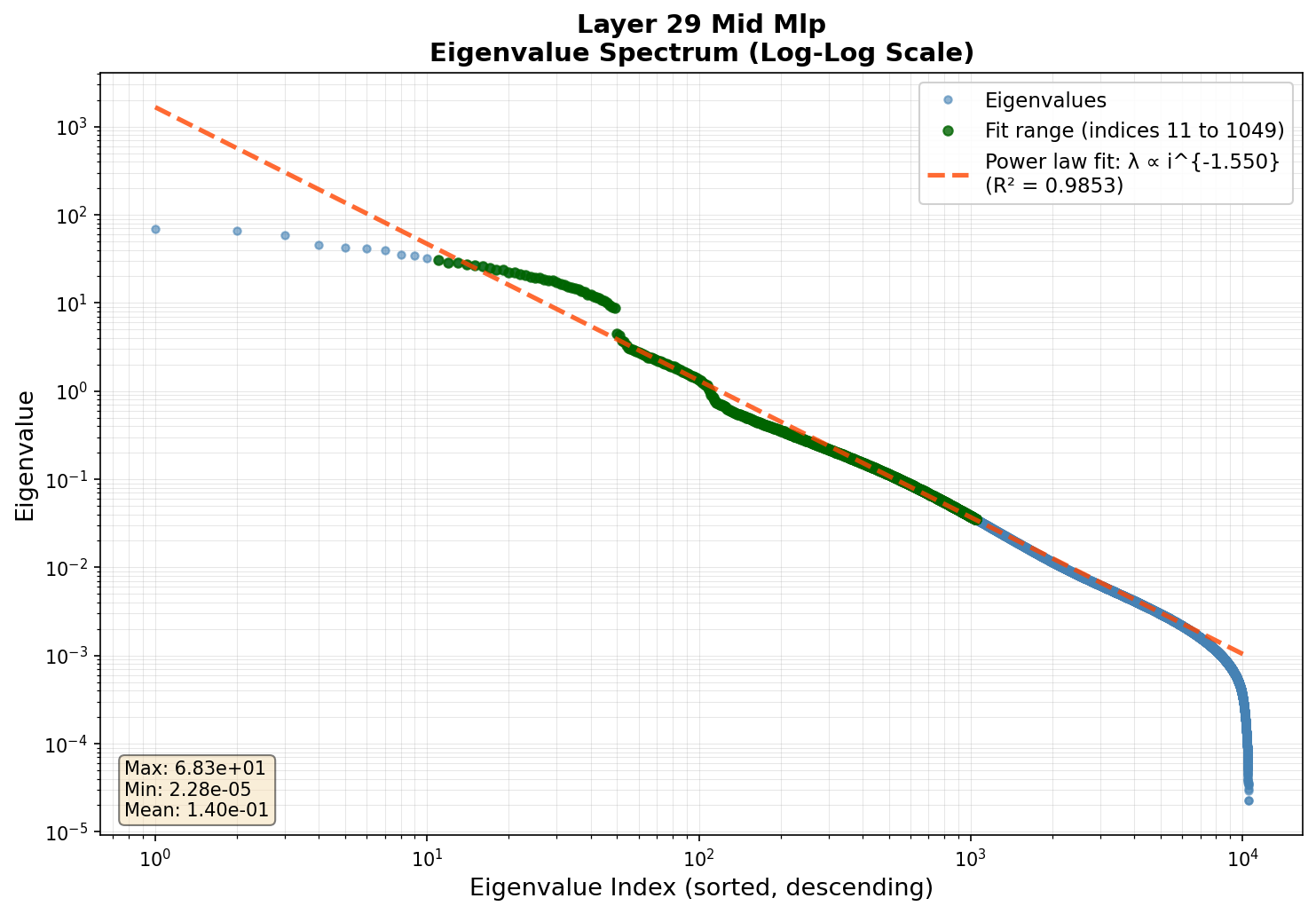}
\caption{Layer 29: Mid-MLP}
\end{subfigure}
\hfill
\begin{subfigure}[t]{0.24\textwidth}
\centering
\includegraphics[width=\textwidth]{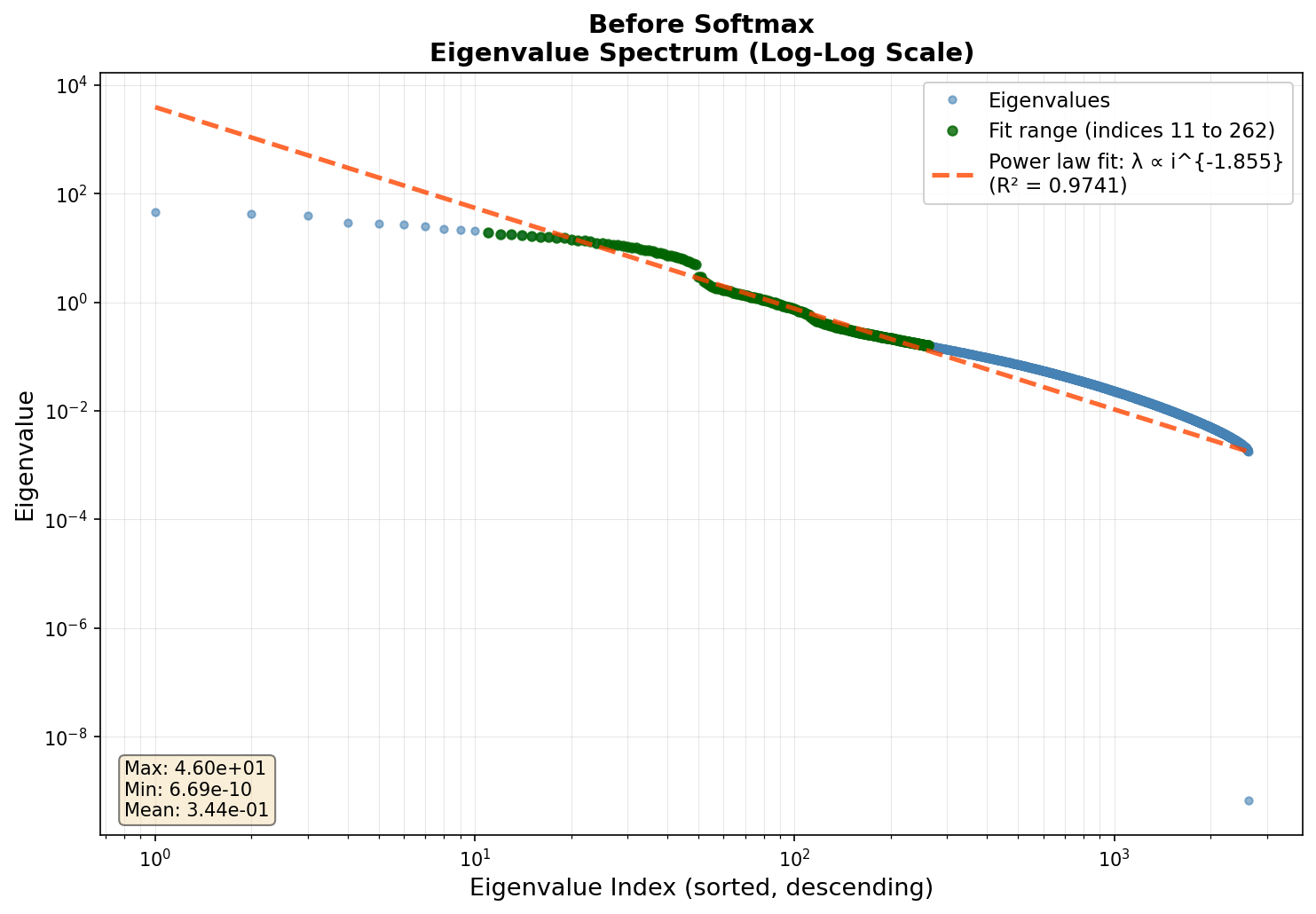}
\caption{Final: Before softmax}
\end{subfigure}
\caption{\textbf{Eigenvalue distributions at different layers before training (41-head model).} Deeper layers exhibit larger power-law exponents at initialization, suggesting they operate in a more low-dimensional regime.}
\label{fig:layer_plots_before_training_head_41}
\end{figure}

\begin{figure}[h!]
\centering
\begin{subfigure}[t]{0.24\textwidth}
\centering
\includegraphics[width=\textwidth]{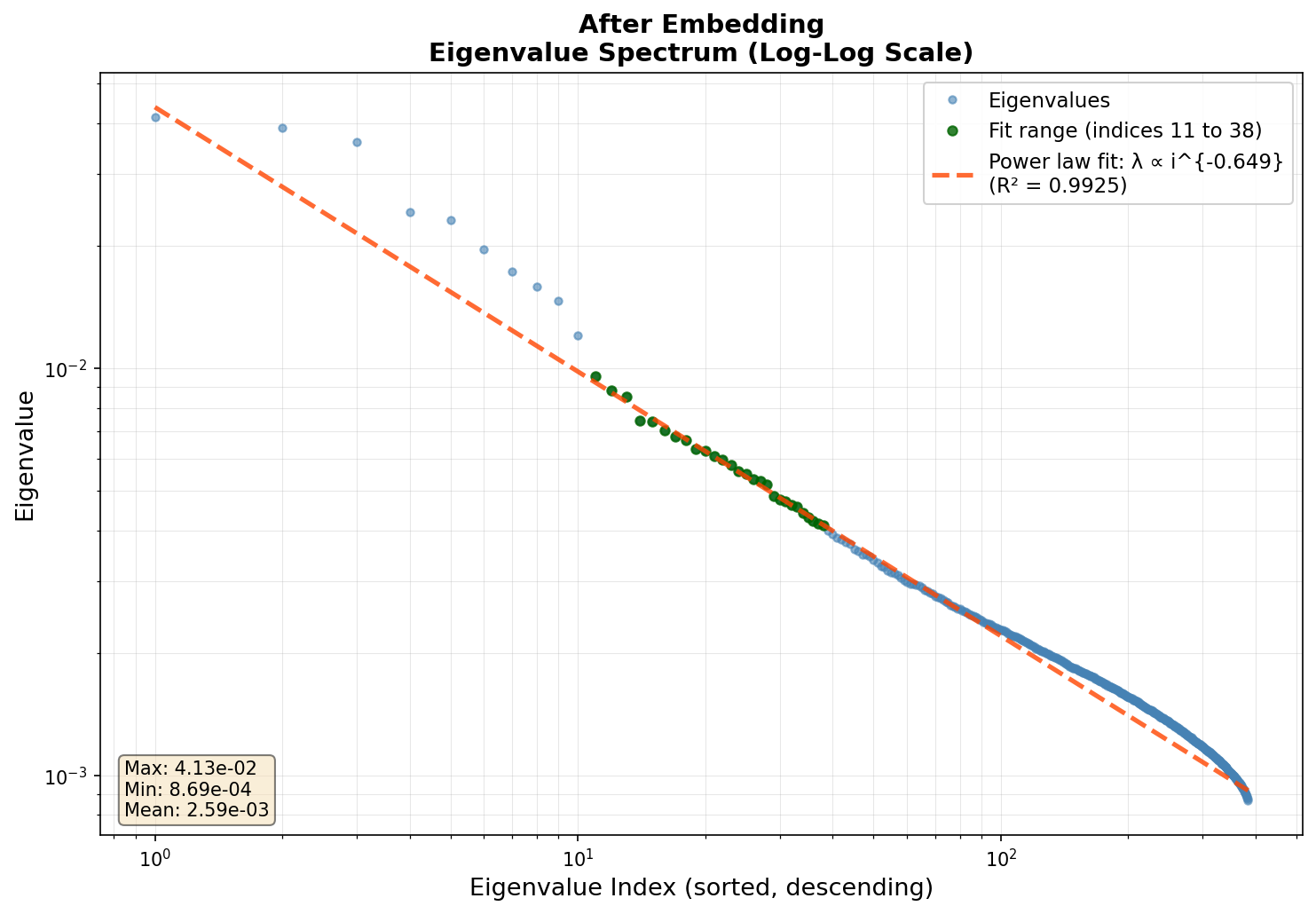}
\caption{Layer 0: After embedding}
\end{subfigure}
\hfill
\begin{subfigure}[t]{0.24\textwidth}
\centering
\includegraphics[width=\textwidth]{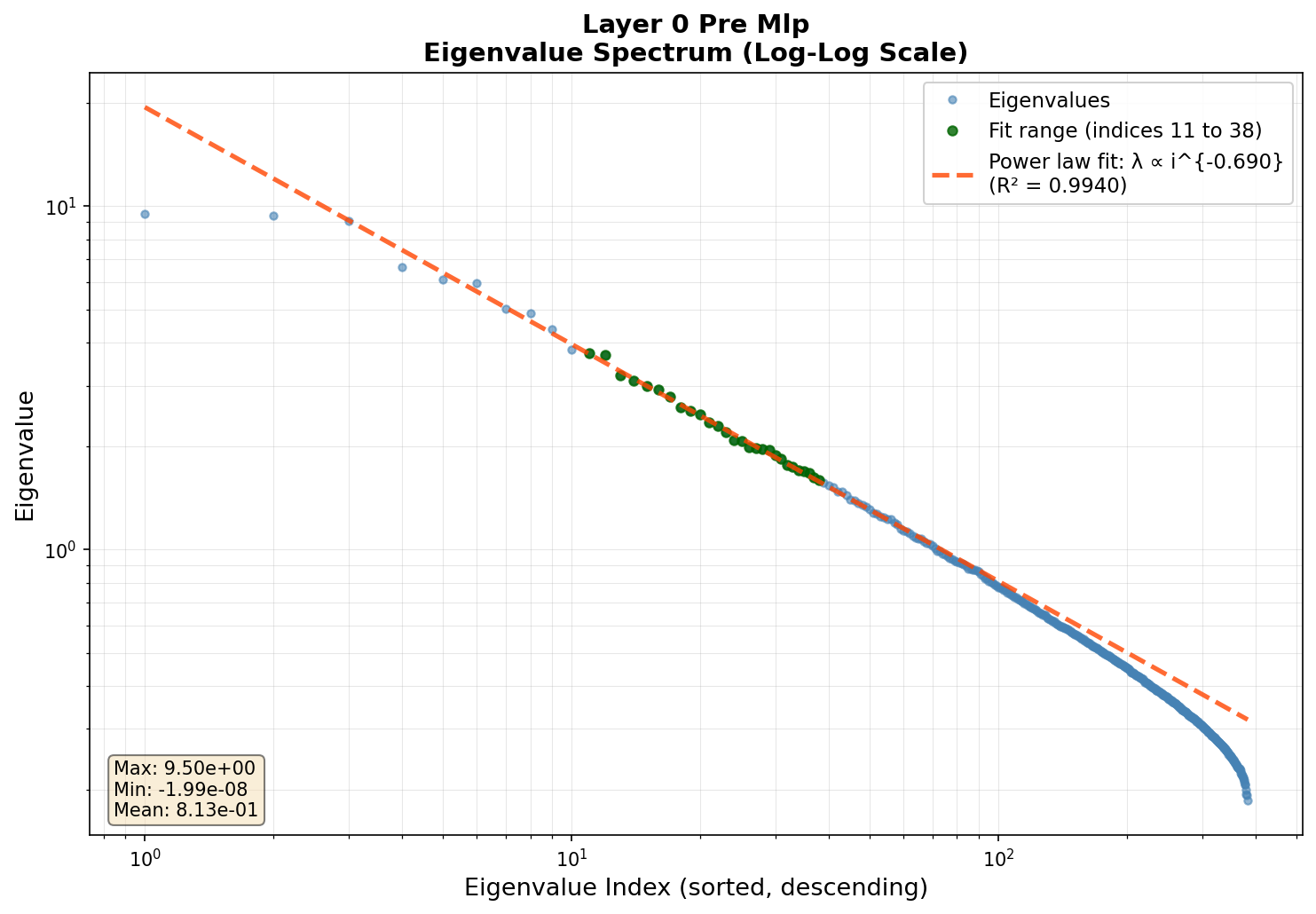}
\caption{Layer 0: Pre-MLP}
\end{subfigure}
\hfill
\begin{subfigure}[t]{0.24\textwidth}
\centering
\includegraphics[width=\textwidth]{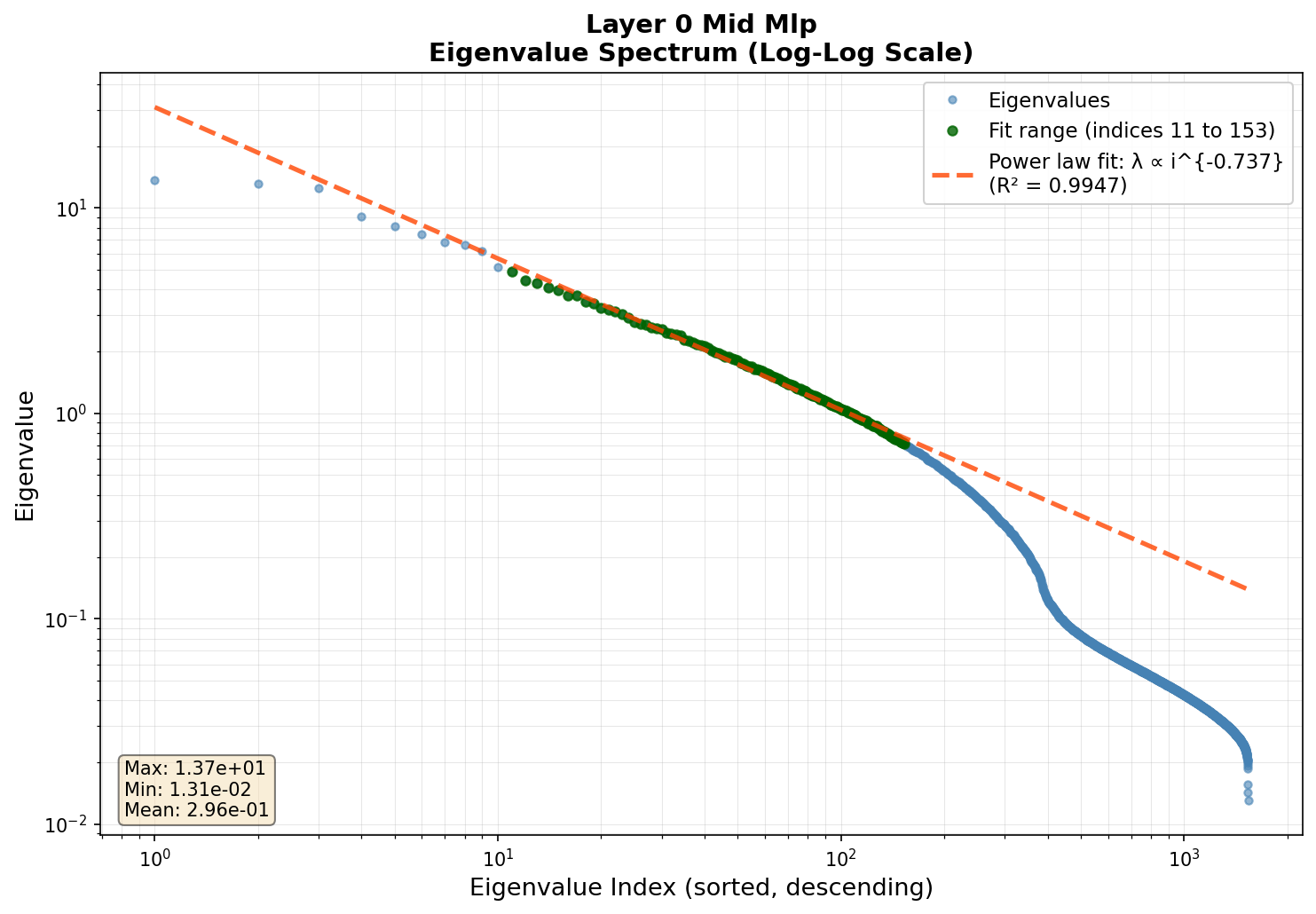}
\caption{Layer 0: Mid-MLP}
\end{subfigure}
\hfill
\begin{subfigure}[t]{0.24\textwidth}
\centering
\includegraphics[width=\textwidth]{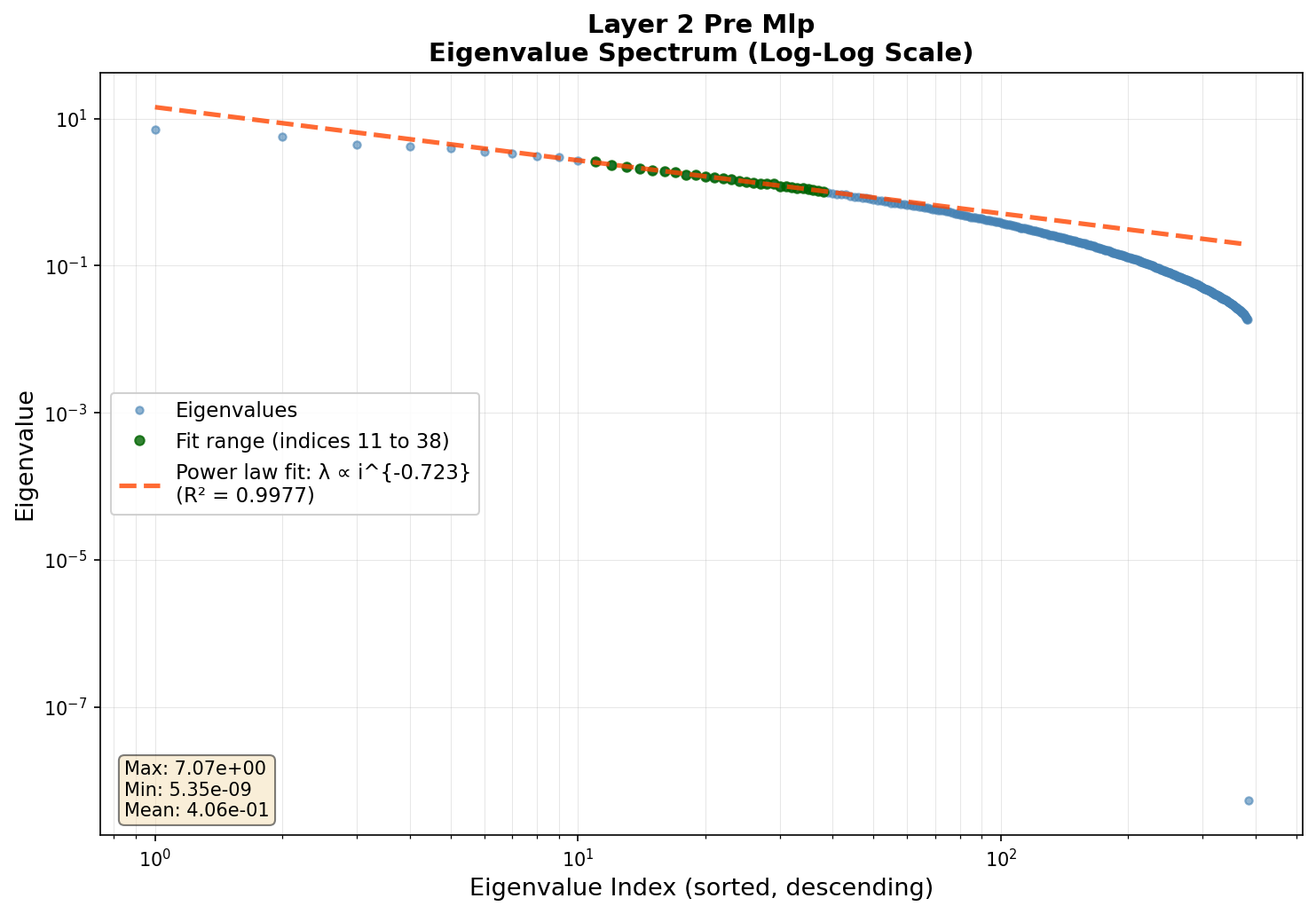}
\caption{Layer 2: Pre-MLP}
\end{subfigure}

\begin{subfigure}[t]{0.24\textwidth}
\centering
\includegraphics[width=\textwidth]{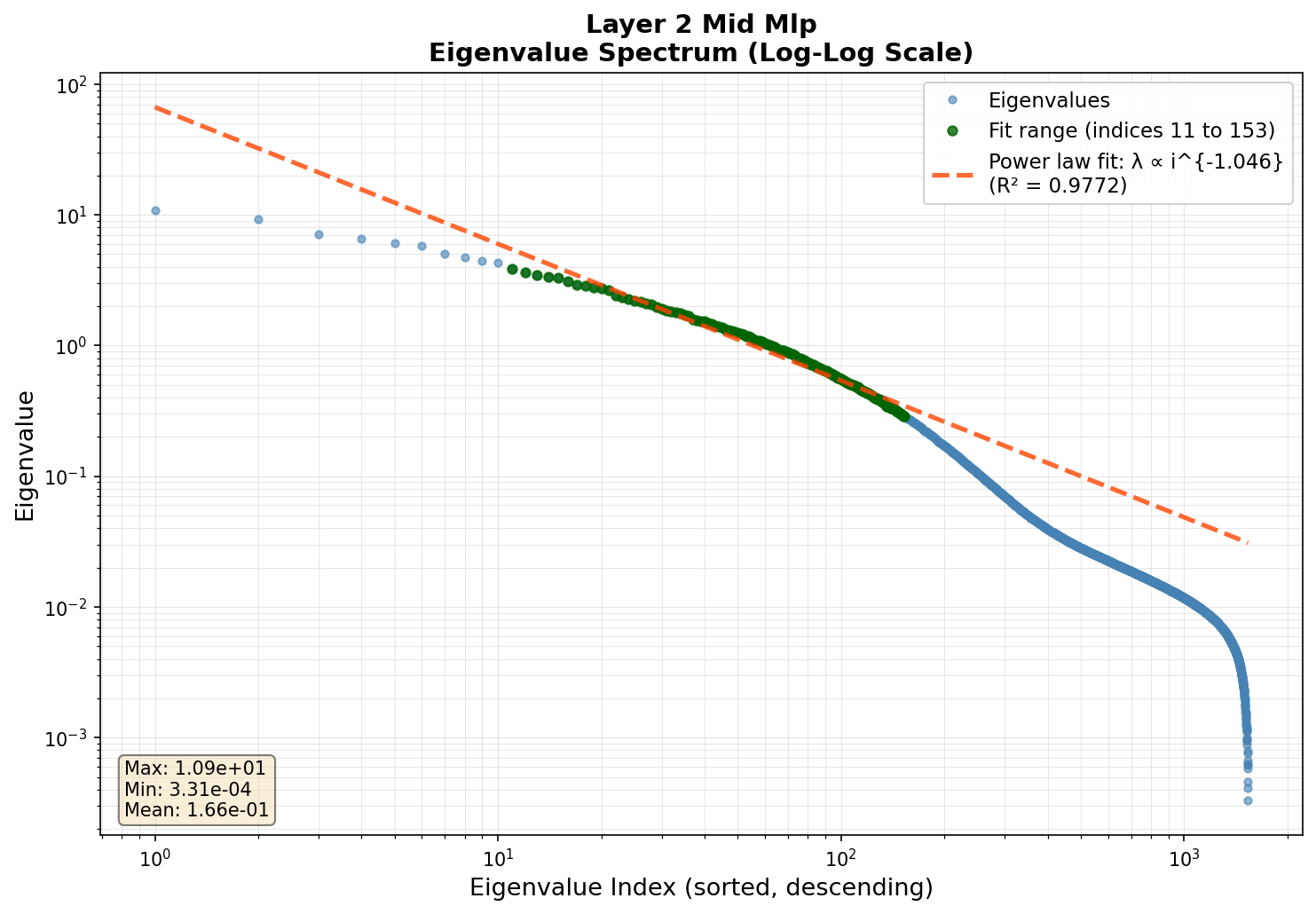}
\caption{Layer 2: Mid-MLP}
\end{subfigure}
\hfill
\begin{subfigure}[t]{0.24\textwidth}
\centering
\includegraphics[width=\textwidth]{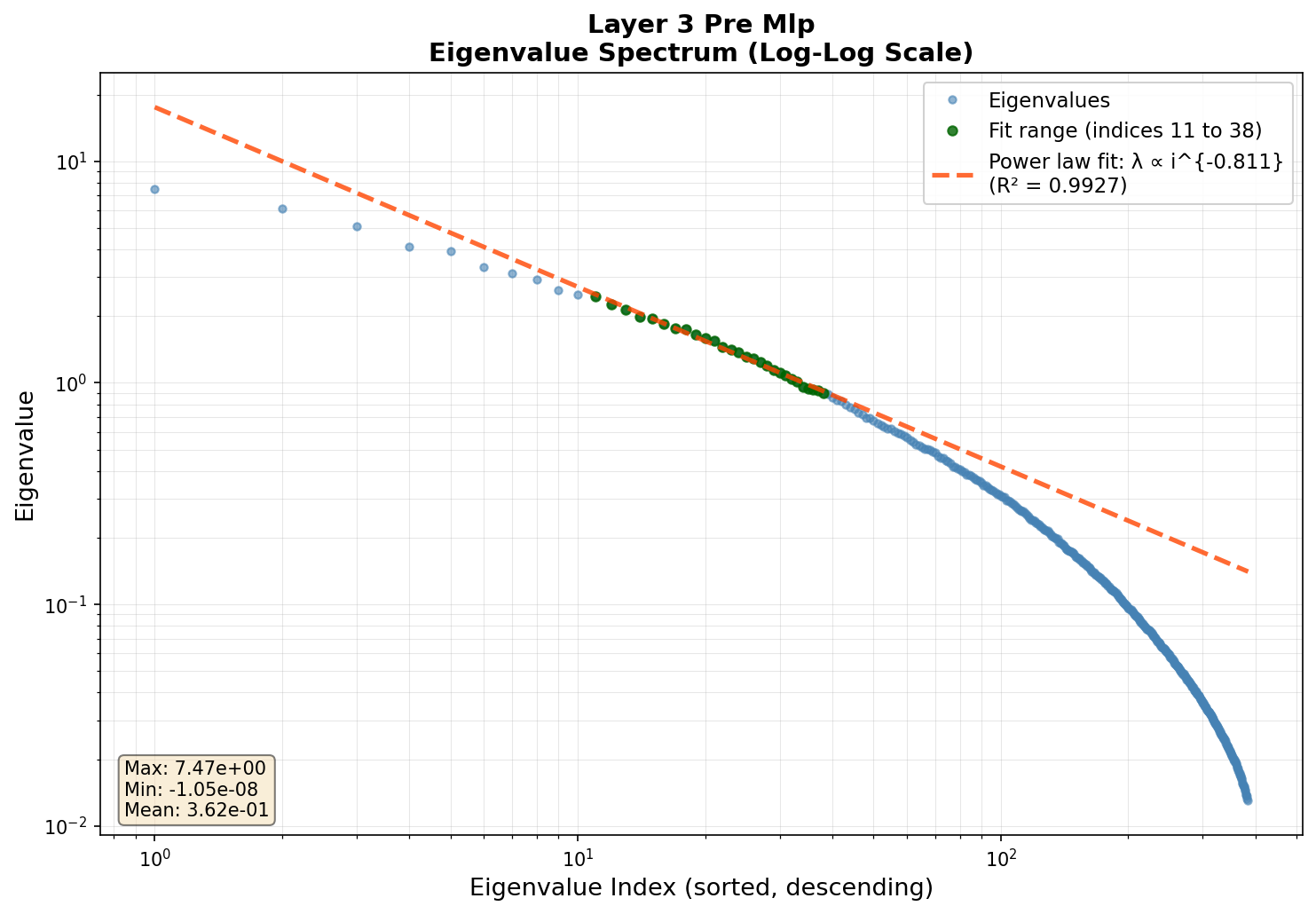}
\caption{Layer 3: Pre-MLP}
\end{subfigure}
\hfill
\begin{subfigure}[t]{0.24\textwidth}
\centering
\includegraphics[width=\textwidth]{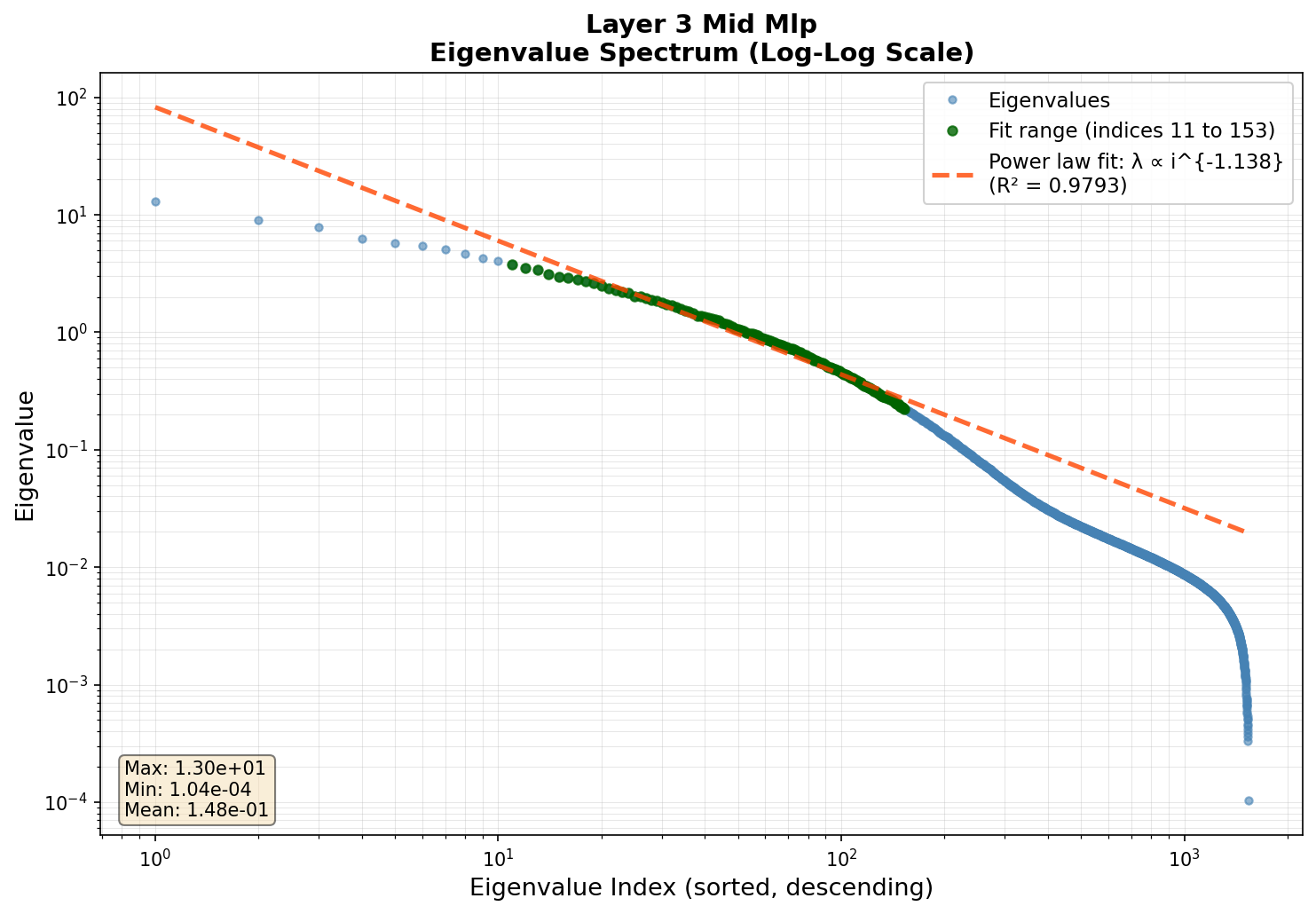}
\caption{Layer 3: Mid-MLP}
\end{subfigure}
\hfill
\begin{subfigure}[t]{0.24\textwidth}
\centering
\includegraphics[width=\textwidth]{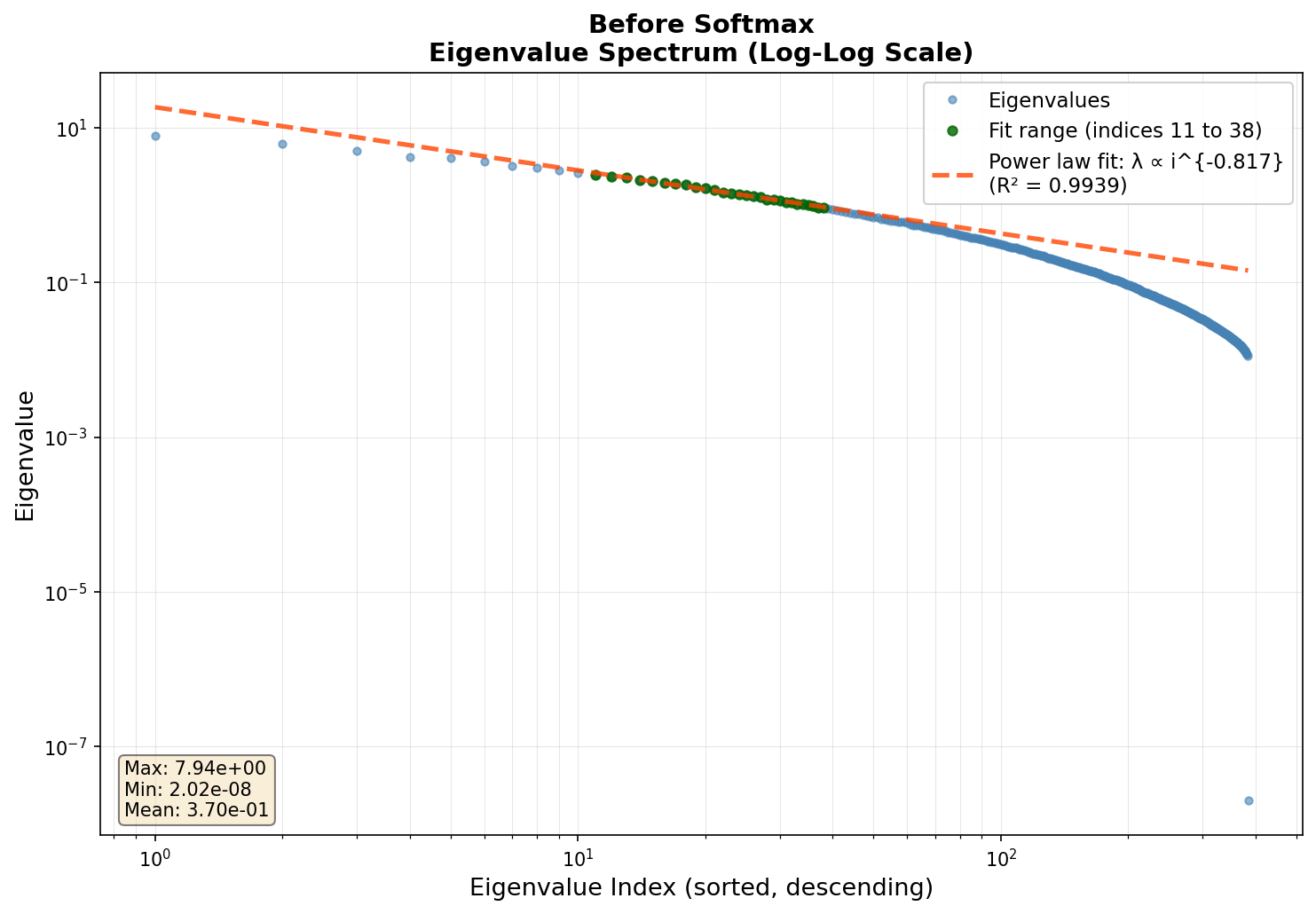}
\caption{Final: Before softmax}
\end{subfigure}
\caption{\textbf{Eigenvalue distributions at different layers before training (6-head model).} Smaller models show similar qualitative trends but with generally smaller exponents.}
\label{fig:layer_plots_before_training_heads_6}
\end{figure}

\begin{figure}[h!]
\centering
\begin{subfigure}[t]{0.24\textwidth}
\centering
\includegraphics[width=\textwidth]{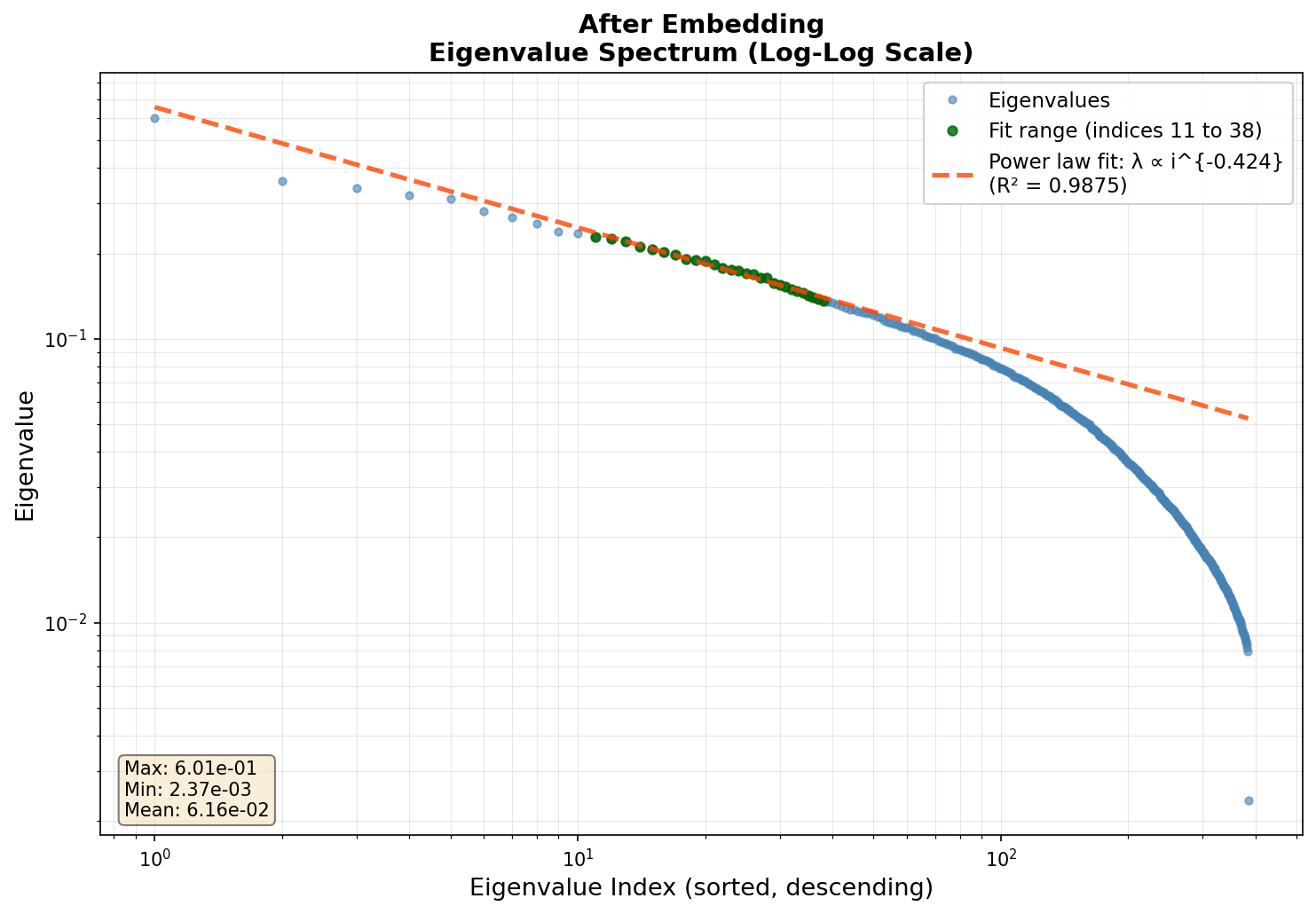}
\caption{Layer 0: After embedding}
\end{subfigure}
\hfill
\begin{subfigure}[t]{0.24\textwidth}
\centering
\includegraphics[width=\textwidth]{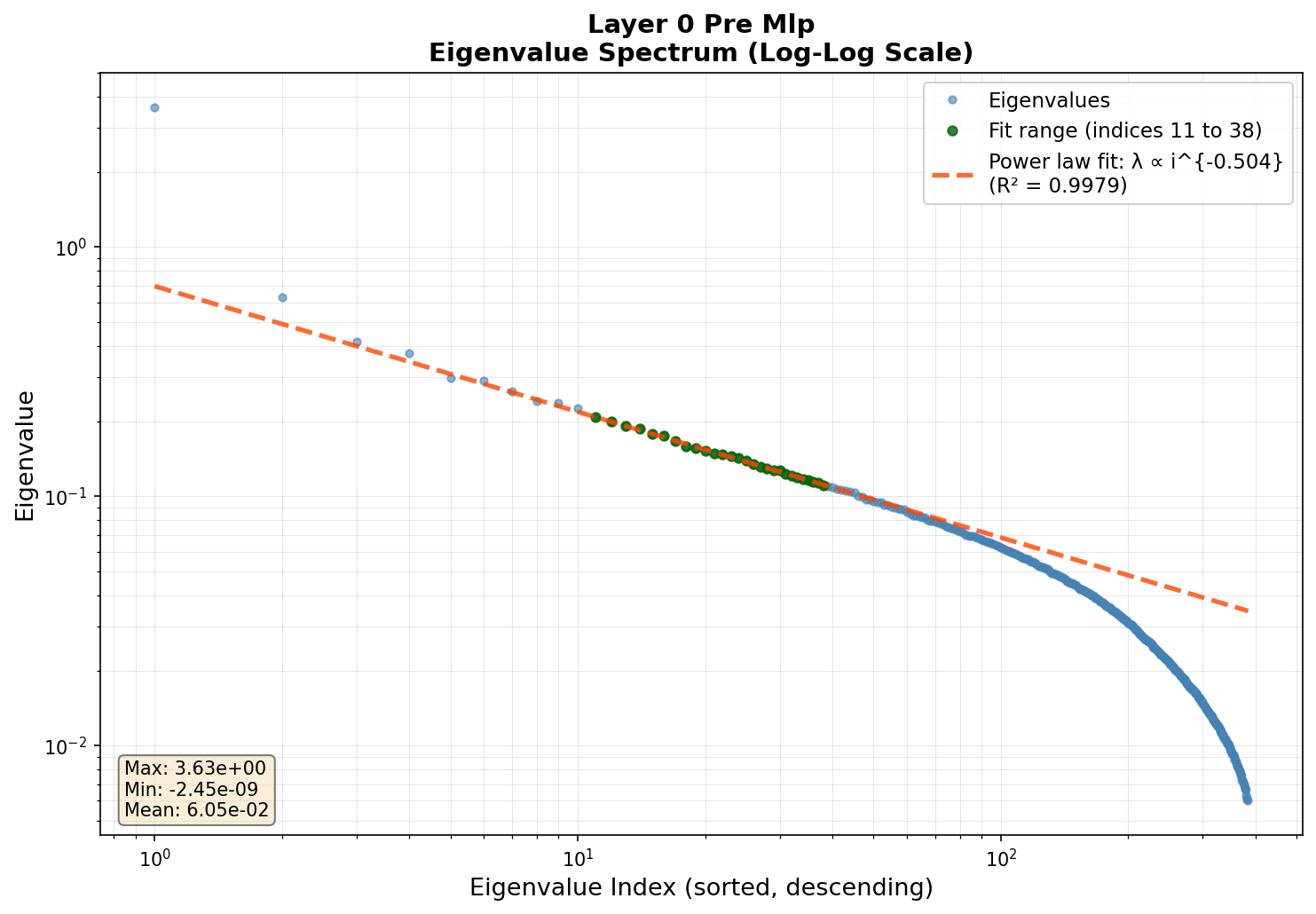}
\caption{Layer 0: Pre-MLP}
\end{subfigure}
\hfill
\begin{subfigure}[t]{0.24\textwidth}
\centering
\includegraphics[width=\textwidth]{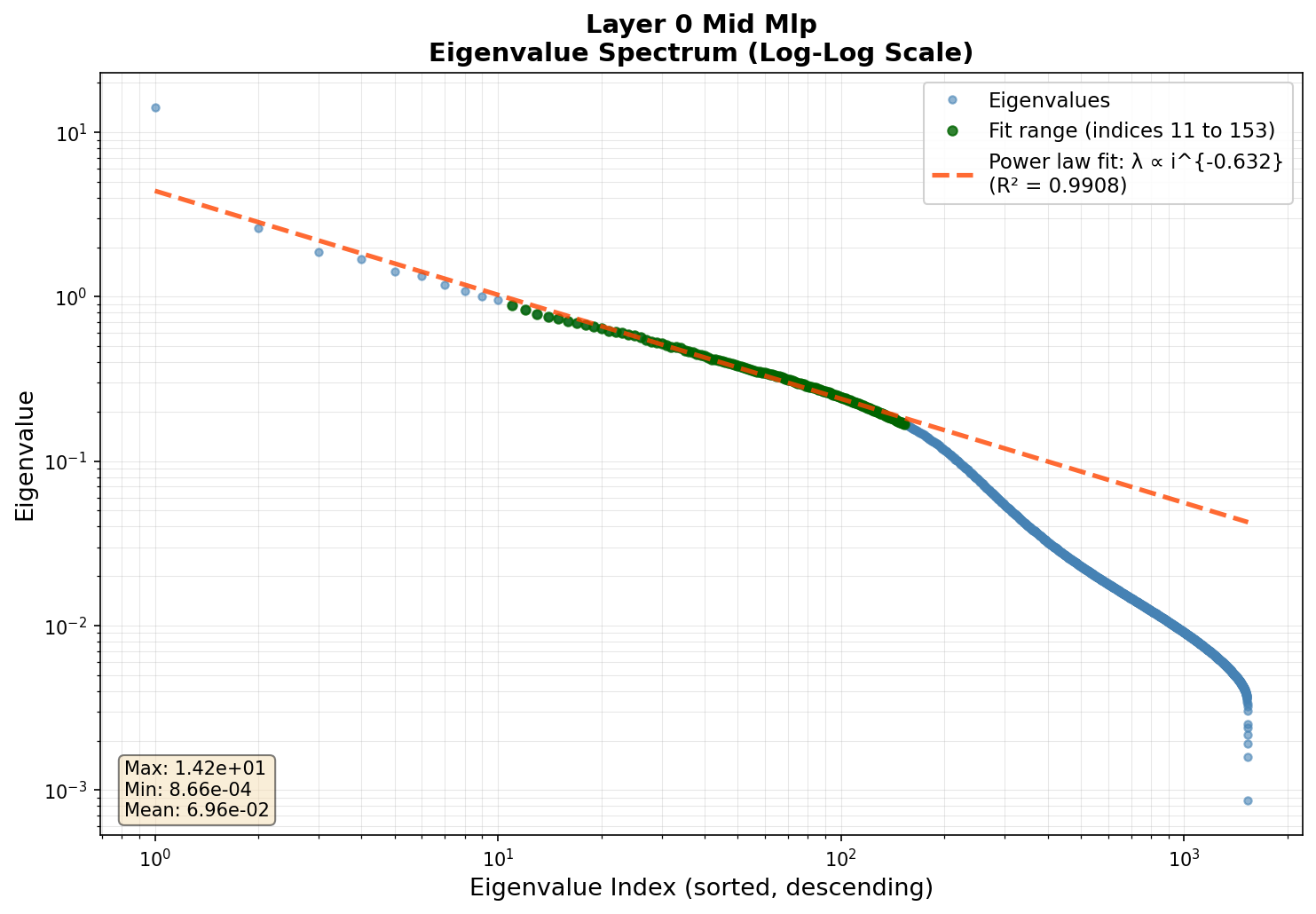}
\caption{Layer 0: Mid-MLP}
\end{subfigure}
\hfill
\begin{subfigure}[t]{0.24\textwidth}
\centering
\includegraphics[width=\textwidth]{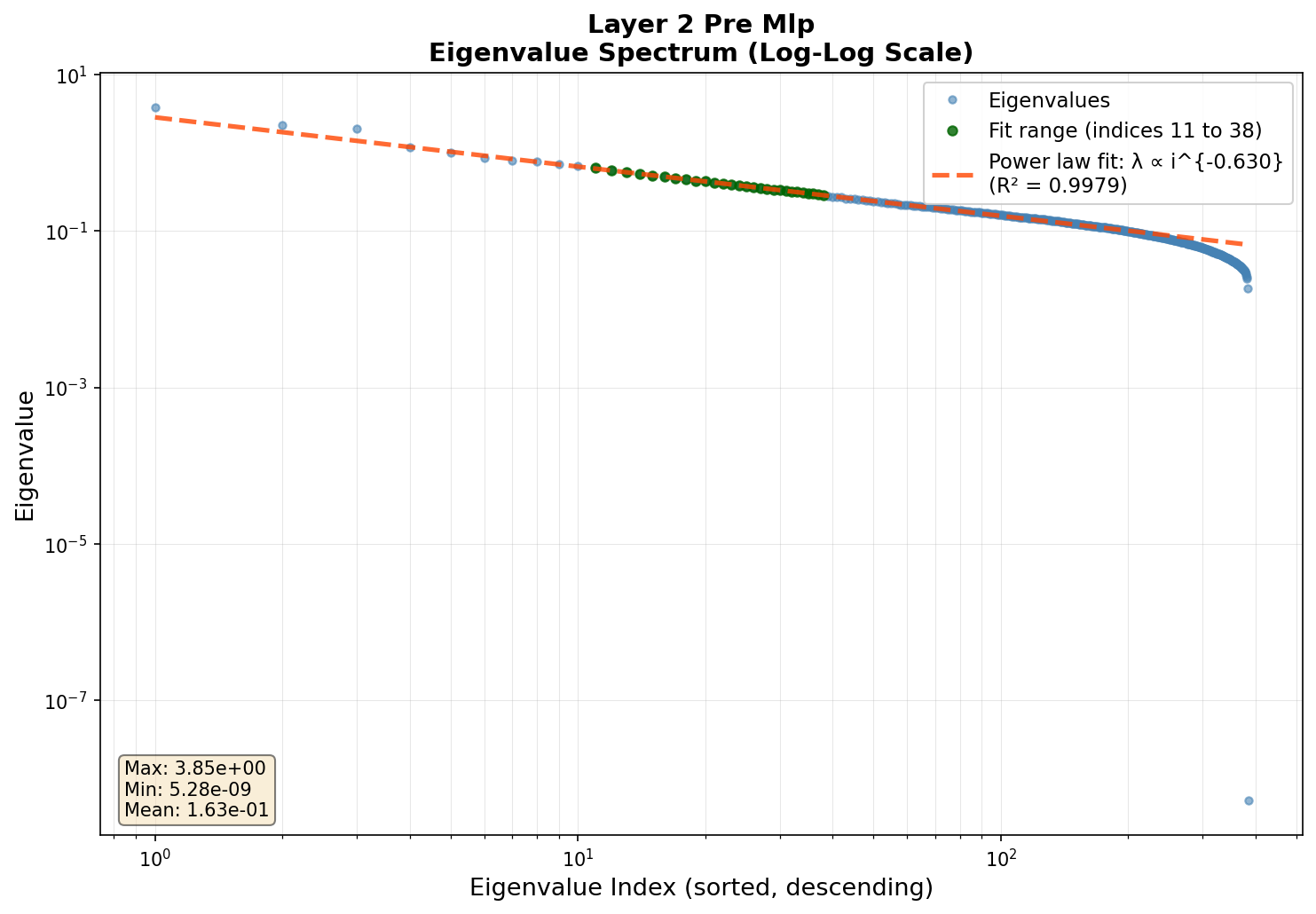}
\caption{Layer 2: Pre-MLP}
\end{subfigure}

\begin{subfigure}[t]{0.24\textwidth}
\centering
\includegraphics[width=\textwidth]{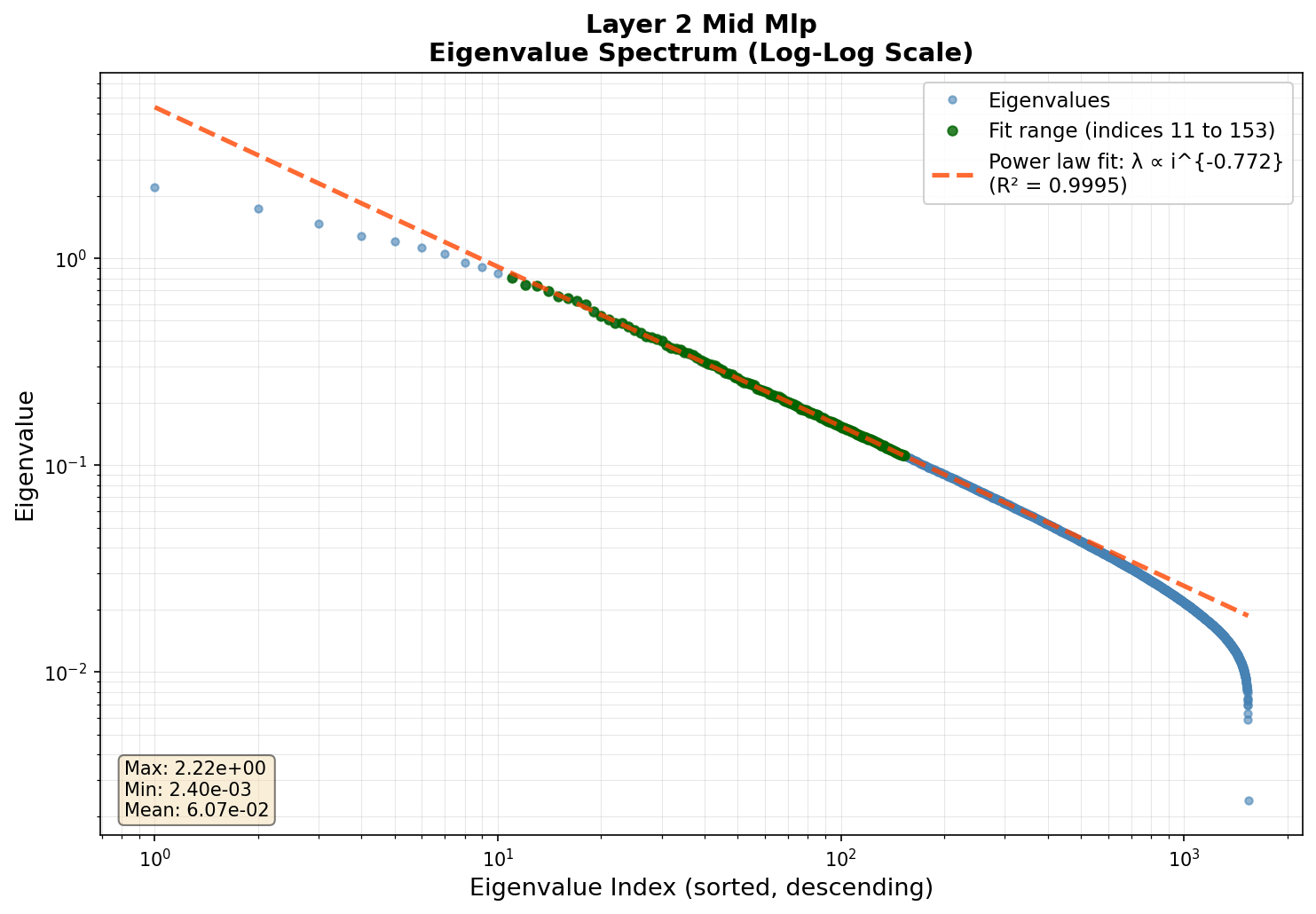}
\caption{Layer 2: Mid-MLP}
\end{subfigure}
\hfill
\begin{subfigure}[t]{0.24\textwidth}
\centering
\includegraphics[width=\textwidth]{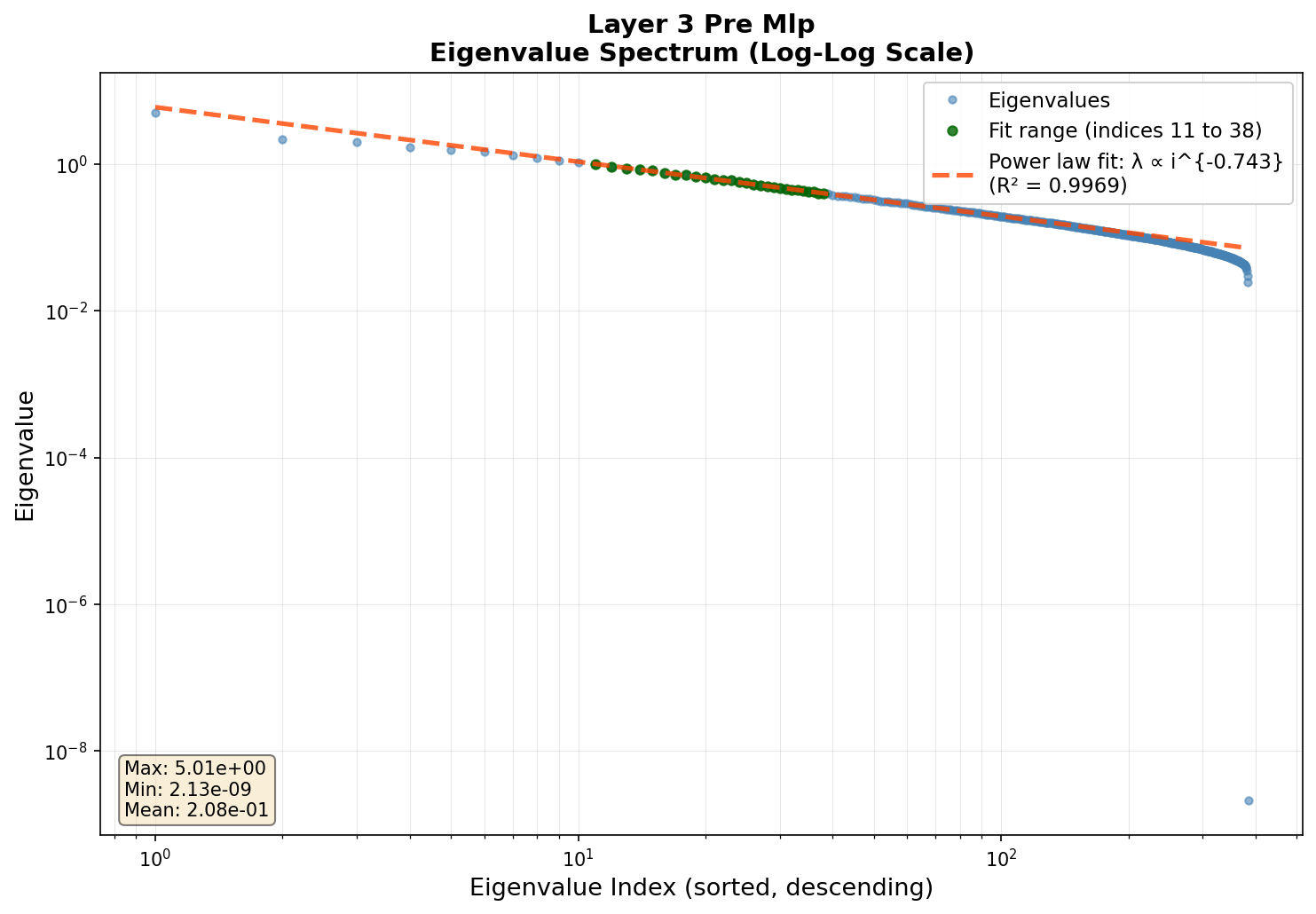}
\caption{Layer 3: Pre-MLP}
\end{subfigure}
\hfill
\begin{subfigure}[t]{0.24\textwidth}
\centering
\includegraphics[width=\textwidth]{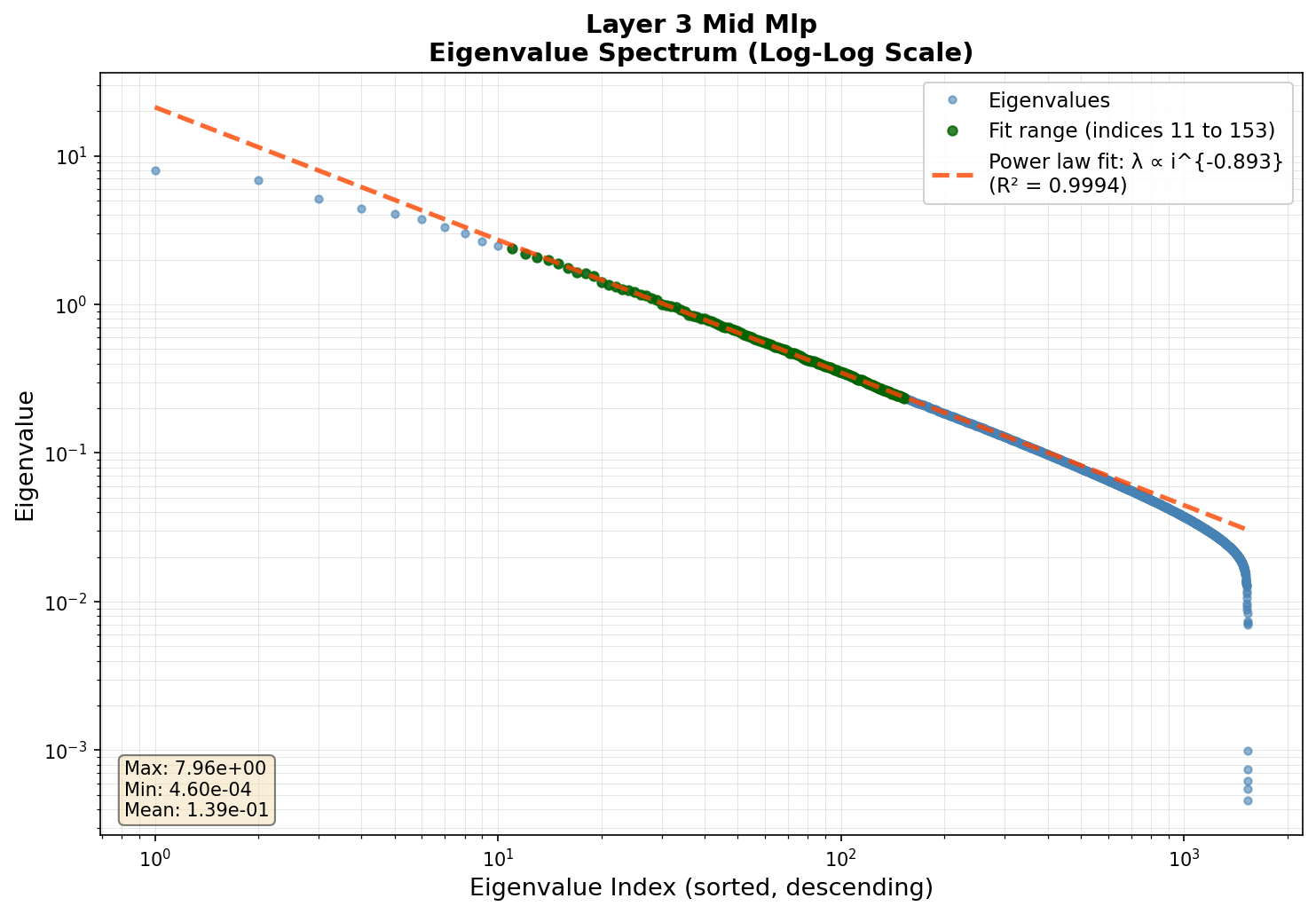}
\caption{Layer 3: Mid-MLP}
\end{subfigure}
\hfill
\begin{subfigure}[t]{0.24\textwidth}
\centering
\includegraphics[width=\textwidth]{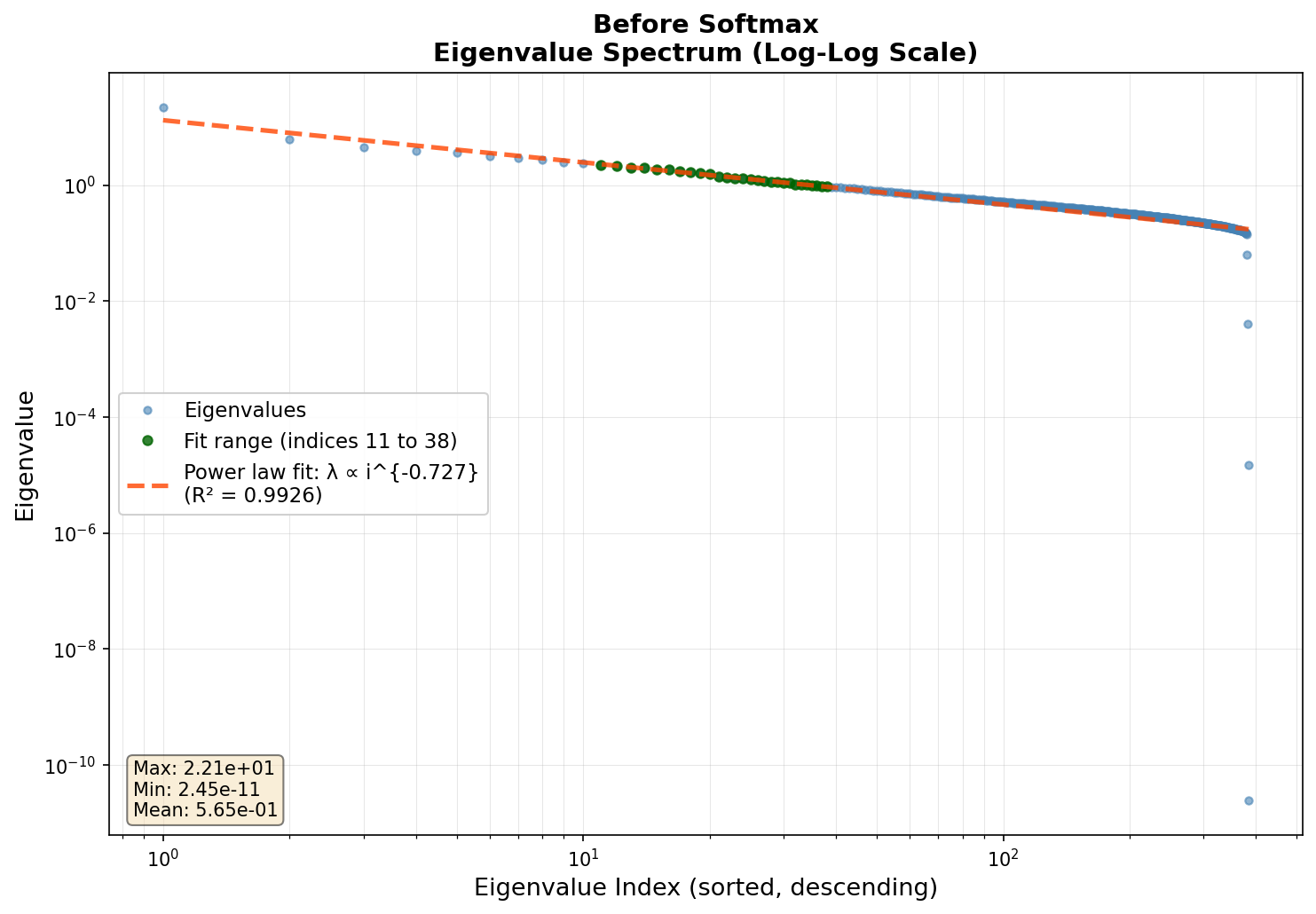}
\caption{Final: Before softmax}
\end{subfigure}
\caption{\textbf{Eigenvalue distributions at different layers after training (6-head model).} After training, the spectra become more uniform, with reduced power-law exponents.}
\label{fig:layer_plots_after_training_heads_6}
\end{figure}


\section{Correspondence with \Ademamix}
\label{sec:appendix_ademamix}


In this section, we analyze the connection between \ADana \Cref{alg:adana} and \Ademamix \Cref{alg:ademamix}. We show the existence of a scaling rule for \Ademamix hyperparameters $\beta_1$ and $\alpha$ such that \Ademamix bears strong connections with \ADana. In particular, our analysis reveals that \Ademamix's warmup schedules implicitly approximate the generalized Nesterov's momentum schedule, providing a theoretical foundation for its empirical success, showed in \Cref{fig:ablation_small,fig:comparison_adana_variants} using our proposed scaling rule for hyperparameters. The remaining gap in performance between \Ademamix and \ADana can further be explained from the ablation study \Cref{sec:ablation_adana}.

\subsection{The \Ademamix Algorithm}
\label{subsec:ademamix_algorithm}

In the original \Ademamix algorithm (see \Cref{alg:ademamix}), \citet{pagliardini2024ademamix} use four main hyperparameters: the long-range momentum EMA $\beta_1$, the short-range momentum EMA $\beta_3$, the second moment EMA $\beta_2$, the mixing coefficient $\alpha$. In particular, \Ademamix originally uses constant values for the short and long-range momentum, respectively $\beta_3(t) \equiv \beta_3,\ \beta_1(t)\equiv \beta_1$, the second moment $\beta_2(t) \equiv \beta_2$, and the mixing coefficient $\alpha(t) \equiv \alpha$.

However, in practice \citet{pagliardini2024ademamix} observed that training with constant long-range momentum was unstable. To address this, they proposed two warmup schedules up to time $T_{\alpha}$ and $T_{\beta_1}$ for the mixing coefficient $\alpha(t)$ and long momentum $\beta_1(t)$ respectively:
\begin{align}
\beta_1(t) &= \min\left\{\exp\left(\frac{\log \beta_3 \cdot \log \beta_1}{(1-t/T_{\beta_1})\log\beta_1 + (t/T_{\beta_1}) \log \beta_3}\right), \beta_1\right\}, \label{eq:ademamix_beta1_warmup} \\
\alpha(t) &= \frac{t}{T_{\alpha}}\alpha. \label{eq:ademamix_alpha_warmup}
\end{align}
The motivation for the $\beta_1(t)$ is to increase linearly the half-time decay of the EMA linearly throughout training. They additionally typically set the warmup horizon $T_{\alpha, \beta_1}$ to the total number of training iterations $T$.

\subsection{Equivalence of Schedules}
\label{subsec:schedule_equivalence}

In this section, we show that under a particular scaling of $\alpha$, $\beta_1$ and $\beta_3$, and under the warmup schedules \Ademamix is approximately equivalent to the \ADana algorithm. More precisely we will explain the following:

\begin{tcolorbox}[colback=yellow!20, colframe=yellow!60, boxrule=0.5pt, arc=8pt, left=5pt, right=5pt, top=3pt, bottom=3pt]
\begin{center}
\textbf{Correspondence between \Ademamix and \ADana:} Let a training horizon $T$, constant $\delta$ (typically $\delta=8$) and $\kappa\in (0,1)$. Set \Ademamix hyperparameters as $\beta_1 = 1 - \nicefrac{\delta}{T}$ (logarithmic-time long momentum), $\beta_3 \in (0,1)$ any constant (typically $\beta_3=0.9$ in practice, but $\beta_3=0.0$ for the correspondence), mixing coefficient $\alpha = T^{1-\kappa}$ and warmup schedules $T_{\alpha} = T_{\beta_1} = T$. Then \Ademamix is approximately equivalent to \ADana with \Danaconstant-type schedule $\alpha(t) = T^{-\kappa} \times t$.
\end{center}
\end{tcolorbox}

In the following we show the two main equivalence on schedules between \Ademamix and \ADana: the long-momentum schedule $\beta_1(t)$ and mixing coefficient $\alpha(t)$.

\subsubsection{Log-time Momentum Exponential Moving Average (EMA) $\beta_1(t)$.}
We first show in \Cref{lem:ademamix_dana_approx} that by setting $\beta_1 = 1-\nicefrac{\delta}{T}$s, the warmup schedule~\eqref{eq:ademamix_beta1_warmup} asymptotically approximates the \ADana Nesterov-style logarithmic-time momentum schedule $\beta_1(t) = 1 - \delta/(\delta + t)$.

\begin{lemma}[\Ademamix schedule approximates \Dana schedule]
\label{lem:ademamix_dana_approx}
Let $\delta > 0$, $0 < \beta_3 < 1$ (short momentum), and define for any $T>0$, $\beta_1 = 1 - \delta/T$ (target long momentum). Suppose that $0 < \beta_1 < 1$ and define the \Ademamix schedule
\[
\beta_1^{\textup{ADEMAMIX}}(t) = \min\left\{\exp\left(\frac{\log \beta_3 \cdot \log \beta_1}{(1-t/T)\log\beta_1 + (t/T) \log \beta_3}\right), \beta_1\right\}.
\]
Then as $t, T \to \infty$ with $0 < t \leq T$, we have
\[
1 - \beta_1^{\textup{ADEMAMIX}}(t) = \frac{\delta}{t}\bigl(1 + \smallO_t(1) + \smallO_T(1)\bigr).
\]
\end{lemma}

\begin{proof}
Since $\beta_1 = 1 - \delta/T$, we can write $\log(\beta_1) = -\delta/T \cdot (1 + \smallO_T(1))$. Substituting into the schedule formula, for $0 < t \leq T$:
\begin{align*}
\beta_1^{\textup{ADEMAMIX}}(t) &= \exp\left(\frac{\log(\beta_3) \cdot (-\delta/T)(1 + \smallO_T(1))}{(t/T)\log(\beta_3)\bigl(1 + \smallO_t(1) + \smallO_T(1)\bigr)}\right) \\
&= \exp\left(\frac{-\delta/t \cdot (1 + \smallO_T(1))}{1 + \smallO_t(1) + \smallO_T(1)}\right) \\
&= 1 - \frac{\delta}{t}\bigl(1 + \smallO_t(1) + \smallO_T(1)\bigr).
\end{align*}
\end{proof}

\Cref{lem:ademamix_dana_approx} shows that setting with $\beta_1 = 1 - \delta/T$, the \Ademamix warmup schedule converges to the \ADana logarithmic time schedule $\beta_1(t) = 1 - \delta/(\delta + t)$ as $t, T \to \infty$. This is precisely the schedule that enables acceleration when $\delta$ is chosen sufficiently large \citep{ferbach2025dimension}. $\delta$ is typically set to $8$ in our experiments.

\subsubsection{Mixing coefficient $\alpha(t)$ schedule.}
The original \Ademamix paper proposes a constant schedule $\alpha(t) \equiv \alpha$. Within the \GeneralNesterov / \Dana framework~\eqref{eq:general_momentum_update}, a constant $\alpha(t) \equiv \alpha$ corresponds to using \Danadecaying with power exponent $\kappa = 1$. In particular, the momentum schedule does not increase fast enough to allow for acceleration.

However, to avoid instabilities, \cite{pagliardini2024ademamix} typically use a linear warmup schedule $\alpha(t) = \alpha \times \frac{t}{T}$. As explained in \Cref{sec:gen_dana_alg}, this schedule recovers the \Danaconstant schedule when scaling the constant $\alpha$ appropriately. In particular, setting $\alpha = T^{1-\kappa}$ recovers the damped schedule $\alpha(t)=(1+t)^{1-\kappa}$ around $t=T$ and is predicted to match the performance \Danadecaying schedule $\alpha(t)=(1+t)^{1-\kappa}$ around $t=T$ \citep{ferbach2025dimension}. 

\begin{figure}[t]
\centering
\includegraphics[width=0.6\textwidth]{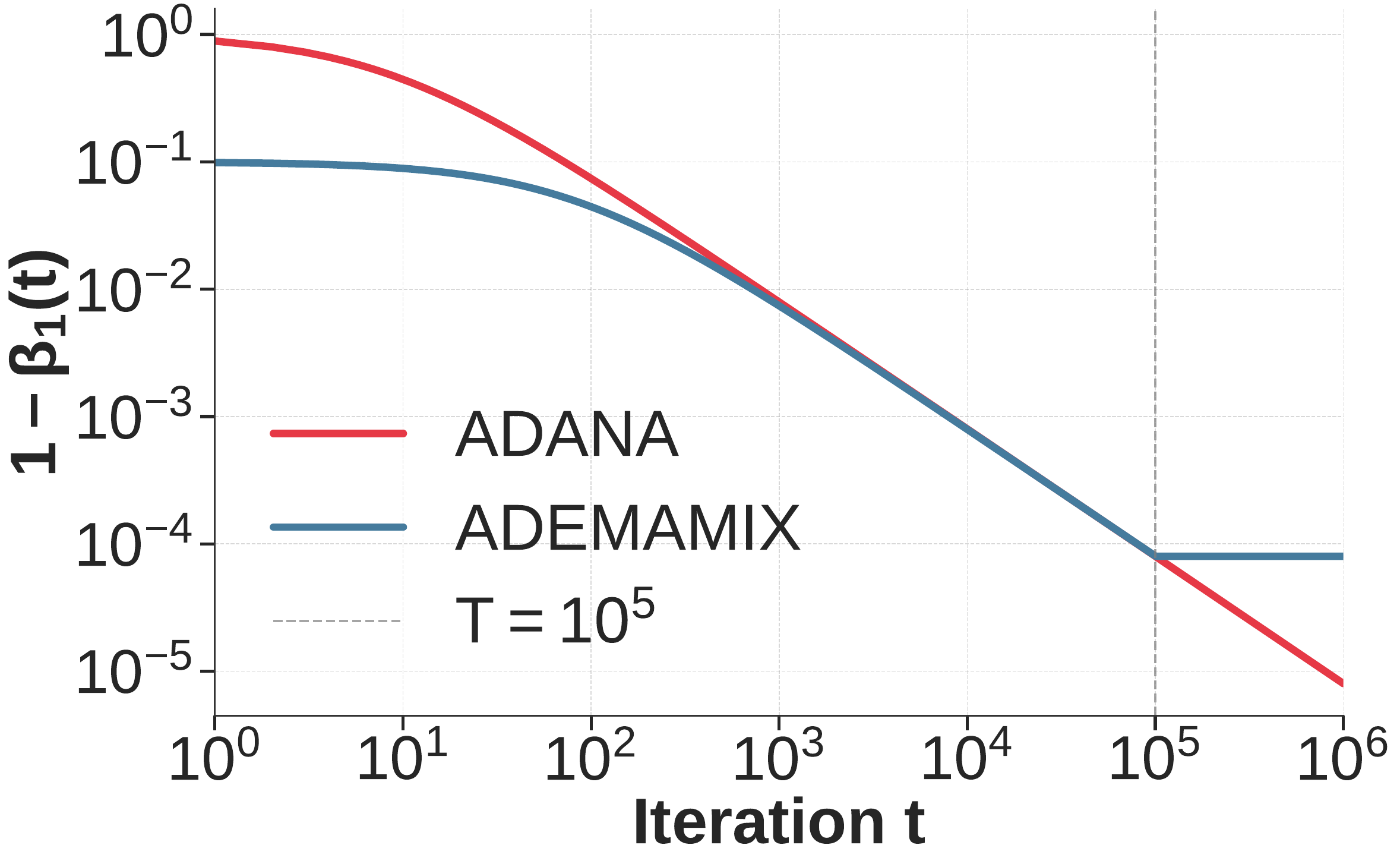}
\caption{Comparison of long-range momentum schedules for \ADana $1-\beta_1(t) = \frac{\delta}{\delta+t}$ and for \Ademamix using warmup $T_{\beta_1} = T,\ \beta_3 = 0.9,\ \beta_1 = 1-\frac{\delta}{T}$  with $\delta=8$ and $T=10^{5}$. While both schedules differ at the start of training they match later in training. Note that after $t\geq T$, \Ademamix's schedule becomes constant $\beta_1(t) = \beta_1 = 1-\frac{\delta}{T}$.}
\label{fig:taylor_expansion_beta1}
\end{figure}

\subsection{Empirical Verification of $\alpha(t)$ Scaling: Saturation on Long Training Runs} \label{appendix:verification_alpha_scaling}

In \Cref{fig:alpha_scaling_chinchilla} we we fit the optimal damping factor (for \Danaconstant-type damping schedule $\alpha(t)=\tilde{\alpha} \times t$)  with respect to training time on different model scales and tokens to parameter ratios. While in \Cref{fig:alpha_scaling_chinchilla} we see that the optimal $\tilde{\alpha}$ follows a power-law dependence on $T$, training far beyond the compute optimal regime reveals a saturation of the exponent $\kappa$, as shown in \Cref{fig:alpha_scaling_iterations}. This effect is particularly visible for smaller models, where the power-law behavior eventually degrades at very long training horizons.


\begin{figure}[t]
\centering
\begin{subfigure}[t]{0.48\textwidth}
\centering
\includegraphics[width=\textwidth]{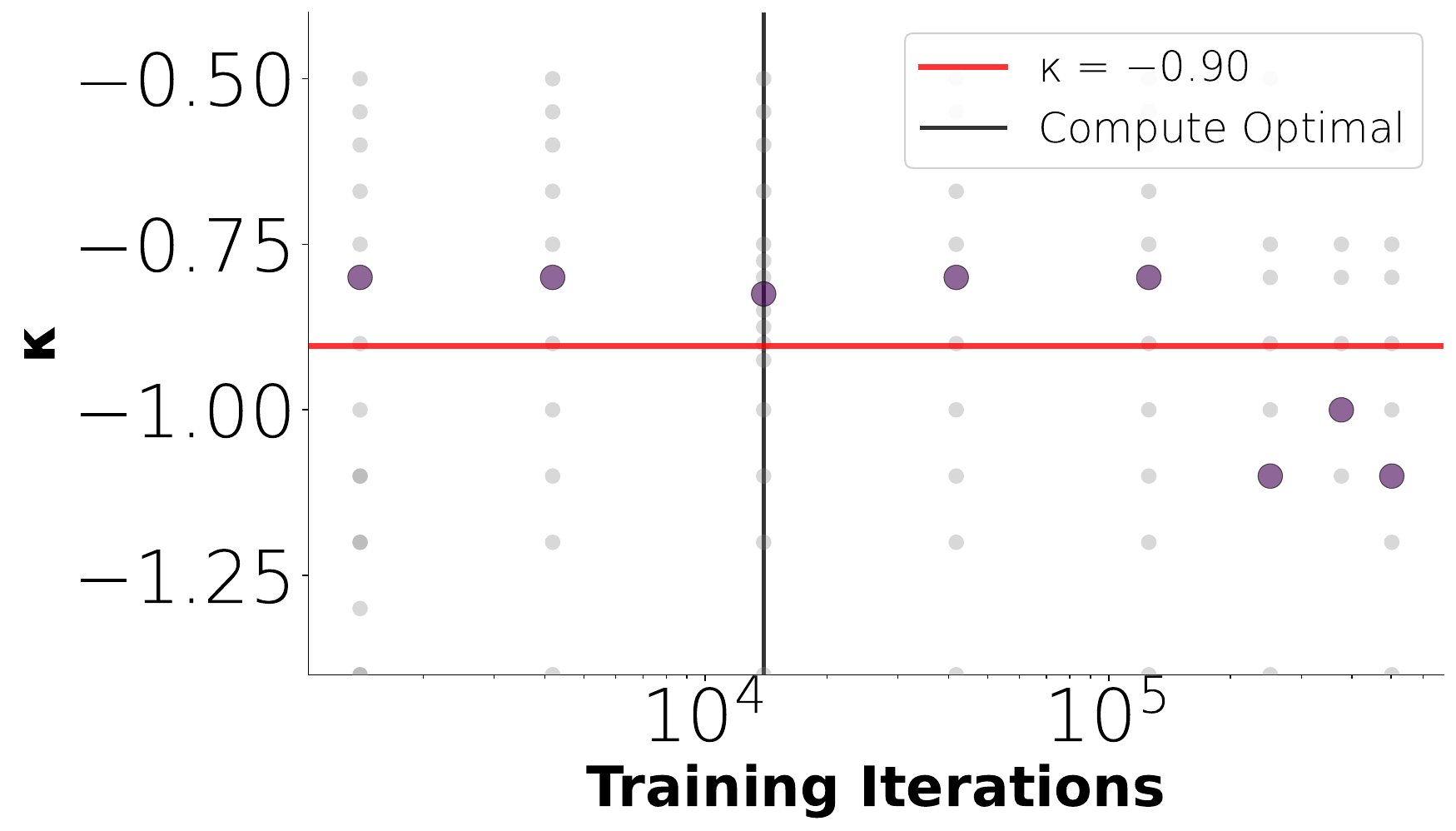}
\caption{Enoki model, $6$ heads}
\label{fig:alpha_scaling_6_long}
\end{subfigure}
\hfill
\begin{subfigure}[t]{0.48\textwidth}
\centering
\includegraphics[width=\textwidth]{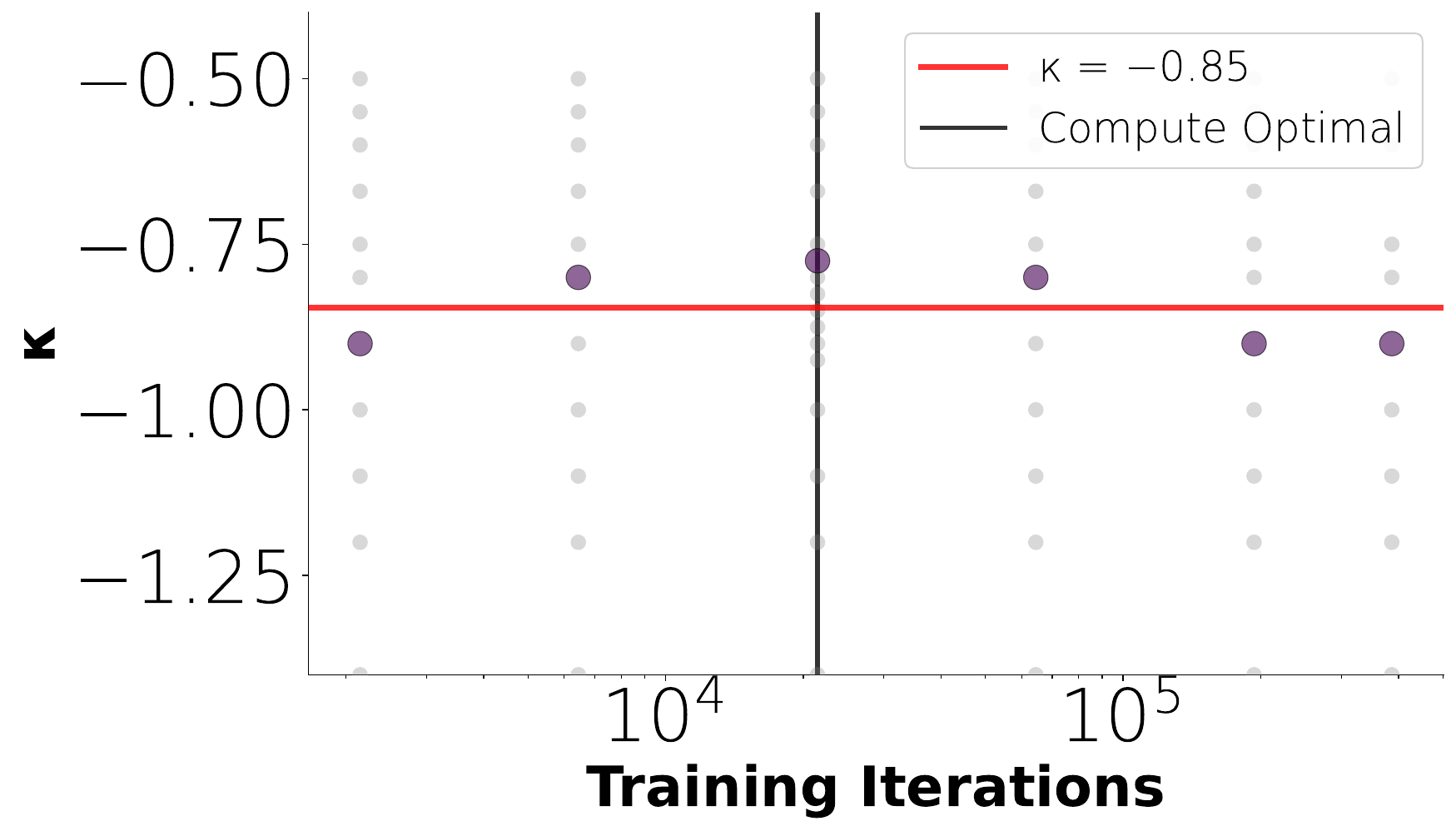}
\caption{Enoki model, $8$ heads}
\label{fig:alpha_scaling_8_long}
\end{subfigure}
\begin{subfigure}[t]{0.48\textwidth}
\centering
\includegraphics[width=\textwidth]{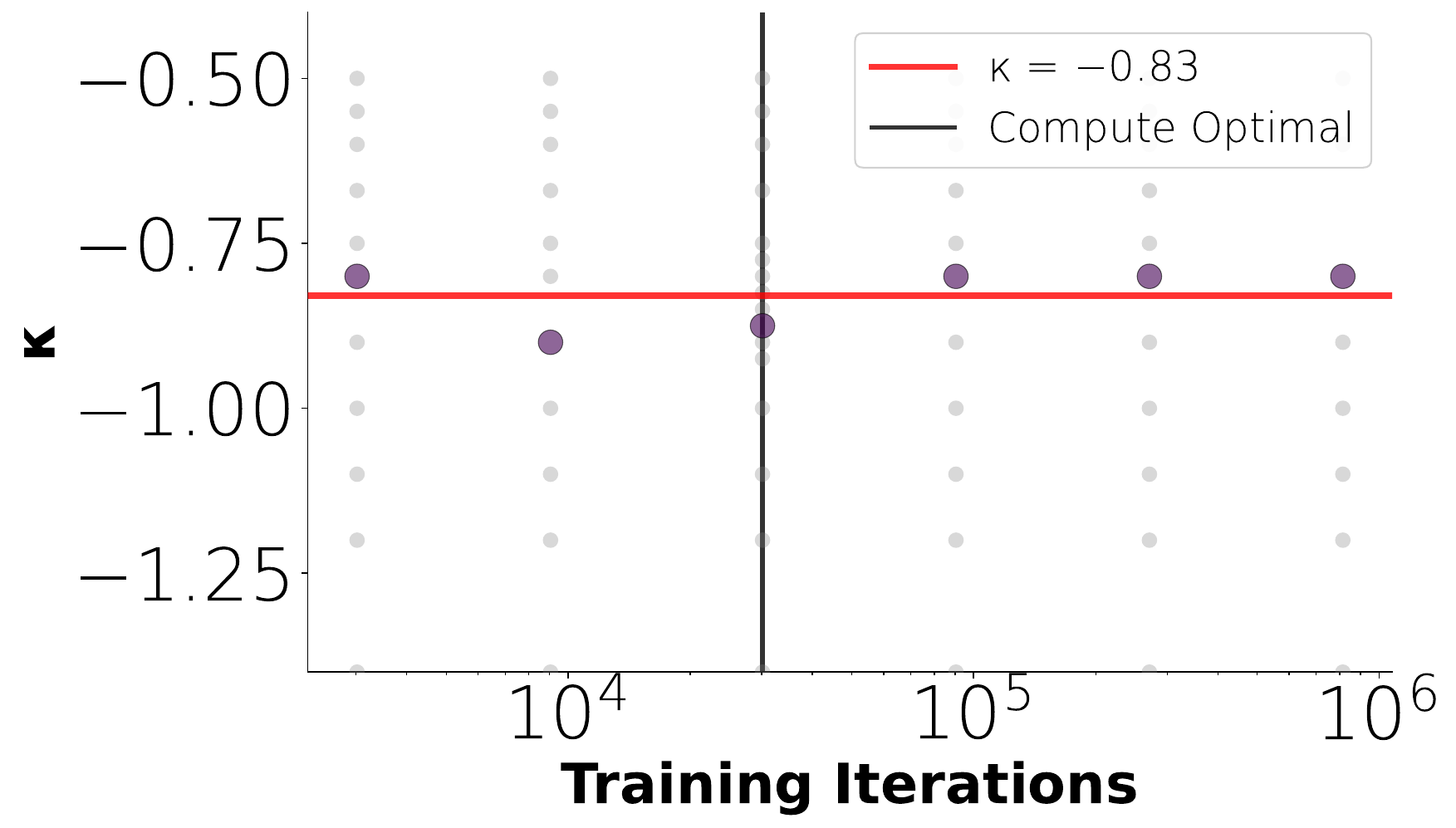}
\caption{Enoki model, $10$ heads}
\label{fig:alpha_scaling_10_long}
\end{subfigure}
\hfill
\begin{subfigure}[t]{0.48\textwidth}
\centering
\includegraphics[width=\textwidth]{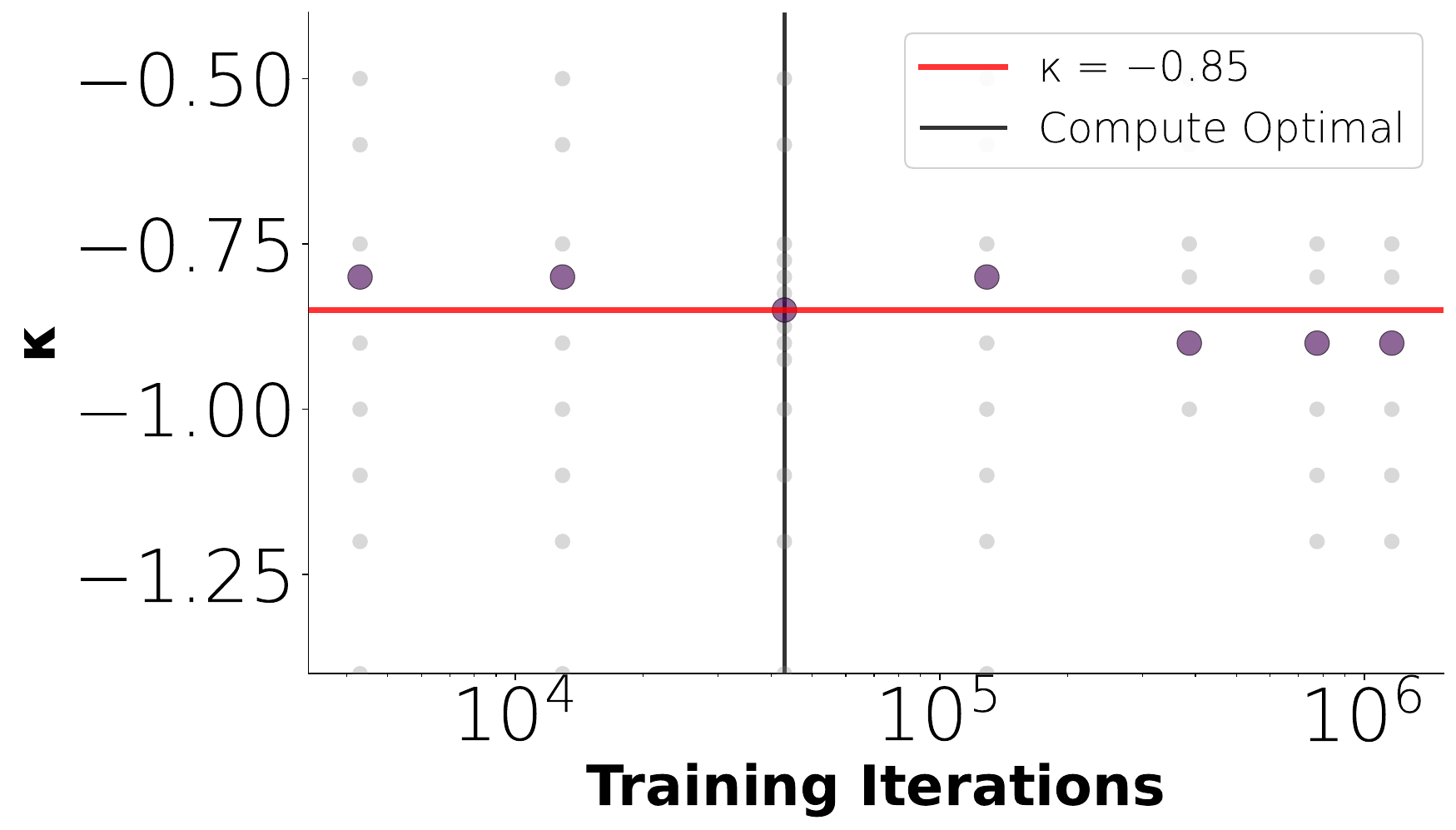}
\caption{Enoki model, $12$ heads}
\label{fig:alpha_scaling_12_long}
\end{subfigure}
\caption{\textbf{Evolution of optimal $\kappa\defas \log_T(\tilde{\alpha})$ for different model sizes when training far beyong the compute optimal point.} The power-law relationship $\tilde{\alpha} \approx T^{-\kappa}$ holds across different model scales and training horizons but breaks when training for too long (see heads $6,8$ especially.}
\label{fig:alpha_scaling_iterations}
\end{figure}

\subsection{Conclusion of the Ablation in \Cref{sec:ablation_adana}}

Due to the close relationship between \ADana and \Ademamix, under the hyperparameters scaling noted above, it is natural to ask how their performance compare. In \Cref{sec:ablation_adana}, we answer this question by conducting an ablation study on \ADana schedules to understand why \Ademamix underperforms \ADana in \Cref{fig:scaling_main}. The main conclusion is that the \Danaconstant-type $\alpha(t)=T^{-\kappa}\times t$ schedule contributes to degrading performance while the short momentum EMA $\beta_3=0.9$ tends to increase performance compared to $\beta_3=0.0$, probably due to improved stability. Finally $\beta_2=0.999$ constant as used in \Ademamix may have a limited impact compared to logarithmic-time \ADana $\beta_2(t) = 1 - \frac{\delta}{\delta+t}$ although it seems to increase instability at scale.

\subsection{Discussion}
\label{subsec:ademamix_discussion}

The connection between \ADana and \Ademamix demonstrates that the key insight---long-term memory with appropriate scaling---can be implemented in multiple ways. \ADana's approach of using time-dependent $\beta_1, \beta_2$ values is more direct and requires fewer hyperparameters, while \Ademamix achieves similar effects through its three-buffer structure with appropriate warmup scheduling.

For practitioners, we offer the following recommendations. When using \Ademamix, setting $\beta_1 = 1 - \delta/T$ (long momentum) and tuning $\alpha = T^{1-\kappa}$yields strong performance. When starting fresh, \ADana is simpler with fewer hyperparameters to tune and achieves better final performance. The relationship $\beta_1 = 1 - \delta/T$ and $\alpha = T^{1-\kappa}$ in \Ademamix provides the bridge between the two frameworks.

\section{Additional Weight Decay Details}
\label{sec:extended_weight_decay}

\begin{figure}[h!]
\centering
\begin{subfigure}[t]{0.32\textwidth}
\centering
\includegraphics[width=\textwidth]{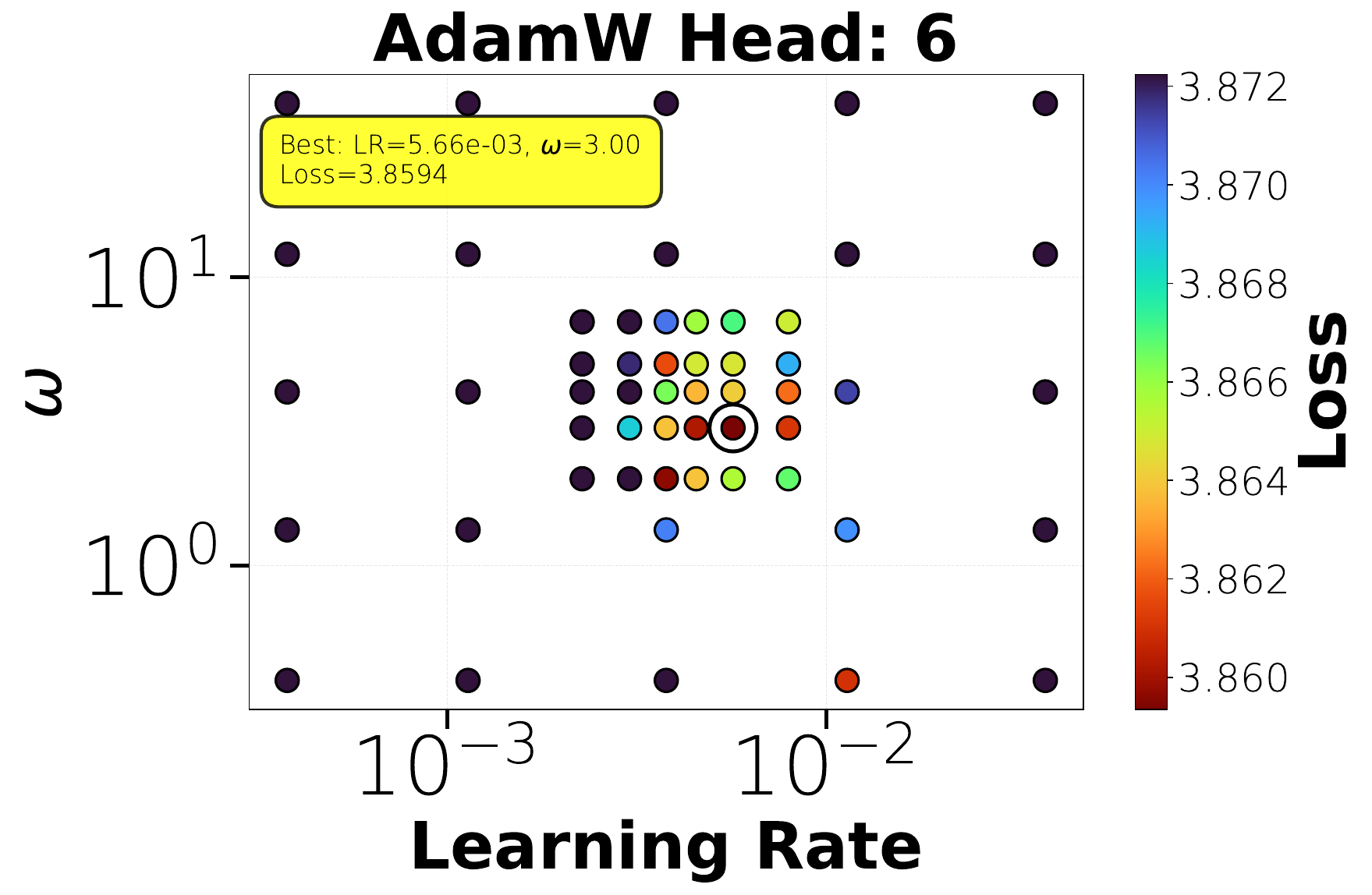}
\caption{6 heads}
\label{fig:wd_heatmap_adamw_6}
\end{subfigure}
\hfill
\begin{subfigure}[t]{0.32\textwidth}
\centering
\includegraphics[width=\textwidth]{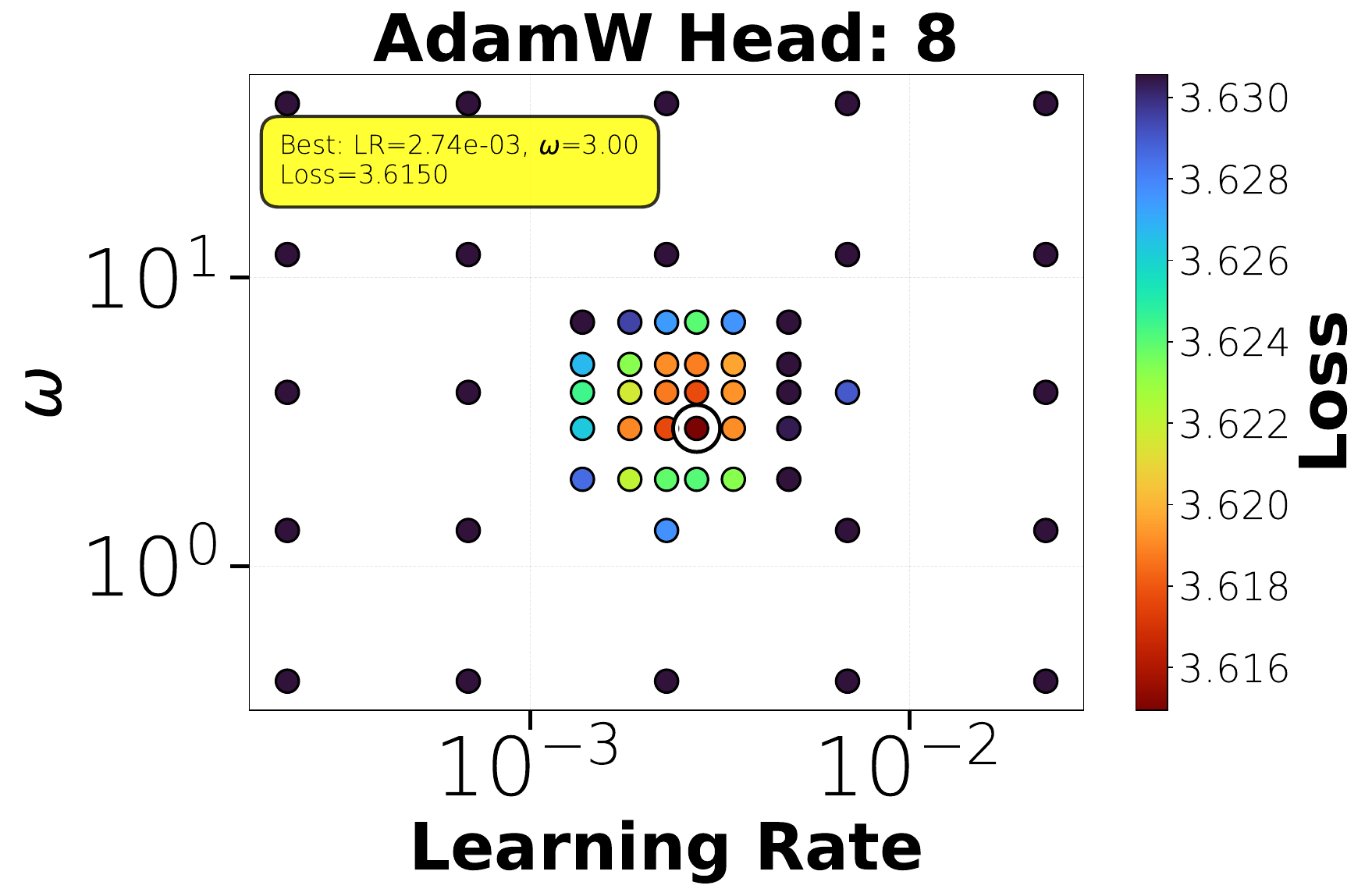}
\caption{8 heads}
\label{fig:wd_heatmap_adamw_8}
\end{subfigure}
\hfill
\begin{subfigure}[t]{0.32\textwidth}
\centering
\includegraphics[width=\textwidth]{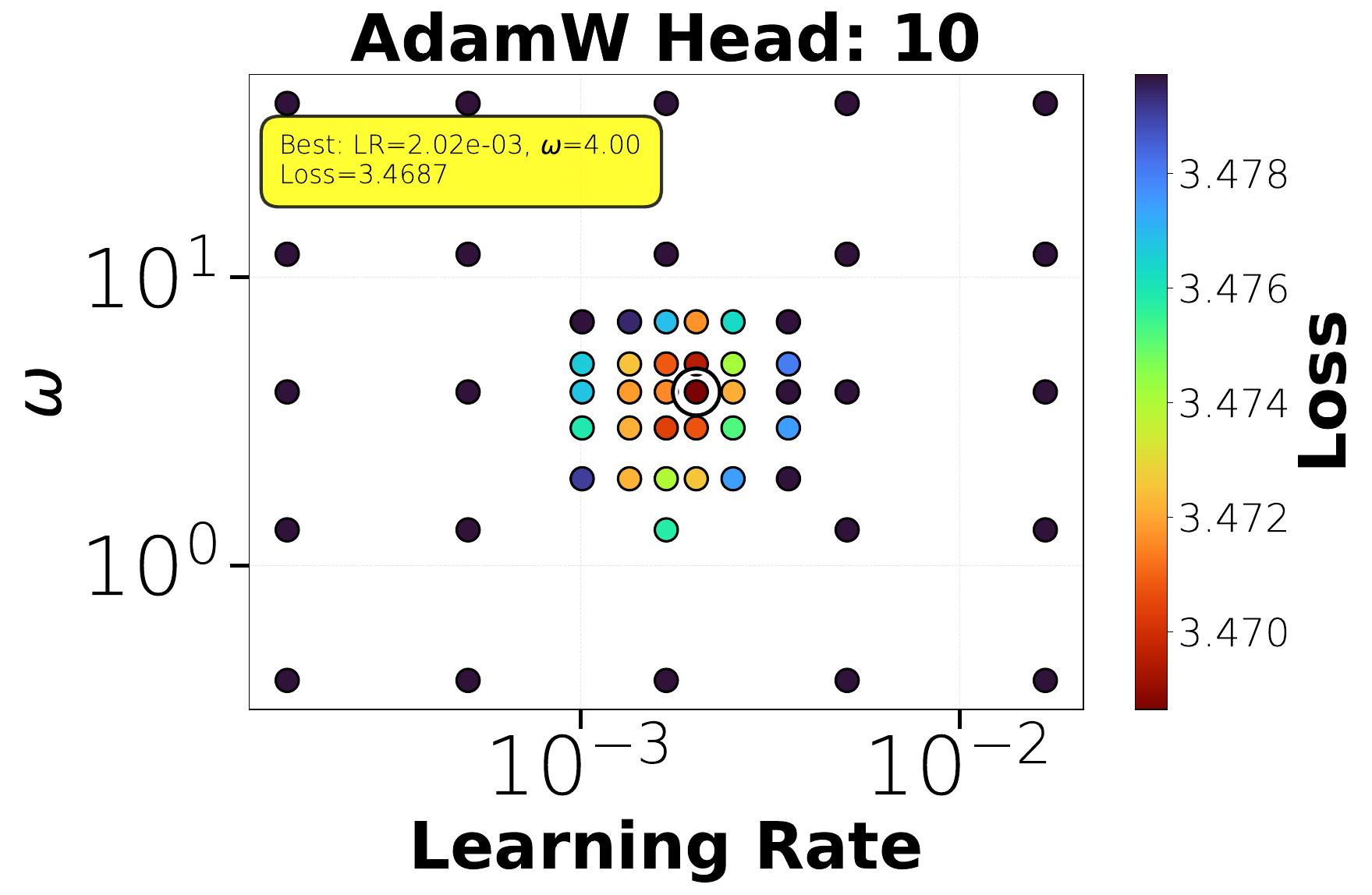}
\caption{10 heads}
\label{fig:wd_heatmap_adamw_10}
\end{subfigure}

\begin{subfigure}[t]{0.32\textwidth}
\centering
\includegraphics[width=\textwidth]{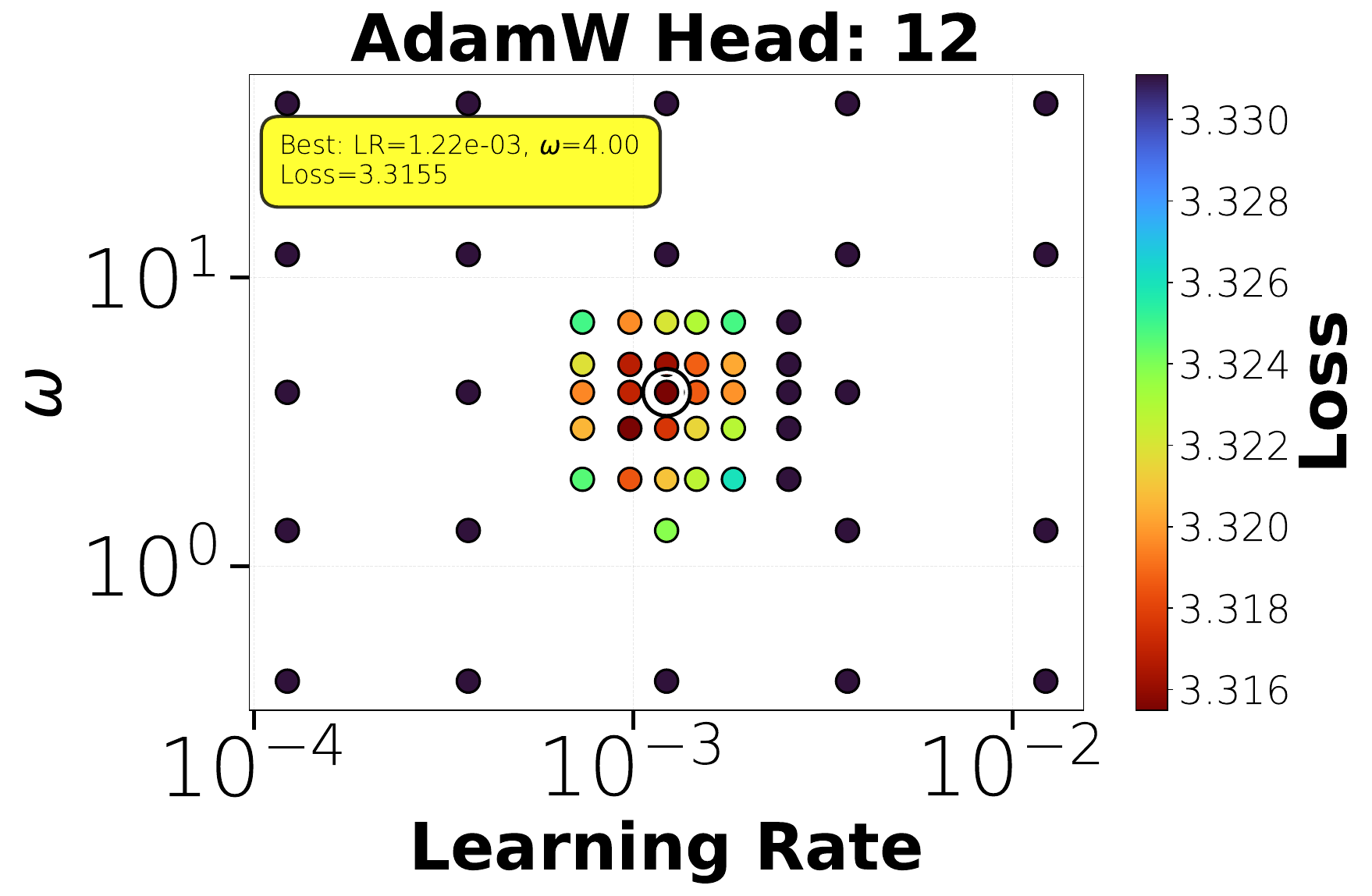}
\caption{12 heads}
\label{fig:wd_heatmap_adamw_12}
\end{subfigure}
\hfill
\begin{subfigure}[t]{0.32\textwidth}
\centering
\includegraphics[width=\textwidth]{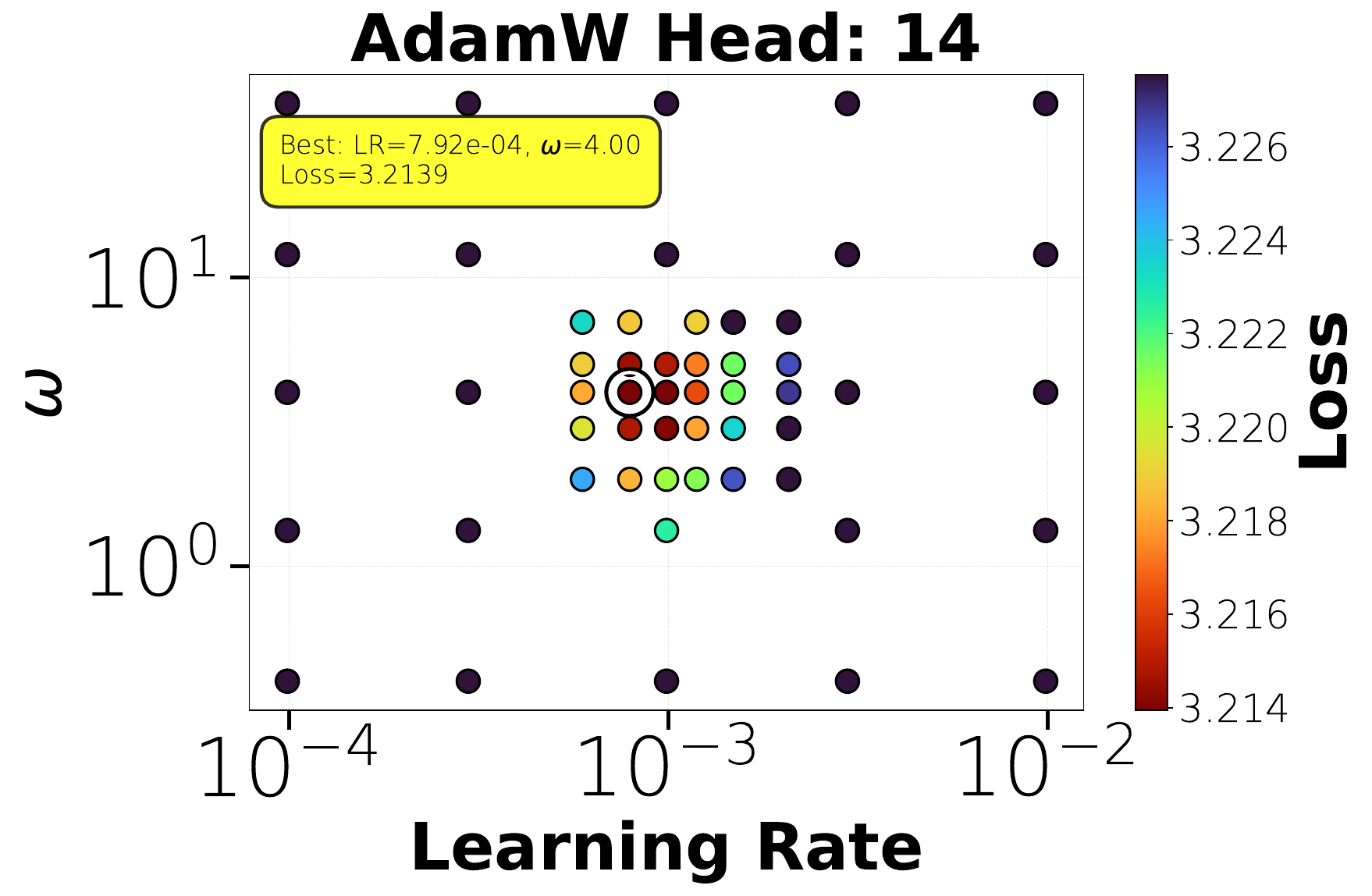}
\caption{14 heads}
\label{fig:wd_heatmap_adamw_14}
\end{subfigure}
\hfill
\begin{subfigure}[t]{0.32\textwidth}
\centering
\includegraphics[width=\textwidth]{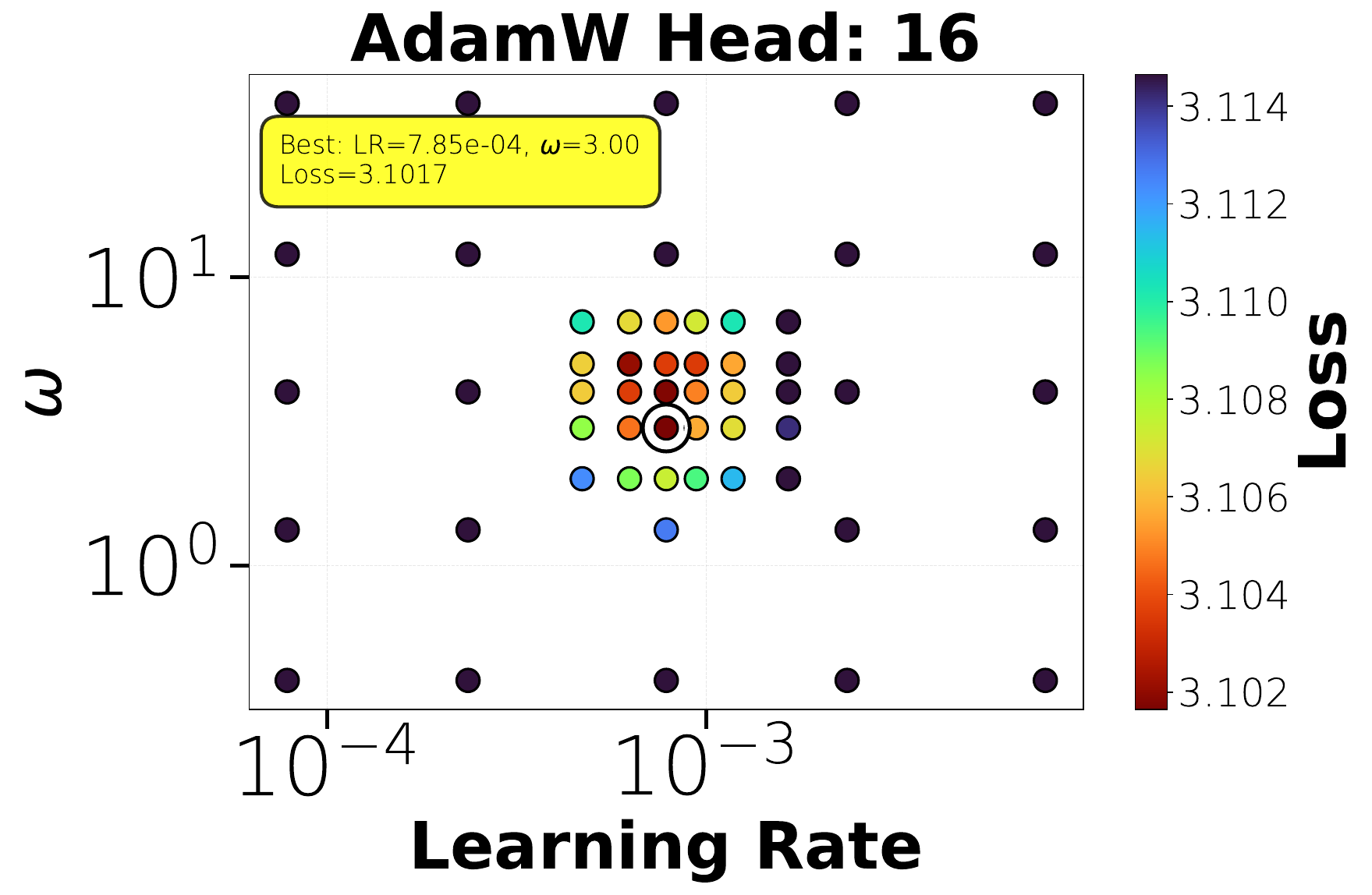}
\caption{16 heads}
\label{fig:wd_heatmap_adamw_16}
\end{subfigure}
\caption{\textbf{Validation loss heatmaps for constant weight-decay (\AdamW) across different model sizes.} Each heatmap shows the final validation loss as a function of learning rate and weight-decay coefficient $\omega$.}
\label{fig:wd_heatmap_adamw}
\end{figure}

\begin{figure}[h!]
\centering
\begin{subfigure}[t]{0.32\textwidth}
\centering
\includegraphics[width=\textwidth]{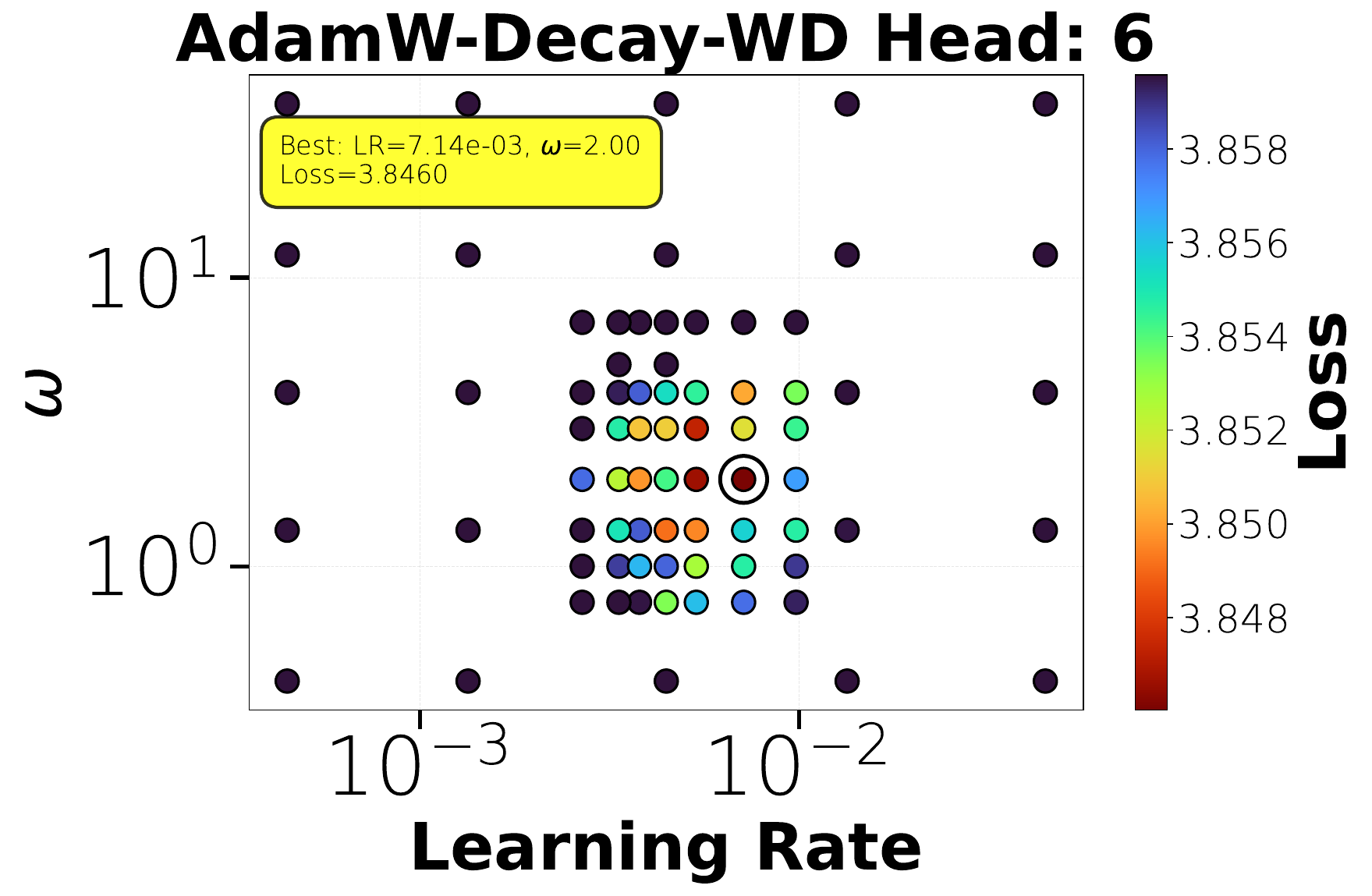}
\caption{6 heads}
\label{fig:wd_heatmap_adamw_decaying_6}
\end{subfigure}
\hfill
\begin{subfigure}[t]{0.32\textwidth}
\centering
\includegraphics[width=\textwidth]{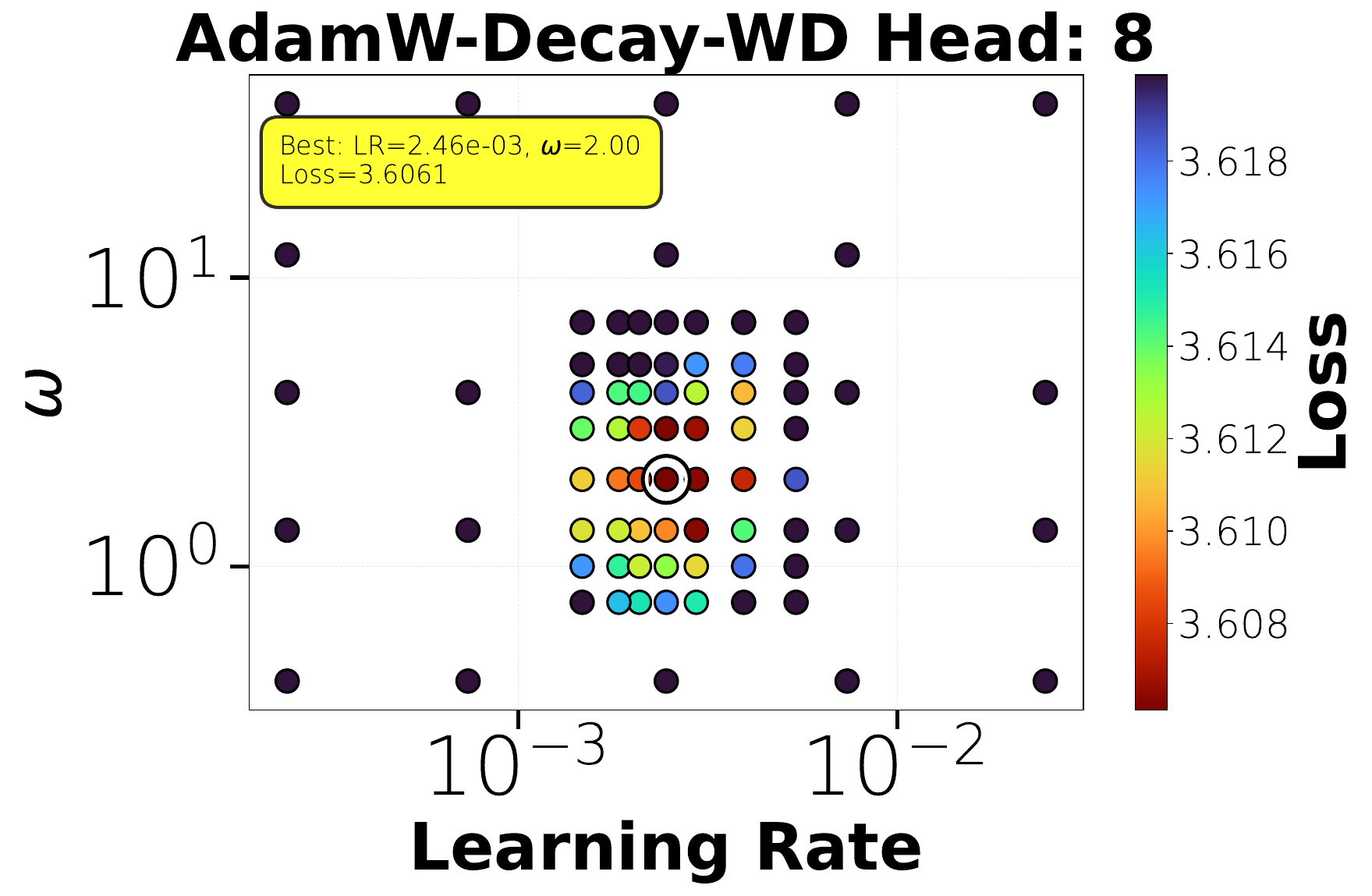}
\caption{8 heads}
\label{fig:wd_heatmap_adamw_decaying_8}
\end{subfigure}
\hfill
\begin{subfigure}[t]{0.32\textwidth}
\centering
\includegraphics[width=\textwidth]{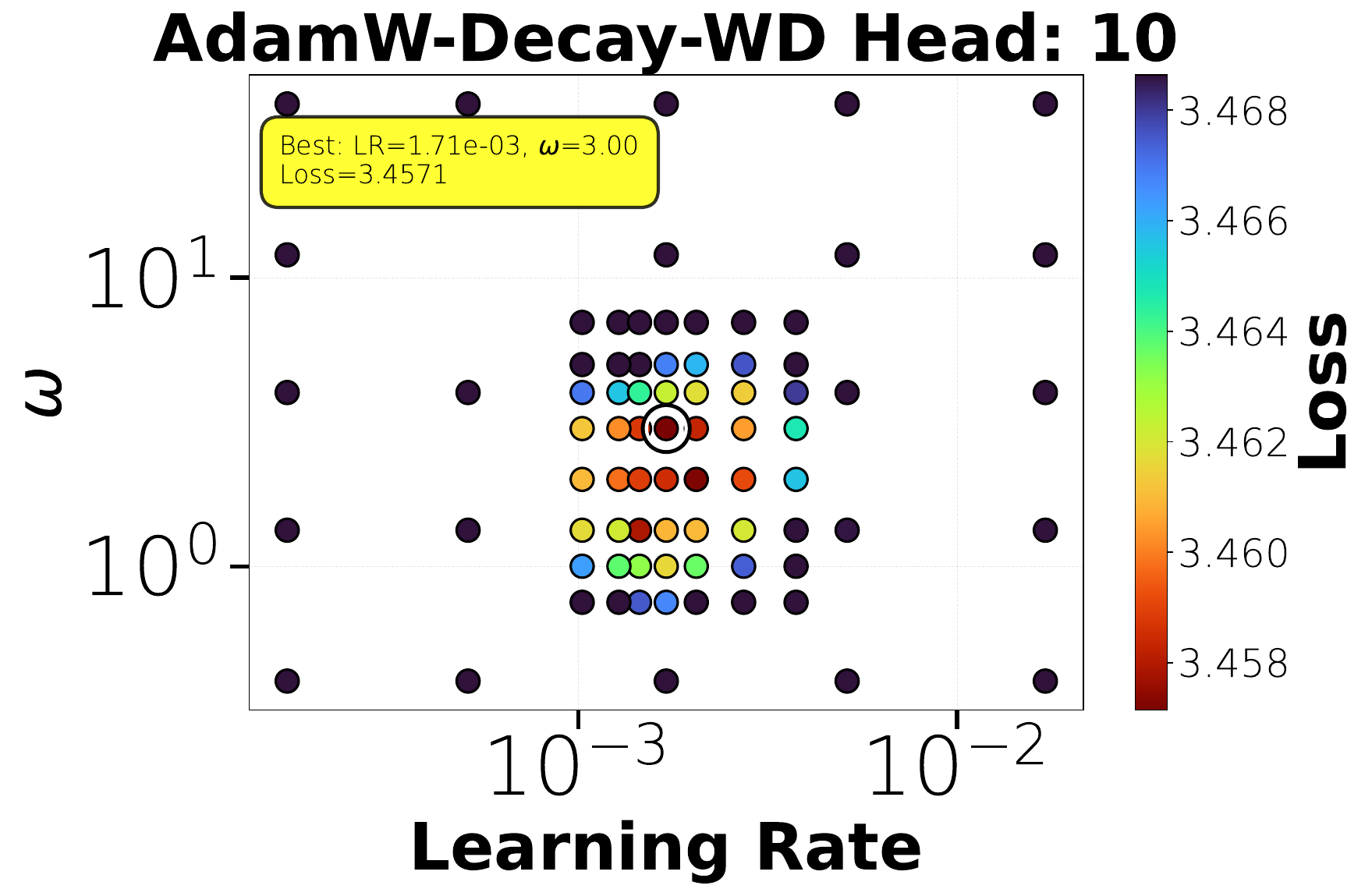}
\caption{10 heads}
\label{fig:wd_heatmap_adamw_decaying_10}
\end{subfigure}

\begin{subfigure}[t]{0.32\textwidth}
\centering
\includegraphics[width=\textwidth]{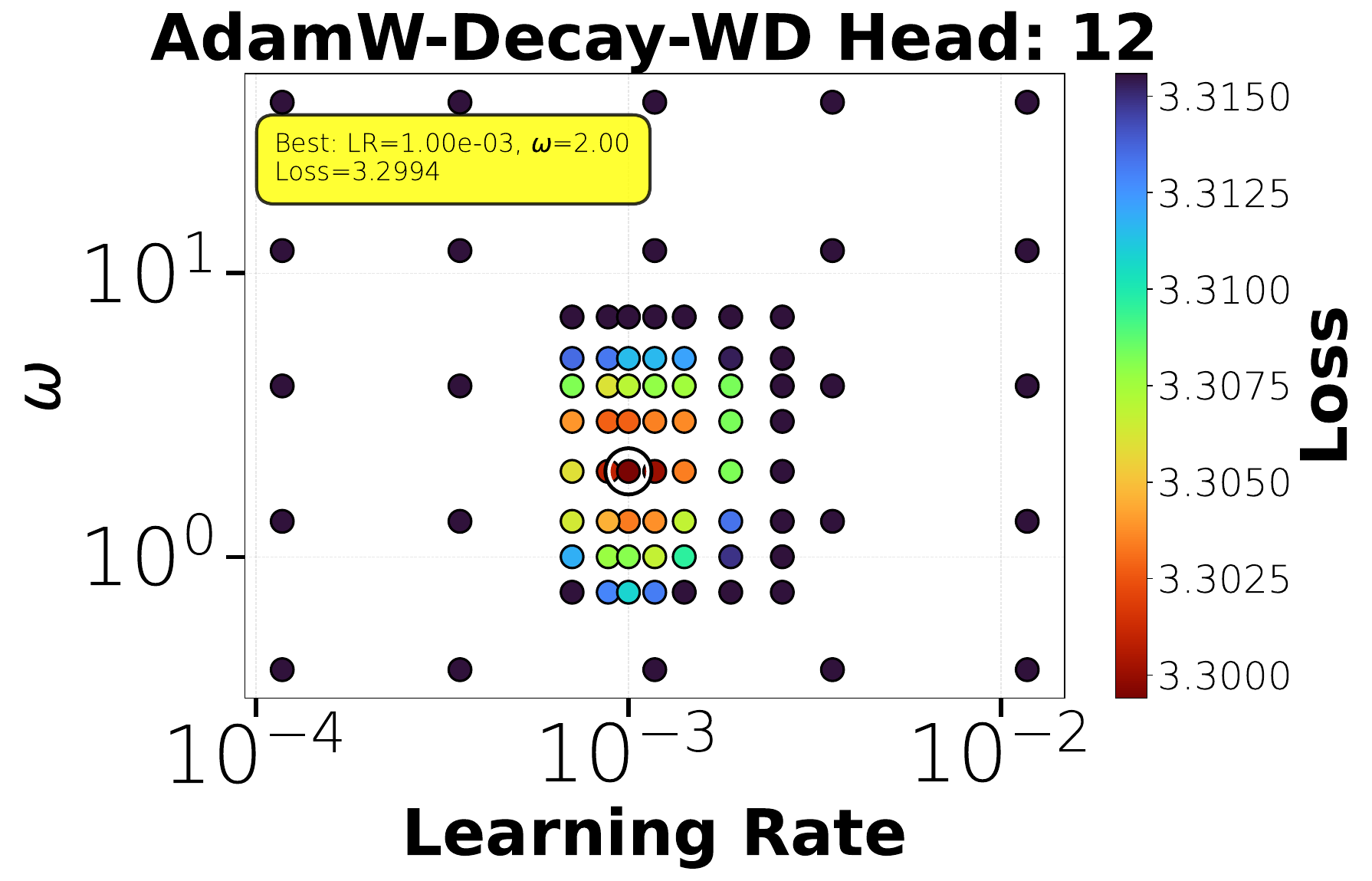}
\caption{12 heads}
\label{fig:wd_heatmap_adamw_decaying_12}
\end{subfigure}
\hfill
\begin{subfigure}[t]{0.32\textwidth}
\centering
\includegraphics[width=\textwidth]{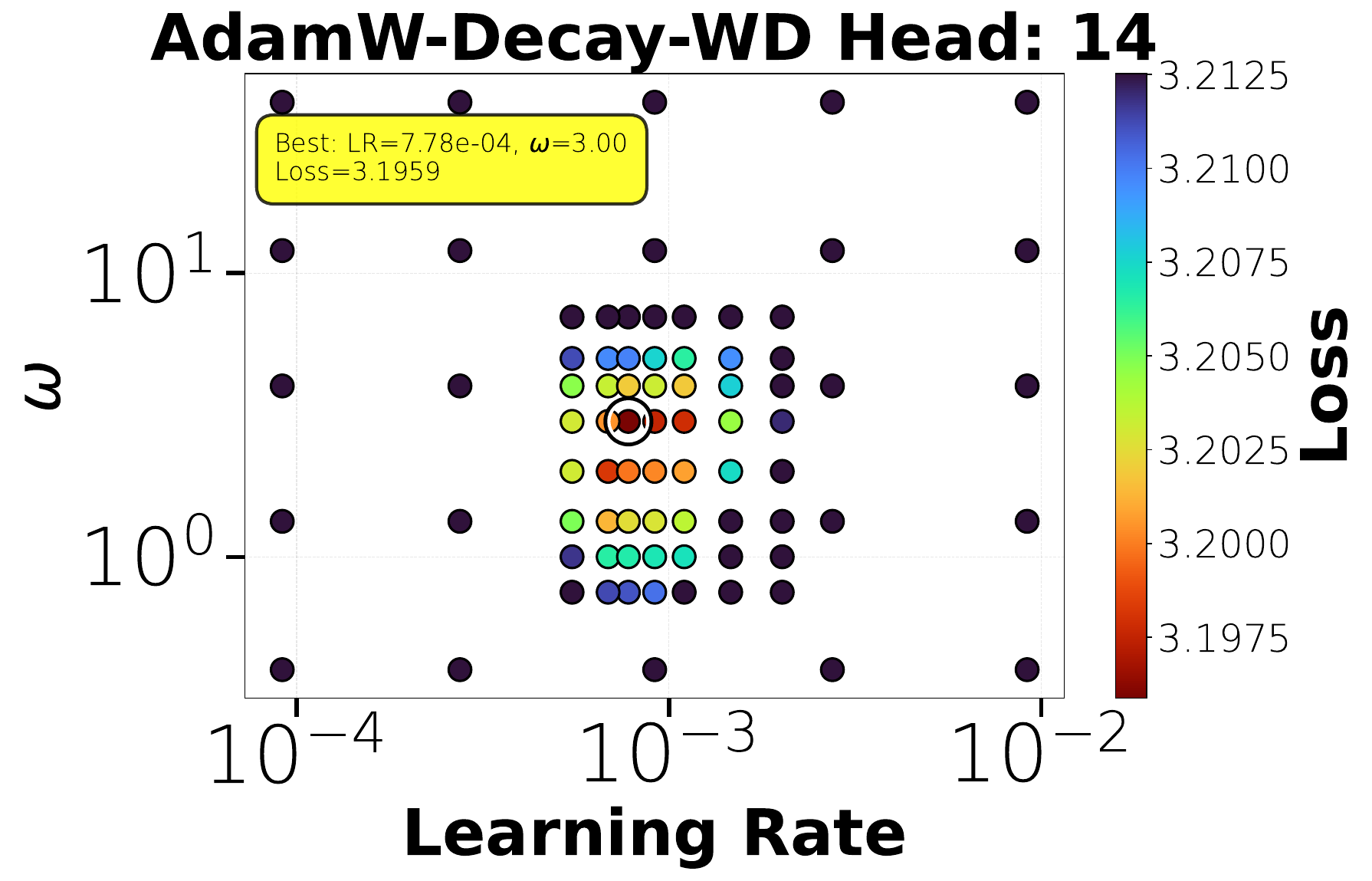}
\caption{14 heads}
\label{fig:wd_heatmap_adamw_decaying_14}
\end{subfigure}
\hfill
\begin{subfigure}[t]{0.32\textwidth}
\centering
\includegraphics[width=\textwidth]{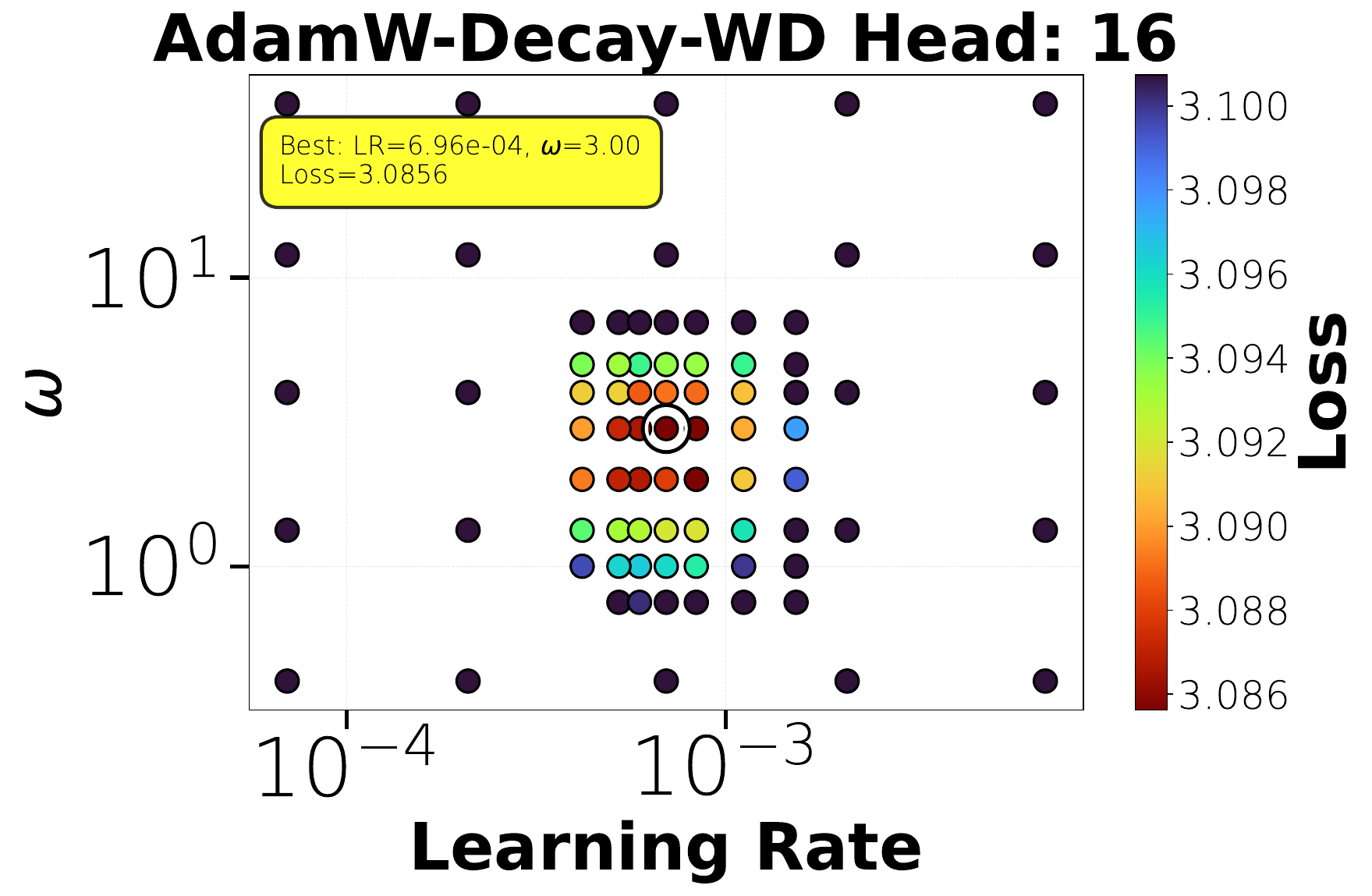}
\caption{16 heads}
\label{fig:wd_heatmap_adamw_decaying_16}
\end{subfigure}
\caption{\textbf{Validation loss heatmaps for time-decaying weight-decay (\AdamW-decaying-wd) across different model sizes.} Compared to the constant weight decay heatmaps in Figure~\ref{fig:wd_heatmap_adamw}, the optimal loss values are consistently lower across all scales.}
\label{fig:wd_heatmap_adamw_decaying}
\end{figure}




\paragraph{Intuition for Logarithmic-Time Weight-Decay} We previously motivated logarithmic-time schedules by the mismatch between the exponential forgetting induced by constant $\beta_1,\beta_2,\lambda$ and the increasing time scale over which meaningful updates occur during training. 


A complementary perspective is obtained by assuming approximate power-law training dynamics. Concretely, suppose the loss $\mathcal{R}(\theta_t)$ and related quantities (such as expected gradients) evolve smoothly with iteration $t$, in a way that is well approximated by power-law scaling. In this regime, these quantities change slowly on multiplicative time scales (e.g., between $t$ and $2t$), rather than on fixed additive time intervals. This behavior is commonly observed in language modeling, where losses, gradient norms, and related statistics exhibit approximate power-law trends over the course of training.

Consider to simplify constant learning rate ($\gamma(t)\equiv 1$). We can approximately rewrite the independent weight-decay update \eqref{eq:decoupled_wd} in closed form as
\begin{equation}
\label{eq:wd_closed_form}
    \theta_t \approx \exp(-\sum_{s=0}^{t-1} \lambda_s) \left(\theta_0 - \gamma^*\sum_{s=0}^{t-1}g_s \exp(\sum_{r=0}^s \lambda_r) \right).
\end{equation}
This expression makes explicit how weight decay reweights past gradient contributions.
If the scaled gradients $g_s$ are approximately constant on any fixed relative time window (e.g., $[t/2, t]$), then weight decay should attenuate updates over such windows by a constant factor to avoid degeneracy. This requires $\sum_{r=t/2}^t\lambda_r$ to be constant uniformly in $t$. Under this condition, a harmonic schedule emerges naturally:
\begin{equation}
\label{eq:wd_natural_schedule}
\lambda_t =\frac{\omega}{t}
\end{equation}



\paragraph{Approximation to Logarithmic Time Weight-Decay} In practice, we use an approximation to logarithmic-time decaying by adding a saturation term $\nicefrac{T}{10}$ to the denominator $\lambda(t)= \frac{\omega}{\nicefrac{T}{10}+t}$ instead of $\lambda(t)=\frac{\omega}{t}$ in the \ADana implementation \Cref{alg:adana} and more generally in all instances of decaying weight-decay in our experiments. Therefore, in effect the weight-decay schedule is only in logarithmic time for a constant fraction of $T$ at the end of training. On the other hand, at the start of training $t\ll T$, the weight-decay is approximately constant $\lambda(t)\approx \frac{10\omega}{T}$ and $10$ times larger than with constant weight decay $\lambda(t) = \frac{\omega}{T}$ for the same renormalized factor $\omega$.
Despite this approximation, the qualitative behavior of time-decaying weight decay is preserved. Early updates experience stronger regularization while later updates (corresponding to rarer patterns in Zipfian data) experience weaker regularization. This temporal redistribution is what leads to the improved performance we observe empirically.

\paragraph{Additional Details on Sweeps from \Cref{fig:opt_weight_decay_sweeps}.}

In \Cref{fig:wd_heatmap_adamw,fig:wd_heatmap_adamw_decaying} we show the joint sweeps on peak LR $\gamma^*$ and weight-decay renormalized factor $\omega$ at each head size of the Enoki Model and with optimizer \AdamW. We perform sweeps for respectively constant weight-decay $\lambda(t) = \frac{\omega}{t}$ in \Cref{fig:wd_heatmap_adamw} and logarithmic time decaying $\lambda(t) = \frac{\omega}{\nicefrac{T}{10}+t}$ weight-decay in \Cref{fig:wd_heatmap_adamw_decaying}. The learning rate schedule $\gamma(t)$ is the same as in Table~\ref{table:training_config_main} (linear warmup and cosine decay).

\begin{figure}[h!]
\centering
\includegraphics[scale = .2]{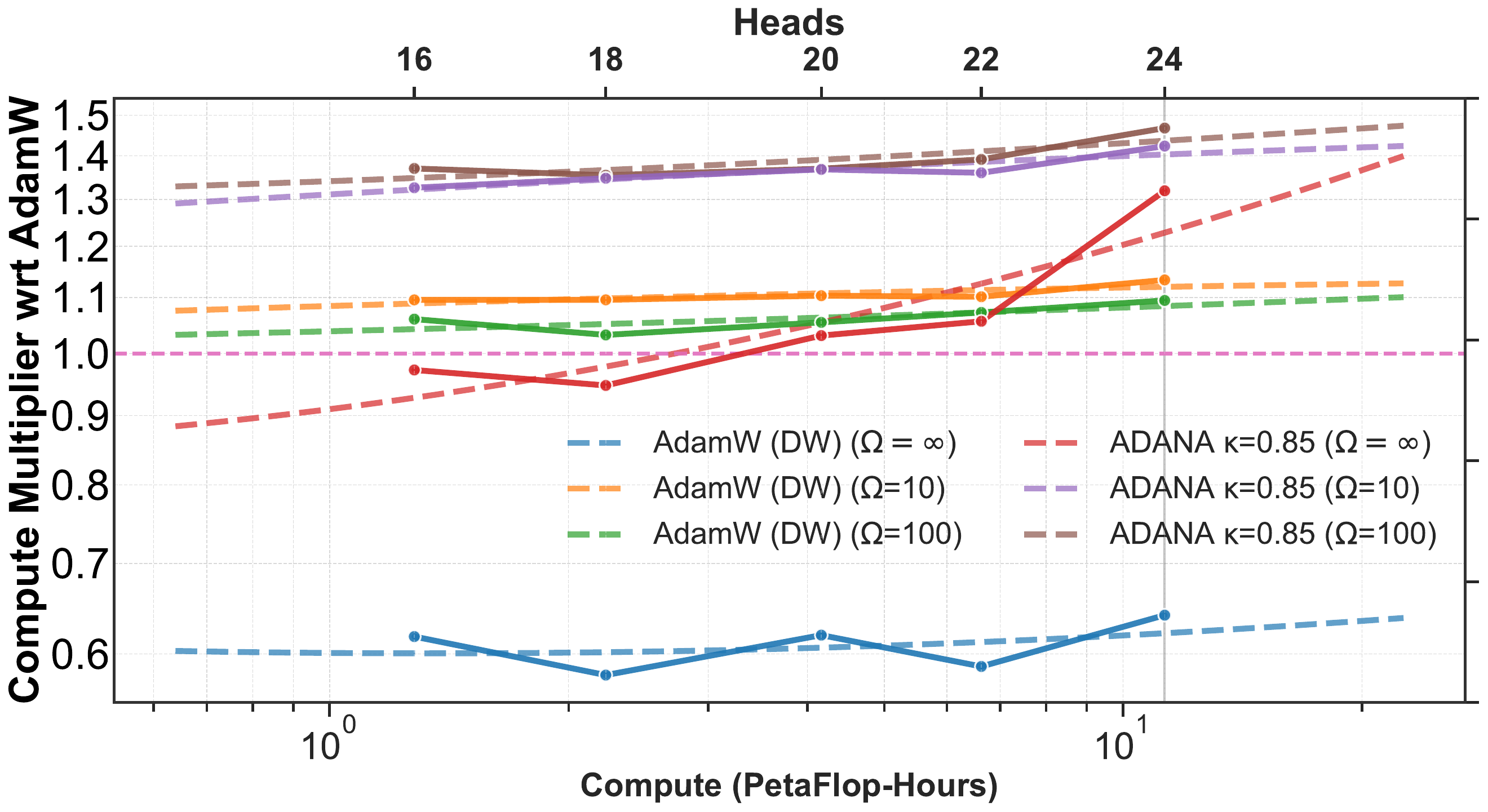}
\caption{\textbf{Warm-up in Weight-decay.} We consider a weight-decay schedule of the form $\lambda(t) = \frac{\omega}{T/\Omega + t}$ and vary $\Omega \in \{10,100\}$, with $\Omega = \infty$ corresponding to $\lambda(t) = \omega/t$. The parameter $\Omega$ acts as a warm-up for weight decay, keeping it approximately constant early in training before transitioning to a $1/t$ decay. We find that choosing $\Omega = 10$ or $100$ has little effect on the performance of \ADana or \AdamW (DW), while consistently improving performance relative to \AdamW with constant weight decay. In contrast, the performance of \AdamW (DW) degrades substantially when using the pure $\omega/t$ schedule, whereas \ADana remains stable under this schedule, at least for larger model sizes but does worse for smaller sizes. It appears that \ADana is less sensitive to the choice of $\Omega$.}
\label{fig:WD_Omega}
\end{figure}




\section{Building \ADana: Additional Details \& Theory}
\label{sec:appendix_building_adana}


This section develops the intuition behind \ADana (Alg.~\ref{alg:adana}) and why each component is necessary in order for \ADana to improve upon standard \Adam. We additionally discuss some problems/challenges that arise when attempting to add log-time schedules for the 1st/2nd moment buffers of \Adam and how \ADana mitigates these problems. 

Suppose $\mathscr{R} \, : \,  \R^d \to \R$ is the expected risk with loss function $\mathcal{L}$ that takes in parameters $\theta \in \R^d$ and data $x$. Consider the following learning problem:
\begin{equation} \label{eq:appendix_loss_function}
\min_{\theta \in \R^d} \mathscr{R}(\theta) \defas \E_{x}[ \mathcal{L}(\theta; x)],
\end{equation}
where we denote the stochastic gradient $g_{t+1} \defas \nabla \mathcal{L}(\theta_t, x_{t+1})$. 

Before going into details about each component of \ADana, let us start from one of the widely used algorithms in machine learning, \Adam (Alg.~\ref{alg:adamW}). Given initializations $m_0, \theta_0, v_0 \in \R^d$, \Adam generates an independent data sample $x_{t+1}$ and updates by
\begin{align*}
\text{\emph{($1^{\text{st}}$ mom.)}}\quad & m_{t+1} = \beta_1 m_t + (1-\beta_1) g_{t+1}\\
\text{\emph{($2^{\text{nd}}$ mom.)}}\quad & v_{t+1} = \beta_2 v_t + (1-\beta_2) g^2_{t+1}\\
\text{\emph{(param.)}}\quad & \theta_{t+1} - \theta_t - \gamma \frac{m_{t+1}}{\sqrt{v_{t+1}} + \epsilon}.
\end{align*}
Here $\beta_1, \beta_2 \in (0,1)$ are the 1st and 2nd moment hyperparameters, $\gamma > 0$ is the learning rate, and $\epsilon > 0$ is a numerical stability constant. Moreover we use the notation that $g_{t+1}^2 = g_{t+1} \odot g_{t+1}$. A common practical hyperparameter choice is to set $\beta_1$ and $\beta_2$ to be fixed constants independent of model size and iteration count, $t$. 

Throughout this discussion, we will refer to the notion of \textit{outscaling}, which we define below.

\paragraph{Outscaling.} A \textit{training regime}, $t \asymp d^{\ell}$, $\ell > 0$, is a scaling of iterations (or samples) to parameters. There are many examples of training regimes, e.g., the \emph{proportional regime} ($t \asymp d$) or the \emph{compute-optimal regime}, in which one selects the $\ell$ that yields the best loss under a fixed compute budget. Now suppose the loss under an algorithm follows the scaling law
\begin{equation} \label{eq:scaling_law}
\mathscr{R}(t,d) \defas \mathscr{R}(\theta_t, d) \asymp t^{-\sigma} + d^{-\tau} \quad \text{and suppose $t \asymp d^{\ell}$, $\ell > 0$ is a training regime.}
\end{equation}
Then under this training regime, the loss satisfies  $\mathscr{R}(d^{\ell}, d) \asymp d^{-\min\{\ell \sigma,\tau\}}$. We call the absolute exponent on $d$, the \textit{loss exponent}. For a given training regime, we say an algorithm \textit{outscales} another algorithm if the loss exponent is larger. 

We emphasize that this notion of outscaling differs in one key aspect from the more traditional notion of ``acceleration'' from optimization theory. Acceleration is typically formulated for a fixed-dimensional problem with constants that can have large $d$-dependence (e.g., $\|\theta_0-\theta^{\star}\|^2$). In other words, acceleration generally denotes outperformance when $t\to\infty$ and $d=O(1)$.

While model architectures and training methods have advanced rapidly, it has long been unclear whether innovations in \textit{optimization algorithms} could fundamentally change the loss exponents \cite{hestness2017deep}. Some evidence suggests that major advances like the \Adam optimizer \cite{kingma2015adam}
primarily improve the constants in the scaling law rather than improving its exponent \cite{hestness2017deep}.

\subsection{Building \ADana: Benefit of log-momentum and importance of damping schedule $\alpha(t)$}
We begin by providing intuitive insight into the benefits of a log-time schedule for $\beta_1$. In particular, we explain how combining a \textit{\textbf{log-time $\beta_1$ schedule with an $\alpha(t)$ damping schedule}} is theoretically grounded, leading to compute-efficiency gains that persist as model size increases and to improved loss-scaling exponents relative to SGD.

This section begins by building intuition directly from two optimization algorithms: stochastic gradient descent with momentum, 
\begin{equation} \label{eq:sgd_m_appendix_building_adana} \tag{SGD-M}
m_{t+1} = \beta_1 m_t + (1-\beta_1) g_{t+1} \quad \text{and} \quad \theta_{t+1} = \theta_t - \gamma (g_{t+1} + m_{t+1}) 
\end{equation}
and a naive implementation of stochastic Nesterov's accelerated method \cite{nesterov1983method}
\begin{equation} \label{eq:snag_appendix_building_adana} \tag{S-NAG}
\theta_{t+1} = y_t - \gamma g_{t+1}, \quad y_{t+1} = \theta_{t+1} + \mu_t (\theta_{t+1} -\theta_t), \quad \text{where} \quad \mu_t = \frac{t}{t+3} = 1 -\frac{3}{t+3}. 
\end{equation}

\subsubsection{Constant $\beta_1$ may not accelerate or outscale SGD}

Consider a simple momentum-based stochastic algorithm SGD-M that updates by
\begin{align*}
\text{\emph{(mom.)}}\quad & m_{t+1} = \beta_1 m_t + (1-\beta_1) g_{t+1}\\
\text{\emph{(param.)}}\quad & \theta_{t+1} = \theta_t - \gamma g_{t+1} - m_{t+1}.
\end{align*}
Assuming that $m_0 = 0$, we can write $m_{t+1} = (1-\beta_1) \sum_{j=0}^t \beta_1^j g_{t+1-j}$. In this sense, gradients which are from a long time ago will get killed exponentially fast. In small batch setting, the stochastic gradient $g_t$ does not change very much on short-time horizons. Note this is not the case in the full-batch setting. Therefore, the momentum parameter $m_{t+1} \approx g_{t+1}$ (in small batch) and we see that the parameter update $\theta_{t+1} \approx \theta_t - \gamma g_{t+1} - \text{constant($\beta_1$)} \times g_{t+1}$, that is, SGD-M with constant $\beta_1$ (in small batch setting) becomes SGD with a different learning rate. This observation has been proven in specific settings (see e.g., \citet{paquette2021dynamics, kidambi2018on,sebbouh2020almost,zhang2019which}).

\subsubsection{Benefits and issues with adding log-momentum, $\beta_1(t) = 1-\delta(\delta+t)^{-1}$} \label{sec:appendix_adding_log_momentum}
It is well known from deterministic optimization that \textit{log momentum} is beneficial for non–strongly convex objective functions. In particular, this corresponds to choosing $\beta_1(t) \asymp (1+t)^{-1}$. This idea is a key component of the celebrated Nesterov accelerated gradient method and explains, in part, why the algorithm achieves acceleration in the deterministic setting. This acceleration is desirable for us. However, long momentum of this form suffers from stability issues when transitioning from deterministic gradients to (small-batch) stochastic gradients (see, e.g., Thm.~7 in \cite{even2021continuized}).

\paragraph{Instability using S-NAG.}
We now explain the source of this instability and how it arises in a naive stochastic implementation of Nesterov’s accelerated method. Consider the stochastic version of Nesterov’s accelerated method \cite{Nesterov1988On}, defined by the updates
\begin{equation} \label{eq:nesterov}
\begin{aligned}
\theta_{t+1} = y_t - \gamma \nabla \mathcal{L}(y_t),
\quad
y_{t+1} = \theta_{t+1} + \mu_t(\theta_{t+1} - \theta_t),
\quad
\text{where } \mu_t = \frac{t}{t+3}.
\end{aligned}
\end{equation}
Here $\gamma > 0$ is the learning rate, and for simplicity we write $\mathcal{L}(y_t) \defas \mathcal{L}(y_t, x_{t+1})$ to denote a stochastic gradient. Note that Nesterov’s accelerated method is originally formulated using full-batch (deterministic) gradients; here, we replace the deterministic gradient with a stochastic one. This constitutes the most naive stochastic variant of Nesterov’s accelerated method. We refer to this formulation as \textit{Two-sequence Nesterov}.

An equivalent formulation of \textit{Two-sequence Nesterov} (see Section~\ref{sec:standard_nesterov_formulations}) is what we call \textit{Extra-gradient Nesterov} (see also \cite{defazio2025smoothing}). The updates take the form
\begin{equation}
m_{t+1} = \mu_{t-1} m_t + g_{t+1}
\quad \text{and} \quad
\Phi_{t+1} = \Phi_t - \gamma (g_{t+1} + \mu_t m_{t+1}),
\quad \text{where } g_{t+1} = \nabla \mathcal{L}(\Phi_t, x_{t+1}).
\end{equation}
We will use this formulation going forward. To build intuition for why instability arises when full-batch gradients are replaced with stochastic ones, suppose that
\[
\nabla \mathcal{L}(\Phi_t, x_{t+1}) = \nabla \mathscr{R}(\Phi_t) + \epsilon \quad \text{where $\epsilon > 0$ is some fixed error in the gradient.}
\]
In practice, $\epsilon \to 0$, though typically at a slow rate. Unraveling the recursion for $m_{t+1}$, the contribution of the noise term $\epsilon$ takes the form
\[
m_{t+1} = \underbrace{\sum_{k=1}^t \left ( \prod_{i=k+1}^t \mu_{i-1} \right ) \nabla \mathscr{R} (\Phi_k)}_{\text{(unravel full gradient)}} + \epsilon \cdot \sum_{k=1}^t \left ( \prod_{i=k+1}^t \mu_{i-1} \right ) \defas G_t + \epsilon \cdot \sum_{k=1}^t \left ( \prod_{i=k+1}^t \mu_{i-1} \right ).
\]
Specializing to $\mu_{t-1} = \frac{t}{t+3}$ yields
\[
m_{t+1} = G_t + \epsilon \cdot \frac{t+3}{4}.
\]
The key observation is therefore that \textbf{\textit{noise in the stochastic gradient accumulates in $m_{t+1}$}}. From an order-of-magnitude perspective, the stochastic gradient satisfies $g_{t+1} = \nabla \mathscr{R}(\Phi_t) + \epsilon$, while the momentum term behaves as $m_{t+1} = G_t + \epsilon \cdot \tfrac{t+3}{4}$. Since $\mu_t \approx 1$, these two contributions have very different dependence on $\epsilon$ in the update for $\Phi_t$. Because there is only \textit{\textbf{one learning rate}} $\gamma$ controlling both terms, the method is expected to be unstable. This instability was formally established for the PLRF model (Section~\ref{sec:appendix_synthetic}) in Theorem~7 of \cite{even2021continuized}.

Finally, we note that there are many equivalent formulations of Nesterov’s method. A detailed discussion of these equivalences can be found in Section~\ref{sec:standard_nesterov_formulations} and Table~\ref{table:nesterov_versions}. Of particular relevance—given its connection to the $\beta_1$ parameter in \AdamW—is what we refer to as \textit{Exponential Moving Average (EMA) Nesterov}:
\[
p_{t+1} = \mu_{t-1} p_t + (1-\mu_t) g_{t+1} \, \, \text{and} \, \, \Phi_{t+1} = \Phi_t - \gamma \big ( g_{t+1} + \frac{\mu_t}{1-\mu_{t-1}} \cdot p_{t+1} \big ), \, \, \text{where $p_{t+1} \approx \frac{m_{t+1}}{1-\mu_{t-1}}, g_{t+1} = \nabla \mathcal{L}(\Phi_t, x_{t+1})$.}
\]
This formulation is not exactly equivalent to the previous ones, but becomes asymptotically equivalent as $t \to \infty$. Importantly, all of these variants suffer from the same gradient noise accumulation issue. This formulation corresponds to \ADana with $\beta_2 = 1$.

\paragraph{Generalized Nesterov: Fixing instability of log-momentum and adding a damped term.}
To mitigate the accumulation of noise from stochastic gradients, we introduce the \textbf{\textit{Generalized Nesterov's accelerated method}}:
\[
m_{t+1} = \mu_{t-1} m_t + g_{t+1} \quad \text{and} \quad \Phi_{t+1} = \Phi_t - \gamma ( \underbrace{g_{t+1}}_{\text{grad. term}} + \underbrace{\textcolor{KPPcolor}{\bm{\hat{\alpha}_t}} \cdot m_{t+1}}_{\text{momentum term}} ) \quad \text{where $g_{t+1} = \nabla \mathcal{L}(\Phi_t)$ and $\mu_t = \frac{t}{t+3}$.}
\]
Here, $\textcolor{KPPcolor}{\bm{\hat{\alpha}_t}}$ is an additional learning rate parameter that must be carefully tuned. We present this generalization using the \textit{Extra-gradient Nesterov} formulation, but we also provide the corresponding generalization for the \textit{EMA Nesterov} formulation due to its similarities with \ADana. Related formulations have appeared previously, e.g., in \cite{defazio2024road,pagliardini2024ademamix,defazio2025smoothing,varre2022accelerated,flammarion2015from}. 
\begin{table}[h!]
\centering
    \begin{tabular}{lll}
         & \textbf{Momentum Update} & \textbf{Parameter Update}  \\
        \hline \\[-2ex]
        \textcolor{FPPcolor}{\textbf{Standard Nesterov-Extra grad. \eqref{eq:dana_nesterov_formulation}}}  & $m_{t+1} = \mu_{t-1} m_t + g_{t+1}$ & $\Phi_{t+1} = \Phi_t -\gamma (g_{t+1} + \textcolor{FPPcolor}{\bm \mu_t} \cdot m_{t+1})$\\
         \textcolor{KPPcolor}{\textbf{Generalized Nesterov-Extra grad. \eqref{eq:generalized_dana_nesterov_formulation}}} & $m_{t+1} = \mu_{t-1} m_t + g_{t+1}$ & $\Phi_{t+1} = \Phi_t -\gamma (g_{t+1} + \textcolor{KPPcolor}{\bm {\hat{\alpha}_t}} \cdot m_{t+1})$ \\  \\[-2ex]
         \hline
         \\[-2ex]
         \textcolor{Mutedred}{\textbf{Standard Nesterov-EMA \eqref{eq:EMA_Nesterov}}}  & $p_{t+1} = \mu_{t-1} p_t + (1-\mu_{t-1})g_{t+1}$ & $\Phi_{t+1} = \Phi_t -\gamma (g_{t+1} + \textcolor{Mutedred}{\bm{\frac{\mu_t}{1-\mu_{t-1}}}} \cdot m_{t+1})$\\
         \textcolor{KPPcolor}{\textbf{Generalized Nesterov-EMA\eqref{eq:generalized_EMA_Nesterov}}} & $p_{t+1} = \mu_{t-1} p_t + (1-\mu_{t-1})g_{t+1}$ & $\Phi_{t+1} = \Phi_t -\gamma (g_{t+1} + \textcolor{KPPcolor}{\bm {\alpha}(t)} \cdot p_{t+1})$ \\
    \end{tabular}
\end{table}

In both formulations, we replace \textcolor{FPPcolor}{$\bm{\mu_t}$} in \eqref{eq:dana_nesterov_formulation} with \textcolor{KPPcolor}{$\bm{\hat{\alpha}_t}$}, or equivalently replace 
\textcolor{Mutedred}{$\bm{\frac{\mu_t}{1-\mu_{t-1}}}$}
with \textcolor{KPPcolor}{$\bm{\alpha(t)}$}. Since $\lim_{t \to \infty} \mu_t = 1$, we may treat $\mu_t$ as effectively equal to one in the asymptotic regime.

Before proceeding, we examine the effect of the additional hyperparameter \textcolor{KPPcolor}{$\bm{\hat{\alpha}_t}$} on Nesterov’s accelerated method; a summary is provided in the table below.

\begin{table}[h]
\centering
    \begin{tabular}{c|c|l}
         \textbf{Damping factor $\textcolor{KPPcolor}{\bm{\hat{\alpha}_t}}$} & \textbf{Damping factor $\textcolor{KPPcolor}{\bm{\alpha(t)}}$} & \textbf{Behavior}  \\
        \hline\\ [-2ex]
        $\hat{\alpha}_t = \mu_t \approx 1$ & $\alpha(t) = \frac{\mu_t}{1-\mu_{t-1}} \approx \frac{t}{3}$ & Standard Nesterov \\
        \textcolor{KPPcolor}{$\bm{\hat{\alpha}_t < \mu_t}$} & \textcolor{KPPcolor}{$\bm{\alpha(t) < \frac{\mu_t}{1-\mu_{t-1}} \approx \frac{t}{3}}$} & \textcolor{KPPcolor}{\textbf{Damped steps (more conservative)}}\\
        $\hat{\alpha}_t > \mu_t$ & $\alpha(t) >\frac{\mu_t}{1-\mu_{t-1}} \approx \frac{t}{3}$ & Amplified steps (more agressive)\\
        $\hat{\alpha}_t = 0$ & $\alpha(t) = 0$ & \text{No momentum contribution/reduces to SGD}
    \end{tabular}
\end{table}

Amplifying the momentum term (i.e., choosing $\alpha(t) \gtrsim \tfrac{t}{3}$) further increases the disparity between the magnitudes of the gradient and momentum contributions. In particular, setting $\alpha(t) = \tfrac{\mu_t}{1-\mu_{t-1}}$ recovers standard Nesterov acceleration, which is known to be unstable in the stochastic setting. Since our primary concern is stability—specifically, the accumulation of stochastic gradient noise in the momentum term—we instead seek to \textit{\textbf{dampen the momentum contribution}} in the parameter update. This motivates focusing on the regime $\hat{\alpha}_t < \mu_t \approx 1$.

At the same time, excessive damping causes the method to behave like standard SGD, forfeiting acceleration and scaling benefits. This raises the central question:
\begin{center}
\textit{Is there a choice of \textcolor{KPPcolor}{$\bm{\alpha(t)}$} that achieves both \textbf{stability} and \textbf{acceleration} at scale?}
\end{center}

The deliberate choice to \textit{dampen} the momentum term—so that the gradient and momentum contributions remain of comparable order in the parameter update rule—is precisely the motivation behind \ADana (Adaptive Damped Nesterov Acceleration).

\subsubsection{Benefit of log-momentum: choosing a damping schedule, \textcolor{KPPcolor}{$\alpha(t)$}, that enables acceleration as model sizes increase}

The main contribution of this section is to motivate the choice of the damping coefficient scaling schedule $\alpha(t)$ in \ADana:
\begin{equation}
\label{eq:alpha_schedule_appendix}
\textcolor{KPPcolor}{\bm{\alpha(t)} = (1+t)^{1-\kappa}},
\end{equation}
where $\kappa$ is a quantity that is tuned, but appears transferable across scales. 

\paragraph{Motivation for damping scaling schedule, $\alpha(t)$.} 
The specific form of $\alpha(t)$ in \eqref{eq:alpha_schedule_appendix} is motivated by a theoretically grounded analysis of \Dana on the PLRF model (see \cite{ferbach2025dimension}). The \Dana algorithm (see Alg.~\ref{alg:dana}) updates according to the generalized Nesterov (extra-gradient) form
\[
m_{t+1} = \left ( 1- \frac{\delta}{t+\delta} \right ) m_t + g_t \quad \text{and} \quad \Phi_{t+1}= \Phi_t - \gamma( g_t + \tilde{\gamma}(d) (1+t)^{-\kappa} m_{t+1}). 
\]
In particular, \cite{ferbach2025dimension} studies damping schedules of the form $\hat{\alpha}(t;d) \asymp \tilde{\gamma}(d) (1+t)^{-\kappa}$, where $\tilde{\gamma}$ may depend on the parameter dimension $d$. They establish the following results on a power-law random features (PLRF) model with data generated from a power-law distribution with data exponent $\rho$ and target exponent $\eta$, satisfying $2\rho + 2\eta > 0$ and $2\rho > 1$ (see Section~\ref{sec:appendix_synthetic}):
\begin{itemize}
\item The optimal damping scaling schedule is $\hat{\alpha}(t) \asymp (1+t)^{-1/(2\rho)}$. Moreover, \Dana with this schedule outscales SGD under the Chinchilla training regime\footnote{A \textit{training regime} is defined by a scaling $t \asymp d^{\ell}$, $\ell > 0$, relating the number of iterations (or samples) to the parameter dimension. Examples include the \emph{proportional regime} ($t \asymp d$) and the \emph{compute-optimal regime}, in which $\ell$ is chosen to minimize loss under a fixed compute budget. Suppose the loss obeys the scaling law
$\mathscr{R}(t,d) \defas \mathscr{R}(\theta_t, d) \asymp t^{-\sigma} + d^{-\tau},$
and $t \asymp d^{\ell}$. Then $\mathscr{R}(d^{\ell}, d) \asymp d^{-\min{\ell \sigma, \tau}}$. We refer to the absolute exponent on $d$ as the \textit{loss exponent}. For a given training regime, one algorithm \textit{outscales} another if it achieves a larger loss exponent.}, achieving both acceleration and compute-efficiency gains. This variant is referred to as \Danadecaying. Among schedules of the form $\hat{\alpha}_t = \tilde{\gamma}(d) (1+t)^{-\kappa}$, this choice is optimal in the sense that it yields the best loss scaling exponent. See also \cite{yarotsky2025sgd} for related algorithm.

\item A dimension-dependent but iteration-independent schedule $\hat{\alpha}_t \asymp \tfrac{1}{d}$ is also introduced. \Dana with this schedule, referred to as \Danaconstant, outscales SGD for nearly all choices satisfying $2\rho + 2\eta > 1$ and $2\rho > 1$, although it does not scale as favorably as \Danadecaying (see also \cite{varre2022accelerated} for related algorithms).

\item When $\kappa \ge 1$ in $\alpha(t)$, \Danadecaying does not accelerate or outscale over SGD. In particular, \cite{ferbach2025dimension} show that $\kappa = 1$, on the PLRF, was equivalent to doing Schedule Free-SGD \cite{defazio2024road}. This suggest that we should not expect \ADana to outscale (i.e., improve as the model size grows) if $\kappa = 1$.  \end{itemize}

Motivated by these findings on the PLRF model—which is widely studied due to its tractability and its ability to mimic scaling properties observed in LLMs—we adopt \textbf{\textit{$\bm{\alpha(t) = (1+t)^{1-\kappa}}$ in \ADana}}. The apparent discrepancy by a factor of $(1+t)$ relative to \Dana arises precisely from the $t$-scaling that appears when transitioning from the generalized extra-gradient Nesterov formulation \eqref{eq:generalized_dana_nesterov_formulation}, (\Dana), to the generalized EMA Nesterov formulation  \eqref{eq:generalized_EMA_Nesterov}, (\ADana).

\paragraph{Choosing $\kappa$ in the damping schedule: Practical guidance.} 
On the PLRF model, the optimal value of $\kappa$ depends on spectral properties of the data, in particular the power-law exponent. Consequently, theory predicts that $\kappa$ should remain unchanged as the problem dimension grows, making it a transferable hyperparameter across scales. In practice, the data power-law exponent is unknown and $\kappa$ must be tuned. Empirically (see Figure~\ref{fig:optimal_kappa} \& Figure~\ref{fig:alpha_scaling_chinchilla}), we find that $\kappa$ is indeed transferable across model sizes and relatively easy to tune. In our experiments on transformer architectures trained on FineWeb data, effective values of $\kappa$ lie in the range \textcolor{Mutedred}{\textbf{0.8-0.9}}, with \textcolor{Mutedred}{\textbf{0.85}} performing best across all scales considered in this work. Moreover, $\kappa$ appears to transfer well across architectures, performing robustly with minimal tuning when moving from Enoki models to Qwen.



\begin{tcolorbox}[colback=yellow!20, colframe=yellow!60, boxrule=0.5pt, arc=8pt, left=5pt, right=5pt, top=3pt, bottom=3pt]
\textbf{Summary: Benefits of log-momentum and damping schedule.}

\vspace{0.5em}
\textbf{Recommended schedule (used in \ADana):}
\[
\beta_1(t) = \delta(\delta+t)^{-1} \quad \text{and} \quad \alpha(t) = (1+t)^{1-\kappa}. 
\]

\vspace{0.5em}
\textbf{Motivation.}
This schedule is motivated by theoretical analyses of damped Nesterov-type methods on the PLRF model \cite{ferbach2025dimension}, where sublinear damping of the momentum term helps control stochastic noise accumulation while retaining acceleration-like behavior. 

\vspace{0.5em}
\textbf{Practical guidance for choosing $\kappa$.}
\begin{itemize}
    \item $\kappa < 1$ is essential: Small values of $\kappa$ strengthen the momentum contribution and may improve acceleration, but can increase sensitivity to stochastic gradient noise. Choosing a good $\kappa$ is important, but it is easily tunable on smaller models; a range of $\kappa$ give similar performance with no dimensioning returns as model sizes grow.  
        \item $\kappa \ge 1$ should lead to no performance gains: larger values over-damp the method and recover SGD-like behavior on the PLRF.
    \item Empirically, $\kappa$ appears relatively \emph{stable} and \emph{transferable across scales}: in our experiments, values in the range $\kappa \in [0.75, 0.9]$ performed well across model sizes, with $\kappa \approx 0.85$ serving as a robust default. This range of $\kappa$'s all give similar performance and compute gains persist (and do not diminish) as the model size grows on LLMs. 
\end{itemize}

\vspace{0.5em}
\textbf{Interpretation.} Adding log-momentum of the form $\beta_1(t) = \delta(\delta + t)^{-1}$ can help improve performance of an algorithm at scale, but one must also include a damping schedule for stability. 
The damping schedule $\alpha(t) = (1+t)^{1-\kappa}$ provides a practical interpolation between unstable Nesterov acceleration and overly conservative SGD, offering a principled way to stabilize log-momentum while preserving potential compute-efficiency gains in large-scale training.
\end{tcolorbox}

\subsection{Scheduling the 2nd moment buffer} \label{sec:log_momentum_beta_2}
We have seen that introducing log-momentum together with an appropriate damping schedule can accelerate stochastic momentum methods (at least on PLRF). We now turn to incorporating this idea into \Adam. In particular, our objective is to determine how the second-moment parameter $\beta_2$ should be scheduled when the first-moment parameter $\beta_1$ follows a log-time schedule.

\begin{center}
    \textit{\textbf{Goal.} Assuming $\beta_1$ uses a log-time schedule and a damping coefficient $\alpha(t)$, choose $\beta_2$ so as to achieve compute-efficiency gains in \Adam across the scaling ladder.}
\end{center}

In doing so, we will also highlight several potential pathologies that arise under different choices of $\beta_2$.

A naive approach to incorporating log-momentum into \Adam is to set $\beta_1 = 1 - \delta(\delta+t)^{-1}$ while keeping $\beta_2$ fixed at an absolute constant in $(0,1)$. However, this choice is empirically unstable even in simple settings such as the Mixture-of-Experts PLRF model (Section~\ref{sec:appendix_synthetic}). As we will later show, it is also provably unstable in the sparse-gradient regime (see Section~\ref{sec:sparse_gradient_necessary_condition}). We observe similar instability on the Enoki scaling ladder when using a constant $\beta_2$ together with a log-time $\beta_1$ (see Figure~\ref{fig:short_average_impact}).

In practice, it is common to choose $\beta_2 > \beta_1$ \cite{orvieto2025adam}, which fails under the naive log-momentum implementation above. This naturally suggests a second candidate: setting $\beta_2$ to follow the same log-time schedule as $\beta_1$. We refer to this variant as \Logadam (Alg.~\ref{alg:log_adam}), defined by
\begin{equation} \label{eq:logadam_appendix}
\begin{aligned}
m_{t+1} &= \left (1- \frac{\delta}{t+\delta} \right ) m_t + \frac{\delta}{t+\delta} \cdot g_{t+1}\\
v_{t+1} & = \left (1- \frac{\delta}{t+\delta} \right ) v_t + \frac{\delta}{t+\delta} \cdot g_{t+1}^2\\
\theta_{t+1} & = \theta_t - \gamma \frac{m_{t+1}}{\sqrt{v_{t+1}} + \epsilon}.
\end{aligned}
\end{equation}

Unfortunately, both variants—\Adam with log-time $\beta_1$ and fixed $\beta_2$, as well as \Logadam with log-time schedules for both moments—are unstable in the presence of sparse gradients. To make this instability precise, we next introduce a simple necessary condition for convergence. All proofs of theorems in this section can be found in Section~\ref{sec:proofs}.

\subsubsection{Sparse gradients and a necessary condition for convergence} \label{sec:sparse_gradient_necessary_condition}

One of the advantages of \Adam (and AdaGrad \cite{duchi2011adaptive}), as highlighted in the original \Adam paper by \cite{kingma2015adam}, is its ability to handle sparse gradients. Sparsity in stochastic gradients, however, can also cause optimization algorithms to diverge. As we will show, this issue arises for \Adam when using a logarithmic-time schedule for $\beta_1$, and it motivates placing $\beta_2$ on a logarithmic-time schedule as well. In this section, we develop the theoretical foundations explaining why log-time $\beta_1$ fails in the presence of sparse gradients and why choosing $\beta_2$ in log-time can partially—but not fully—mitigate this problem.

To model sparse gradients, consider a stochastic gradient oracle $\mathcal{O}$ which, for any parameter $\theta \in \mathbb{R}^d$, outputs a stochastic gradient $\mathcal{G} \sim \mathcal{O}(\theta)$ satisfying, for some constant $M > 0$,
\[
\E[\|\mathcal{G}\|^2] \le M \quad \text{and} \quad \E[\mathcal{G}] = \nabla \mathscr{R}(\theta).
\]
We induce sparsity via a \textit{sparse oracle} $\mathcal{O}_p$, which generates sparse stochastic gradients $\mathcal{G}' \sim \mathcal{O}_p$ according to
\[
\mathcal{G}' = B \cdot \mathcal{G} \quad \text{where} \quad \mathcal{G} \sim \mathcal{O}(\theta) \quad \text{and} \quad B \sim \text{Bernoulli}(p).
\]
Here $B \sim \text{Bernoulli}(p)$ is a Bernoulli random variable with probability $p$ of being $1$ and probability $(1-p)$ of being $0$. Therefore, the value of $p \in [0,1]$, represents the probability that the stochastic gradient is \textit{not sparse}. 

Now let us suppose that our gradients $g_t$ in \Adam are generated from our sparse oracle $\mathcal{O}_p$ and we perform the \Adam parameter update, $Y_t \defas \tfrac{m_t}{\sqrt{v_t} + \epsilon}$. Even though the gradients $g_t$ may be zero due to sparsity, because \Adam has cumulative 1st/2nd moments $m_t$ (and $v_t$), the parameters still evolve and change. If sparsity exists and the 1st moment accumulates enough (such as in log-momentum), then the updates $\theta_t$ will not decrease the loss and $\|\theta_t\|$ can grow and grow causing non-convergence. 

To make this behavior explicit, let $T_t$ be the times at which nonzero gradients occur and define the cumulative sum of \Adam updates between those times 
\begin{equation} \label{eq:Z_updates}
Z_\ell \defas \sum_{t = T_{\ell}}^{T_{\ell+1} -1} Y_t, \quad \ell \ge 1 \quad \text{where $Y_t \defas \frac{m_t}{\sqrt{v_t} + \epsilon}$ is the \Adam update.}
\end{equation}

\begin{theorem} \label{thm:long_adam_divergence} Suppose $\beta_1(t) = 1-\delta/(t+1)$ for some $\delta > 1$ and let $\beta_2(t) \in (0,1)$ be \textbf{any} monotone sequence. Let the gradients $g_{t} \sim \mathcal{O}_p$ be generated by our sparse gradient oracle and define $Z_j$ as in \eqref{eq:Z_updates}. Then the family of random variables $\{Z_j \, : \, j \ge 1, p \in (0,1), \epsilon> 0\}$ is unbounded, i.e., 
\[
\limsup_{p \to 0} \sup_{j \ge 1, \epsilon > 0} \E[ \|Z_j\|] = \infty.
\]
In fact, it follows that along a subsequence of $p \to 0$, $\|Z_2\| \to \infty$. 
\end{theorem}

Intuitively, this result says that as the probability of sparse gradients increase, the likelihood that the updates of \Longadam, with any monotonically decreasing $\beta_2 \in (0,1)$ sequence, can grow without bound. In particular, if sparse gradients are sufficiently common (small $p$), the iterates in \Logadam may diverge before one sees another nonzero gradient. This immediately yields the following corollary.

\begin{corollary}[Divergence of \Logadam for Lipschitz functions] \label{cor:divergence_long_adam} Suppose $\mathscr{R} \, : \, \R^d \to \R$ is Lipschitz. Let the gradients $g_t \sim \mathcal{O}_p$ be generated by the sparse oracle and $\{\theta_t\}$ be iterates generated by Alg.~\ref{alg:log_adam} (or with $\beta_2 \in (0,1)$ which is monotonically decreasing sequence) using these sparse gradients. Let $T_2$ be the time at which the second nonzero gradient occurs. Then there is a sequence of $p \to 0$ so that 
\[
\|\theta_{T_2}\| \to \infty. 
\]
\end{corollary}
Hence \Adam with $\beta_1(t) = 1-\delta(1+t)^{-1}$ is essentially uniformly bad; on \textit{every} problem with a sufficiently sparse oracle, the iterates can become arbitrarily large. In particular, \Logadam is unstable for any Lipschitz loss whose minimizer does not lie at infinity.

\paragraph{Necessary condition for convergence.} As a consequence of Corollary~\ref{cor:divergence_long_adam}  (and Theorem~\ref{thm:long_adam_divergence}), we get a necessary condition for convergence of \Adam: 
\begin{equation} \label{eq:necessary_condition}
\text{(Necessary condition for convergence)} \quad \E[ \|Z_j\|] \quad \text{is bounded.}
\end{equation}

\subsubsection{Intuition for Theorem~\ref{thm:long_adam_divergence} and examples of bad choices for $\beta_2$ with log-momentum.} 

We now provide intuition for the divergence result and examine several problematic choices of $\beta_2$ in the presence of log-momentum. For simplicity, we apply all arguments entry-wise and assume $\epsilon > 0$ is negligible.

Unraveling the \Adam updates with time-varying $\beta_1$ and $\beta_2$ yields the approximation
\begin{equation} \label{eq:Y_t}
Y_t = \frac{m_t}{\sqrt{v_t} + \epsilon} \approx \sum_{s=1}^t \frac{  (1-\beta_1(s)) \big ( \prod_{r = s+1}^t  \beta_1(r) \big )  g_s}{\sqrt{ \sum_{s=1}^t (1-\beta_2(s)) \big (\prod_{r=s+1}^t \beta_2(r) \big ) (g_s^2)}}. 
\end{equation}

\paragraph{\Adam with constant $\beta_1$ and $\beta_2$.}
Let us first understand the case when $\beta_1$ and $\beta_2$ are fixed constants (independent of time). In this case, $\prod_{r=s+1}^t \beta_1 = \beta_1^{t-s}$ and similarly for $\prod_{r=s+1}^t \beta_2 = \beta_2^{t-s}$. Thus, 
\[
Y_t \approx \frac{1-\beta_1}{\sqrt{1-\beta_2}} \times \sum_{s=1}^t \frac{ \beta_1^{t-s} g_s}{\sqrt{\sum_{s=1}^t \beta_2^{t-s} g_s^2} }.
\]
 Because of the constant $\beta_1$ and $\beta_2$, the tail history of $g_s$ gets small exponentially fast and thus, only the most recent gradients really matter. The small batch means that the gradients do not change very much over this recent history. Thus, we can replace this changing $g_s$ with the last gradient $g_t$, which implies that $g_t / \sqrt{g_t^2} \le 1$.  Now we need some summability of the $\beta_1,\beta_2$ to ensure that the update is bounded. As a result, we get the following.

\begin{theorem}[\Adam with constant $\beta_1$, $\beta_2$ satisfies necessary condition] \label{thm:constant_beta_necessary}
If $\beta_1, \beta_2 \in (0,1)$ are fixed constants independent of $t$ and $\beta_1^2 < \beta_2$, the family of random variables $\{Z_{\ell} \, : \, \ell \ge 1, p \in (0,1), \epsilon > 0\}$ satisfies the uniform first-moment bound
\[
\E[|Z_\ell|] \le \frac{1-\beta_1}{(\sqrt{1-\beta_2})\big (1- \tfrac{\beta_1}{\sqrt{\beta_2}} \big )^2 }.
\]
\end{theorem}
While this does not guarantee convergence, it explains why standard \Adam handles sparse gradients reasonably well, consistent with \cite{kingma2015adam}. However, constant $\beta_1$ does not yield the scaling-law improvements associated with log-momentum.

\paragraph{Bad choices for $\beta_2$ with log-momentum.} Let us now suppose that $\beta_1(t) = ( 1 - \delta/(t+1) )$, which implies that $\prod_{r=s+1}^t (1-\delta/(r+1)) \asymp \left ( \tfrac{s}{t} \right )^{\delta}$. To understand what happens in this case, let us focus only on what happens between the first non-zero gradient and the 2nd non-zero gradient (i.e., $t \in [T_1, T_2-1]$). For convenience, we can assume that $T \defas T_1$ and the time between the first and 2nd non-zero gradients is the same as the time to get the first non-zero gradient, i.e., $T = T_1 \asymp T_2-T_1$. In particular, we expect $T = 1/p$ where $p$ is the probability of a non-spare gradient. If sparse gradients happen, then $T$ is large. 

Fix $t \in [T_1, T_2-1] = [T, M \cdot T-1]$ where the constant $M > 1$. Since $t \in [T, MT-1]$, we know that only one non-zero gradient has occurred from the collection of gradients $\{g_s\}_{s=1}^t$. Let $X_{T}$ be this non-zero gradient. Thus, we have that
\begin{gather*}
m_{t} = \sum_{s=1}^t (1-\beta_1(s)) \big ( \prod_{r=s+1}^t \beta_1(r) ) g_s =  \tfrac{1}{T} \times X_{T} \times \big ( \tfrac{T}{t} \big )^{\delta}\\
v_{t} = \sum_{s=1}^t (1-\beta_2(s)) \big (\prod_{r=s+1}^t \beta_2(r) \big ) g_s^2 = (X_{T})^2 (1-\beta_2(T))   \prod_{r = T+1}^t \beta_2(r)  \\
\Rightarrow \quad Y_t = \frac{ \frac{1}{T} \times X_T \times \big ( \frac{T}{t} \big )^{\delta} }{|X_T| \sqrt{ (1-\beta_2(T)) \prod_{T+1}^t \beta_2(r)}} \quad \Rightarrow \quad |Z_1| = \sum_{t = T}^{MT -1} \frac{ \frac{1}{T} \times \big ( \frac{T}{t} \big )^{\delta} }{ \sqrt{ (1-\beta_2(T)) \prod_{T+1}^t \beta_2(r)}}.
\end{gather*}
It follows by a Riemann approximation (provided $\delta > 1$) that for large $T$ (or sparse gradients), $\frac{1}{T} \sum_{t=T}^{MT -1} \big ( \tfrac{T}{t} \big )^{\delta} \approx \int_0^{M-1} (1+u)^{-\delta} \, \dif u < \infty$. Since the numerator of $|Z_1|$ is order $1$, it means that if $|Z_1|$ is to be bounded, then we need the denominator to be also order $1$. Before continuing, we discuss below two choices for $\beta_2$ which fail the necessary condition~\ref{eq:necessary_condition}. 

\noindent \textbf{\textit{Bad choice 1: Log-momentum and constant $\beta_2$.}} Let us suppose that $\beta_2$ is a constant. Then we see that $\prod_{r=T+1}^t \beta_2 = \beta_2^{t-T}$. Since $\beta_2^{t-T} \to 0$ as $t \to \infty$, the denominator of $|Z_1|$ is going to $0$, causing $|Z_1|$ to be infinity. Thus \Adam with $\beta_1(t) = 1-\delta(t+1)^{-1}$ and $\beta_2$ constant is divergent, i.e., doesn't satisfy our necessary condition in \eqref{eq:necessary_condition}.

\noindent \textit{\textbf{Bad choice 2: Log-momentum and $\beta_2 = 1-c/(t+1)^2$.}} Now suppose we take an aggressive $\beta_2(t) = 1-c/(t+1)^2$. Then we see that 
\[
(1-\beta_2(T)) \prod_{r=T+1}^t \beta_2(r) \asymp \frac{1}{T^2}, \quad \text{as we know $\prod_{r=T+1}^t \beta_2(r) = \mathcal{O}(1).$} 
\]
Putting this together gives that
$|Z_1| = \sum_{t=T}^{MT -1} \big (\frac{T}{t}\big)^{\delta}$. To make this summation convergent for large $T$, requires a $1/T$. Multiplying and dividing by $1/T$ results in 
\[
|Z_1| = T \times \frac{1}{T} \sum_{t=T}^{MT-1} \bigg ( \frac{T}{t} \bigg )^{\delta} \asymp T \asymp \frac{1}{p}. 
\]
This again blows up as $p \to 0$ and \Adam fails to converge with this choice of $\beta_2$. 



\subsubsection{Best choice for $\beta_2$ in the sparse gradient setting, $\beta_2(t) = 1- \delta(1+t)^{-1}$}

These examples motivate choosing
\[\beta_2 = 1-\delta/(t+1).\]
We provide the follow result in this case. 

\begin{theorem} \label{thm:sparsity_main} If $\beta_1(t) = \beta_2(t) = (1-\tfrac{\delta}{t})$ and $\delta > 2$, the family of random variables $\{Z_{\ell} \, : \, \ell \ge 1, p \in (0,1), \epsilon > 0\}$ satisfies the first-moment bound, 
\[
\E[|Z_{\ell}|] \le \frac{C}{\sqrt{p}}, \quad \text{where $C > 0$ is a constant.} 
\]
\end{theorem}
This bound is substantially better than the $1/p$ divergence observed with more aggressive $\beta_2$ schedules, making this choice the \textit{least sensitive} to sparsity among log-time schedules.

However, even this choice does not fully resolve the issue: both \Logadam and \ADana may still diverge under sparse gradients.  
This theorem suggests that we still need to correct for sparsity by multiplying our updates by $\sqrt{p}$. Since the probability $p$ is unknown, we need to construct an online estimate of it. This motivates the development of \Danastar, a variant of \ADana explicitly designed to correct for sparsity by estimating the unknown probability $p$ online via a quantity denoted $\tau$. We discuss this and other variants of \ADana in Section~\ref{sec:appendix_failure_modes}.

\begin{tcolorbox}[colback=yellow!20, colframe=yellow!60, boxrule=0.5pt, arc=8pt, left=5pt, right=5pt, top=3pt, bottom=3pt]
\textbf{Summary: Log-time $\beta_2$ and sparsity.}
When using log-time momentum $\beta_1(t)=1-\delta/(t+1)$, the choice of the second-moment schedule $\beta_2$ plays a critical role in stability under sparse gradients. Among log-time schedules, setting
\[
\beta_2(t)=1-\delta/(t+1)
\]
is the \emph{least sensitive to sparsity}: it yields the weakest divergence rate of the accumulated updates between nonzero gradients among potential $\beta_2$ schedules. While this choice does not fully eliminate divergence under extreme sparsity, it substantially mitigates the instability observed with constant or faster-decaying $\beta_2$ schedules. 
\end{tcolorbox}

\subsection{Explanation for the log-time}
\label{sec:log_time_explanation}

We finish this section by providing an explanation for the term \textit{log-time}. 

Consider the \AdamW updates (with indepdent weight-decay) on time $\tau = \log(\delta + t)$, with constant learning rate $\gamma^*$ (the schedule $\gamma(t)=1$), $\beta_1 = \beta_2 = 1-\delta$, $\lambda = \omega$:
\begin{equation}
\begin{aligned}
\text{\emph{(mom.)}}\quad &m_{\tau+1} = (1-\delta) m_\tau + \delta g_{\tau+1}, \\
\text{\emph{($2^{\text{nd}}$ moment)}}\quad &v_{\tau+1} = (1 - \delta) v_\tau + \delta g_{\tau+1}^2, \\
\text{\emph{(param.)}}\quad &\theta_{\tau+1} = \theta_\tau - \gamma^* \cdot \frac{m_{\tau+1}}{\sqrt{v_{\tau+1}} + \epsilon} -  \omega \cdot \theta_\tau  .
\end{aligned}
\end{equation}
We can rewrite these updates with $\Delta \tau = 1$ as

\begin{equation}
\begin{aligned}
\text{\emph{(mom.)}}\quad &(m_{\tau+1} - m_\tau) \Delta \tau = \delta (g_{\tau+1} - m_\tau) \Delta \tau , \\
\text{\emph{($2^{\text{nd}}$ moment)}}\quad &(v_{\tau+1} - v_\tau)\Delta \tau = \delta (g_{\tau+1}^2 - v_\tau)\Delta \tau, \\
\text{\emph{(param.)}}\quad &(\theta_{\tau+1} - \theta_\tau)\Delta \tau = - \left(\gamma^*  \frac{m_{\tau+1}}{\sqrt{v_{\tau+1}} + \epsilon} - \omega \theta_\tau\right) \Delta \tau .
\end{aligned}
\end{equation}

Now we can do the change of variable from $\tau$ to $t$ and replace $\Delta \tau \approx \frac{\Delta t}{\delta+t}$ to obtain

\begin{equation}
\begin{aligned}
\text{\emph{(mom.)}}\quad &(m_{t+1} - m_t) \frac{\Delta t}{\delta + t} \approx \delta (g_{t+1} - m_t) \frac{\Delta t}{\delta + t} , \\
\text{\emph{($2^{\text{nd}}$ moment)}}\quad &(v_{t+1} - v_t)\frac{\Delta t}{\delta + t} \approx \delta (g_{t+1}^2 - v_t)\frac{\Delta t}{\delta + t}, \\
\text{\emph{(param.)}}\quad &(\theta_{t+1} - \theta_t)\frac{\Delta t}{\delta + t} \approx - \left(\gamma^*  \cdot \frac{m_{t+1}}{\sqrt{v_{t+1}} + \epsilon} - \omega \cdot \theta_t\right) \frac{\Delta t}{\delta + t}.
\end{aligned}
\end{equation}

Finally we recover 

\begin{equation}
\begin{aligned}
\text{\emph{(mom.)}}\quad &m_{t+1} \approx (1-\frac{\delta}{\delta+t}) m_t + \frac{\delta}{\delta+t} g_{t+1}, \\
\text{\emph{($2^{\text{nd}}$ moment)}}\quad &v_{t+1} \approx (1 - \frac{\delta}{\delta+t}) v_t + \frac{\delta}{\delta+t} g_{t+1}^2, \\
\text{\emph{(param.)}}\quad &\theta_{t+1} \approx \theta_t - \frac{\gamma^*}{\delta+t} \cdot  \frac{m_{t+1}}{\sqrt{v_{t+1}} + \epsilon} - \frac{\omega}{\delta+ t}\cdot \theta_t.
\end{aligned}
\end{equation}

Note the additional factor on the learning rate $\frac{\gamma^*}{\delta+t}$ which in practice we ignored. Indeed, intuitively we want to apply logarithmic time updates only on the averaging updates (weight-decay and momentum) and not the gradient descent part of the update. This is equivalent to adding a factor $\gamma^* \times \tau$ to the learning rate update in logarithmic time.

\section{The Hilberg Hypothesis and Logarithmic Time}
\label{sec:hilberg_hypothesis}


In this section, we provide a speculative but motivated connection between the Hilberg exponent from information theory and the logarithmic time schedules used in \ADana. The central idea is that the rate-limiting factor in language model training is feature learning on power-law distributed data, and the Hilberg exponent characterizes precisely how difficult this learning problem is.

\subsection{The Hilberg Exponent}


Following \citet{Shannon1951} we model natural language as a stationary stochastic process $(X_t)_{t \in \mathbb{Z}}$ taking values in a finite alphabet $\mathcal{A}$. The \emph{block entropy} $H(X_1^n) = H(X_1, \ldots, X_n)$ measures the total uncertainty in $n$ consecutive symbols, and the \emph{Shannon entropy rate} is defined as $h = \lim_{n \to \infty} H(X_1^n)/n = \lim_{n \to \infty} H(X_n \mid X_1^{n-1})$, where the second equality holds for stationary processes~\citep{Shannon1948}. The entropy rate $h$ captures the irreducible per-symbol uncertainty that remains even when the predictor has access to unbounded context.

\citet{Shannon1951} estimated the entropy of English through human guessing experiments at various context lengths, observing that per-symbol uncertainty decreases substantially as more context is provided. \citet{Hilberg1990} reanalyzed Shannon's data and proposed that block entropy grows sublinearly as $H(X_1^n) \sim n^\beta$ with $\beta \approx 0.5$, which would imply an entropy rate of $h = 0$. This original conjecture is now considered too strong. Following \citet{Crutchfield2003} and \citet{Ebeling1991}, the modern \emph{relaxed Hilberg hypothesis} retains a positive entropy rate and posits that
\begin{equation}
H(X_1^n) = h \cdot n + A \cdot n^\beta + o(n^\beta),
\end{equation}
so that the per-symbol encoding rate---the bits per symbol required by an optimal compressor with access to $n$ symbols of context---decays as $r(n) = H(X_1^n)/n = h + A n^{\beta - 1} + o(n^{\beta-1})$~\citep{Debowski2015}. The \emph{Hilberg exponent} $\beta \in (0,1)$ thus characterizes how quickly compression rates converge to the entropy rate as context grows.

\citet{Takahira2016} measured $\beta$ empirically using PPM compression~\citep{ClearyWitten1984} on corpora up to 7.8~GB across six languages (English, French, German, Japanese, Chinese, and Korean). They introduced a stretched exponential ansatz $r(n) = \exp(A n^{\beta-1} + h')$ that better fits the empirical compression curves, and found $\beta \approx 0.88$ with remarkably small cross-language variance (SD $= 0.034$). This consistency suggests that the Hilberg exponent is a \emph{language universal}, independent of the particular language or writing system.

The universality of $\beta$ suggests that while entropy rate $h$ measures how hard it is to \emph{predict} text, $\beta$ measures how hard it is to \emph{learn} to predict text---and all human languages appear equally hard to learn.

\subsection{The Hilberg Hypothesis: Systematic Noise in Logarithmic Time}

To connect the Hilberg exponent to optimization, we work in \emph{logarithmic time} $\tau = \log t$. This change of variables is natural for two reasons: (1) the momentum schedule $\beta_1(t) = 1 - \delta/t$ becomes a constant effective momentum in log-time, and (2) the Hilberg exponent describes how information accumulates across scales, where each unit of log-time corresponds to one scale.

We hypothesize that within each unit of logarithmic time (i.e., from $t$ to $2t$), the model is learning features at a particular frequency scale. The gradient noise during this period is \emph{systematic}, not independent: the model lacks certain features needed to predict tokens at this scale, and this manifests as correlated errors throughout the log-time window. The model makes the same mistakes repeatedly because it is missing the same features.

Concretely, let $\sigma(\tau)$ denote the scale of the systematic error per unit of log-time. Based on the Hilberg exponent, this error decays as
\begin{equation}
\sigma(\tau) \sim t^{\beta - 1} = e^{\tau(\beta - 1)},
\end{equation}
where $\beta \approx 0.88$ is the Hilberg exponent. In physical time, this is $t^{-0.12}$---a slow decay reflecting the difficulty of learning rare features.

This model differs fundamentally from standard stochastic optimization assumptions. The noise is not IID across iterations; rather, it is perfectly correlated within each log-time unit (same missing features) but approximately independent across log-time units (different frequency scales). This correlation structure means the noise does not average out within a log-time window.

\paragraph{Connection to damped Nesterov acceleration.} Consider the damping schedule $\alpha(t) = (1+t)^{1-\kappa}$ used in \ADana. In logarithmic time, this becomes $\alpha(\tau) = e^{\tau(1-\kappa)}$. For $\kappa = 0$, the damping grows exponentially in log-time (full Nesterov acceleration); for $\kappa = 1$, the damping is constant (no acceleration).

The momentum buffer in log-time spans $O(1)$ units, accumulating the systematic errors from recent scales. Since the noise is correlated within each log-time unit, we cannot appeal to variance reduction from averaging. Instead, the contribution to the optimizer state per log-time unit has scale
\begin{equation}
\alpha(\tau) \cdot \sigma(\tau) \sim e^{\tau(1-\kappa)} \cdot e^{\tau(\beta-1)} = e^{\tau(\beta - \kappa)}.
\end{equation}
For stability, this contribution should remain $O(1)$ as $\tau \to \infty$, which requires
\begin{equation}
\beta - \kappa \leq 0, \quad \text{i.e.,} \quad \boxed{\kappa \geq \beta}.
\end{equation}
This is the quantitative statement of the Hilberg hypothesis: the damping exponent $\kappa$ in \ADana should be at least as large as the Hilberg exponent $\beta$. With $\beta \approx 0.88$, this predicts $\kappa \geq 0.88$, and the optimal choice $\kappa = \beta$ exactly balances acceleration against the systematic noise.

\paragraph{Interpretation.} The spectral dimension parameter $\kappa$ encodes an assumption about the noise structure of the learning problem. When $\kappa = 0$, the algorithm assumes deterministic gradients and uses full acceleration. When $\kappa = 1$, the algorithm assumes persistent noise and forgoes acceleration entirely. The Hilberg hypothesis suggests that for natural language, the noise decays as $t^{\beta-1}$ with $\beta \approx 0.88$, and thus $\kappa$ should be chosen close to $\beta$.

This provides a principled justification for why $\kappa \in [0.8, 0.9]$ works well across model scales: it reflects a fundamental property of language data (the Hilberg exponent) rather than a hyperparameter to be tuned per architecture.





\subsection{Discussion}

The connection between the Hilberg exponent and the power-law damping factor that we use in combination with logarithmic time scheduling remains speculative and requires further theoretical development. However, we believe this perspective is valuable for several reasons:

\begin{enumerate}
    \item It provides a principled explanation for why $\kappa$ should be approximately constant across scales---it reflects properties of language data, not model architecture.
    \item It explains why $\kappa$ should be chosen close to the Hilberg exponent: our experiments use $\kappa \in [0.8, 0.9]$, which is consistent with $\beta \approx 0.88$.
    \item It suggests that the dominant noise in language model training is less the standard sampling variance but rather the variance due to information-theoretic limits on how well the learned features can explain the next token.
\end{enumerate}



\section{Failure Modes and Hardened Variants of \ADana: Full Details}
\label{sec:appendix_failure_modes}

This appendix provides complete details on the \Danastar, \DanaMKfour, and \DanastarMKfour algorithms that address the failure modes of logarithmic-time momentum under sparse gradients and inhomogeneous spectral dimensions.

\subsection{\Danastar: Accounting for Time Sparsity}

To address time sparsity, we introduce \Danastar, which incorporates a probability estimator $\tau$ that tracks how frequently each parameter receives meaningful updates. The key insight is that the effective time $t_{\texttt{eff}}$ experienced by a parameter should scale with the fraction of iterations in which it receives non-trivial gradient information.

\paragraph{Probability estimator, $\tau$.}
We define an instantaneous estimate of the `presence' of an update of a parameter
\begin{equation}
\tau_{\texttt{update}} = \frac{|g_{t+1}|}{|g_{t+1}| + \sqrt{v_{t+1}} + \epsilon}.
\end{equation}
This quantity measures the magnitude of the current gradient relative to the accumulated second moment. When $|g| \gg \sqrt{v}$, we have $\tau_{\texttt{update}} \approx 1$, indicating that the parameter is receiving consistent updates. Conversely, when $|g| \ll \sqrt{v}$, we have $\tau_{\texttt{update}} \approx 0$, indicating sparse or rare updates.  We now collect a time average of the presence of an update, to produce an estimate for the \emph{probability} of an update.
\begin{equation}
\tau_{t+1} = \left(1 - \frac{\delta}{t+1}\right) \tau_t + \frac{\delta}{t+1} \tau_{\texttt{update}}.
\end{equation}

\paragraph{Effective time and transformed probability.}
The effective time $t_{\texttt{eff}} = \max\{t \cdot \tau_{t+1}, 1\}$ rescales the iteration count by the update probability, ensuring that rarely-updated parameters do not accumulate excessive momentum. For frequently-updated parameters, $t_{\texttt{eff}} \approx t$ and \Danastar behaves like standard \ADana. For rarely-updated parameters, $t_{\texttt{eff}} \ll t$, reducing the momentum contribution.

We also define a clipped and transformed version of $\tau$ for use in the gradient scaling:
\begin{equation}
\tau_{\texttt{clip}} = \min\{\tau_{t+1}, 0.5\}, \quad \tilde{\tau}_{t+1} = \max\left\{ \frac{\tau_{\texttt{clip}}}{1-\tau_{\texttt{clip}}}, \frac{1}{1+t} \right\}.
\end{equation}
The clipping at $0.5$ prevents numerical instability when $\tau$ approaches $1$, and the floor at $1/(1+t)$ ensures the scaling remains well-behaved throughout training; note that this strategy can never accurately estimate a probability at time $t$ when that probability is less than $1/(t+1)$. The transformed quantity $\tilde{\tau}_{t+1}$ enters the parameter update through a $\sqrt{\tilde{\tau}}$ scaling factor that modulates both the gradient and momentum terms.  The transformation $\tau/(1-\tau)$ just corrects the clipped probability estimate to be $1$ for $\tau = 1/2$, and could be ignored as a reasonable optimization of the algorithm.

The full \Danastar algorithm is presented in Alg~\ref{alg:danastar}. 

\begin{figure}[t]
    \centering
    \includegraphics[width=\textwidth]{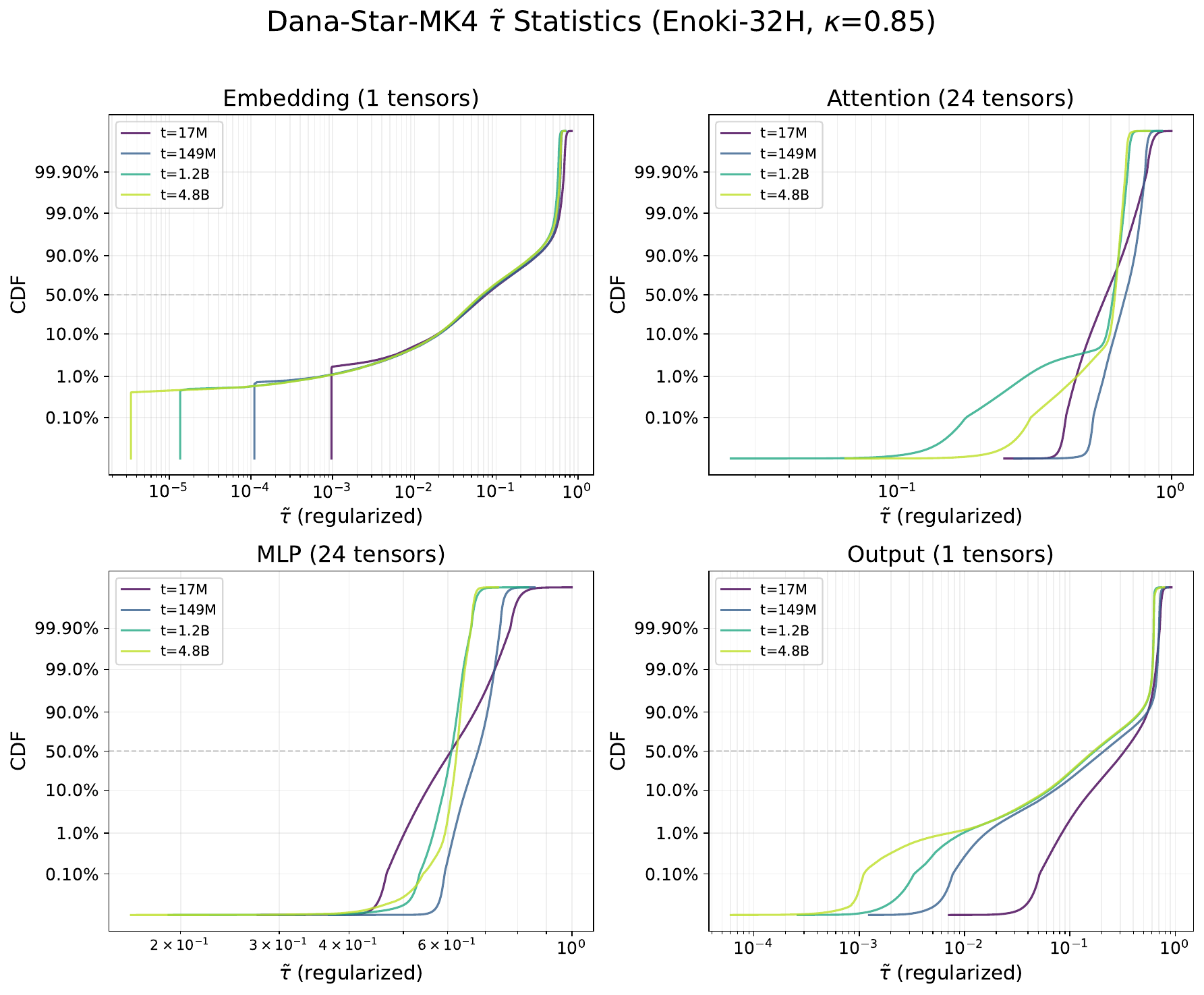}
    \caption{\textbf{Distribution of $\tilde{\tau}$ across parameter types.} Cumulative distribution of the regularized probability estimator $\tilde{\tau}$ for a 32-head Enoki model with $\kappa=0.85$. Each panel shows a different parameter type: embedding (wte), attention, MLP, and output (lm\_head). Curves show the distribution at different training stages from early (dark) to late (light). The embedding and output layers exhibit substantially lower $\tau$ values due to token-level sparsity, while attention and MLP layers have more concentrated distributions near $\tau \approx 1$.}
    \label{fig:tau_stats_summary}
\end{figure}

\paragraph{Visualizing the $\tau$ distribution across layers.}
Figure~\ref{fig:tau_stats_summary} shows the cumulative distribution of the regularized $\tilde{\tau}$ values across different parameter types during training. The four panels display key architectural components: the embedding layer (wte), attention weights, MLP weights, and the output projection (lm\_head). Each curve shows the distribution at a different training stage, from early (dark) to late (light) training.

Several patterns emerge. First, the embedding and output projection layers exhibit substantially lower $\tau$ values compared to the dense attention and MLP layers, reflecting the inherent sparsity in token-level gradients due to Zipfian token frequencies. Second, the distributions shift rightward over training as the $\tau$ estimator accumulates more accurate statistics. Third, the heavy tails at low $\tau$ values (visible as the rapid decrease below 1\%) indicate that a significant fraction of parameters receive sparse updates throughout training, justifying the need for the effective time scaling $t_{\texttt{eff}} = t \cdot \tau$.

\subsection{\DanastarMKfour / \DanaMKfour: Adapting the Exponent During Training}

While \Danastar addresses time sparsity, it still requires setting the exponent $\kappa$ correctly. \DanastarMKfour introduces an additional mechanism to adapt the effective exponent during training through SNR-based clipping. The key modification is that instead of directly using the momentum term scaled by $(1+t_{\texttt{eff}})^{1-\kappa}$, we clip the effective scaling factor based on a signal-to-noise ratio estimate.

\paragraph{SNR-based clipping.}
Define the normalized momentum factor as
\begin{equation}
\texttt{norm} = \frac{\sqrt{\tilde{\tau}_{t+1}}}{\sqrt{v_{t+1}} + \epsilon}, \quad \texttt{mfac} = \frac{|m_{t+1}| \cdot \texttt{norm}}{\tilde{\tau}_{t+1}}.
\end{equation}
The quantity $\texttt{mfac}$ measures the magnitude of the momentum relative to the second moment, serving as a proxy for the signal-to-noise ratio. We then define the clipped scaling factor
\begin{equation}
\alpha_{\texttt{fac}} = \min\left\{ t_{\texttt{eff}}^{1-\kappa} \cdot \texttt{mfac}, \, \texttt{clipsnr} \right\},
\end{equation}
where $\texttt{clipsnr}$ is a hyperparameter controlling the maximum allowed scaling. This clipping prevents the momentum term from dominating when the gradient signal is weak relative to variance, providing automatic adaptation to the local effective dimension.

\paragraph{Visualizing the SNR clipping mechanism.}
Figure~\ref{fig:snr_clipping_stats} illustrates the behavior of the key quantities in \DanastarMKfour during training. We track four statistics averaged across all parameters:
\begin{itemize}
    \item $\|m_t\|$: The norm of the first moment buffer. This quantity decays over time as the logarithmic time averaging accumulates more gradient history, even though the gradient norm itself does not decay. This motivates the need for the growing $(1+t)^{1-\kappa}$ scaling factor.
    \item $\alpha_{\texttt{fac}} / \texttt{mfac}$: The effective exponent scaling. When no clipping occurs, this equals $t_{\texttt{eff}}^{1-\kappa}$ (shown as dashed line). When SNR clipping is active, this is reduced to $\texttt{clipsnr} / \texttt{mfac}$.
    \item $\texttt{mfac}$: The normalized momentum magnitude serving as our SNR proxy. This quantity captures the ratio of momentum signal to accumulated variance.
    \item $\alpha_{\texttt{fac}}$: The actual scaling factor applied to the momentum term after clipping.
\end{itemize}

\begin{figure}[t]
\centering
\includegraphics[width=\textwidth]{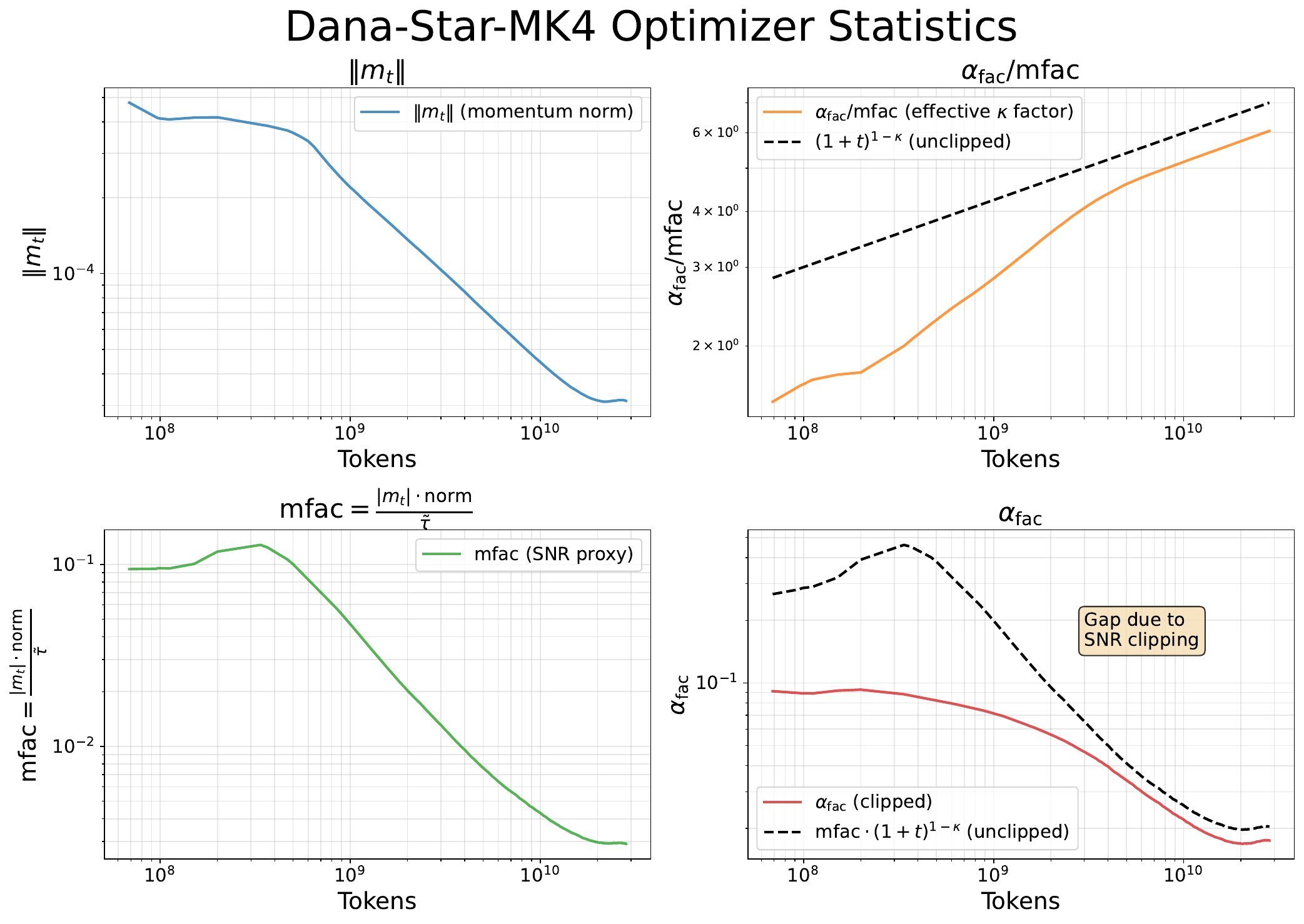}
\caption{\textbf{SNR clipping statistics during training.} Optimizer statistics for a 32-head Enoki model with $\kappa=0.85$ and $\texttt{clipsnr}=0.125$. \textbf{Top left:} The momentum norm $\|m_t\|$ decays over time. \textbf{Top right:} The effective $\kappa$ factor $\alpha_{\texttt{fac}}/\texttt{mfac}$ (solid) compared to the theoretical $(1+t)^{1-\kappa}$ schedule (dashed); the SNR clipping causes the actual schedule to fall below the theoretical curve. \textbf{Bottom left:} The $\texttt{mfac}$ provides the SNR proxy. \textbf{Bottom right:} The $\alpha_{\texttt{fac}}$ shows the actual scaling applied after clipping; the gap from the dashed line ($\texttt{mfac} \cdot (1+t)^{1-\kappa}$) reflects the effect of SNR clipping.}
\label{fig:snr_clipping_stats}
\end{figure}

\paragraph{Sign-based momentum update.}
The \Danastar-MK4 update takes the form
\begin{equation}
\theta_{t+1} = \theta_t - \tilde{\gamma}_2 \cdot g_{t+1} \odot \texttt{norm} - \tilde{\gamma}_3 \cdot \text{sign}(m_{t+1}) \odot \left(\tilde{\tau}_{t+1} \cdot \alpha_{\texttt{fac}} + |m_{t+1}| \odot \texttt{norm}\right).
\end{equation}
The use of $\text{sign}(m_{t+1})$ rather than $m_{t+1}$ directly provides additional robustness by decoupling the direction from the magnitude of the momentum contribution. This sign-based update is reminiscent of techniques used in signSGD and related methods, but here it is applied specifically to the logarithmic-time momentum term.

The full \DanastarMKfour algorithm is presented in \Cref{alg:dana_star_mk4}. The SNR clipping provides automatic adaptation: when the signal-to-noise ratio is high (clean gradient signal), the full $(1+t_{\texttt{eff}})^{1-\kappa}$ scaling is used; when the SNR is low (noisy signal), the momentum contribution is automatically reduced. This makes \DanastarMKfour more robust to misspecification of $\kappa$ and to heterogeneous effective dimensions across layers. 

We also define the \DanaMKfour algorithm (see \Cref{alg:dana_mk4}) in which we set $\tau \equiv 1$ in \DanastarMKfour. This is used to illustrate separately the reduction in sensitivity to the choice of $\kappa$ without the sparsity corrector. 

\subsection{Comparison of Variants}

\begin{table}[t]
\centering
\caption{\textbf{\ADana variant comparison.} Each variant addresses specific failure modes while maintaining the core benefits of logarithmic time scheduling.}
\label{table:adana_variants_appendix}
\begin{tabular}{l|l}
\toprule
\textbf{Variant} & \textbf{Key Feature} \\
\midrule
\ADana & Base logarithmic-time scheduling \\
\DanaMKfour & + SNR clipping ($\kappa$ robustness) \\
\Danastar & + $\tau$ estimator (sparse gradients) \\
\DanastarMKfour & + both \\
\bottomrule
\end{tabular}
\end{table}


On standard transformer training, all variants perform similarly. The robust variants show their value in challenging settings with sparse gradients.

For most applications, the base \ADana algorithm suffices. \DanaMKfour adds stability with essentially no tradeoff ($\texttt{clipsnr} \in [0.25,2.0]$ are effective and can be set and generally ignored). We do not have an obvious motivating use case for \Danastar or \DanastarMKfour, and they require additional memory overhead.  It could be that at larger scales, batch sizes, or models that directly integrate sparsity (e.g. through MOE or gating mechanisms),these methods are important.


The computational overhead of the robust variants is minimal, requiring only one additional division per parameter per step for the $\tau$ estimator. Consequently, they can be used as a safe default when robustness is valued over the marginal computational savings.

\section{Derivations for different formulations of Nesterov's Accelerated Method \& Generalized Nesterov Accelerated Method.}
\label{sec:appendix_nesterov}
In this section, we show how to derive different formulations for Nesterov's Accelerated Method and a generalized version -- useful when the gradients are stochastic. 

\subsection{Standard Nesterov Formulations} \label{sec:standard_nesterov_formulations}

We recall the naive stochastic Nesterov's accelerated method in \eqref{eq:nesterov}, which we called  \textcolor{F0color}{\textbf{\underline{Two-sequence Nesterov}:}} 
\begin{equation} \label{eq:nesterov_1}
\begin{aligned}
\theta_{t+1} = y_t - \gamma \nabla \mathcal{L}(y_t), \quad y_{t+1} = \theta_{t+1} + \mu_t(\theta_{t+1} - \theta_t), \quad \text{where \quad $\mu_t = \frac{t}{t+3}$}.
\end{aligned}
\end{equation}
Note that \cite{Nesterov1988On,beck2009gradient} use $1 - \frac{3}{t+3}$; however as noted in \cite{lan2012optimal, flammarion2015from}, one can also use this formulation with $\mu_t = 1-\frac{2}{t+1}$ to achieve the optimal rate of convergence for non-strongly convex, differentiable functions. 

\paragraph{ \textcolor{FACcolor}{\underline{Look-ahead Nesterov.}}} Our first goal is to rewrite this in terms of different set of variables and in particular, using a `look-ahead' term. Define new variables
\[
\boxed{v_{t+1} \defas \theta_{t+1} - \theta_t \quad \text{and} \quad \theta_t \defas \theta_t.} 
\]
Using the definition of $y_{t+1}$ in \eqref{eq:nesterov_1}, we have that $y_{t+1} = \theta_{t+1} + \mu_t v_{t+1}$ and $y_t = \theta_t + \mu_{t-1} v_t$. By definition of $\theta_{t+1}$ in \eqref{eq:nesterov_1} together with our $v_{t+1}$,
\begin{align*}
v_{t+1} 
&
= \theta_{t+1}-\theta_{t} = y_t - \gamma \nabla \mathcal{L}(y_t) - \theta_t\\
&
= y_t - \gamma \nabla \mathcal{L}(y_t) - y_t + \mu_{t-1} v_t
\\
&
= \mu_{t-1} v_t - \gamma \nabla \mathcal{L}(y_t).
\end{align*}
Therefore, we have in our new variables $(v_t, \theta_t)$, 
\begin{equation} \label{eq:look_ahead_nesterov}
\boxed{\theta_{t+1} = \theta_t + v_{t+1}, \quad v_{t+1} = \mu_{t-1} v_t - \gamma \nabla \mathcal{L}(\theta_t + \mu_{t-1} v_t), \quad \text{where} \quad \mu_t = \frac{t}{t+3}.}
\end{equation}
We call this formulation, \textit{Look-ahead Nesterov.}
\paragraph{ \textcolor{FPPcolor}{\underline{Extra-gradient Nesterov.}}} We now rewrite \eqref{eq:look_ahead_nesterov} in terms of general \Dana framework \cite{ferbach2025dimension} or an extra gradient. For this we define, 
\[
\boxed{\Phi_t \defas \theta_t + \mu_{t-1} v_t, \quad g_{t+1} \defas \nabla \mathcal{L}(\Phi_t), \quad \text{and} \quad m_t \defas \frac{-v_t}{\gamma}.} 
\]
Here again $\mu_{t} = \frac{t}{t+3}$. Since $\theta_{t+1} = \theta_t + v_{t+1}$ and $\theta_t = \Phi_t - \mu_{t-1} v_t$, we have that $\theta_{t+1} = \Phi_t -\mu_{t-1} v_t + v_{t+1}$. It follows that
\begin{align*}
\Phi_{t+1} 
&
= \theta_{t+1} + \mu_t v_{t+1} = \Phi_t - \mu_{t-1} v_t + v_{t+1} + \mu_t v_{t+1}
\\
&
= \Phi_t - \mu_{t-1} v_t + (1+\mu_t) v_{t+1}
\\
(v_{t+1} = \mu_{t-1} v_t - \gamma g_t ) \qquad &
=
\Phi_t - \mu_{t-1} v_t + (1+\mu_t) ( \mu_{t-1} v_t - \gamma g_{t+1})
\\
&
= \Phi_t + \mu_t \mu_{t-1} v_t - (1+\mu_t) \gamma g_{t+1}.
\end{align*}
Using that $m_t = -v_t / \gamma$ and $v_{t+1} = \mu_{t-1} v_t - \gamma g_{t+1}$, we have that 
\begin{equation} \label{eq:dana_nesterov_1}
m_{t+1} = \frac{-v_{t+1}}{\gamma} = \frac{-\mu_{t-1} v_t}{\gamma} + g_{t+1} = \mu_{t-1} m_t + g_{t+1} \quad \text{and} \quad \mu_{t} \mu_{t-1} v_t = - \mu_t \mu_{t-1} \gamma m_t.
\end{equation}
This yields that 
\begin{equation} \label{eq:dana_nesterov_2}
\mu_t m_{t+1} = \mu_t \mu_{t-1} m_t + \mu_t g_{t+1} \quad \Rightarrow \quad \mu_{t} m_{t+1} - \mu_t g_{t+1} = \mu_t \mu_{t-1} m_t. 
\end{equation}
Combining \eqref{eq:dana_nesterov_1} and \eqref{eq:dana_nesterov_2}, we have that
\begin{align*}
\Phi_{t+1} 
&
= \Phi_t -\gamma \mu_t (m_{t+1} - g_{t+1}) - (1+\mu_t) \gamma g_{t+1}
\\
&
= \Phi_t - \gamma \mu_t m_{t+1} - \gamma g_{t+1} 
\\
&
= \Phi_t - \gamma (g_{t+1} + \mu_t m_{t+1})
\end{align*}
Consequently, we have that
\begin{equation} \label{eq:dana_nesterov_formulation}
\boxed{m_{t+1} = \mu_{t-1} m_t + g_{t+1} \quad \text{and} \quad \Phi_{t+1} = \Phi_t - \gamma(g_{t+1} + \mu_t m_{t+1}), \quad \text{where $g_{t+1} = \nabla \mathcal{L}(\Phi_t)$ and $\mu_t = \frac{t}{t+3}$.}}
\end{equation}
We denote this formulation as \textit{Extra-gradient Nesterov}. 

\paragraph{\textcolor{Mutedred}{\textbf{\underline{EMA Nesterov.}}}} Since we used no properties that $g_{t+1} = \nabla \mathcal{L}(\Phi_t)$, we can redefine $g_{t+1}$ in the generalized Nesterov formulation \eqref{eq:dana_nesterov_formulation} so that
\[
g_{t+1} = (1-\mu_{t-1}) \tilde{g}_{t+1}.  
\]
Using this we get that the momentum update $m_{t+1}$ becomes an exponential moving average which is similar to $\beta_1$ in \AdamW. Let us define
\[
p_{t+1} = \mu_{t-1} p_t + (1-\mu_{t-1}) g_{t+1}.
\]
Thus we can think of $p_{t+1} \approx (1-\mu_{t-1}) m_{t+1}$ and $\Phi_t$ remains the same as in \eqref{eq:dana_nesterov_formulation}. While this is not a direct equivalence, it is close and it better aligns with the implementation of $\beta_1$ in \AdamW. We call this formulation, \textit{exponential moving average (EMA) Nesterov}, and its updates are 
\begin{equation} \label{eq:EMA_Nesterov}
\begin{gathered}
\boxed{p_{t+1} = \mu_{t-1} p_t + (1-\mu_{t-1}) g_{t+1} \quad \text{and} \quad \Phi_{t+1} = \Phi_t - \gamma \left ( g_{t+1} + \frac{\mu_{t}}{1-\mu_{t-1}} \cdot p_{t+1} \right )}
\\
\text{where} \quad m_{t+1} \approx (1-\mu_{t-1} ) p_{t+1}. 
\end{gathered}
\end{equation}

We summarize the results of the different formulations below. 

\begin{table}[h!]
\centering
    \begin{tabular}{cll}
         \textbf{Nesterov Version} & \textbf{Update Rules} & \textbf{Relationship} \\
        \hline \\[-2ex]
        \begin{minipage}{0.17\textwidth}
        \begin{center}
        \textcolor{FPPcolor}{\textbf{Extra-gradient Nesterov} }\\ \eqref{eq:dana_nesterov_formulation} \\
        $(\Phi_t, m_t)$
        \end{center}
        \end{minipage}
        &
        \begin{minipage}{0.27\textwidth}
        $g_{t+1} = \nabla \mathcal{L}(\Phi_t)$\\
        $m_{t+1} = \mu_{t-1} m_t + g_{t+1}$\\
        $\Phi_{t+1} = \Phi_t - \gamma(g_{t+1} + \mu_t m_{t+1})$
        \end{minipage} 
        \\ \\[-2ex]
        \hline\\[-2ex]
             \begin{minipage}{0.15\textwidth}
        \begin{center}
        \textcolor{FACcolor}{\textbf{Look-ahead Nesterov}} \eqref{eq:look_ahead_nesterov} \\
        $(\theta_t, v_t)$
        \end{center}
        \end{minipage}
        &
        \begin{minipage}{0.3\textwidth}
        $v_{t+1} = \mu_{t-1} v_t - \gamma \nabla \mathcal{L}(\theta_t + \mu_{t-1} v_t)$\\
        $\theta_{t+1} = \theta_t + v_{t+1}$
        \end{minipage} 
        & 
        \begin{minipage}{0.2\textwidth}
        $v_{t} = -\gamma m_t$\\
        $\theta_{t} = \Phi_t + \mu_{t-1} \cdot \gamma \cdot m_t$
        \end{minipage} 
        \\
        \\[-2ex]
        \hline\\[-2ex]
                  \begin{minipage}{0.15\textwidth}
        \begin{center}
        \textcolor{F0color}{\textbf{2-sequence Nesterov}} \eqref{eq:nesterov_1} \\
        $(\theta_t, y_t)$
        \end{center}
        \end{minipage}
        &
        \begin{minipage}{0.3\textwidth}
        $\theta_{t+1} = y_t - \gamma \nabla \mathcal{L}(y_t)$\\
        $y_{t+1} = \theta_{t+1} + \mu_t (\theta_{t+1}-\theta_t)$\\
        \end{minipage} 
        & 
        \begin{minipage}{0.2\textwidth}
        $y_{t} = \Phi_t$\\
        $\theta_{t} = \Phi_t + \mu_{t-1} \cdot \gamma \cdot m_t$
        \end{minipage} 
          \\
        \\[-2ex]
        \hline \\[-2ex]
        \begin{minipage}{0.17\textwidth}
        \begin{center}
        \textcolor{Mutedred}{\textbf{EMA Nesterov} }\\ \eqref{eq:EMA_Nesterov} \\
        $(\Phi_t, p_t)$
        \end{center}
        \end{minipage}
        &
        \begin{minipage}{0.3\textwidth}
        $g_{t+1} = \nabla \mathcal{L}(\Phi_t)$\\
        $p_{t+1} = \mu_{t-1} p_t + (1-\mu_{t-1}) g_{t+1}$\\
        $\Phi_{t+1} = \Phi_t - \gamma(g_{t+1} +  \frac{\mu_t}{1-\mu_{t-1}} \cdot p_{t+1})$
        \end{minipage}
        & \begin{minipage}{0.2\textwidth}
        \text{Not an exact equivalence}\\
        $p_{t+1} \approx \frac{m_{t+1}}{1-\mu_{t-1}}$
        \end{minipage}  
                 \\
        \bottomrule
    \end{tabular}
    \caption{We summarize the different formulations of standard Nesterov accelerated method, Section~\ref{sec:standard_nesterov_formulations}. The relationship column shows how to go between the different parameters. Here $\gamma > 0$ is a learning rate and $\mu_{t} = \frac{t}{t+3}$.}
    \label{table:nesterov_versions}
\end{table}

\subsection{Generalized Nesterov Formulation: Adding a Damping Factor} \label{sec:generalized_nesterov_derivations} 
As explained in Section~\ref{sec:appendix_adding_log_momentum}, when using stochastic gradients in Nesterov, \eqref{eq:dana_nesterov_formulation}, the noise from the gradients grows and dominates the parameter updates $\Phi_t$. Since there is only one learning rate $\gamma$ to control both the noise from the stochastic gradients in $g_t$ and $m_{t+1}$, this causes divergence. Another way of stating this is that there is no choice of $\gamma$ so that both $g_{t+1}$ and $m_{t+1}$ survive. You would end up driving the $g_{t+1}$ to $0$ (since $m_{t+1}$ has more stochastic noise), but you need both of them for Nesterov's Accelerated Method to converge. To combat this, we introduce \textit{Generalized Nesterov Accelerated Method}, which by choosing a new hyperparamter so that Nesterov is ``damped". This allows for the algorithm to converge with stochastic gradients. 

\paragraph{ \textcolor{FPPcolor}{\underline{Generalized Extra-gradient Nesterov.}}} The \textit{Generalized Nesterov Acceleration Method} introduces a learning rate $\textcolor{KPPcolor}{\bm{\hat{\alpha}_t}}$ to scale the momentum term $m_{t+1}$ in the parameter update rule. By choosing this learning such that $\textcolor{KPPcolor}{\bm{\hat{\alpha}_t}} \le \mu_t$, we can damp the momentum term. Choosing $\textcolor{KPPcolor}{\bm{\hat{\alpha}_t}} > \mu_t$, amplifies momentum. Since we are concerned with the accumulation of noise from the stochastic gradients in the momentum term, we will want to dampen momentum, or equivalently, choose $\textcolor{KPPcolor}{\bm{\hat{\alpha}_t}} < \mu_t$. On the other hand, we do not want to dampen too much, in which we would loose the acceleration effects we desire and become SGD. By choosing an appropriate $\textcolor{KPPcolor}{\bm{\hat{\alpha}_t}}$, \citet{ferbach2025dimension} show that you can accelerate \textit{and} even outscale SGD on the power-law random features model (PLRF) (see Section~\ref{sec:appendix_synthetic}). 

The \textit{Generalized Nesterov Accelerated Method} updates as 
\begin{equation} \label{eq:generalized_dana_nesterov_formulation}
\boxed{m_{t+1} = \mu_{t-1} m_t + g_{t+1} \quad \text{and} \quad \Phi_{t+1} = \Phi_t - \gamma ( g_{t+1} + \textcolor{KPPcolor}{\bm{\hat{\alpha}_t}} \cdot m_{t+1} ) \quad \text{where $g_{t+1} = \nabla \mathcal{L}(\Phi_t)$ and $\mu_t = \frac{t}{t+3}$,}}
\end{equation}
where $\textcolor{KPPcolor}{\bm{\hat{\alpha}_t}}$ is a learning rate that needs to be properly tuned. 

\begin{table}[h!]
\centering
    \begin{tabular}{llll}
         & \textbf{Momentum Update} & \textbf{Parameter Update} & \textbf{Hyperpameters} \\
        \hline \\[-2ex]
        \begin{minipage}{0.15\textwidth}
        \begin{center} \textcolor{FPPcolor}{\textbf{Standard Nesterov \eqref{eq:dana_nesterov_formulation}}} 
        \end{center}
        \end{minipage} & $m_{t+1} = \mu_{t-1} m_t + g_{t+1}$ & $\Phi_{t+1} = \Phi_t -\gamma (g_{t+1} + \textcolor{FPPcolor}{\bm \mu_t} \cdot m_{t+1})$ & $g_{t+1} = \nabla \mathcal{L}(\Phi_t), \quad \mu_t = \frac{t}{t+3}$\\
        \\[-2ex]
         \begin{minipage}{0.15\textwidth}
         \begin{center} \textcolor{KPPcolor}{\textbf{Generalized Nesterov \eqref{eq:generalized_dana_nesterov_formulation}}}
         \end{center}
         \end{minipage} & $m_{t+1} = \mu_{t-1} m_t + g_{t+1}$ & $\Phi_{t+1} = \Phi_t -\gamma (g_{t+1} + \textcolor{KPPcolor}{\bm {\hat{\alpha}_t}} \cdot m_{t+1})$ & $g_{t+1} = \nabla \mathcal{L}(\Phi_t), \,\, \mu_t = \frac{t}{t+3}$\\
    \end{tabular}
\end{table}
Here we replaced \textcolor{FPPcolor}{$\bm \mu_t$} in \eqref{eq:dana_nesterov_formulation} with \textcolor{KPPcolor}{$\bm{\hat{\alpha}_t}$}. Note that $\lim_{t \to \infty} \mu_t = 1$, so we can view $\mu_t$ as essentially $1$ and we want to damped this term with an $\alpha_t$ which is decaying in $t$. In \cite{ferbach2025dimension} (see \Danadecaying Alg.), the authors showed that the optimal choice for this dampening factor in the PLRF (see Section~\ref{sec:appendix_synthetic}) is $ \textcolor{KPPcolor}{\bm{\hat{\alpha}_t = (1+t)^{-\kappa}}}$ for some $\kappa$ depending on the power-law covariance structure of the data. 

By a simple substitution, we have that
\begin{equation} \label{eq:dana_nesterov_view_1}
\begin{gathered}
\Phi_{t+1} = \Phi_t - \gamma ( g_{t+1} + \textcolor{KPPcolor}{\bm{\hat{\alpha}_t}} (\mu_{t-1} m_t + g_{t+1})) = \Phi_t - \gamma (1+\textcolor{KPPcolor}{\bm{\hat{\alpha}_t}}) g_{t+1} - \gamma \textcolor{KPPcolor}{\bm{\hat{\alpha}_t}}\mu_{t-1} m_t = \Phi_t - \gamma (1+\textcolor{KPPcolor}{\bm{\hat{\alpha}_t}}) g_{t+1} + \textcolor{KPPcolor}{\bm{\hat{\alpha}_t}} \mu_{t-1} v_t.
\end{gathered}
\end{equation}
The last equality follows from the relationship between $v$ and $m$, namely $v_t = -\gamma m_t$. 

\paragraph{ \textcolor{FACcolor}{\underline{Generalized Look-ahead Nesterov.}}} Let us show that \textit{Generalized Nesterov Accelerated Method} \eqref{eq:generalized_dana_nesterov_formulation} can be formulated as a \textit{generalized Look-Ahead Nesterov} (see update rule in \eqref{eq:look_ahead_nesterov}). For this, we proposed as an Ansatz for the dynamics in $(\theta, v)$ space that
\begin{gather*}
v_{t+1} = \mu_{t-1} v_t - \gamma g_{t+1} \quad \text{and} \quad \theta_{t+1} = \theta_t +  \textcolor{KPPcolor}{\bm{\delta_t}} v_{t+1}, \quad 
\text{where} \quad g_{t+1} = \nabla \mathcal{L}(\theta_t + \eta_{t-1} v_t).
\end{gather*}
We will choose $\textcolor{KPPcolor}{\bm{\delta_t}}$ and $\eta_{t-1}$ so that it matches the dynamics in \eqref{eq:dana_nesterov_view_1}. Note in the original Look Ahead Version $\eta_{t-1} = \mu_{t-1}$ and $\textcolor{KPPcolor}{\bm{\delta_t = 1}}$. Let us define $\Phi_t \defas \theta_t + \eta_{t-1} v_t$ since $\Phi_t$ was our look ahead parameter. Then we have the following 
\begin{align*}
\Phi_{t+1} &= \theta_{t+1} + \eta_t v_{t+1} 
\\
\text{($\theta_{t+1} = \theta_t + \delta_t v_{t+1}$)} \quad &= \theta_t + (\textcolor{KPPcolor}{\bm{\delta_t}} + \eta_t) v_{t+1}
\\
\text{($\theta_t = \Phi_t - \eta_{t-1} v_t$)} \quad &= \Phi_t - \eta_{t-1} v_t + (\textcolor{KPPcolor}{\bm{\delta_t}} + \eta_t) v_{t+1}
\\
\text{($v_{t+1} = \mu_{t-1} v_t - \gamma g_{t+1}$)} \quad & \Phi_t - \eta_{t-1} v_t + (\textcolor{KPPcolor}{\bm{\delta_t}} + \eta_t) (\mu_{t-1} v_t - \gamma g_{t+1})
\\
& = \Phi_t + \big [ (\textcolor{KPPcolor}{\bm{\delta_t}} + \eta_t) \mu_{t-1} - \eta_{t-1} \big ] v_t - \gamma (\textcolor{KPPcolor}{\bm{\delta_t}} + \eta_t) g_{t+1}
\end{align*}
Now we match up coefficients in \eqref{eq:dana_nesterov_view_1}, specifically we see that
\begin{align*}
\text{coefficients $g_t$:} \quad & 1+\alpha_t = \textcolor{KPPcolor}{\bm{\delta_t}} + \eta_t\\
\text{coefficients $v_t:$} \quad & \alpha_t \mu_{t-1} = (\textcolor{KPPcolor}{\bm{\delta_t}} + \eta_t)\mu_{t-1} - \eta_{t-1} = (1+\textcolor{KPPcolor}{\bm{\hat{\alpha}_t}}) \mu_{t-1} - \eta_{t-1}
\end{align*}
Solving this, we get that $\eta_{t-1} = \mu_{t-1}$ and $\textcolor{KPPcolor}{\bm{\delta_t}} = 1+\textcolor{KPPcolor}{\bm{\hat{\alpha}_t}} - \mu_t$. Putting this altogether, we have the following for the (generalized) Look-ahead Nesterov method
\begin{equation} \label{eq:generalized_look_ahead}
\begin{gathered}
\boxed{v_{t+1} = \mu_{t-1} v_t - \gamma g_t \quad \text{and} \quad \theta_{t+1} = \theta_t + \textcolor{KPPcolor}{\bm{\delta_t}} v_{t+1}, \quad 
\text{where} \quad g_t = \nabla \mathcal{L}(\theta_t + \mu_{t-1} v_t) \quad \text{and} \quad \textcolor{KPPcolor}{\bm{\delta_t = 1+\hat{\alpha}_t - \mu_t}}.}\\
\text{Change of variables:} \quad v_t = -\gamma m_t \quad \text{and} \quad \theta_t = \Phi_t + \gamma \mu_{t-1} m_t. 
\end{gathered}
\end{equation}
This formulation of the algorithm is nice because it has an intepretation as in terms of physical quantities such as momentum and damping. We summarize the results below.

\begin{table}[h!]
\centering
    \begin{tabular}{c|c|l}
         \textbf{Damping factor $\textcolor{KPPcolor}{\bm{\hat{\alpha}_t}}$} & \textbf{Learning rate $\textcolor{KPPcolor}{\bm{\delta_t = 1+\hat{\alpha}_t-\mu_t}}$} & \textbf{Behavior}  \\
        \hline\\ [-2ex]
        $\hat{\alpha}_t = \mu_t$ & $\delta_t = 1$ & Standard Nesterov \\
        \textcolor{KPPcolor}{$\bm{\hat{\alpha}_t < \mu_t}$} & \textcolor{KPPcolor}{$\bm{\delta_t < 1}$} & \textcolor{KPPcolor}{\textbf{Damped steps (more conservative)}}\\
        $\hat{\alpha}_t > \mu_t$ & $\delta_t > 1$ & Amplified steps (more agressive)\\
        $\hat{\alpha}_t = 0$ & $\delta_t = 1-\mu_t$ & \text{No momentum contribution/reduces to SGD}
    \end{tabular}
\end{table}

Given the issues with stability of standard Nesterov when the gradients are stochastic, we are interested in the regime where $\textcolor{KPPcolor}{\bm{\hat{\alpha}_t < \mu_t}}$, so that we take more conservative momentum steps. We do not want to make this too small because then we are not doing any momentum. In the work of \cite{ferbach2025dimension} with \Danadecaying, they propose a good choice for $\textcolor{KPPcolor}{\bm{\hat{\alpha}_t = (1+t)^{-\kappa}}}$ where $\kappa = 1/(2\rho)$ on the PLRF problem (see Section~\ref{sec:appendix_synthetic}). The $\rho$ corresponds to the exponent on the power-law covariance of the input data $x$ where $x_i \sim N(0, i^{-2\rho})$.

\paragraph{ \textcolor{F0color}{\underline{Generalized Two-sequence Nesterov.}}} We finish by also expressing the generalized form back into the original $(\tilde{x}, y)$ variables in \eqref{eq:nesterov}. In order to match the value that the gradient $g_t$ is evaluated on, we defined $y_t = \theta_t + \mu_{t-1} v_t$. Thus $g_{t+1} = \nabla \mathcal{L}(y_t)$. From $y_t$, we have that $v_t = \frac{y_t - \theta_t}{\mu_{t-1}}$. Thus, the $v_{t+1}$-update and $\theta_t$-parameter update become
\begin{align*}
v_{t+1} = y_t - \theta_t - \gamma g_{t+1} \quad \text{and} \quad \theta_{t+1} = \theta_t + \textcolor{KPPcolor}{\bm{\delta_t}} \big [ y_t - \theta_t - \gamma g_{t+1} \big ]. 
\end{align*}
We see that $\theta_{t+1} = \theta_t + \textcolor{KPPcolor}{\bm{\delta_t}} v_{t+1} \Rightarrow v_{t+1} = \frac{\theta_{t+1}-\theta_t}{\textcolor{KPPcolor}{\bm{\delta_t}}}$. Therefore, the next look ahead point is
\[
y_{t+1} = \theta_{t+1} + \mu_t v_{t+1} = \theta_{t+1} + \frac{\mu_t}{\textcolor{KPPcolor}{\bm{\delta_t}}} \big ( \theta_{t+1}-\theta_t \big ). 
\]
Putting everything together, we have the following
\begin{equation} \label{eq:generalized_two_sequence}
\begin{gathered}
\boxed{\theta_{t+1} = (1 - \textcolor{KPPcolor}{\bm{\delta_t}}) \theta_t + \delta_t y_t - \gamma \textcolor{KPPcolor}{\bm{\delta_t}} g_{t+1} \quad \text{and} \quad y_{t+1} = \theta_{t+1} + \frac{\mu_t}{\textcolor{KPPcolor}{\bm{\delta_t}}} (\theta_{t+1}-\theta_t), \quad \text{where} \quad g_{t+1} = \nabla \mathcal{L}(y_t).}\\
\text{Change of variables:} \quad y_t = \theta_t + \mu_{t-1} v_t \quad \text{and} \quad \theta_t = \theta_t. 
\end{gathered}
\end{equation}
In the original formulation, $\delta_t = 1$. 

\paragraph{\textcolor{Mutedred}{\textbf{\underline{Generalized EMA Nesterov.}}}} Since we used no properties that $g_{t+1} = \nabla \mathcal{L}(\Phi_t)$, we can redefine $g_t$ in the generalized Nesterov formulation \eqref{eq:generalized_dana_nesterov_formulation} so that
\[
g_{t+1} = (1-\mu_{t-1}) \tilde{g}_{t+1}.  
\]
Using this we get that the momentum update $m_{t+1}$ becomes an exponential moving average which is similar to $\beta_1$ in \AdamW. Let us define
\[
p_{t+1} = \mu_{t-1} p_t + (1-\mu_{t-1}) g_{t+1}.
\]
Thus we can think of $p_{t+1} \approx (1-\mu_{t-1}) m_{t+1}$ and $\Phi_t$ remains the same as in \eqref{eq:generalized_dana_nesterov_formulation}. While this is not a direct equivalence, it is close and it better aligns with the implementation of $\beta_1$ in \AdamW. We call this formulation, \textit{generalized exponential moving average (EMA) Nesterov}, and its updates are 
\begin{equation} \label{eq:generalized_EMA_Nesterov}
\begin{gathered}
\boxed{p_{t+1} = \mu_{t-1} p_t + (1-\mu_{t-1}) g_{t+1} \quad \text{and} \quad \Phi_{t+1} = \Phi_t - \gamma \left ( g_{t+1} + \frac{\textcolor{KPPcolor}{\bm{\hat{\alpha}(t)}}}{1-\mu_{t-1}} \cdot p_{t+1} \right )}
\\
\text{where} \quad m_{t+1} \approx (1-\mu_{t-1} ) p_{t+1}. 
\end{gathered}
\end{equation}

\begin{table}[h!]
\centering
    \begin{tabular}{cllc}
         \textbf{Generalized Version} & \textbf{Update Rules} & \textbf{Relationship} &
        \textbf{Standard Nesterov}\\
        \hline \\[-2ex]
        \begin{minipage}{0.17\textwidth}
        \begin{center}
        \textcolor{FPPcolor}{\textbf{Extra-gradient Nesterov} } \eqref{eq:generalized_dana_nesterov_formulation} \\
        $(\Phi_t, m_t)$
        \end{center}
        \end{minipage}
        &
        \begin{minipage}{0.25\textwidth}
        $g_t = \nabla \mathcal{L}(\Phi_t)$\\
        $m_{t+1} = \mu_{t-1} m_t + g_t$\\
        $\Phi_{t+1} = \Phi_t - \gamma(g_t +  \textcolor{KPPcolor}{\bm{\hat{\alpha}_t}} \cdot m_{t+1})$
        \end{minipage}
        & & 
        $\textcolor{KPPcolor}{\bm{\hat{\alpha}_t = \mu_{t}}}$
        \\ \\[-2ex]
        \hline\\[-2ex]
             \begin{minipage}{0.15\textwidth}
        \begin{center}
        \textcolor{FACcolor}{\textbf{Look-ahead Nesterov}} \eqref{eq:generalized_look_ahead} \\
        $(\theta_t, v_t)$
        \end{center}
        \end{minipage}
        &
        \begin{minipage}{0.3\textwidth}
        $v_{t+1} = \mu_{t-1} v_t - \gamma \nabla \mathcal{L}(\theta_t + \mu_{t-1} v_t)$\\
        $\theta_{t+1} = \theta_t + \textcolor{KPPcolor}{\bm{\delta_t}} \cdot v_{t+1}$
        \end{minipage} 
        & 
        \begin{minipage}{0.2\textwidth}
        $v_{t} = -\gamma m_t$\\
        $\theta_{t} = \Phi_t + \mu_{t-1} \cdot \gamma \cdot m_t$\\
        $\textcolor{KPPcolor}{\bm{\delta_t = 1+ \hat{\alpha}_t - \mu_t}}$
        \end{minipage} 
        &
        $\textcolor{KPPcolor}{\bm{\delta_t = 1}}$
        \\
        \\[-2ex]
        \hline\\[-2ex]
                  \begin{minipage}{0.15\textwidth}
        \begin{center}
        \textcolor{F0color}{\textbf{2-sequence Nesterov}} \eqref{eq:generalized_two_sequence} \\
        $(\theta_t, y_t)$
        \end{center}
        \end{minipage}
        &
        \begin{minipage}{0.35\textwidth}
        $\theta_{t+1} = \textcolor{KPPcolor}{\bm{(1-\delta_t)}} \cdot \theta_t + \textcolor{KPPcolor}{\bm{\delta_t}} \cdot y_t - \gamma \cdot \textcolor{KPPcolor}{\bm{\delta_t}} \nabla \mathcal{L}(y_t)$\\
        $y_{t+1} = \theta_{t+1} + \frac{\mu_t}{\textcolor{KPPcolor}{\bm{\delta_t}}} (\theta_{t+1}-\theta_t)$\\
        \end{minipage} 
        & 
        \begin{minipage}{0.2\textwidth}
        $y_{t} = \Phi_t$\\
        $\theta_{t} = \Phi_t + \mu_{t-1} \cdot \gamma \cdot m_t$
        \\
        $\textcolor{KPPcolor}{\bm{\delta_t = 1+ \hat{\alpha}_t - \mu_t}}$
        \end{minipage} 
        &
        $\textcolor{KPPcolor}{\bm{\delta_t = 1}}$
        \\
        \\[-2ex]
        \hline \\[-2ex]
        \begin{minipage}{0.17\textwidth}
        \begin{center}
        \textcolor{Mutedred}{\textbf{EMA Nesterov} }\\ \eqref{eq:generalized_EMA_Nesterov} \\
        $(\Phi_t, p_t)$
        \end{center}
        \end{minipage}
        &
        \begin{minipage}{0.3\textwidth}
        $g_t = \nabla \mathcal{L}(\Phi_t)$\\
        $p_{t+1} = \mu_{t-1} p_t + (1-\mu_{t-1}) g_t$\\
        $\Phi_{t+1} = \Phi_t - \gamma(g_t +  \frac{\textcolor{KPPcolor}{\bm{\hat{\alpha}_t}}}{1-\mu_{t-1}} \cdot p_{t+1})$
        \end{minipage}
        & \begin{minipage}{0.2\textwidth}
        \text{Not an exact equivalence}\\
        $p_{t+1} \approx \frac{m_{t+1}}{1-\mu_{t-1}}$
        \end{minipage}  & 
        $\textcolor{KPPcolor}{\bm{\hat{\alpha}_t = \mu_{t}}}$
                 \\
        \bottomrule
    \end{tabular}
    \caption{We summarize the different formulations of standard Nesterov accelerated method, Section~\ref{sec:generalized_nesterov_derivations}. The relationship column shows how to go between the different parameters. Here $\gamma > 0$ is a learning rate and $\mu_{t} = \frac{t}{t+3}$.}
    \label{table:nesterov_generalized_versions}
\end{table}

\subsection{Correspondence between EMA \Dana and the original \Dana formulation}
\label{sec:equivalence_dana_convex_combination}
In this section, we discuss how to express the original formulation of \Dana (\cite{ferbach2025dimension}) in terms of an exponential moving average (EMA) formulations. In \citep{ferbach2025dimension} the authors write the \Dana algorithm as 
\begin{equation} \tag{Original \Dana} \label{eq:original_dana}
\begin{aligned}
\text{\emph{(mom.)}}\quad & m_{t+1} = \beta_1(t)m_t + g_{t+1}\\
\text{\emph{(param.)}}\quad & \theta_{t+1} = \theta_t - \gamma(t) (g_{t+1} + \alpha(t) m_{t+1}).
\end{aligned}
\end{equation}
where $\alpha(t) \asymp (1+t)^{-\kappa}$ for \Danadecaying and $\alpha(t) \asymp \frac{1}{d}$ on the PLRF for DANA-constant. They are respectively the largest monomial and constant schedules that avoid divergence of \eqref{eq:original_dana}. To better understand this in the format of \Adam, we want to write the momentum as 
\begin{equation}
\label{eq:momentum_general_dana}
    m_{t+1} = \beta_1(t) m_t + (1-\beta_1(t)) g_{t+1}.
\end{equation}

A precise correspondence can be made between both writings. Indeed, from \eqref{eq:momentum_general_dana} we have

\begin{align*}
    m_{t}&=\sum_{s=1}^{t-1}(1-\beta_1(s))g_s\prod_{r=s+1}^{t-1}\beta_1(r) \approx\sum_{s=1}^{t-1}\frac{\delta}{\delta + s}g_s\left(\frac{s+\delta}{t+\delta}\right)^{\delta}.
\end{align*}
Hence using the EMA representation for $m_{t+1}$, we get for the parameter update
\begin{equation}
    \theta_{t+1} \approx \theta_t - \gamma(t) \left ( g_{t+1} + \hat{\alpha}(t) \sum_{s=1}^{t-1}\frac{\delta}{\delta + s}g_s\left(\frac{s+\delta}{t+\delta}\right)^{\delta} \right ).
\end{equation}

Similarly for \eqref{eq:original_dana}, we can write the momentum as  
\begin{align*}
    m_{t}&=\sum_{s=1}^{t-1}g_s\prod_{r=s+1}^{t-1}\beta_1(r) \approx\sum_{s=1}^{t-1}g_s\left(\frac{s+\delta}{t+\delta}\right)^{\delta}
\end{align*}

and the parameter updates as 

\begin{equation}
    \theta_{t+1} \approx \theta_t - \gamma(t) \left ( g_{t+1} + \hat{\alpha}(t) \sum_{s=1}^{t-1}g_s\left(\frac{s+\delta}{t+\delta}\right)^{\delta} \right ).
\end{equation}
We see that setting $\alpha(t) \defas \hat{\alpha}(t)\frac{\delta+t}{\delta}$, the two parameter updates are very similar:
\begin{equation}\tag{Original Dana}
\label{eq:parameter_update_original_dana}
    \theta_{t+1} \approx \theta_t - \gamma(t) (g_{t+1} + \hat{\alpha}(t) \sum_{s=1}^{t-1}g_s\left(\frac{s+\delta}{t+\delta}\right)^{\delta} ),
\end{equation}
as compared with
\begin{equation}\tag{General Dana}
\label{eq:parameter_update_general_dana}
    \theta_{t+1} \approx \theta_t - \gamma(t) (g_{t+1} + \hat{\alpha}(t) \sum_{s=1}^{t-1}g_s\left(\frac{s+\delta}{t+\delta}\right)^{\delta-1} ).
\end{equation}

Hence setting $\alpha(t) \defas \hat{\alpha}(t)\frac{\delta+t}{\delta}$ make a direct correspondence between both algorithms with the only change being from $\delta$ to $\delta-1$. Note that this change is minor since \cite{ferbach2025dimension} noticed that the exact constant $\delta$ does not matter as long as it is large enough.

\section{Proofs under sparse gradient oracle}
\label{sec:proofs}

To prove the results in Section~\ref{sec:log_momentum_beta_2}, we require bounds on the tightness of certain stochastic sequences arising from the Adam update rule. We formalize the simplified per-parameter stochastic process that models the transformation \Adam performs on gradients.

Let the sequence of random variables be defined as
\[
Y_t
= \sum_{r=1}^t \frac{A_t(r) g_r (1-\beta_1(r))}{\sqrt{\sum_{s=1}^t B_t(s) (g_s)^2(1-\beta_2(s))}+\epsilon}
\eqqcolon \frac{m_t}{\sqrt{v_t}+\epsilon}
\]
where $A_t(r) = \prod_{j=r+1}^t \beta_1(j)$, $B_t(s) = \prod_{j=s+1}^t \beta_2(j)$, and $g_r = X_r B_r$ with $X_r$ independent random variables with second moment $1$ and $B_r \sim \text{Bernoulli}(p)$.
This represents a simplified per-parameter stochastic process which models the transformation that \Adam performs to a gradient.

At the very least, this transformation should not greatly expand the order of magnitude of the updates---in the simple setup proposed, we should have that $Y_t$ remain stochastically bounded (which is the same order as the original input sequence). However, due to sparsity, there is another effect. In the time between nonzero updates, the optimizer continues to apply updates, even though no additional gradient information has been provided.

Hence if we let $T_\ell$ be the times at which nonzero gradients occur (i.e.\ the times at which consecutive nonzero $B_r$ occur), then the cumulative sum of updates between those times should remain bounded:
\[
Z_\ell \defas \sum_{t=T_{\ell}}^{T_{\ell+1}-1} Y_t, \qquad \ell \ge 1.
\]
We would like that these remain bounded, which will allow us to distinguish bounded rules from unbounded rules.

We begin by showing that the usual application of \Adam remains bounded in this sense. Throughout this section, we apply inequalities and actions entry-wise.

\begin{theorem}[Boundedness of Standard \Adam, Theorem~\ref{thm:constant_beta_necessary}]
\label{thm:normal_tightness}
    If $\beta_1 \in (0,1)$ and $\beta_2 \in (0,1)$ are fixed constants (not depending on $k$) and $\beta_1^2 < \beta_2$, the family of random variables $\{Z_\ell : \ell \geq 1, p \in (0,1), \epsilon > 0\}$ satisfies the uniform first-moment bound
    \[
    \Exp |Z_\ell|\leq \frac{1-\beta_1}{(\sqrt{1-\beta_2})(1-\tfrac{\beta_1}{\sqrt{\beta_2}})^2}.
    \]
\end{theorem}

\begin{proof}[Proof of Theorem~\ref{thm:normal_tightness}]
    We start by observing that the coefficient sequences $A_t(r)= \beta_1^{t-r}$ and $B_t(s)= \beta_2^{t-s}$. Fix $t \ge T_{\ell}$ with $\ell \ge 1$. For any $1 \le j \le \ell$, we have that $T_j \le T_{\ell}$ and
    \[
    v_t = \sum_{s=1}^t B_t(s) (g_s)^2 (1-\beta_2(s)) \ge B_t(T_j) g_{T_j}^2 (1-\beta_2) = \beta_2^{t-T_j}X_{T_j}^2 (1-\beta_2).
    \]
    Here we used that $g_{T_j} = X_{T_j}$ since $B_{T_j} = 1$ as $T_j$ is the $j$th nonzero gradient seen.

    As for the first moment, we have for any $t \in [T_{\ell}, T_{\ell +1}-1]$,
    \[
    m_t = \sum_{r=1}^t A_t(r) g_r(1-\beta_1(r)) = \sum_{j=1}^{\ell} A_t(T_j) X_{T_j} (1-\beta_1(T_j)) = \sum_{j=1}^{\ell} \beta_1^{t-T_j} X_{T_j} (1-\beta_1).
    \]
    Here again we used that only the nonzero gradients are counted in the summation for $m_t$. Putting this together, we have the following bound
    \begin{align*}
    \frac{|m_t|}{\sqrt{v_t} + \epsilon} \le \sum_{j=1}^{\ell} \frac{\beta_1^{t-T_j} X_{T_j} (1-\beta_1) }{\sqrt{\beta_2^{t-T_j} X_{T_j}^2 (1-\beta_2) } } \le \frac{1-\beta_1}{\sqrt{1-\beta_2}} \sum_{j=1}^{\ell} \bigg ( \frac{\beta_1}{\sqrt{\beta_2}} \bigg )^{t-T_j}.
    \end{align*}
    Summing over $t \in [T_{\ell}, T_{\ell +1}-1]$, we have that
    \begin{align*}
        |Z_{\ell}|
        &
        = \sum_{t = T_{\ell}}^{T_{\ell + 1}-1} \frac{|m_t|}{\sqrt{v_t} + \epsilon} \le \frac{1-\beta_1}{\sqrt{1-\beta_2}} \sum_{t=T_{\ell}}^{T_{\ell+1}-1}  \sum_{j=1}^{\ell} \bigg ( \frac{\beta_1}{\sqrt{\beta_2}} \bigg )^{t-T_j}
        \\
        &= \frac{1-\beta_1}{\sqrt{1-\beta_2}}   \sum_{j=1}^{\ell}\sum_{t=T_{\ell}}^{T_{\ell+1}-1}  \bigg ( \frac{\beta_1}{\sqrt{\beta_2}} \bigg )^{t-T_j}
        \\
        &
        \le
        \frac{1-\beta_1}{\sqrt{1-\beta_2} (1-\tfrac{\beta_1}{\sqrt{\beta_2}}) } \sum_{j=1}^{\ell} \left ( \frac{\beta_1}{\sqrt{\beta_2}} \right )^{T_{\ell} - T_j}
        \\
        &
        \le \frac{1-\beta_1}{\sqrt{1-\beta_2} (1-\tfrac{\beta_1}{\sqrt{\beta_2}})^2}.
    \end{align*}
    This proves the result.
\end{proof}

We show that this rule remains salvageable for long first moment buffers, if we scale the updates by $\sqrt{p}$.

\begin{theorem}[Boundedness of log-time schedules for $\beta_1$ and $\beta_2$ if multiplied by $\sqrt{p}$, Theorem~\ref{thm:sparsity_main}]
\label{thm:long_tightness}
    If $\beta_1(t)=\beta_2(t)=1-\tfrac{c}{t}$ for $c > 2$ the family of random variables $\{Z_\ell : \ell \geq 1, p \in (0,1), \epsilon > 0\}$ satisfies the first-moment bound
    \[
    \Exp |Z_\ell|\leq \frac{C}{\sqrt{p}}.
    \]
\end{theorem}

\begin{proof}[Proof of Theorem~\ref{thm:long_tightness}]
    We have that the moment sequences satisfy the estimates
    \[
    A_t(r) \asymp \left(\frac{r}{t}\right)^{c}, \quad B_t(s) \asymp \left(\frac{s}{t}\right)^{c},
    \]
    with the constants depending on $c$. Moreover, we have that by Riemann sum approximation,
    \[
    \sum_{j=1}^k A_k(j) (1-\beta_1(j)) \to \int_0^1 c t^{c-1}\dif t = 1.
    \]

    For a given $X^{T_r}$, we observe the following estimates on the conditional moment of $v_k$:
    \begin{equation}\label{eqn:vk}
    \frac{1}{\sqrt{\Exp(v_k \mid X^{T_r})} + \epsilon}
    \leq
    \Exp\left(\frac{1}{\sqrt{v_k}+\epsilon} \mid X^{T_r}\right) \lesssim \left(\frac{T_r}{k}\right)^{-c/2} {\frac{\sqrt{T_r}}{|X^{T_r}|}}\mathfrak{p}_k + \sqrt{\frac{1}{p\Exp |X^{T_r}|^2}},
    \end{equation}
    with the left hand side following from Jensen's inequality,
    the $\mathfrak{p}_k$ in the right hand side giving the probability
    \begin{equation}\label{eqn:pk}
    \mathfrak{p}_k \coloneqq
    \Pr\left(
        M_k \leq \frac{1}{2}\Exp M_k
    \right)
    \quad \text{where} \quad
    M_k \coloneqq \sum_{j={2k/3}}^k \frac{B_k(j)|g^j|^2}{j}.
    \end{equation}
    (The bound then follows by selecting a group of $k/3$ terms not including $T_r$ and bounding $v_k$ below on the event described in the bound or on its complement.)
    From moment-based bounds, we have that for any $C>0$, $\mathfrak{p}_k \lesssim_C (\max\{1,pk\})^{-C}$.

    As for the first moment buffer, we have that for $k \in [T_\ell,T_{\ell+1}-1]$:
    \[
        m_k = \sum_{r=1}^{\ell} B_k(T_r) X^{T_r}(1-\beta_1(T_r)).
    \]
    Thus summing over $k$ in the range, we have that
    \[
    \Exp\left( |Z_{\ell+1}| \right)
    \leq
    \sum_{r=1}^{\ell} \Exp\left(
        \sum_{k=T_{\ell}}^{T_{\ell+1}-1}
        \frac{B_k(T_r) |X^{T_r}|(1-\beta_1(T_r))}{\sqrt{v_k}+\epsilon}
    \right).
    \]
    We then take a conditional expectation, using \eqref{eqn:vk} and \eqref{eqn:pk}, and we have that
    \[
    \begin{aligned}
    \Exp\left(
        \sum_{k=T_{\ell}}^{T_{\ell+1}-1}
        \frac{B_k(T_r)}{\sqrt{v_k}+\epsilon}
        \,\middle\vert\, X^{T_r},T_r
    \right)
    &\lesssim
    \Exp\left(
        \sum_{k=T_{\ell}}^{T_{\ell+1}-1}
         \left(\frac{T_r}{k}\right)^{c/2} \frac{\sqrt{T_r}}{|X^{T_r}|}\mathfrak{p}_k
         \,\middle\vert\, X^{T_r},T_r
    \right)
    &\Biggr\}\phantom{\Bigg|}&\text{\,Term-(i)}
    \\
    &+
    \Exp\left(
        \sum_{k=T_{\ell}}^{T_{\ell+1}-1}
        \left(\frac{T_r}{k}\right)^{c}
        \sqrt{\frac{1}{p\Exp |X^{T_\ell}|^2}}
        \,\middle\vert\, X^{T_r},T_r
    \right)
    &\Biggr\}\phantom{\Bigg|}&\text{\,Term-(ii)}
    .
    \end{aligned}
    \]

    For Term-(ii) we bound the sum over $k^{-c}$ by $(T_{\ell+1}-T_\ell) (T_\ell)^{-c}$, combining it with the factor $|X^{T_r}|(1-\beta_1(T_r))$ to give
    \[
    \sum_{r=1}^\ell \Exp\left( |X^{T_r}|(1-\beta_1(T_r)) \text{Term-(ii)}\right)
    \lesssim
    \sum_{r=1}^\ell c\Exp\left(
        \left(\frac{T_r}{T_\ell}\right)^{c-1}
        \frac{T_{\ell+1}-T_\ell}{T_\ell}
    \right)
    \frac{\Exp |X|}{\sqrt{p\Exp |X^2|}}
    \to
    \frac{\Exp |X|}{\sqrt{p\Exp |X^2|}},
    \]
    as $\ell \to \infty.$ This is because the sum is a stochastic Riemann sum approximation to the integral $\int_0^1 c t^{c-1}\dif t = 1$; the ratio $(T_{\ell+1}-T_\ell)\ell/T_\ell$ converges in moments to $1$, while the ratio $T_r/T_\ell$ converges to $r/\ell$.

    As for Term-(i),
    \[
        \sum_{r=1}^\ell \Exp\left( |X^{T_r}|(1-\beta_1(T_r)) \text{Term-(i)}\right)
        \lesssim
        \sum_{r=1}^\ell \Exp\left(
            \left(\frac{T_r}{T_\ell}\right)^{c/2-1}
            \frac{(T_{\ell+1}-T_\ell)\sqrt{T_r}}{T_\ell} \max\{1,pT_\ell\}^{-1/2}
        \right).
    \]
    The expression $\sqrt{T_r}\max\{1,pT_\ell\}^{-1/2}$ is bounded deterministically by $\sqrt{1/p}$, and hence we arrive at both Term-(i) and Term-(ii) giving
    \[
    \Exp(|Z_{\ell+1}|) \leq \frac{C}{\sqrt{p}}.
    \]
\end{proof}

Conversely, we show that with $\beta_1(\ell)=1-\tfrac{c}{\ell}$ it is not possible to choose a $\beta_2(\ell)$ which is monotone increasing and produces a bounded sequence.

\begin{theorem}[Unboundedness of \Logadam, Theorem~\ref{thm:long_adam_divergence}]
\label{thm:non_tightness}
Suppose $\beta_1(\ell) = 1-c/\ell$ for $c>1$. For any monotone sequence $\beta_2(\ell) \in (0,1)$, the family of random variables $\{{Z_\ell}: \ell \geq 1, p \in (0,1), \epsilon > 0\}$ is unbounded, i.e.\
\[
\limsup_{p \to 0} \sup_{\ell \geq 1,\epsilon > 0}\Exp(|Z_\ell|) = \infty.
\]
In fact it follows that along a subsequence of $p \to 0$,
\[
|Z_2| \Prto \infty.
\]
\end{theorem}

\begin{proof}
We consider the time of the first non-zero gradient. Let $T_\ell$ be the $\ell$-th time $j$ such that $B^j=1$. Then $T_\ell-T_{\ell-1}$ is a geometric random variable with parameter $p$. The probability that $T_1=k$ is $p(1-p)^{k-1}$. Then on sending $p \to 0$ we have that
\[
(pT_1,p(T_2-T_1)) \Wkto (\operatorname{Exp}(1),\operatorname{Exp}(1)),
\]
in distribution as $p \to 0$ with the convergence in probability, with the limiting random variables independent.
In particular we have convergence
\[
\Pr\left( T_1 \in \frac{1}{p}[1,2], T_2 \in \frac{1}{p}[4,5]\right) \to q \in (0,1).
\]
Let $\mathcal{E}$ be this event. Then on this event we have that for $k \in [T_1,T_2-1]$
\[
    Y_{k} = \frac{X^{T_1}(1-\beta_1(T_1))(A_{k}(T_1))}{|X^{T_1}|\sqrt{B_{k}(T_1)(1-\beta_2(T_1))} + \epsilon}.
\]
We have the representation
\[
|Z_2| = \sum_{k=T_1}^{T_2-1} \frac{|X^{T_1}|(1-\beta_1(T_1))(A_{k}(T_1))}{|X^{T_1}|\sqrt{B_{k}(T_1)(1-\beta_2(k-T_1))} + \epsilon}
\Asto
\sum_{k=T_1}^{T_2-1} \frac{(1-\beta_1(T_1))(A_{k}(T_1))}{\sqrt{B_{k}(T_1)(1-\beta_2(k-T_1))}},
\]
as $\epsilon \to 0$.
Hence for this to remain bounded over all $p$, we need that this expression is bounded uniformly above on the event $\mathcal{E}$.

On this event, since $A_k(j) \asymp (j/k)^c$ we have that it is bounded below by a constant, and using monotonicity of $\beta_2$ and $B_k(j)$ we conclude that for $\Exp|Z_2|$ to remain bounded independent of $p$, there is a constant $C>0$
\[
C \leq (B_{4k}(2k))(1-\beta_2(k)) \leq B_{4k}(2k).
\]

Putting everything together, we have that for all $k$ sufficiently large
\[
-C \leq \log B_{4k}(2k) \leq \sum_{\ell = 2k}^{4k} \log (1-(1-\beta_2(\ell))) \leq -2k(1-\beta_2(4k)).
\]
Rearranging and using monotonicity, we conclude that there is a $C>0$ for all $k$ sufficiently large
\begin{equation}\label{eqn:beta2logk}
(1-\beta_2(k)) \leq \frac{C \log k}{k}.
\end{equation}
However this implies that
\[
    \sum_{k=T_1}^{T_2-1} \frac{(1-\beta_1(T_1))(A_{k}(T_1))}{\sqrt{B_{k}(T_1)(1-\beta_2(k-T_1))}}
    \gtrsim
    \sum_{k=T_1}^{T_2-1} \frac{1}{\sqrt{k\log k}} \gtrsim \frac{1}{\sqrt{p\log(1/p)}},
\]
which is a contradiction.
\end{proof}

As a corollary, we conclude the following strong form of instability:

\begin{theorem}[Strong Blowup, Corollary~\ref{cor:divergence_long_adam}]
\label{thm:strongblowup}
    Suppose that $F :\R^d \to \R$ is a coercive, Lipschitz loss, and suppose we are given access to a noisy unbiased gradient oracle $\mathcal{O},$ which for any given parameter $\theta \in \R^d$ provides a distributional, $L^2(\Pr)$ approximation of $\nabla F$. That is, for any $\theta \in \R^d$ a sample $\mathcal{G} \sim \mathcal{O}(\theta)$ satisfies
    for some constant $M > 0$,
    \[
    \Exp \| \mathcal{G}\|^2 \leq M
    \quad
    \text{and}
    \quad
    \Exp( \mathcal{G} ) = \nabla F (\theta).
    \]
    For $p \in [0,1]$, let the sparse oracle $\mathcal{O}_p$ be defined by the sampling rule that $G' \sim \mathcal{O}_p(\theta)$ is given by $G' = B G$ where $G \sim \mathcal{O}(\theta)$ and where $B \sim \operatorname{Bernoulli}(p)$ is independent of $G$.
    Suppose we run \Adam with $\beta_1(\ell) = 1-c/\ell$ for $c>1$ and all $\ell > 1$ and suppose $\beta_2(\ell) \in (0,1)$ is any monotone sequence. Let $\{\theta^k\}$ be the iterates of \Adam generated from the noisy oracle $\mathcal{O}_p$. Let $T_2$ be the time at which the second nonzero gradient occurs. Then there is a sequence of $p \to 0$ so that
    \[
    \theta^{T_2} \Prto \infty.
    \]
\end{theorem}

Hence \Adam with $1-\beta_1(\ell) = c/\ell$ is essentially uniformly bad---on \emph{every} problem with a sufficiently sparse oracle, it will cause arbitrarily bad iterates.


\section{Details on the estimation of the compute used for the experiments}
\label{sec:appendix_compute}

In this section, we provide details on how we estimated the compute used for this project. Note first that this estimation concern the total copute used for this project(including the experiments that did not end up being in the paper)

We use the metric H100 GPU year which correspond to running a H100 for a year because most of our runs were using H100. Given that the Max Thermal
Design Power of a H100 is between 350 and 700W, and between 300 and 400W for an A100. We used the following conversion system: 1hour with an A100 is equivalent to .5 hours with an H100. 
\begin{table}[h]
\centering
\caption{Compute Usage Report}
\renewcommand{\arraystretch}{1.3} 
\begin{tabular}{lrrr}
\hline
\textbf{GPU} & \textbf{Raw GPU days} & \textbf{H100 Factor} & \textbf{H100 Equiv Hours} \\ \hline
A100               & 1801
               & 0.5                 & 900.5   \\
H100               & 7138          & 1.0                & 55710      \\ \hline \hline
\multicolumn{3}{r}{\textbf{Total H100 Equivalent Days:}}  & \textbf{8038.5}       \\
\multicolumn{3}{r}{\textbf{Total H100 Equivalent Years:}}  & \textbf{22}          \\ \hline
\end{tabular}
\end{table}


\end{document}